\documentclass[11pt]{article}

\usepackage{graphicx}
\usepackage{BOONDOX-cal} 
\usepackage{float} 
\usepackage{subfigure}
\usepackage{caption} 
\usepackage[ruled]{algorithm2e} 
\usepackage{xr} 
\usepackage{booktabs} 
\usepackage{multirow} 
\usepackage{color} 
\usepackage{setspace} 
\usepackage{amsfonts} 
\usepackage{amsmath} 
\usepackage{bm} 
\usepackage{longtable} 

\SetKwComment{Comment}{(}{)}
\renewcommand{\KwData}{\textbf{Input: }}
\renewcommand{\KwResult}{\textbf{Output: }}

\newtheorem{Def}{Definition}[section]
\newtheorem{Thm}{Theorem}[section]

\newenvironment{proof}{\noindent\emph{Proof}\newline}{\hfill$\square$}
\usepackage[a4paper,left=30mm,right=30mm,top=30mm,bottom=30mm]{geometry}

\title{\textbf{Outlier Detection with Cluster Catch Digraphs}}
\author{Rui Shi\footnote{The Department of Mathematics and Statistics, Auburn University, rzs0112@auburn.edu}, Nedret Billor\footnote{The Department of Mathematics and Statistics, Auburn University, billone@auburn.edu}, and Elvan Ceyhan\footnote{The Department of Mathematics and Statistics, Auburn University, ezc0066@auburn.edu}}
\date{}

\begin{document}
\maketitle

\begin{abstract}
\noindent
This paper introduces a novel family of outlier detection algorithms based on Cluster Catch Digraphs (CCDs), specifically tailored to address the challenges of high dimensionality and varying cluster shapes,
which deteriorate the performance of most traditional outlier detection methods. 
We propose the Rapid Uniformity-Based CCD with Mutual Catch Graph (RU-MCCD), the Uniformity- and Neighbor-Based CCD with Mutual Catch Graph (UN-MCCD), and their shape-adaptive variants (SU-MCCD and SUN-MCCD), which are designed to detect outliers in data sets with arbitrary cluster shapes and high dimensions. 
We present the advantages and shortcomings of these algorithms 
and provide the motivation or need to define each particular algorithm. 
Through comprehensive Monte Carlo simulations, we assess their performance and demonstrate the robustness and effectiveness of our algorithms across various settings and contamination levels. 
We also illustrate the use of our algorithms on various real-life data sets. The RU-MCCD algorithm efficiently identifies outliers while maintaining high true negative rates, 
and the SU-MCCD algorithm shows substantial improvement in handling non-uniform clusters. Additionally, the UN-MCCD and SUN-MCCD algorithms address the limitations of existing methods in high-dimensional spaces by utilizing Nearest Neighbor Distances (NND) for clustering and outlier detection. 
Our results indicate that these novel algorithms offer substantial advancements in the accuracy and adaptability of outlier detection, providing a valuable tool for various real-world applications.
\end{abstract}


\section{Introduction}

Research on \emph{outlier detection} has a long and rich history. As early as $1620$, Francis Bacon recognized the existence of substantial deviations from commonly occurring phenomena in nature \cite{bacon1878novum}. In the 19th century, Legendre and Gauss discovered the least squares methodology \cite{stigler1981gauss}. Legendre was the first mathematician to realize the impact of outliers (which he referred to as ``errors") on the method. He suggested rejecting models that produce errors too large to be admissible \cite{hadi2009detection}. Later, Edgeworth and Ysidro proposed dropping a certain portion of abnormal data points (i.e., most likely outliers) to avoid their substantial influence on least squares estimates \cite{edgeworth1887xli}.

Today, outlier detection remains a popular research topic due to its wide range of applications. For instance, it can help financial institutions identify suspicious loan applications \cite{paula2016deep, samariya2023comprehensive}. It can be employed to detect faults in mechanical units \cite{samariya2023comprehensive}. It can also be used in network anomaly detection to build a security management system that protects against intrusion attempts \cite{alrawashdeh2016toward, olteanu2023meta, gogoi2011survey}, and. Furthermore, outlier detection is crucial in diagnosing diseases such as brain cancer and leukemia \cite{gebremeskel2016combined}.

There have been various definitions of outliers since the start of outlier detection research. Ayadi \textit{et al.} \cite{ayadi2017outlier} summarized twelve different definitions according to different researchers in chronological order. Among all the definitions, the one from Hawkins is widely accepted by statisticians:
``\emph{An outlier is an observation that deviates so much from other observations, and it arouses suspicions that it was generated by a different mechanism.}" \cite{hawkins1980identification}

According to the previous surveys, outliers can be classified as point, collective, and local outliers \cite{zhang2013advancements, samariya2023comprehensive}:
\begin{itemize}
    \item[(1)] \textbf{Point outliers:} An individual point that is outlying.
    \item[(2)] \textbf{Collective outliers:} Several or a group of close data points showing a nonconforming pattern compared to the entire data set. Identifying an outlying group is generally a more challenging task.
    \item[(3)] \textbf{Local outliers:} A single (or a group of) points exhibits anomaly in terms of its (their) neighbors.
\end{itemize}

Outlier detection is essential for data analysis and pre-processing. It is easy to spot outliers visually in one or two-dimensional space. However, virtual inspection becomes challenging in higher dimensions. Thus, developing outlier detection algorithms is necessary, especially for a space with many dimensions.

Although many methods have been proposed, outlier detection remains challenging for the following reasons. (\romannumeral1) It is difficult to find precise support for regular data points in real-life data \cite{chandola2009anomaly}; (\romannumeral2) the definition of outliers varies substantially from one domain to another \cite{wang2019survey}; (\romannumeral3) distinguishing outliers from \emph{noise} is not trivial \cite{wang2019survey}. Furthermore, most outlier detection algorithms require input parameters that are too technical for non-professionals to understand, and the trial-and-error processes can be tedious and time-consuming. For this reason, we propose outlier detection methods that are either input-parameter-free or require only understandable input parameters that can be determined easily beforehand.

Additionally, \emph{masking} and \emph{swamping} are common problems in outlier detection. The masking problem occurs when an outlier is hidden by similar outliers that are close. Generally, it occurs among collective outliers. On the other hand, the swamping problem occurs when a regular observation is falsely labeled as outliers given either the effect of nearby true outliers or other close regular points that exhibit different behaviour \cite{bhar2013detection}.

Several strategies are proposed to avoid masking and swamping in outlier detection: employing robust statistics like median, trimmed means, and Median Absolute Deviation about the median (MAD) \cite{huber2011robust}; visualizing data with graphics (e.g., box plots) \cite{wang2018masking}; set the number of outliers to detect as an input parameter \cite{fung1999outlier}. These approaches help identify true outliers accurately without mislabeling non-outliers.

We propose outlier detection algorithms based on \emph{Cluster Catch Digraphs} (CCDs), which were first introduced by Devinney \cite{devinney2003class} and improved by Marchette \cite{marchette2005random}, developed from a similar classification digraph called \textit{Class Cover Catch Digraphs} (CCCDs). Later, Manukyan and Ceyhan \cite{manukyan2019parameter} modified and improved this approach further, developing two variants that use a Kolmogorov-Smirnov (KS) based statistic and Ripley's \textit{K} function, respectively, calling the associated digraphs KS-CCDs and RK-CCDs. RK-CCDs and KS-CCDs work similarly in clustering, and RK-CCDs are almost parameter-free, making them especially appealing. However, our experimental analysis shows that RK-CCDs may not be suitable for moderate to high dimensionality. Thus, we introduce another CCD-based approach that uses \emph{nearest neighbors} instead of Ripley's $K$ function to test underlying point-process patterns.

Given a data set, RK-CCDs and UN-CCDs construct an open (hyper)sphere for each latent cluster, called covering balls. Experimental results show that the covering balls catch the majority of points of a data set, which are considered regular points \cite{manukyan2019parameter}. On the other hand, we can find outliers among those points not covered by any covering balls, which are generally far away from any clusters and located in low-density regions. This is appealing and is also the motivation of this paper. We adapt RK-CCDs and UN-CCDs on two CCD-based outlier detection algorithms called the RU-MCCD and UN-MCCD algorithms; then, we propose two ``flexible" variations called the SU-MCCD and SUN-MCCD algorithms aiming at outlier detection on the data sets with arbitrary-shaped clusters.

By conducting comprehensive Monte Carlo experiments, we demonstrate that our algorithms exhibit wide adaptability and can deliver promising results across different data sets, even with high dimensionality. The paper is organized as follows: 

Section \ref{sec:prelim} covers previously proposed algorithms in outlier detection. We focus on the graph-based, density-based, cluster-based methods and previous works on CCCDs and CCDs. In Section \ref{sec:CCD4outliers}, we proposed \emph{Mutual Catch Graphs} (MCGs) based on KS-CCDs and its application on outlier detection given a single cluster. Then, we combine MCGs and CCDs, proposing four CCD-based outlier detection algorithms, called RU-MCCDs, UN-MCCDs, SU-MCCDs, and SUN-MCCDs, respectively. We conduct extensive simulations to assess the performance of all the CCD-based outlier detection algorithms starting from Section \ref{sec:Simul_CCDs}.

To help readers navigating the specialized terminology used throughout this paper, we enumerate a list of acronyms and their full terms below.

\begin{table}[htb]
  \centering
  \resizebox{\textwidth}{!}{\begin{tabular}{|c|c|}
  \hline
  \textbf{Abbreviation} & \textbf{Full Term} \\ \hline
  CCDs & Cluster Catch Digraphs\\ \hline
  RK-CCDs & The CCDs based on the Ripley's $K$ function\\ \hline
  KS-CCDs & The CCDs based on the a KS-based statistic\\ \hline
  UN-CCDs & Uniformity- and Neighbor-based CCDs\\ \hline
  D-MCGs & Density-based Mutual Catch Graphs\\ \hline
  U-MCCDs & Uniformity-Based CCDs with Mutual catch graph\\ \hline
  RU-MCCDs & Rapid Uniformity-Based CCDs with Mutual catch graph\\ \hline
  SU-MCCDs & Shape-adaptive Uniformity-based CCDs with Mutual catch graph\\ \hline
  UN-MCCDs & Uniformity- and Neighbor-based CCDs with Mutual catch graph\\ \hline
  SUN-MCCDs & Shape-adaptive Uniformity- and Neighbor-Based CCD with Mutual catch graph\\ \hline
  SR-MCT & Spatial Randomness Monte Carlo Test\\ \hline
  HPP & Homogeneous Poisson Process\\ \hline
  CSR & Complete Spatial Randomness\\ \hline
  NND & Nearest Neighbor Distance\\ \hline
  MAD & Median Absolute Deviation about the median\\ \hline
  MADN & Normalized Median Absolute Deviation about the median\\ \hline
  TPR & True Positive Rate\\ \hline
  TNR & True Negative Rate\\ \hline
  BA & Balance Accuracy\\ \hline
  \end{tabular}} \label{tab:acronyms}
\end{table}

\section{Background and Preliminaries}
\label{sec:prelim}
Researchers have proposed various outlier detection methods, and they are mainly categorized into graph-based, density-based, cluster-based, and statistical-based methods based on their core ideas \cite{smiti2020critical,olteanu2023meta}. The cluster-based methods generally operate in two phases: identifying clusters and pinpointing outliers within them \cite{wang2019survey}. We focus on the non-parametric categories (i.e., graph-based, density-based, and cluster-based methods), enumerating well-known methods, and include some (possible) subsequent variants developed following the prototypes.

\subsection{Graph-Based Methods}

Graph-based outlier detection methods employ graph theoretic techniques that capture outliers by constructing interdependence ties among observations \cite{wang2019survey}. These methods are suitable in scenarios where data is inherently relational, such as social networks, biological networks, and communication networks. We enumerate some well-known algorithms in this category below.

Noble and Cook proposed two graph-based anomaly detection methods with the \textit{Subdue system} \cite{noble2003graph}, which flag unusual subgraphs or substructures. \emph{OddBall} \cite{akoglu2010oddball} discovers substantial outlying patterns by four features. Hautamaki \textit{et al.} introduced \emph{Outlier Detection using In-degree Number} (ODIN), which assumes that outliers have a substantially lower in-degree than regular points in a $k$-Nearest Neighbor ($k$NN) graph \cite{hautamaki2004outlier}. Liu \textit{et al.} proposed an unsupervised-learning algorithm called \emph{Isolation Forest} \cite{liu2008isolation}, with the notion that outlier points have distinct characteristics, making them easier to isolate than regular data points in a binary tree. Other well-known method include \emph{OutRank} \cite{moonesinghe2006outlier}, \emph{Community Outlier Detection Algorithm} (CODA) \cite{gao2010community}, and \emph{Local Information Graph-based Random Walk model} (LIGRW) \cite{wang2018new}.

\subsection{Density-Based Methods}

Density-based methods identify outliers among points in low-density regions. Typically, these approaches measure a point's outlyingness by comparing its local density with those of its nearest neighbors.

\emph{Local outlier factor} (LOF) \cite{breunig2000lof} is one of the prototype methods in this category, which introduces \emph{local reachability density} to compute the local outlyingness of a point. Tang \textit{et al.} proposed \emph{Connectivity-based Outlier Factor} (COF) \cite{tang2002enhancing} that performs better than LOF on the outliers that deviate from their neighbor patterns but with similar local density. A similar method called \emph{LOcal Correlation Integral} (LOCI) \cite{papadimitriou2003loci} was proposed by Papadimitriou \textit{et al.}, coming with a data-orientated threshold for outlyingness score. Kriegel \textit{et al.} formulated a new outlyingness score called \emph{Local Outlier Probabilities} (LoOP) \cite{kriegel2009loop}, which represents the probability of a point being an outlier, greatly enhancing the interpretability. Other density-based outlier detection algorithms include \emph{Relative Density Factor} (RDF) \cite{ren2004rdf}, \emph{INFLuenced Outlier-ness} (INFLO) \cite{jin2006ranking}, \emph{Resolution-based Outlier Factor} (ROF) \cite{fan2009resolution}, \emph{Dynamic Window Outlier Factor} (DWOF) \cite{momtaz2013dwof}, \emph{High Contrast Subspaces} (HiCS) \cite{keller2012hics}, \emph{Simplified LOF} \cite{schubert2014local}, \emph{Global-Local Outlier Scores from Hierarchies} (GLOSH) \cite{campello2015hierarchical}, and \emph{Simple uni-variate Probabilistic Anomaly Detector} (SPAD) \cite{aryal2016revisiting}.

\subsection{Cluster-Based Methods}

Clustering is an unsupervised method that groups points that are close or behave similarly. Small clusters with substantially fewer points or isolated points far apart from other clusters could be labeled as outliers. Outliers often come as by-products of clustering algorithms.

So far, cluster-based methods have been classified into several subgroups, known as partitional, hierarchical, and density-based. Many are formulated with robust mechanisms against outliers \cite{wang2019survey}.

Partitional clustering methods create a single-level partition of the data set \cite{wang2019survey}. These algorithms typically begin with a pre-specified number of clusters, often represented by their centers, which can be obtained through a simple method like random selection. The partitions are then iteratively updated until a specific object function is optimized. The most commonly known algorithms include \emph{$k$-means} \cite{gareth2013introduction}, \emph{MacQueen} \cite{macqueen1967classification}, \emph{Partitioning Around Medoids} (PAM) \cite{kaufman2009finding}, \emph{Clustering LARge Applications} (CLARA) \cite{kaufman2009finding} and \textit{Clustering Large Applications based on RANdomized Search} (CLARANS) \cite{ng2002clarans}.

Hierarchical clustering methods construct a hierarchical tree-like structure called \emph{dendrogram} and partition the whole data set based on the desired granularity. It can be divided into two subgroups called \emph{agglomerative} and \emph{divisive clustering} \cite{zhang2013advancements}. One of the popular algorithms is the \emph{Minimal Spanning Tree} (MST) method \cite{MST}, which constructs a minimal spanning tree that connects all data points and removes ``inconsistent" edges to obtain clusters and outliers. Other algorithms include \emph{Clustering Using Representatives} (CURE) \cite{cure}, \textit{CHAMELEON} \cite{karypis1999chameleon}, \emph{Robust Clustering using links} (ROCK) \cite{guha2000rock}.

The core idea of the density-based clustering method involves identifying the regions where data points are dense as clusters. Some well-known examples include \emph{Density-Based Spatial Clustering of Applications with Noise} (DBSCAN) \cite{ester1996density}, which captures clusters by first finding some core points and expanding them to clusters. Other well-known algorithms include \emph{Ordering Points To Identify the Clustering Structure} (OPTICS) \cite{ankerst1999optics}, \emph{Distribution Based Clustering of LArge Spatial Databases} (DBCLASD) \cite{xu1998distribution}, and \emph{DENsity-based CLUstEring} (DENCLUE) \cite{hinneburg1998efficient}.

\subsection{Evaluation Metrics in Outlier Detection}

Although many outlier detection algorithms have been introduced over the last two decades, there has yet to be an agreed-upon answer to how to measure the performance of an outlier detection algorithm \cite{wang2019survey}. Although researchers have always concluded that their approaches are comparable to or outperform existing algorithms, some of their conclusions are subjective due to the choice of the evaluation metric, and a more comprehensive empirical analysis is needed \cite{wang2019survey}. We choose \emph{True Positive Rate} (TPR), \emph{True Negative Rate} (TNR), \emph{Balanced Accuracy} (BA), and $F_2$-scores as evaluation matrics in Monte Carlo simulations.

TPR (i.e., recall) and TNR measure the ratio of correctly identified outliers and regular points. However, outlier detection is essentially a classification problem over highly imbalanced data sets, the performance of which should not be solely measured by plain accuracies or errors \cite{chawla2004special, daskalaki2006evaluation}. Therefore, we also consider using BA and $F_2$-scores. BA is the mean of TPR and TNR, and $F_2$-scores is the weighted harmonic mean of precision and recall. Both of them focus on positive and negative observations and are widely used in highly imbalanced data sets; they are suitable for evaluating the performance of outlier detection algorithms \cite{sokolova2006beyond}.

The outlier detection algorithms we propose are based on CCDs. CCDs are digraphs with all data points as vertices and arcs determined by the spherical balls centered at the vertices. Actually, \textit{Class Cover Catch Digraphs} (CCCDs), formulated by Priebe \textit{et al.} \cite{marchette2003classification}, are prototypes of CCDs. CCCDs are powerful tools for supervised classification. Following the chronological order, we will first discuss CCCDs briefly.

\subsection{Class Cover Catch Digraphs}

Given a data set $\mathcal{X} \subset \mathbb{R}^d$ that consists of \emph{i.i.d} points from two classes $\mathfrak{X_0}=\{x_1,x_2,...,x_n\}$ and $\mathfrak{X_1}=\{y_1,y_2,...,y_m\}$, i.e., $\mathcal{X} = \mathfrak{X_0} \cup \mathfrak{X_1}$. Without loss of generality, here we refer to the class of interest  $\mathfrak{X_0}$ as \emph{target class} and $\mathfrak{X_1}$ as \emph{non-target class}.

\emph{Class Cover Problem} (CCP) aims to distinguish the target class ($\mathfrak{X_0}$) from the non-target class ($\mathfrak{X_1}$) by finding a minimum collection of open balls or hyperspheres $B_i =B(a_i,r_i) = \{x|d(a_i,x)<r_i,\ x \in \mathcal{X} \}$ such that $\cup_iB_i$ covers all the points of the target class $\mathfrak{X_0}$ while excluding the non-target class ($\mathfrak{X_1}$) \cite{marchette2005random}.

CCCDs address the CCP. A CCCD for $\mathfrak{X_0}$, denoted as $D_0=(V_0,A_0)$, is a digraph with vertex set $V_0=\mathfrak{X_0}$ and arc set $A_0$. It starts by constructing a \textit{covering ball} $B(x_i,r_{x_i})$ centered at each $x_i \in V_0$. For any two distinct vertices $x_i,x_j \in V_0$, the arc $(x_i,x_j) \in A_0$ if and only if $x_j \in B(x_i,r_{x_i})$. We could also build a CCCD for $\mathfrak{X_1}$ by swapping the roles of the two classes. Currently, there are two variants of CCCDs: pure-CCCDs (P-CCCDs) and random walk-CCCDs (RW-CCCDs). They differ in the criterion used to determine the radius $r_{x_i}$ for each covering ball $B(x_i,r_{x_i})$ \cite{manukyan2016classification}. We will not discuss their details here.

\subsubsection{The Approximate Minimum Dominating Sets}
\label{sec:aMDS}
With the above construction, a digraph $D_i=(V_i,A_i)$ and a cover $\cup_i B_i$ for the target class $\mathcal{X_i}$ ($i=0$ or $1$) either by P-CCCDs or by RW-CCCDs can be obtained. However, to avoid the over-fitting problem, we may want to reduce the complexity of the covers by keeping only a certain number of covering balls and dropping the others \cite{marchette2003classification}. The centers of these retained covering balls are called the \textit{prototype set}. Obtaining a \textit{minimum dominating set} (MDS) $S_i$ for $D_i$ is one way to achieve this goal.

Finding an MDS is generally an NP-Hard optimization problem \cite{karp1972reducibility}. Fortunately, the Greedy Algorithm \ref{alg:greedy1} below provides an efficient way to find an approximate MDS in $O(|V_0|^2)$ time \cite{chvatal1979greedy, hochbaum1982approximation}. The algorithm initializes with all vertices as uncovered and an empty dominating set. It iteratively selects the vertex with the maximum outdegree, adds it to the dominating set, and removes its closed neighborhood from the set of uncovered vertices. This process repeats until all vertices are covered.
Additionally, there are two more variants of greedy algorithms proposed by Manukyan and Ceyhan \cite{manukyan2019parameter}, differing in the way of choosing vertex at each iteration. The first variant is presented as the Greedy Algorithm \ref{alg:greedy2} below, and this variant is tailored for CCDs. At each iteration, it selects the vertex with the maximum outdegree in the initial digraph, such that the members of the dominating set will be closer to the cluster centers. The second variant is greedy in a score function, i.e., chooses a vertex $v$ that maximizes a score function $sc(v)$ at each iteration. It is presented as the Greedy Algorithm \ref{alg:greedy3}.

In general, RW-CCCD outperforms P-CCCD in classification, especially when the data set is highly imbalanced \cite{manukyan2016classification}.

\begin{algorithm}[htb]
\setcounter{algocf}{0}
\SetAlgorithmName{Greedy Algorithm}{}{}
\caption{(A greedy algorithm finding an approximate MDS) $D^{sub}(S)$ is the induced sub-digraph of vertex set $S$ from a digraph $D$, $\bar{N}(v)$ is the closed neighborhood of a vertex $v$. $V_{temp}$ represents the uncovered vertices at current iteration.}
\label{alg:greedy1}
\KwData{A digraph $D=(V(D),A(D)).$ for a given data set $\mathcal{X} = \{x_1,x_2,...,x_n\}$}\\
\KwResult{A approximate minimum dominate set $\hat{S}.$}\\
\nl\textbf{Initialization:} $V_{temp} \gets V(D)$, $\hat{S} \gets \emptyset$\\
\nl\While{$V_{temp} \neq \emptyset$}{
\nl $v_{temp}\gets \arg\max_{v \in V(D)}\{d_{out}(v)\}$\Comment*[r]{$d_{out}(v)$:the outdegree of $v$ in $A(D)$}
\nl $V_{temp} \gets V_{temp} \symbol{92} \bar{N}(v_{temp})$\;
\nl $\hat{S} \gets \hat{S} \cup \{v_{temp}\}$\;
\nl $D \gets D^{sub}(V_{temp})$\;
}
\end{algorithm}

\begin{algorithm}[htb]
\SetAlgorithmName{Greedy Algorithm}{}{}
\caption{(A greedy algorithm finding an approximate MDS) This greedy algorithm is adapted for Cluster Catch Digraphs (CCDs).}\label{alg:greedy2}
\KwData{A digraph $D=(V(D),A(D))$ for a given data set $\mathcal{X} = \{x_1,x_2,...,x_n\}$}\\
\KwResult{A approximate minimum dominate set $\hat{S}$}\\
\nl\textbf{Algorithm Steps:} It is similar to Greedy Algorithm \ref{alg:greedy1}, except that it iteratively selects the vertex with the maximum outdegree in the initial digraph.
\end{algorithm}

\begin{algorithm}[htb]
\SetAlgorithmName{Greedy Algorithm}{}{}
\caption{(A greedy algorithm finding an approximate MDS) This algorithm is similar to Greedy Algorithm \ref{alg:greedy1}, except that it is greedy in a score function $sc(v)$ at each iteration.}\label{alg:greedy3}
\KwData{A digraph $D=(V(D),A(D))$ for a given data set $\mathcal{X} = \{x_1,x_2,...,x_n\}$}\\
\KwResult{A approximate minimum dominate set $\hat{S}$}\\
\end{algorithm}

\subsection{Cluster Catch Digraphs Using a KS-Based Statistic}

The CCCD approach for classification was adapted to clustering, and CCDs were introduced by DeVinney \cite{devinney2003class}. Suppose there is an unlabeled data set $\mathcal{X} = \{x_1,x_2,...,x_n\}$ in $\mathbb{R}^d$ drawn from a mixture distribution, where each component of the mixture represents a cluster, the goal is to determine the number of clusters and the optimal partition. Unlike CCCDs, CCDs determine the optimal radius of each covering ball using a Kolmogorov-Smirnov (KS)-based statistic. The KS-based statistic measures the ``clustered-ness" around a point $x_i \in \mathcal{X}$ \cite{marchette2005random}, and it is defined as follows,
\begin{equation}\label{equ:KS_stat}
	T_{KS}(x_i,r)= F_{rw}(x_i,r)-F_0(x_i,r),
\end{equation}
where $F_{rw}(x_i,r)$ equals the number of points caught by the covering ball $B(x_i,r)$. The second term $F_0(x_i,r)$ represents the expected number of points in $B(x_i,r)$ under a null distribution. For example, under the common assumption of \textit{Complete Spatial Randomness} (CSR), which is also known as \textit{Homogeneous Poisson Process} (HPP), we can take $F_0(x_i,r)=\delta r^d$ \cite{marchette2005random}, where $d$ represents the dimensionality and $\delta$ is an input density parameter. Based on the Kolmogorov-Smirnov (KS) type test, the optimal radius $r_{x_i}$ is chosen to maximize $T_{KS}(x_i,r)$, i.e.,
\begin{equation}\label{equ:KS_rad}
	r_{x_i}=\arg\max_{r \geq 0}\{T_{KS}(x_i,r)\}.
\end{equation}

By maximizing $T_{KS}(x_i,r)$, the value of the radius is selected with the notion that the most clustered points around $x_i$ are covered by $B(x_i,r_x)$ \cite{marchette2005random}.

Once the radii are determined, a CCD for $\mathcal{X}$, denoted as $D=(V(D), A(D))$, can be constructed. The weakly connected components of $D$ (i.e., $\cup_{i=1}^{m} C_i=\mathcal{X}$) can be returned as clusters. However, for each cluster found, its covering balls are not equally important. Thus, for the same reason as CCCDs, obtaining a lower complexity cover is desired. Similar to CCCDs, this goal can be achieved by finding an approximate MDS. Marchette proposed two versions of modified greedy algorithms to find an approximate MDS for CCDs \cite{marchette2005random}. Despite the two modified versions, Manukyan and Ceyhan prefer the Greedy Algorithm \ref{alg:greedy2}.

Although an approximate MDS $\hat{S}$ reduces the cover complexity, not all its covering balls are necessary. To further reduce the complexity of the cluster cover, one can identify the ``core" covering balls by constructing an \emph{intersection graph}, denoted as $G_{MD}=(V_{MD}, E_{MD})$, where $V_{MD}=\hat{S}$, and for any points $u,v \in \hat{S}$, the edge $uv\in E_{MD}$ if and only if $B(u,r_u)$ and $B(v,r_v)$ cover some common points in $\mathcal{X}$. With the intersection graph $G_{MD}$, one can implement Greedy Algorithm \ref{alg:greedy1} to prune $\hat{S}$ again. The approximate MDS of $G_{MD}$ is denoted as $\hat{S}(G_{MD})$, and each covering ball of $\hat{S}(G_{MD})$ represents a \emph{latent cluster}.

Although we have reduced the cover complexity in two sequenced phases and can obtain a partition $P=\{P_1, P_2,..., P_k\}$ for $\mathcal{X}$, the clustering result is not robust to noise and outlier clusters. Therefore, Manukyan and Ceyhan \cite{manukyan2019parameter} employs the \textit{silhouette index} \cite{gan2020data} to identify and remove redundant clusters. Silhouette index of $x_i$, written as $sil(x_i)$, is a metric measuring how well $x_i$ is clustered in terms of the partition $P$. Manukyan and Ceyhan \cite{manukyan2019parameter} first rank the partitions in $P$ in a decreasing order based on their size. Starting from the first two, they add partitions incrementally as valid clusters until the average silhouette index of the entire data set (denoted as $sil(P)$) is maximized. Indicating that no more clusters are necessary, and we call the covering balls retained as the \emph{dominating covering balls} of the intersection graph.

For the point $x_i \in \mathcal{X}$ that is not covered by any selected clusters (covering balls), it can be assigned to the nearest cluster (covering ball) with minimal relative similarity measure. The relative similarity measure between $x_i$ and the covering ball $B(x_j,r_{x_j})$, denoted as $\rho(x_i,B(x_j,r_{x_j}))$, can be computed as follows,
\begin{equation}\label{equ:ReDist}
	\rho(x_i,B(x_j,r_{x_j})) =d(x_i,x_j)/r_{x_j}.
\end{equation}
For simplicity, Manukyan and Ceyhan \cite{manukyan2019parameter} refer to the CCDs based on a KS-based statistic as KS-CCDs.

\subsection{Cluster Catch Digraphs using Ripley's \textit{K} Function}
Although utilizing silhouette index enhances the robustness of KS-CCDs to outliers or noise clusters, there are still a few shortcomings due to the intrinsic property of the KS-based statistic. It is a density-based statistic falling short of delivering insight into the spatial distribution of data points. As a result, it may falsely return two or more clusters as one \cite{manukyan2019parameter}. Additionally, the input density parameter $\delta$ is usually unknown beforehand. As a result, an appropriate value of this parameter can only be obtained via a costly trial-and-error process in most cases.

To tackle the shortcomings above, instead of using the KS-based statistic, Manukyan and Ceyhan \cite{manukyan2019parameter} applied Ripley's \textit{K} function \cite{ripley1976second}, denoted as $K(t)$, and designed a distribution-based test to determine whether the points inside each covering ball follow an HPP. This test will be referred to as the \textit{Spatial Randomness Monte Carlo Test} (SR-MCT) with Ripley's \textit{K} function in this article. For each covering ball, an optimal radius can be specified as the maximum possible value that the points covered satisfy an HPP. The resulting algorithms are called the RK-CCD algorithm. It is worth noting that the only difference between RK-CCDs and KS-CCDs is the way to determine the values of radii.

Manukyan and Ceyhan \cite{manukyan2019parameter} also proposed another variant of RK-CCDs, which aims to find clusters with arbitrary shape. In this variant, rather than the approximate MDS, the connected components of the intersection graph $G_{MD}$ are considered to be clusters.

\subsection{Our Contribution}

In this paper, we first introduce the RU-MCCD algorithm, which combines RK-CCDs and the \emph{Mutual Catch Graph} from KS-CCD, and find potent outliers within some low-density regions. To tackle the data sets in high dimensional space, we introduce another CCD-based clustering algorithm called UN-CCDs, which utilizes the \emph{Nearest-Neighbor Distance} (NND) to test CSR. Then, we adapt UN-CCDs similarly for outlier detection, and the resulting algorithm is called the UN-MCCD algorithm.

The RU-MCCD and UN-MCCD algorithms find clusters in (approximate) spherical shapes. To construct covers for arbitrary-shaped clusters, we introduce the SU-MCCD and SUN-MCCD algorithms, the ``flexible" variants of the first two CCD-based outlier detection algorithms. Extensive experiments show they deliver better performance in general when the shape of the clusters is arbitrary or the dimensionality of a data set is high.

Besides the four CCD-based outlier detection algorithms, we have also introduced two types of scores, \emph{Outbound Outlyingness Score} (OOS) and \emph{Inbound Outlyingness Score} (IOS), to quantify the outlyingness of a point. To be used, they must be combined with a CCD-based algorithm. In experimental analysis, we found that IOS performs exceptionally well; it is robust to the masking and swamping problem and achieves promising results even on a data set with a dimensionality of 100.

In summary, we enumerate our contributions as the follows:
\begin{itemize}
  \item[\romannumeral1] \textbf{The RU-MCCD algorithm}: Combines RK-CCDs and Mutual Catch Graphs (MCGs) for outlier detection in low-density regions.
  \item[\romannumeral2] \textbf{UN-CCDs for clustering}: Utilize the Nearest-Neighbor Distance instead of Ripley's $K$ function to test CSR and are more effective in high-dimensional spaces.
  \item[\romannumeral3] \textbf{The UN-MCCD algorithm}: An adaptation of UN-CCDs specifically tailored for outlier detection.
  \item[\romannumeral4] \textbf{The SU-MCCD and SUN-MCCD algorithms}: The shape-adaptive version of the UN-MCCD and SU-MCCD algorithms to handle data sets with clusters of arbitrary shapes.
  \item[\romannumeral5] \textbf{Outbound Outlyingness Score (OOS) and Inbound Outlyingness Score (IOS)}: New metrics to quantify how much a data point deviates from regular points, particularly with IOS demonstrating robustness to masking and swamping issues.
\end{itemize}

\section{Outlier Detection with Cluster Catch Digraphs}
\label{sec:CCD4outliers}
\subsection{The Mutual $k$-Nearest-Neighbor Graphs}

Brito \textit{et al.} \cite{brito1997connectivity} proposed an approach that uses the mutual $k$-nearest neighbor (m$k$NN) graph to detect latent clusters, and Marchette \cite{marchette2005random} pointed out that this approach is appropriate for outlier detection. The main idea of Brito's approach is to identify latent clustering structures or outliers by examining the local connectivity of each point of a data set $\mathcal{X}$. To achieve this goal, they formulated a test measuring the connectivity of the m$k$NN graph (which is denoted as $\tilde{G}_k(\mathcal{X})$) for the given data set $\mathcal{X}$. Under the null hypothesis $$H_0: \text{``no clustering structure or no outliers" (i.e., data forms a single cluster),}$$ the test assumes that $\tilde{G}_k(\mathcal{X})$ should be connected given a $k$ value less than or equal to a certain threshold $k_{max}$. Appropriate value(s) of $k_{max}$ is(are) determined by Monte Carlo simulation and a model using the Ordinary Least Squares (OLS). Once $k_{max}$ has been found, the m$k$NN graph, $\tilde{G}_{k_{max}}(\mathcal{X})$, is examined to determine whether it can be partitioned into multiple components. In general, these components are returned as separate clusters or outliers. However, labeling these components can be somewhat challenging \cite{brito1997connectivity}. Later, Marchette \textit{et al.} \cite{marchette2005random} suggested that this method is more suitable for outlier detection because it is susceptible to the presence of contextual (local) outliers. In the following section, we introduce our first approach based on a similar idea.

\subsection{The Mutual Catch Graphs}

Inspired by Brito's approach that focuses on the connectivity of the m$k$NN graph, we have adopted a similar idea to KS-CCDs or RK-CCDs. Rather than constructing an m$k$NN graph, we introduce \emph{mutual catch graphs} (MCGs), and it is defined as follows:

\begin{Def}[Mutual Catch Graphs (MCGs)]\label{sec:Def_MCG}
Given a data set $\mathcal{X} = \{x_1,...,x_n\}$ of i.i.d points and a CCD denoted as $D(\mathcal{X})$, a Mutual Catch Graph (MCG), denoted as $G_M(\mathcal{X}) := (V(\mathcal{X}), E_M(\mathcal{X}))$, is constructed with $V(\mathcal{X}) = \mathcal{X}$. $E_M(\mathcal{X})$ is comprised with the edges $x_ix_j$ for distinct $x_i,x_j \in \mathcal{X}$ iff $d(x_i,x_j) < \min(r_{x_i},r_{x_j})$. Here, $r_{x_i}$ and $r_{x_j}$ represent the radii of covering balls for $x_i$ and $x_j$ regarding $D(\mathcal{X})$, respectively, implying an edge exists if their covering balls satisfy the ``mutual catch" property (i.e., catch (or cover) each other mutually).
\end{Def}

\subsection{The Density-based Mutual Catch Graph Algorithm}
Recall that a covering ball $B(x_i,r_i)$ in CCDs captures the largest possible latent cluster structure around a point $x_i$. Thus, any points captured by $B(x_i,r_i)$ seem to belong to the same cluster with $x_i$. With this notion, a pair of points connected in $G_M(\mathcal{X})$ are likely to belong to the same cluster.

Similar to Brito \textit{et al.}'s approach, we take the same null hypothesis that there is only one cluster with no outliers. Under $H_0$, every point is drawn from a distribution $F$ with compact and connected support $S$ and bounded density $f$. Therefore, with all observations aggregating within $S$, the MCG $G_M(\mathcal{X})$ obtained from a KS-CCD should be connected even when the density parameter $\delta$ for the KS-based statistic is relatively large. Therefore, $\delta$ is analogous to the $k$ in an m$k$NN graph. Hence, we want to find a threshold for $\delta$ and to test $H_0$, identifying latent clusters or outliers when possible. Similar to Brito \textit{et al.}'s approach, the threshold can be determined by Monte Carlo simulations. More specifically, we simulate a data set $\mathcal{X}^*$ from the distribution $F$ (estimate it if unknown) with the same size as the given data set $\mathcal{X}$. We record the maximal value of $\delta$ such that the MCG, $G_M(\mathcal{X}^*)$, is connected. We repeat this procedure $m$ times and obtain a sample of $m$ $\delta$ values. Finally, we use a chosen sample quantile as the threshold for $\delta$.

Although Brito \textit{et al.}'s and our approaches are similar, our approach is density-based. In contrast, Brito's approach only measures the connectivity of the whole data set globally and ignores local density. Our approach is proposed as Algorithm \ref{alg:DMCG_Algo} below, and we call this approach the \emph{Density-based Mutual Catch Graph} (D-MCG) algorithm for clustering and outlier detection.

In Algorithm \ref{alg:DMCG_Algo}, the vertex set $V(\mathcal{X}^*)=\mathcal{X}^*$ in step 8. In step 9, for any $u,v \in \mathcal{X}^*$, the edge $uv \in E(\mathcal{X}^*)$ if and only if $v$ and $u$ are ``mutually caught" with their covering balls $B_u$ and $B_v$, respectively. Finally, if more than one component is detected, further investigation is needed to decide whether these components are clusters or outliers.

\begin{algorithm}[htb]
\small
\setcounter{algocf}{0}
\SetAlgorithmName{Algorithm}{}{}
\KwData{$\delta_0$, $\Delta$, $M$, $\alpha$ and a dataset $\mathcal{X}$}\;
\KwResult{Connected components of $\mathcal{X}$ (potential clusters or outliers)}\;
\textbf{Algorithm Steps:}\\
\nl Initialize under the assumption that $\mathcal{X}$ has no outliers or other clusters, based on a distribution $F$ with connected (estimated) support $S$\;
\nl $n \gets |\mathcal{X}|$ (i.e., the size of $\mathcal{X}$)\;
\nl $i \gets 1$\;
\nl $\delta_{seq} \gets \emptyset$\;
\nl \While{$i \leq M$}{
\nl	$\delta \gets \delta_0$\;
\nl	Simulate a data set $\mathcal{X}^*$ of size $n$ from the distribution $F$\;
\nl	Construct $D(\mathcal{X}^*)=(V(\mathcal{X}^*),A(\mathcal{X}^*))$: the KS-CCD of $\mathcal{X}^*$ (with density parameter $\delta$)\;
\nl	Construct $G_M(\mathcal{X}^*)=(V(\mathcal{X}^*),E(\mathcal{X}^*))$: the MCG of $D(\mathcal{X}^*)$\;
\nl	\While{$G_M(\mathcal{X}^*)$ is not connected}{
\nl		$\delta \gets \delta-\Delta$\;
\nl		Repeat steps 8 and 9 to update $D(\mathcal{X}^*)$ and $G_M(\mathcal{X}^*)$\;
	}
\nl $\delta_{seq} \gets \delta_{seq}\cup \{\delta\}$\;
\nl $i \gets i+1$\;
}
\nl Find the $\alpha$ quantile of $\delta_{seq}$, and denote it as $\delta_{\alpha}$\;
\nl Construct $D_\alpha(\mathcal{X})=(V(\mathcal{X}),A(\mathcal{X}))$: the KS-CCD of $\mathcal{X}$ (with density parameter $\delta_{\alpha}$)\;
\nl	Construct $G_{\alpha,M}(\mathcal{X})=(V(\mathcal{X}),E(\mathcal{X}))$: the MCG of $D_\alpha(\mathcal{X})$\;
\nl\eIf{$G_{\alpha,M}(\mathcal{X})$ is connected}{
\nl	Retain $H_0$, and return $\mathcal{X}$ as the single component\;
    }{
\nl Reject $H_0$ at $\alpha$ level and return the connected components of $G_M(\mathcal{X})$ either as clusters or outliers\;
}
\caption{\textbf{(D-MCG algorithm)} Tests for presence of clusters or outliers in $\mathcal{X}$, utilizing a density parameter $\delta$ adjusted through simulation. Parameters: initial density $\delta_0$, density decrement $\Delta$, simulation count $M$, quantile $\alpha$.}
\label{alg:DMCG_Algo}
\end{algorithm}

Recall that when we apply RK-CCDs and KS-CCDs for clustering, we use an intersection graph to reduce the cover complexity of the approximate MDS because covering balls for the same cluster are likely to overlap. We want to find only one representative covering ball for each cluster. On the contrary, considering applying CCDs for outlier detection, we see that the covering balls of an outlier and a regular observation often cover some common points but rarely catch the center of each other simultaneously. Therefore, we employ the MCG technique instead of an intersection graph for outlier detection to avoid any edges between outliers and regular observations.

We illustrate this algorithm under two simple artificial data sets with outliers. Under the first simulation setting, the regular data points are generated uniformly within a unit hypersphere $B(0_d,1)$ ($0_d$ is the origin of a $d$-dimensional space), i.e., $x_i$ are drawn from Uniform[$B(0_d,1)$]. Outliers are drawn uniformly from another unit hypersphere with a certain distance (3 units) from the first. Figure \ref{fig:2d_fig_DMCG_Algo_a} presents a realization under this setting when $d=2$, where we have $3$ outliers out of $50$ (the contamination level is $6\%$). \label{sec:D-MCG_Simul}

Under the second simulation setting, the regular data points are also generated uniformly within a unit hypersphere $B(0_d,1)$. Outliers are distributed uniformly within the annulus between two hyperspheres $B(0_d,R_1)$ and $B(0_d,R_2)$, where $R_1=1.5$ and $R_2=3$. Thus, the distance from any outliers to $0_d$ is at least $1.5$, making the outliers separable from the regular data points. A realization in $\mathbb{R}^2$ is presented in Figure \ref{fig:2d_fig_DMCG_Algo_b}, where a data set of size $50$ is generated with $6\%$ of it being outliers.\label{sec:D-MCG_Simu2}

\begin{figure}[htb]
\centering
\subfigure[]{
\label{fig:2d_fig_DMCG_Algo_a}
\includegraphics[width=0.35\textwidth]{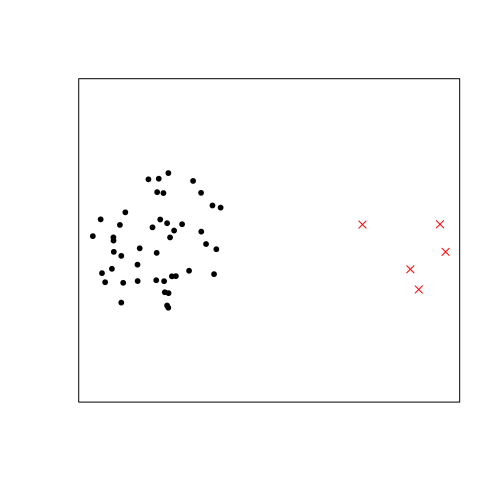}}
\subfigure[]{
\label{fig:2d_fig_DMCG_Algo_b}
\includegraphics[width=0.35\textwidth]{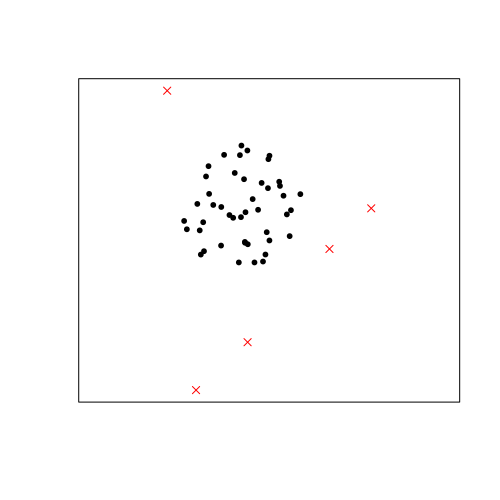}}
\subfigure[]{
\label{fig:2d_fig_DMCG_Algo_result_a}
\includegraphics[width=0.35\textwidth]{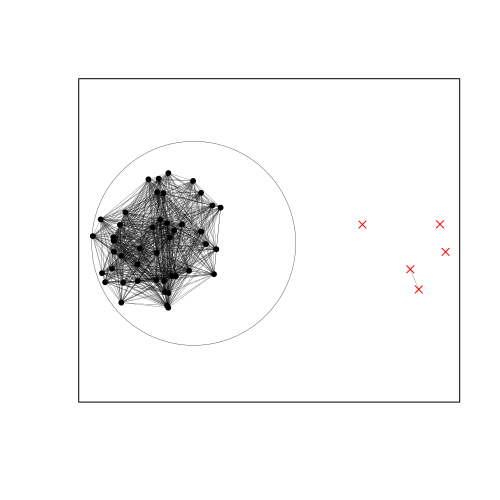}}
\subfigure[]{
\label{fig:2d_fig_DMCG_Algo_result_b}
\includegraphics[width=0.35\textwidth]{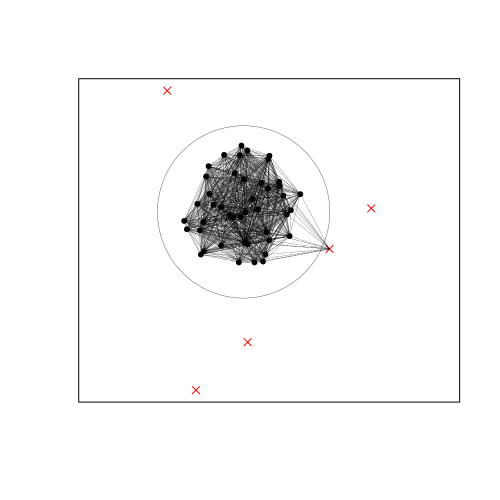}}
\caption{\small (a) A data set with 45 regular points (black) generated uniformly within a unit circle $B((0,0),1)$, and 5 outliers (red crosses) are drawn (uniformly) from another unit circle $B((3,0),1)$ that is 3 units away from the first one. (b) A data set that consists of 45 regular points (black) which are distributed uniformly within a unit circle $B((0,0),1)$, and $5$ outliers (red) that are drawn uniformly in the annular region between $B((0,0),1.5)$ and $B((0,0),3)$. (c) \& (d) The connected components returned by the D-MCG algorithm, the circles are the estimated support for regular data points, which are obtained by SVDD with the polynomial kernel of degree 1.}
\label{fig:2d_fig_DMCG_Algo}
\end{figure}

Here we know the support of the regular data points is hypersphere under both simulation settings, but this information is usually unavailable in real-world applications. Thus, when the support is unknown, we try to estimate the support using \emph{Support Vector Data Description} (SVDD) \cite{tax2004support} by assuming the regular data points are uniformly distributed. SVDD is a one-class classification method that constructs a boundary encompassing all regular points while excluding potential outliers. Similar to \emph{Support Vector Machine} (SVM), there are many kernels to choose from when conducts SVDD. We adopt a polynomial kernel with degree $1$ such that the boundary is a rigorous hypersphere, and the tuning parameter $C$ (which controls the volume of the hypersphere) is set to 0.05.

We present the estimated supports of SVDD and the connected components returned by the D-MCG algorithm in Figures \ref{fig:2d_fig_DMCG_Algo_result_a} and \ref{fig:2d_fig_DMCG_Algo_result_b}. For both data sets, the D-MCG algorithm delivers promising results in the sense that it can identify almost all outliers as single components (except one outlier in Figure \ref{fig:2d_fig_DMCG_Algo_result_b}) while connecting all the regular points. However, the estimated support is not robust to outliers. For instance, given the existence of outliers, the estimated supports of SVDD in Figure \ref{fig:2d_fig_DMCG_Algo} are not compact enough; additionally, the estimated support in Figure \ref{fig:2d_fig_DMCG_Algo_result_a} is dragged to the right by the five outliers, resulting in a misalignment.

\begin{Thm}[Time Complexity of Algorithm \ref{alg:DMCG_Algo}]\label{thm:D-MCG_Time}
  Given a data set $\mathcal{X} \subset \mathbb{R}^d$ of size $n$ ($d<n$). Suppose we simulate $M$ data sets from the (estimated) $F$ and $S$, then the time complexity of Algorithm \ref{alg:DMCG_Algo} is $O(M(n^2(d+\log n))+M\log M)$.
\end{Thm}

\begin{proof}
   When implementing Algorithm \ref{alg:DMCG_Algo}, we construct KS-CCDs and find connected components of the obtained MCGs for each simulated data set (of size $n$). Constructing KS-CCDs takes $O(n^2(d+\log n))$ time \cite{manukyan2019parameter}, establishing MCGs and finding the connected components can be done within $O(n^2)$ time in the worst cases. Thus, a total of $O(M(n^2(d+\log n)))$ time is needed for $M$ simulated data sets; sorting $\delta_{seq}$ and finding the $\alpha$ quantile requires $O(M\log M)$ time at most. Establishing a KS-CCD and finding connected components for $\mathcal{X}$ takes another $O(n^2(d+\log n))$ time. Thus, Algorithm \ref{alg:DMCG_Algo} runs in $O(M(n^2(d+\log n))+M\log M)$ time. Therefore, that is $O(n^2\log n)$ time for fixed $M$ as $n \rightarrow \infty$.
\end{proof}

\subsection{Outlier Detection with RK-CCDs and D-MCGs}
\subsubsection{Mutual Catch Graph with Cluster Catch Digraphs}
\label{sec:U-MCCD}
Although the D-MCG algorithm gives promising results on data sets with simple simulation settings, several limitations may affect their performance under more complex settings. These limitations include (1) The difficulty in determining whether the resulting connected components are outliers or clusters. It is a common problem for most outlier detection algorithms. Decisions can be made based on the cardinality, density, and (spatial) layout of connected components, but this approach is often unreliable and subjective, especially for high-dimensional data sets. (2) An appropriate distribution $F$ with support $S$ must be specified for the given data set before any simulations. Although we assume $F$ and $S$ are known in the D-MCG algorithm, they are usually unavailable beforehand. One solution would be estimating $F$ and $S$, and we conduct SVDD to estimate the support $S$ in Section \ref{sec:D-MCG_Simul}, but the performance is mediocre when the data size gets larger. Other possible ways include \emph{empirical CDF} and \emph{kernel density estimation}, but they are feasible only when $d \leq 5$, especially the latter, which requires large samples for reliable results with high dimensions \cite{samiuddin1990nonparametric}. (3) The intensity parameter $\delta_{\alpha}$ obtained by simulations in the D-MCG algorithm (line 14 of Algorithm \ref{alg:DMCG_Algo}) is a global parameter. While relying heavily on $\delta_{\alpha}$, this approach may not work well for local outliers or clusters differing drastically in densities.

To address the abovementioned limitations, we can use the RK-CCDs clustering approach to the given data set and then apply the D-MCG algorithm to each resulting cluster. For each cluster, points within the dominating covering ball are considered part of the cluster rather than outliers. With this approach, we only need to focus on points not covered by the covering balls. Under the MCG obtained from a KS-CCD, any point not connected to the dominating covering ball of its respective cluster will be considered an outlier.

By conducting RK-CCDs first on a given data set $\mathcal{X}$, we can obtain a reasonable partition of clusters by their local distribution. Under the MCG for a cluster, any connected component other than the dominating covering balls is more likely to be outliers than a cluster. We could address the limitation (1) above with this approach. RK-CCDs capture clusters by Spatial Randomness Monte Carlo Test (SR-MCT). Thus, following this notion, we could specify $F$ as an HPP under the null assumption $H_0$, specified in the D-MCG algorithm (Algorithm \ref{alg:DMCG_Algo}). Thus, limitation (2) could also be resolved.

Finally, since we are applying the D-MCG algorithm on each cluster separately rather than on the entire data set globally, we can get the intensity threshold $\delta_\alpha$ for each cluster in the data set. In this sense, limitation (3) should no longer be a problem.

We call this approach the \emph{Uniformity-based Cluster Catch Digraphs with Mutual catch graphs} (U-MCCD) algorithm.

However, obtaining the threshold $\delta_{\alpha}$ for each cluster via hundreds or thousands of simulations is computationally expensive. Therefore, we propose a faster alternative, called the \textit{Rapid Uniformity-based Cluster Catch Digraphs with Mutual catch graphs} (RU-MCCD) algorithm (Algorithm \ref{alg:RUMCCD}), which sets the threshold $\delta_{\alpha}$ as the largest density parameter $\delta$ such that the points within the dominating covering ball are connected under the D-MCG. Therefore, we can skip the intensive simulation step.

More specifically, the algorithm first partitions the data set into clusters using RK-CCDs. For each cluster, it determines the dominating covering ball and creates a KS-CCD with a given density parameter $\delta_j$. It then constructs the MCG of this cluster. 
If the MCG is not connected, the intensity parameter is adjusted (i.e. reduced by $\Delta$) iteratively until connectivity is achieved. 
The algorithm identifies outliers as points not connected within the final MCG of each cluster.

\begin{algorithm}[htb]
\SetAlgorithmName{Algorithm}{}{}
\KwData{$\delta_0$, $\Delta$ and a dataset $\mathcal{X}=\{x_1,x_2,...,x_n\}$}\;
\KwResult{Clusters and outliers in $\mathcal{X}$}\;
\textbf{algorithm Steps:}\\
\nl Partition $\mathcal{X}$ into clusters $\mathbf{P}=\{P_1,P_2,...,P_m\}$ using RK-CCDs\;
\nl \For{$P_j \in \mathbf{P}$}{
\nl $B_j \gets$  the dominating covering ball of  $P_j$\;
\nl $P_{j,c} \gets \{x:x \in \mathcal{X} \cap B_j\}$\;
\nl $\delta_j \gets \delta_0$\;
\nl	$D(P_{j,c})=(V(P_{j,c}),A(P_{j,c}))$: the KS-CCD of $P_{j,c}$ (with density parameter $\delta_j$)\;
\nl	$G_M(P_{j,c})=(V(P_{j,c}),E(P_{j,c}))$: the MCG of $D(P_{j,c})$\;
\nl	\While{$G_M(P_{j,c})$ is not connected}{
\nl		$\delta_j \gets \delta_j-\Delta$\;
\nl		Repeat steps 6 and 7 to update $D(P_{j,c})$ and $G_M(P_{j,c})$\;
	}
\nl $D(P_j) = (V(P_j),A(P_j))$: the KS-CCD of $P_j$ (with density parameter  $\delta_j$)\;
\nl $G_M(P_j) =(V(P_j),E(P_j))$: the MCG of $D(P_j)$\;
\nl Label $x \notin P_{j,c}$ as an outlier if disconnected in $G_M(P_j)$\;
}
\nl Return the constructed clusters $\mathbf{P}$ and outliers\;
\caption{(\textbf{RU-MCCD Algorithm}, a faster version of the U-MCCD algorithm) Outlier detection by utilizing RK-CCDs for initial cluster partition and KS-CCDs for determining clusters and outliers, based on density adjustments $\delta_0$ and $\Delta$.}\label{alg:RUMCCD}
\end{algorithm}

\begin{Thm}[Time Complexity of Algorithm \ref{alg:RUMCCD}]\label{thm:RUMCCD_Time}
  Given a data set $\mathcal{X} \subset \mathbb{R}^d$ of size $n$ ($d<n$). The time complexity of Algorithm \ref{alg:RUMCCD} is $O(n^3(\log n +N)+n^2(d+\log n))$, where $N$ is the number of simulated data sets for the confidence envelopes of $\widehat{K}(t)$.
\end{Thm}

\begin{proof}
  In Algorithm \ref{alg:RUMCCD}, we first obtain a partition of $\mathbf{P}=\{P_1,P_2,...,P_m\}$ for $\mathcal{X}$ with RK-CCDs, which takes $O(n^3(\log n +N)+n^2d)$ time \cite{manukyan2019parameter}; then, we loop through each partition $P_j$, constructing KS-CCDs and finding connected components for both $P_j$ and $P_{j,c}$. According to Theorem \ref{thm:D-MCG_Time}, and given the fact that distance matrix (which costs $O(n^2d)$ time to compute) of $\mathcal{X}$ is already available with RK-CCDs, the above process runs in $O(n^2\log n)$ time at most for all partitions in total. Therefore, Algorithm \ref{alg:RUMCCD} costs $O(n^3(\log n +N)+n^2(d+\log n))$ time in the worst case, which boils down to $O(n^3\log n)$ for fixed $d$ and $N$.
\end{proof}

We present several synthetic data sets in Figure \ref{2d_fig_RUMCCD_Algo2}. These data sets vary in several factors, including the sizes of data sets, the number of clusters, and the percentage of outliers within the entire data set. Each cluster's observations follow a uniform distribution within a unit circle. It is important to note that the number of observations within each cluster may not necessarily be identical since we want to evaluate the effectiveness of the RU-MCCD algorithm (Algorithm \ref{alg:RUMCCD}) on local outliers, which may not be easy to capture when considered globally.

Figure \ref{2d_fig_RUMCCD_Algo2_result} presents the connected components and outliers identified by the RU-MCCD algorithm (Algorithm \ref{alg:RUMCCD}). The RU-MCCD algorithm demonstrates effectiveness across all six data sets, accurately identifying nearly all outliers while excluding regular data points. This is true even in scenarios where the clusters vary in size. Furthermore, the RU-MCCD algorithm successfully connects regular data points outside the dominating covering balls to the main clusters, thereby minimizing the number of false positives.

\begin{figure}[htb]
\centering
\subfigure[]{
\label{2d_fig_RUMCCD_Algo2_a}
\includegraphics[width=0.3\textwidth]{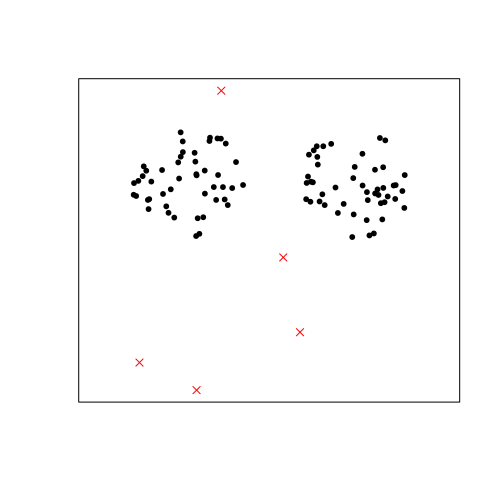}}
\subfigure[]{
\label{2d_fig_RUMCCD_Algo2_b}
\includegraphics[width=0.3\textwidth]{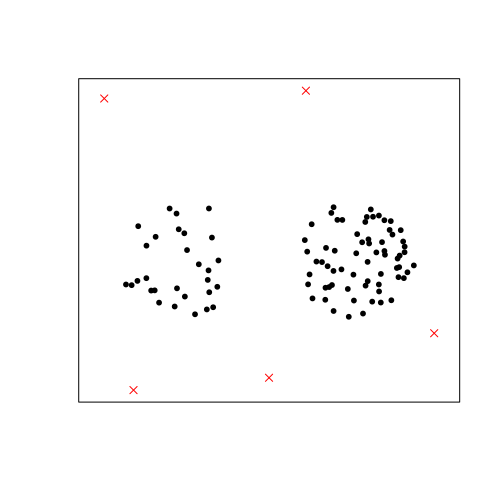}}
\subfigure[]{
\label{2d_fig_RUMCCD_Algo2_c}
\includegraphics[width=0.3\textwidth]{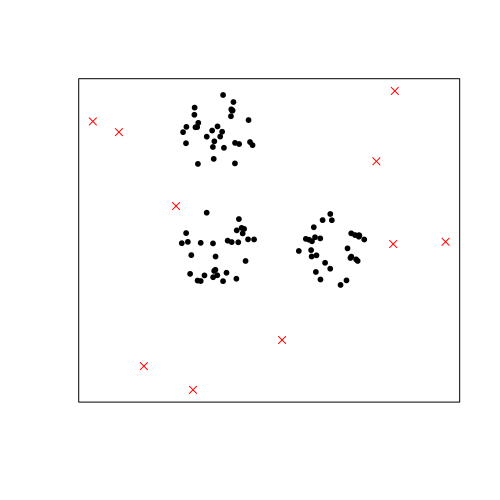}}
\subfigure[]{
\label{2d_fig_RUMCCD_Algo2_d}
\includegraphics[width=0.3\textwidth]{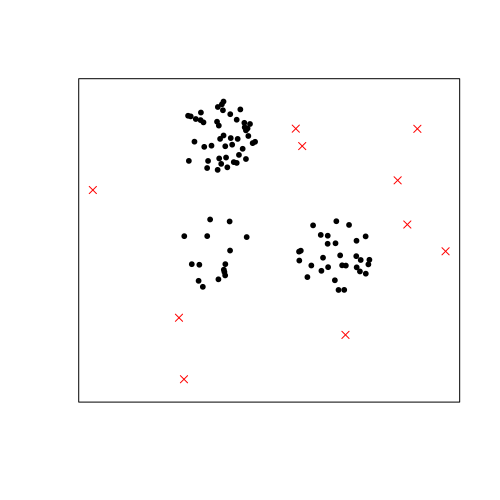}}
\subfigure[]{
\label{2d_fig_RUMCCD_Algo2_e}
\includegraphics[width=0.3\textwidth]{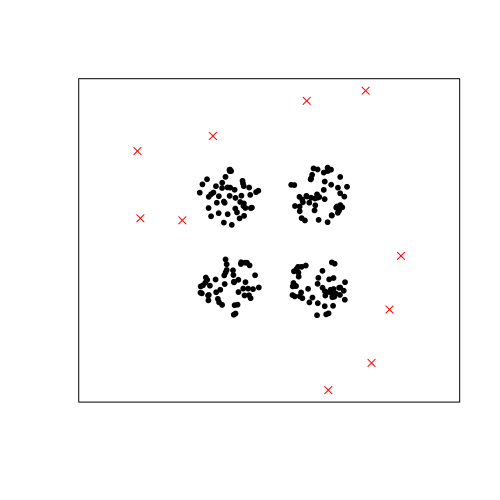}}
\subfigure[]{
\label{2d_fig_RUMCCD_Algo2_f}
\includegraphics[width=0.3\textwidth]{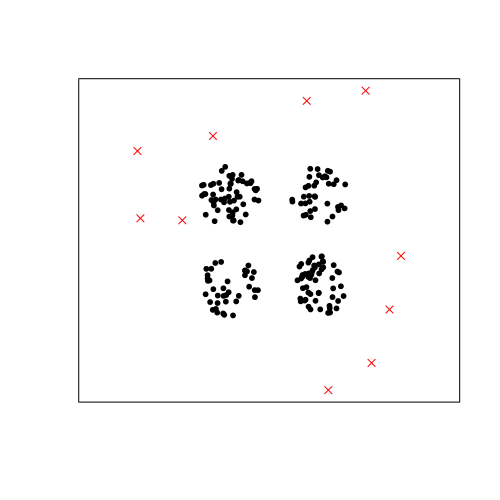}}
\caption{Some simulated uniform data sets, black points are regular data points, red crosses are outliers, (a) 2 clusters, 5\% outliers, $n=100$. (b) 2 clusters with different sizes, 5\% outliers, $n=100$. (c) 3 clusters, 10\% outliers, $n=100$. (d) 3 clusters with different sizes, 10\% outliers, $n=100$. (e) 4 clusters, 10\% outliers, $n=200$. (f) 4 clusters with different sizes, 5\% outliers, $n=200$}
\label{2d_fig_RUMCCD_Algo2}
\end{figure}

\begin{figure}[htb]
\centering
\subfigure[]{
\label{2d_fig_RUMCCD_Algo2_result_a}
\includegraphics[width=0.3\textwidth]{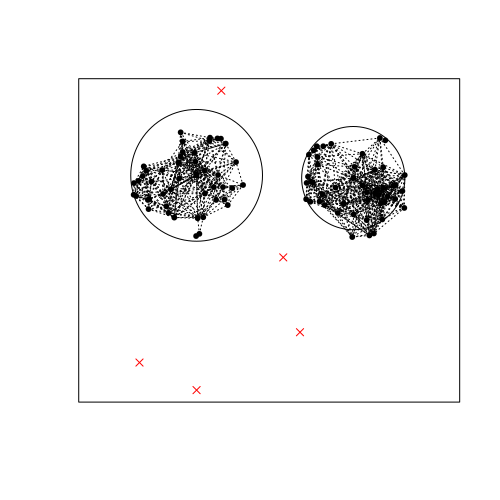}}
\subfigure[]{
\label{2d_fig_RUMCCD_Algo2_result_b}
\includegraphics[width=0.3\textwidth]{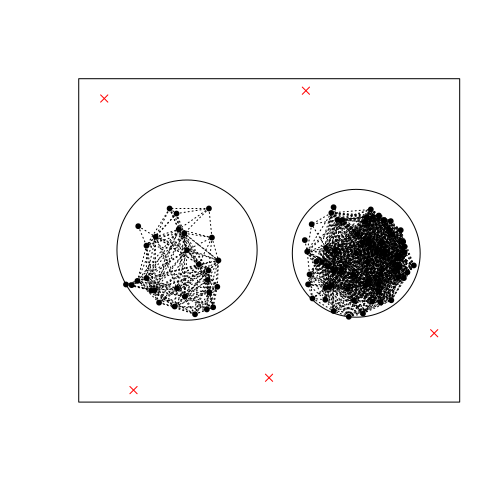}}
\subfigure[]{
\label{2d_fig_RUMCCD_Algo2_result_c}
\includegraphics[width=0.3\textwidth]{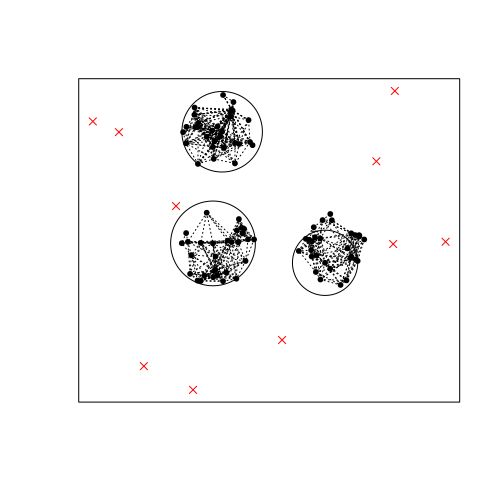}}
\subfigure[]{
\label{2d_fig_RUMCCD_Algo2_result_d}
\includegraphics[width=0.3\textwidth]{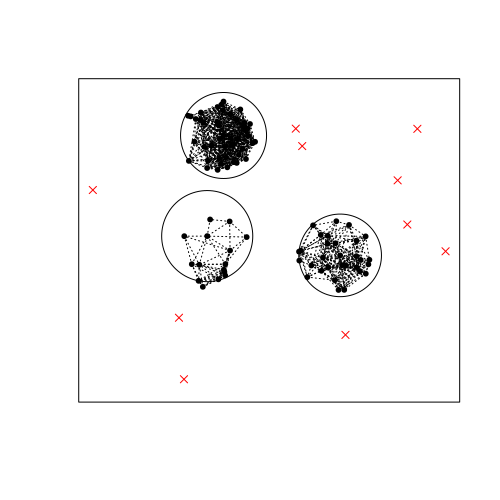}}
\subfigure[]{
\label{2d_fig_RUMCCD_Algo2_result_e}
\includegraphics[width=0.3\textwidth]{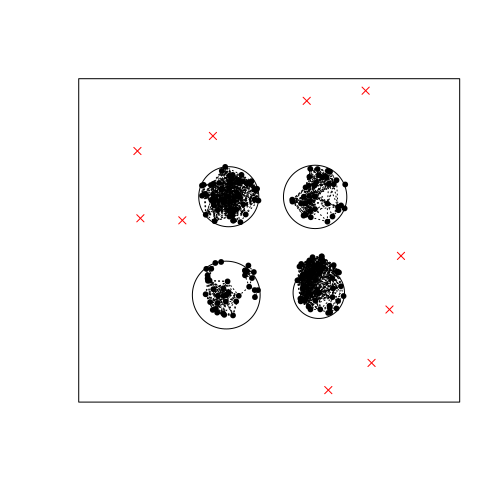}}
\subfigure[]{
\label{2d_fig_RUMCCD_Algo2_result_f}
\includegraphics[width=0.3\textwidth]{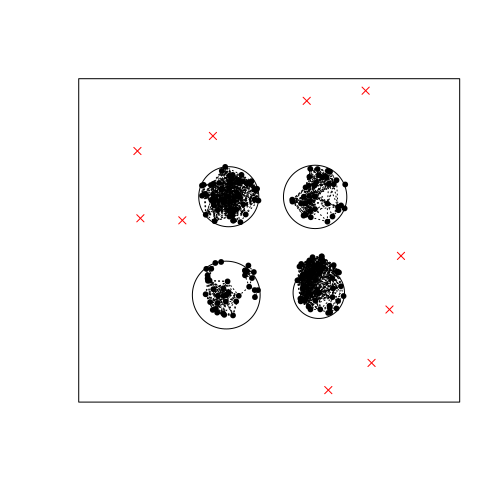}}
\caption{The connected components and outliers determined by the RU-MCCD algorithm (Algorithm \ref{alg:RUMCCD}) for the settings in Figure \ref{2d_fig_RUMCCD_Algo2}. The solid black circles are the dominating covering balls of RK-CCDs.}
\label{2d_fig_RUMCCD_Algo2_result}
\end{figure}

The experimental analysis in Section \ref{sec:Simul_CCDs} demonstrates that the RU-MCCD algorithm performs effectively on simulated data sets when each cluster is uniformly distributed and the dimensionality $d\leq10$, where the $F_2$-scores exceed 0.9 under most simulation settings. The TPRs (for outlier detection) are generally satisfactory under low dimensions ($d=2,3,5$), with most TPRs exceeding $90\%$ or even $95\%$. Additionally, due to the effectiveness of RK-CCDs on clustering with less dimensions, the TNRs under almost all simulation settings are substantially above $95\%$, even when the size of a data set is as low as $50$.

However, the performance of the RU-MCCD algorithm begins to decline with more dimensions ($d\geq20$), as shown in Section \ref{sec:Simul_CCDs}. Although the TPRs tend to increase towards $1$ in almost all the cases, the TNRs become substantially lower than those within a lower-dimensional space. Increasing the data size to $1000$ does not yield substantial improvement. This can be explained by the increased sparsity of regular data points as $d$ increases, complicating clustering with RK-CCDs and leaving more regular observations uncovered by the dominating covering balls. Additionally, higher dimensions bring considerable intensity differences between a cluster's center and boundary. As a result, regular data points not covered by the dominant covering balls are unlikely to be connected in an MCG. This phenomenon further decreases the TNRs. Some shortcomings of RK-CCDs also contribute to this decreased performance, which will be discussed in subsequent sections.

In Section \ref{sec:Simul_CCDs}, we also conduct the simulations with Gaussian clusters. We aim to investigate the performance of the RU-MCCD algorithm (Algorithm \ref{alg:RUMCCD}) when points within a cluster are non-uniformly distributed. As expected, the RU-MCCD algorithm yields less satisfactory results with substantially lower TNRs due to the SR-MCT of RK-CCDs, which implies approximately uniformity within each covering ball. However, this is not true for a Gaussian cluster due to nonuniform intensity. As a result, the resulting dominating covering balls tend to be much smaller than the span of Gaussian clusters and are generally located around the center of Gaussian clusters, leaving many regular points uncovered. Additionally, due to the substantial intensity difference over a (multivariate) normal distribution, it is unlikely for the mutual catch digraphs to connect relatively sparse points with the points covered by the dominating covering balls, which generally have much higher intensities.

\subsubsection{Mutual Catch Graph with Shape-Adaptive Cluster Catch Digraphs}
\label{sec:shape_adapt_CCD}
As shown in Section \ref{sec:Simul_CCDs}, with the Gaussian clusters, the RU-MCCD algorithm can result in a substantially low TNR with regular points labeled as outliers due to the nonuniform intensity within a cluster. Although the RU-MCCD algorithm can still identify the correct number of clusters in most cases, the dominating covering balls only cover the densest part of Gaussian clusters, leaving many regular points of lower intensity uncovered. Thus, a single dominating covering ball may not be sufficient to cover a Gaussian cluster entirely.

In order to address this limitation, an intuitive solution is to increase the number of covering balls for each latent cluster. Thus, we propose a flexible approach using multiple covering balls of RK-CCDs for each cluster. Similar to the RU-MCCD algorithm, we first implement the clustering process with RK-CCDs and obtain a dominating covering ball for each cluster. Although a single dominating covering ball may not be large enough to cover a cluster fully, it can be perceived as the core and location of the corresponding cluster. Therefore, one may want to expand the coverage outward from the core. We consider the MCGs based on RK-CCDs for the objective. A pair of points are more likely to be drawn from the same local HPP when connected. Generally, this happens when two close points from the same cluster have similar local intensities and spatial distributions. With this notion, each connected component of an MCG can be considered a latent cluster, which was first proposed by Marchette \cite{marchette2005random}. However, this approach is not robust to noise. In experimental analysis (not presented here), when noise is in the gaps between different clusters, the above approach may falsely identify two or more clusters as one since noise may ``connect" them together. This is due to the "over-fitting" effect when we use all the covering balls of a data set. Due to this reason, we only consider the points connected to the center point of one of the dominating covering balls, and we extend the coverage of a cluster to the union of the corresponding covering balls. All the points belonging to an enlarged coverage are assigned to the same cluster. Theoretically, this new approach could return clusters with more precise boundaries compared to the RU-MCCD algorithm, especially when the shape of clusters is not spherical or the point intensities over the support are not even. Additionally, the adverse effect of noise on clustering could also be minimized.

To determine the optimal number of clusters, we apply an approach similar to the RK-CCDs and KS-CCDs algorithms \cite{manukyan2019parameter}, which employs the silhouette index. All connected components are ranked in decreasing order in terms of their cardinalities. Following the rank, we incrementally add components as valid clusters (starting from the first two components) until the maximum average silhouette index for the whole data set is reached. However, when a group of connected outliers is far from true clusters, they could be identified as small but valid clusters. To handle this problem, we introduce another input parameter $S_{min}$. A point set can only be considered a valid cluster when its cardinality is at least $S_{min}$. The value of $S_{min}$ is flexible and can be specified by the user.

Given a data set, we may have some points (mainly outliers) that are far from others. As a result, they are not in the scope of any existing clusters. Either the partition size they belong to is smaller than $S_{min}$, or the corresponding partition has yet to be added as a valid cluster. Therefore, we must find a method that assigns these points to appropriate clusters. We have tried the \emph{Local Distance-based Outlier Factor} (LDOF) \cite{zhang2009new} and adapted it differently so that this measurement can be used to determine the optimal cluster for each unlabeled point. Unfortunately, the algorithms utilizing LDOF do not work well in simulation. Therefore, we have to give up this measure. 

Recall that Manukyan and Ceyhan \cite{manukyan2019parameter} introduced the convex distance between an uncovered point and a dominating covering ball and assigned every uncovered point accordingly. We utilize this idea in the new algorithm. 

The new approach is presented in Algorithm \ref{alg:SUMCCD_Algo} below, and we call it the \emph{Shape-adaptive Uniformity-based CCDs with Mutual catch graph} (SU-MCCD) algorithm.

Figure \ref{fig:2d_Demo_SUMCCD} presents the realization of the SU-MCCD algorithm on a synthetic data set (Figure \ref{fig:2d_fig_SUMCCD1}) with two Gaussian clusters (black points) of different intensities and a few outliers (red crosses); Figure \ref{fig:2d_fig_SUMCCD2} presents the dominating covering balls for the two clusters, which are not large enough to cover all the regular points; Figure \ref{fig:2d_fig_SUMCCD3} shows all the covering balls (dashed lines) of the points that are connected to the center of any dominating covering balls under the MCG of RK-CCDs, the union of these covering balls exhibits an extension of cluster covers. Figure \ref{fig:2d_fig_SUMCCD4} presents the connected components and outliers identified by the SU-MCCD algorithm, by using multiple covering balls for each cluster, it manages to connect most regular points from the same cluster and excludes all the outliers.

\begin{figure}[htb]
\centering
\subfigure[]{
\label{fig:2d_fig_SUMCCD1}
\includegraphics[width=0.35\textwidth]{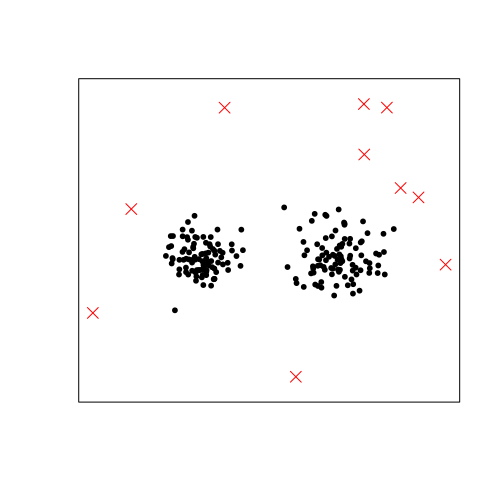}}
\subfigure[]{
\label{fig:2d_fig_SUMCCD2}
\includegraphics[width=0.35\textwidth]{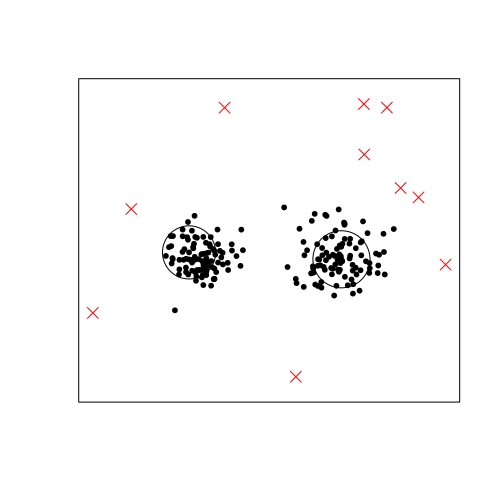}}

\subfigure[]{
\label{fig:2d_fig_SUMCCD3}
\includegraphics[width=0.35\textwidth]{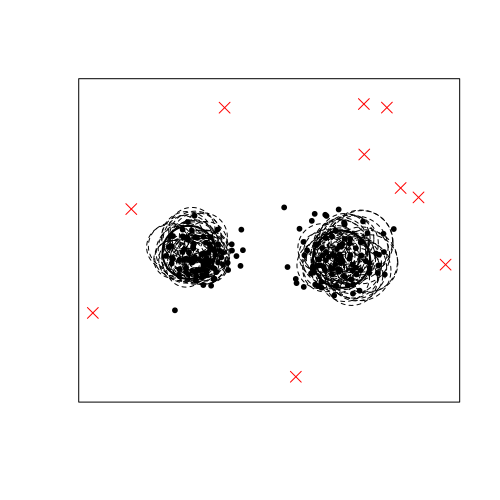}}
\subfigure[]{
\label{fig:2d_fig_SUMCCD4}
\includegraphics[width=0.35\textwidth]{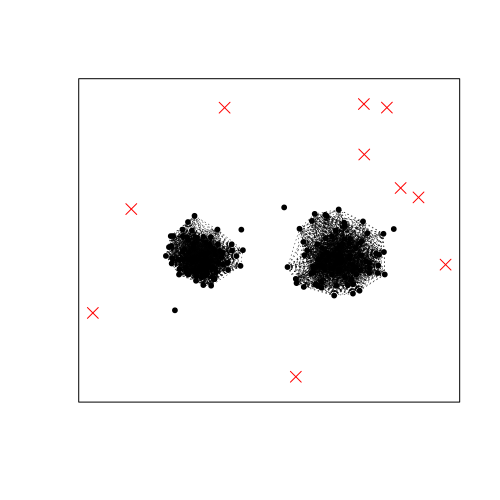}}
\caption{A illustration of the SU-MCCD algorithm.}
\label{fig:2d_Demo_SUMCCD}
\end{figure}

\begin{algorithm}[htb]
\small
\SetAlgorithmName{Algorithm}{}{}
\KwData{$\delta_0$, $\Delta$, $k$, $S_{min}$, and a data set $\mathcal{X}=\{x_1,x_2,...,x_n\}$}\;
\KwResult{Clusters and outliers of $\mathcal{X}$}\;
\textbf{Algorithm Steps:}\\
\nl Construct $D(\mathcal{X})=(V(\mathcal{X}), A(\mathcal{X}))$: a RK-CCD of $\mathcal{X}$\;
\nl Obtain the dominating covering balls from $D(\mathcal{X})$\;
\nl Construct $G_M(\mathcal{X})=(V(\mathcal{X}),E(\mathcal{X}))$: the MCG of $D(\mathcal{X})$\;
\nl $\mathbf{P}=\{P_1,P_2,...,P_s\} \gets$ the partition obtained by extending the coverage from each dominating covering ball, ordered from high to low by size (is by the number of points of a partition)\;
\nl Incrementally form clusters $\mathbf{C}$ from $\mathbf{P}$, assigning isolated points based on smallest relative distance, until maximizing $sil(\mathbf{C})$ or partition sizes drop below $S_{min}$\;
\nl  $\mathbf{C}=\{C_1,C_2,...,C_m\} \gets$ the clusters obtained in step 4 (which also serves as a partition of $\mathcal{X}$)\;
\nl \For{$C_j \in \mathbf{C}$}{
\nl $C_{j,c} \gets P_j$ (the cluster $C_j$ is constructed from the partition $P_j$)\;
\nl $\delta_j \gets \delta_0$\;
\nl	Construct $D(C_{j,c})=(V(C_{j,c}),A(C_{j,c}))$: a KS-CCD of $C_{j,c}$ (with density parameter $\delta_j$)\;
\nl	Construct $G_M(C_{j,c})=(V(C_{j,c}),E(C_{j,c}))$: the MCG of $D(C_{j,c})$\;
\nl	\While{$G_M(C_{j,c})$ is not connected}{
\nl		$\delta_j \gets \delta_j-\Delta$\;
\nl		Repeat lines 9 and 10 to update $D(C_{j,c})$ and $G_M(C_{j,c})$\;
	}
\nl Construct $D(C_j)= (V(C_j),A(C_j))$: the KS-CCD of $C_j$ (with density parameter $\delta_j$)\;
\nl Construct $G_M(C_j)=(V(C_j),E(C_j))$: the MCG of $D(C_j)$\;
\nl Under $G_M(C_j)$, for each $x \notin C_{j,c}$, $x$ is labeled as an outlier if it is not connected to any vertices in $C_{j,c}$\;
}
\nl Return the constructed clusters $\mathbf{C}$ and outliers\;
\caption{(\textbf{SU-MCCD Algorithm}) Outlier detection using RK-CCDs for cluster formation and KS-CCDs for density-based validation, incorporating $S_{min}$ for minimum cluster size, with initial density $\delta_0$ and decrement $\Delta$ as in Algorithm \ref{alg:DMCG_Algo}. Adapted for arbitrary-shaped clusters.}\label{alg:SUMCCD_Algo}
\end{algorithm}

\begin{Thm}[The Time Complexity of Algorithm \ref{alg:SUMCCD_Algo}]\label{thm:SU-MCCD_Time}
  Given a data set $\mathcal{X} \subset \mathbb{R}^d$ of size $n$ (with $d<n$ being fixed). The time complexity of Algorithm \ref{alg:SUMCCD_Algo} is $O(n^3(\log n +N)+n^2(d+\log n))$ (the same order as Algorithm \ref{alg:RUMCCD}), where $N$ is the number of simulated data sets for the confidence envelopes of $\widehat{K}(t)$.
\end{Thm}

\begin{proof}
   When implementing Algorithm \ref{alg:SUMCCD_Algo}, we need to construct an RK-CCD and obtain dominating covering balls for $\mathcal{X}$ first, which takes $O(n^3(\log n +N)+n^2d)$ time \cite{manukyan2019parameter}. Constructing the MCG of the RK-CCD and extend the coverage from each dominating covering ball costs $O(n^2)$ time at most. Each time we add a new cluster, we need to re-partition the data set. Finding the nearest cluster based on the relative distance needs $O(n)$ at most for each $x_i \in \mathcal{X}$, thus a maximum of $O(n^2)$ time for the entire data set at each iteration and $O(n^3)$ time for all iterations. Updating and maximizing the average silhouette measure take less than $O(n^3)$ time \cite{manukyan2019parameter}. Finally, similar to Theorem \ref{thm:RUMCCD_Time}, looping through each partition in $\mathbf{P}$ to identify outliers takes no more than $O(n^2\log n)$ time. Hence, Algorithm \ref{alg:SUMCCD_Algo} runs in $O(n^3(\log n +N)+n^2(d+\log n))$ time. Note that for fixed $d$ and $N$, Algorithm \ref{alg:SUMCCD_Algo} runs in $O(n^3\log n)$ when $n \rightarrow \infty$.
\end{proof}

In Section \ref{sec:Simul_CCDs}, we evaluate the SU-MCCD algorithm's performance under the same simulation settings as the RU-MCCD algorithm. And we set $S_{min}$ (the minimal size of a cluster) to be half of the contamination level.

The results are summarized from Tables \ref{tab:Uni_General_Results1} to \ref{tab:Gau_General_Results2}. Generally, when points within each cluster are uniformly distributed (Tables \ref{tab:Uni_General_Results1} and \ref{tab:Uni_General_Results2}), the performance of the SU-MCCD algorithm is comparable to or slightly better than that of the RU-MCCD algorithm under lower dimensions ($d\leq5$). Additionally, it can achieve substantially higher TNRs when $d=10$ and $20$. Besides, this flexible approach can identify all the outliers (i.e., TPRs are near $100\%$) while maintaining relatively high TNRs.

Under the second simulation setting, where clusters are (multivariate) normally distributed (Tables \ref{tab:Gau_General_Results1}), the SU-MCCD algorithm yields considerably higher TNRs when $d\leq10$ compared to the RU-MCCD algorithm. This can be attributed to the flexibility of the SU-MCCD algorithm in capturing clusters with arbitrary shapes or uneven intensities and returning precise boundaries.

Furthermore, the SU-MCCD algorithm is robust against the contamination level, as shown in the simulation study in Section \ref{sec:Simul_CCDs}. A higher contamination level results in high-intensity outliers, and a masking problem typically arises in such cases due to a substantial increase in small outlier groups. Thanks to the mechanism that filters small clusters, the SU-MCCD algorithm can correctly label small outlier groups whose sizes are smaller than $S_{min}$ as outliers. Furthermore, as an input parameter, $S_{min}$ is relatively easy to specify in various disciplines (e.g., with a pilot study).

However, similar to the RU-MCCD algorithm, the SU-MCCD algorithm tends to underperform in higher dimensions ($d\geq$20). Like the RU-MCCD algorithm, the SU-MCCD algorithm still employs RK-CCDs for clustering, inheriting the same limitations. When the dimensionality is high, the covering balls of RK-CCDs tend to be much smaller, making it challenging for any two points to connect under the MCG, even if they are nearest neighbors. Consequently, the simulation results of the SU-MCCD algorithm are disappointing when $d>20$ and close to or even the same as the RU-MCCD algorithm when $d \geq 50$ as almost every point is isolated in the MCG. Additionally, increasing the data size to as large as $1000$ only provides a small improvement.

\subsection{Outlier Detection with UN-CCDs}
\subsubsection{Complete Spatial Randomness and the Nearest Neighbor Distance}

The Monte Carlo experiments conducted starting from Sections \ref{sec:Simul_CCDs} show that the outlier detection algorithms based on RU-MCCDs and SU-MCCDs typically deliver low TNRs when applied to high-dimensional data sets, particularly when $d\geq10$. A similar limitation encountered by these outlier detection algorithms in higher dimensions is the size of the dominating covering balls returned by RK-CCDs. These balls are not sufficiently large and leave too many regular observations uncovered. And unfortunately, this shortcoming is only partially addressed by the subsequent D-MCG algorithm.

We have identified several limitations of RK-CCDs that eventually lead to the shortcomings mentioned above (on high-dimensional data sets). Firstly, recall that to find the optimal radius $r_{x_i}$ for each covering ball $B(x_i, r_{x_i})$, RK-CCDs expand the size of $B(x_i, r_{x_i})$ from the center $x_i$ incrementally until the points captured within can no longer pass SR-MCT (Spatial Randomness Monte Carlo Test). It is known that Ripley's $K$ function can be used to describe the second-order moments of a point process \cite{ripley1976second}. The SR-MCT was developed based on one of Ripley's $K$ functions ($K(t)$), which measures the number of pairs of points whose distance is less than $t$ within a window or a region of interest. However, the first point $x_i$ to be involved in the test is not random as it is always the center of $B(x_i, r_{x_i})$. Thus, any successive points to be covered will be less than $1$ unit distance (scaled by radius) apart from $x_i$. It may not be a big issue when $d$ is small because, inside a unit ball, it is expected to see a pair of points whose distance is less than $1$ under CSR with sufficiently high probability; adding a few more such non-random pairs would not considerably affect the validity of the test with a high probability. However, close pairs of points become extremely rare when $d$ is large. For example, the probability that two random points are at most $1$ unit away is approximately $0.122$ (estimated by simulation) when $d=10$; this probability decreases to roughly $0.0222$ when $d=20$. Under high-dimensional settings, adding a few more close non-random pairs can make a huge difference. Therefore, the test conducted on these covering balls is no longer accurate. Except for the non-random center $x_i$, the point-wise confidence band for $K(t)$ raises another problem on the test. In RK-CCDs, $t_{max}$ was specified to be half of the radius of a unit sphere, namely $0.5$ \cite{manukyan2019parameter}. The commonly chosen values for $t$ are $0.1, 0.2,..., 0.5$. The point-wise confidence band (for $K(t)$) built on these fixed values is only appropriate when $d$ is small because the distances between any points increase substantially as $d$ increases. Consequently, the small and fixed $t$ values are no longer suitable.

Some potential ways to improve RK-CCDs include the following: (1) Remove the center point $x_i$ when conducting the SR-MCT on a covering ball $B(x_i, r_{x_i})$. (2) Make the values for $t$ dynamically adaptable to $d$. The first should be easy to implement, while the second may be challenging. Determining appropriate $t$ values for different dimensions is difficult because the distribution of Ripley's $K$ function is unknown, and so are the theoretical quantiles, whose values change with $d$.

Nevertheless, we have attempted to obtain appropriate values for $t$ through Monte Carlo simulations. First, we simulate $M$ data sets of the same size as the given data set. Then, with each simulated data set, we record the distances between any two points and aggregate these distances from all data sets into a sample. Finally, we take the $10\%$, $20\%$, $30\%$, $40\%$, and $50\%$ quantiles of the sample and set these quantiles as the values for $t$. The experimental results (not presented here) exhibit substantial improvement but are still not good enough, and determining the values of $t$ is another hurdle against RK-CCDs in real-life applications. Therefore, the test based on Ripley's $K$ function seems unsuitable for high-dimensional clustering.

To address this shortcoming, we introduce an alternative way to test CSR using the \emph{Nearest Neighbor Distance} (NND) and employ the CCDs with NND for outlier detection. First, we review NND and the existing methods for testing CSR.

Suppose we have a set of \emph{i.i.d} points $\mathcal{X} = \{x_1,x_2,...,x_n\}$ in a subspace of $\mathbb{R}^d$ with a specified intensity $\rho$. Let $\mathcal{d}_i$ be the distance of $x_i$ to its nearest neighbor. Then, the mean NND of the data set $\mathcal{X}$ can be computed as $\bar{\mathcal{d}}=\frac{\sum_{i=0}^{n}\mathcal{d}_i}{n}$. To measure how much $\mathcal{X}$ departs from randomness, we also want to know the expected mean NND of $\mathcal{X}$ (denoted as $\mu_{\mathcal{d}}$) under CSR. Fortunately, when $d=2$,  Clark and Evans \cite{clark1954distance} have shown the following,
\begin{equation}\label{equ:NND1}
\mu_{\mathcal{d}} = \frac{1}{2\sqrt{\rho}},\ \ \
\sigma_{\bar{\mathcal{d}}} = \frac{0.26136}{\sqrt{\rho}},
\end{equation}
where $\sigma_{\bar{\mathcal{d}}}$ is the standard deviation of $\bar{\mathcal{d}}$.

The significance of the difference between $\mu_{\mathcal{d}}$ and $\bar{\mathcal{d}}$ can be measured by the widely used Gaussian Z-score when $n$ is sufficiently large \cite{clark1954distance},
\begin{equation}\label{equ:Z-score}
Z = \frac{\bar{\mathcal{d}}-\mu_{\mathcal{d}}}{\sigma_{\bar{\mathcal{d}}}}.
\end{equation}

Clark and Evans \cite{clark1979generalization} had also generalized the expressions in Equation \eqref{equ:NND1} to arbitrary dimensionality as
\begin{equation}\label{equ:NND2}
\mu_{\mathcal{d}} = \frac{\Gamma(d/2+1)^{1+1/d}}{\rho^{1/d}\pi^{1/2}}, \ \ \
\sigma_{\bar{\mathcal{d}}} = \frac{\Gamma(d/2+1)^{1/d}((2/d+1)-\Gamma(d/2+1)^2)^{1/2}}{\rho^{1/d}\pi^{1/2}}.
\end{equation}

Although the normality test conducted by measuring $\bar{\mathcal{d}}$ is convenient and easy to interpret for non-statisticians, its accuracy is questionable when the sample size is too small. Actually, the distribution of $\bar{\mathcal{d}}$ is skewed to the left, and its skewness cannot be ignored when $n$ is relatively small \cite{clark1954distance}. In addition, Besag and Diggle \cite{besag1977simple} argued that Clark and Evans’s derivation of $\mu_{\mathcal{d}}$ and $\sigma_{\bar{\mathcal{d}}}$ ignored the fact that the NNDs $\mathcal{d}_1,\mathcal{d}_2,...,\mathcal{d}_n$ are not \emph{i.i.d}. Therefore, they proposed an alternative, more reliable way by employing the Monte Carlo test \cite{besag1977simple}. They simulate $m$ data sets of size $n$ that are uniformly distributed, the mean NND values $\bar{\mathcal{d}}_1,\bar{\mathcal{d}}_2,...,\bar{\mathcal{d}}_m$ can be obtained for the $m$ simulated data sets. Then, the significance level of $\bar{\mathcal{d}}$ can be measured by its quantile in the $m$ simulated mean NND values. This Monte Carlo test for CSR is easy to conduct and does not require formulas or parameters. It is also well adapted to subspaces with any shape, as correction for edge effects is not needed \cite{besag1977simple}. With these advantages, we consider using Besag and Diggle's Monte Carlo approach to test CSR rather than calculating theoretical values of the quantiles.

\subsubsection{Mutual Catch Graph with the Nearest Neighbor Cluster Catch Digraphs}
\label{sec:UN-CCDs}

We propose another outlier detection method based on CCDs, which conducts the SR-MCT with the NND instead of Ripley's $K$ function. However, Clark's and Besag's approaches \cite{besag1977simple, clark1954distance, clark1979generalization} only consider the mean NNDs when measuring the significance of outlyingness, which is not robust and can be highly affected by a few extreme values, especially with in lower-dimensional space where the distances between points could be very different. For example, a group of observations consists of a cluster and a few outliers can still pass the SR-MCT if those outliers are far from the cluster. Thus, when implementing the Monte Carlo test, we consider using the median and mean of NND simultaneously when conducting the SR-MCT.

Furthermore, we make three additional modifications to the previous Monte Carlo test: (1) The center point $x_i$ of the covering ball $B(x_i, r_{x_i})$ will not be used in the test. (2) When the dimensionality $d$ is large, larger covering balls are preferred to compensate for the increasing sparseness. Thus, we offer the option to test the candidate values for the radius in descending order and stop decreasing the radius once the $H_0$ (i.e., CSR) is not rejected. (3) We make the test lower-tailed as we are not interested in the upper tail when the point pattern is significantly ``regular". The Monte Carlo test is presented in Algorithm \ref{alg:NNtest}.

\begin{algorithm}[htb]
\KwData{A hypersphere in $\mathbb{R}^d$ with radius $r$ covering i.i.d point set $\mathcal{X}_{sub}$ of size $n_{sub}$ from $\mathcal{X}$}\;
\KwResult{Decision on CSR rejection for $\mathcal{X}_{sub}$ at level $\alpha$}\;
\textbf{Algorithm Steps:}\\
\nl Compute mean $\bar{\mathcal{d}}$ and median $\widetilde{\mathcal{d}}$ of nearest neighbor distance (NND) in $\mathcal{X}_{sub}$, scaled by $r$\;
\nl Simulate $m$ sets within a unit sphere in $\mathbb{R}^d$, each of size $n_{sub}$, under CSR\;
\nl Calculate mean $\{\bar{\mathcal{d}}_i\}$ and median $\{\widetilde{\mathcal{d}}_i\}$ NNDs for simulations\;
\nl Determine empirical p-values $p_1$ for $\bar{\mathcal{d}}$ and $p_2$ for $\widetilde{\mathcal{d}}$, then order $p_{(1)} \leq p_{(2)}$\;
\nl Reject CSR for $\mathcal{X}_{sub}$ if $p_{(1)} \leq \alpha/2$ or $p_{(2)} \leq \alpha$ using Holm's Step-Down Procedure \cite{witten2013introduction}\;
\caption{Spatial Randomness Monte Carlo Test (SR-MCT) Using NND}
\label{alg:NNtest}
\end{algorithm}

With the above construction, we propose a clustering approach based on the NND as Algorithm \ref{alg:UN-CCDs}, and call it \emph{Uniformity- and Neighbor-based CCD} (UN-CCD) clustering algorithm.
The UN-CCD clustering algorithm identifies cluster centers in a data set $\mathcal{X}$ using Cluster Catch Digraphs (CCDs) based on the Nearest Neighbor Distance (NND). 
For each point in $\mathcal{X}$, the algorithm calculates the distances to all other points and sorts them.
A Monte Carlo test is performed on increasing radii until rejection at a specified level $\alpha$. 
Using the determined radii, a CCD is constructed, and an approximate minimum dominating set is found. 
An intersection graph is created from this set, and another minimum dominating set is found using a greedy algorithm, 
aiming to maximize the silhouette score. The final set of cluster centers are returned.

\begin{algorithm}[htb]
\KwData{$\alpha$, data set $\mathcal{X} = \{x_1,x_2,...,x_n\}$}\;
\KwResult{Cluster centers of $\mathcal{X}$}\;
\textbf{Algorithm Steps:}\\
\nl \ForEach{$x_i \in \mathcal{X}$}{
\nl Calculate distances $D(x_i) = \{d(x_i, x_j) | x_j \in \mathcal{X}, x_i \neq x_j\}$\;
\nl \ForEach{distance $r_{(j)}$ in $D(x_i)$ sorted}{
\nl Perform Monte Carlo test (Algorithm \ref{alg:NNtest}) on $B(x_i, r_{(j)})$\;
\nl \If{test rejected at level $\alpha$}{
\nl Set $r_{x_i} = r_{(j-1)}$; \textbf{break}\;}
}}
\nl Construct a CCD $D=(V,A)$ using the pre-determined radii\;
\nl Find the approximate minimum dominating set $\hat{S}(V)$ in $D$ with the Greedy Algorithm \ref{alg:greedy2}\;
\nl Create intersection graph $G_{MD}=(V_{MD},E_{MD})$ with $\hat{S}(V)$\;
\nl Find an approximate minimum dominating set $\hat{S}(G_{MD})$ in $G_{MD}$ using the Greedy Algorithm \ref{alg:greedy3} with a score function measuring the number of points covered, stops when the average silhouette index $sil(P)$ is maximized\;
\nl Output $\hat{S}(G_{MD})$ as cluster centers\;
\caption{(\textbf{UN-CCD Clustering Algorithm}) Cluster Catch Digraphs based on the Nearest Neighbor Distance (NND). $\alpha$ is the level of the Monte Carlo test with NND.}
\label{alg:UN-CCDs}
\end{algorithm}

\begin{Thm}[Time Complexity of Algorithm \ref{alg:UN-CCDs}]\label{UNCCD_Time}
  Given a data set $\mathcal{X} \subset \mathbb{R}^d$ of size $n$. The time complexity of Algorithm \ref{alg:UN-CCDs} is $O((N+d)n^2 + n^3)$, where $N$ represents the number of simulated data sets.
\end{Thm}

\begin{proof}
   The UN-CCDs are similar to KS-CCDs and RK-CCDs. The only difference between them is the way to determine $r_{x_i}$ for each $x_i \in \mathcal{X}$.

   For each simulated data set of size $n$, the median and mean of the NNDs can both be obtained in $O(n)$ time (e.g., the median can be found by the \emph{Quick-select} Algorithm \cite{rcoh2020}, which only costs $O(n)$ time). Repeating for subsets of sizes $2,3,...,n$ (take one subset for each size) takes less than $O(n^2)$ time, that is $O(Nn^2)$ time in total for $N$ simulated data sets.

   Considering the given data set $\mathcal{X}$, the distance matrix can be computed in $O(dn^2)$ time. For each $x_i \in \mathcal{X}$, sorting the distances $D(x_i)$ takes $O(\log n)$ time, and we need $O(n)$ time at most to obtain the median and mean of the NNDs of the points covered by $B(x_i,r_{(j)})$. Thus, a total of $O(n^2)$ time is needed for all $r_{(j)} \in D(x_i)$. Therefore, constructing a UN-CCD for the entire data set takes $O(n(\log n + n^2))$ time. Finding an approximate minimum dominating set $\hat{S}(V)$ by the Greedy Algorithm 2 costs $O(n^2)$ time in worst cases. Finally, we can construct $G_{MD}$ and $\hat{S}(G_{MD})$ in $O(n^3)$ time \cite{manukyan2019parameter}. Therefore, Algorithm \ref{alg:UN-CCDs} runs in $O((N+d)n^2+n^3)$ time. Note that if $N$ and $d$ are fixed, the time complexity reduces to $O(n^3)$.
\end{proof}

Additionally, we propose an outlier detection algorithm based on UN-CCDs as Algorithm \ref{alg:UNMCCD_Algo} and refer to it as the Uniformity- and Neighbor-based CCD with mutual catch graph (UN-MCCD) algorithm. Although its acronym is suggestive, we want to emphasize that it is based on UN-CCDs to distinguish it from one of the previous approaches, the RU-MCCD algorithm (Algorithm \ref{alg:SUMCCD_Algo}). Furthermore, it is worth noting that the UN-MCCD algorithm is the same as the RU-MCCD algorithm, except that RK-CCDs are replaced by UN-CCDs for clustering.

\begin{algorithm}[htb]
\SetAlgorithmName{Algorithm}{}{}
\KwData{$\delta_0$, $\Delta$ and a data set $\mathcal{X}=\{x_1,x_2,...,x_n\}$}\;
\KwResult{Clusters and outliers in $\mathcal{X}$}\;
\textbf{Algorithm Steps:}\\
\nl The same as Algorithm \ref{alg:RUMCCD}, except that RK-CCDs are replaced by UN-CCDs for clustering (line 1 of Algorithm \ref{alg:RUMCCD}).
\caption{(\textbf{UN-MCCD Algorithm}), an outlier detection algorithm with UN-CCDs and KS-CCDs, incorporating $\delta_0$ and $\Delta$ adjustments as in Algorithm \ref{alg:DMCG_Algo}.}\label{alg:UNMCCD_Algo}
\end{algorithm}

\begin{Thm}[Time Complexity of Algorithm \ref{alg:UNMCCD_Algo}]\label{M-UNCCD_Time}
  Given a data set $\mathcal{X} \subset \mathbb{R}^d$ of size $n$ ($d<n$). The time complexity of Algorithm \ref{alg:UNMCCD_Algo} is $O((N+d+\log n)n^2+n^3)$, where $N$ represents the number of simulated data sets when constructing UN-CCDs.
\end{Thm}

\begin{proof}
  From Theorem \ref{UNCCD_Time}, we know UN-CCD partitions $\mathcal{X}$ in $O((N+d)n^2 + n^3)$ time. Similar to Algorithm \ref{alg:RUMCCD}, building an MCG for each partition and identifying outliers takes $O(n^2\log n)$ time in the worst cases. Therefore, Algorithm \ref{alg:UNMCCD_Algo} runs in $O((N+d+\log n)n^2 + n^3)$ time, the same as UN-CCD (Algorithm \ref{alg:UN-CCDs}).
\end{proof}

In Section \ref{sec:Simul_CCDs}, we evaluate the average performance of the UN-MCCD algorithm and compare its results with those of the RU-MCCD algorithm. We perform Monte Carlo simulations under the same two simulation settings with uniform clusters and Gaussian clusters, respectively. The performance of them are summarized in Tables \ref{tab:Uni_General_Results1} to \ref{tab:Gau_General_Results2}.

The simulation results from both simulation settings show that the SU-MCCD algorithm outperforms the RU-MCCD algorithm under most simulation settings. In the first simulation setting, where points in each cluster are uniformly distributed following CSR, the UN-MCCD algorithm performs comparable or better than the RU-MCCD algorithm when $d\leq5$. Under most simulation cases, the TPRs and TNRs are much higher than $0.95$. Notably, both TPRs and TNRs are relatively insensitive to the number of clusters, the size of each cluster, and even the contamination level (which are shown in Section \ref{sec:Simul_Changing_Factors}), which we will discuss in detail.

When compared with the RU-MCCD algorithm with $d\geq10$, the UN-MCCD algorithm reduces the number of false negatives substantially while still maintaining high TPRs ($\approx 1$), the TNPs remain acceptable even when $d=20$, as most of them are around $0.9$.

The advantages of the UN-MCCD algorithm are even more apparent under the second simulation setting, where it outperforms the RU-MCCD algorithm in nearly all the dimensions we considered. However,  we do not expect the UN-MCCD algorithm to achieve as high TNRs as in the first simulation setting because UN-CCDs are conducting SR-MCT while Gaussian clusters are distributed nonuniformly, which deviates from CSR.

\subsubsection{Shape-Adaptive Uniformity- and Neighbor-Based CCD with Mutual Catch Graph}

Recall that in the previous section, we adapted the RU-MCCD algorithm to the cases where the cluster's shapes are arbitrary, or the intensities within clusters are nonuniform. We called the resulting algorithm the SU-MCCD algorithm. Different from the RU-MCCD algorithm that uses only one covering ball for each cluster, the SU-MCCD algorithm extends the coverage of each dominating covering ball by finding points that are connected to the center in the MCG obtained from an RK-CCD, and the union of their covering balls represents the scope of a latent cluster. With the above construction, we find the optimal number of clusters (connected components) by maximizing the silhouette index. Meanwhile, we assign each isolated point to a ``nearest" cluster with the smallest relative distance. Furthermore, we have introduced an input parameter, $S_{min}$, representing the minimal size of a cluster. The $S_{min}$ value should be easy to specify in real-life applications.

However, as discussed earlier, the SU-MCCD algorithm's performance shows little or no improvement when $d$ is large (see Tables \ref{tab:Uni_General_Results1} to \ref{tab:Gau_General_Results2}) due to the limitations of RK-CCDs: the covering balls are too small for any two points to be connected in the MCG even if they are nearest neighbors.

Fortunately, we were able to fix these limitations by introducing another version of CCDs that uses the NND to conduct SR-MCT, and the resulting approach is called the UN-MCCD algorithm. Like the SU-MCCD algorithm, we modify the UN-MCCD algorithm in a similar fashion, hence the name SUN-MCCD (\textit{Shape-adaptive Uniformity- and Neighbor-based CCD with Mutual catch graph}) algorithm, presented as Algorithm \ref{alg:SUN-MCCD} below. SUN-MCCDs differ from SU-MCCDs only in the clustering phase, and we expect this new algorithm to outperform the SU-MCCD algorithm.

\begin{algorithm}[htb]
\SetAlgorithmName{Algorithm}{}{}
\KwData{$\delta_0$, $\Delta$, $k$, $S_{min}$, and a data set $\mathcal{X}=\{x_1,x_2,...,x_n\}$}\;
\KwResult{Clusters and outliers of $\mathcal{X}$}\;
\textbf{Algorithm Steps:}\\
\nl The same as Algorithm \ref{alg:SUMCCD_Algo}, except that RK-CCDs are replaced by UN-CCDs for clustering (line 1).
\caption{(\textbf{SUN-MCCD Algorithm}), outlier detection using RK-CCDs for cluster formation and KS-CCDs for density-based validation, incorporating $S_{min}$ for minimum cluster size, with initial density $\delta_0$ and decrement $\Delta$ as in Algorithm \ref{alg:DMCG_Algo}. Adapted for arbitrary-shaped clusters.}\label{alg:SUN-MCCD}
\end{algorithm}

\begin{Thm}[Time Complexity of Algorithm \ref{alg:SUN-MCCD}]\label{M-FUNCCD_Time}
  Given a data set $\mathcal{X} \subset \mathbb{R}^d$ of size $n$ ($d<n$), the time complexity of Algorithm \ref{alg:SUN-MCCD} is $O((N+d+\log n)n^2+n^3)$, where $N$ represents the number of simulated data sets when constructing UN-CCDs.
\end{Thm}

\begin{proof}
  As shown in Theorem \ref{UNCCD_Time}, constructing a UN-CCD for $\mathcal{X}$ costs $O((N+d)n^2+n^3)$ time. According to Theorem \ref{thm:SU-MCCD_Time}, the remaining steps take $O(n^3 + n^2 \log n + n^2)$. So, Algorithm \ref{alg:SUN-MCCD} requires $O((N+d+\log n)n^2+n^3)$ time to capture outliers, and it reduces to $O(n^3)$ for fixed $N$ and $d$.
\end{proof}

Similar to the previous Monte Carlo experiments, we assess the average performance of the SU-MCCD algorithm and compare it with the SU-MCCD algorithm that is based on RK-CCDs. We perform Monte Carlo simulations using the same two settings presented in Section \ref{sec:Simul_CCDs}. In the first setting, the points of each cluster are uniformly distributed, while in the second simulation setting, they follow Gaussian distributions. The results are summarized from Tables \ref{tab:Uni_General_Results1} to \ref{tab:Gau_General_Results2}.

According to the simulation results, the SUN-MCCD algorithm performs well. Under most simulation settings, the TPRs are close to $1$, which is comparable to the previous algorithms. Additionally, when compared to the UN-MCCD algorithm, the SUN-MCCD algorithm delivers higher TPRs when the size of a data set is large enough or larger TNRs when the dimensionality $d$ is relatively large. We will discuss its performance in detail in the next section.

\section{The Space Complexity of CCD-Based Algorithms}
\label{sec:space_complexity}

In this section, we analyze the space complexity of all the CCD-based algorithms, which determines the memory consumption. We prove that each algorithm requires $O(n^2)$ space in the following.\\

\noindent \textbf{The KS-CCD, RK-CCD, and UN-CCD algorithms:}
\begin{itemize}
\item[(1)] \textit{Data storage:} The space requirement for a $d$-dimensional data set is $O(dn)$.
\item[(2)] \textit{Distance matrix:} Computing and storing pairwise distances between all points require $O(n^2)$ space.
\item[(3)] \textit{Simulations:} Both RK-CCDs and UN-CCDs require $N$ simulated data sets, storing each data set and its distance matrix requires $O(Ndn+Nn^2)$ space, which boils down to $O(n^2)$ when $N$ and $d$ are fixed. The space requirements for the $n$ upper envelopes of the Ripley’s $K$ function \cite{manukyan2019parameter}, and the $2n$ confidence intervals of the mean and median NNDs, are both $O(n)$.
\item[(4)] \textit{Radii of the covering balls}: There are $n$ covering balls in total, whose radii require $O(n)$ space to store.
\item[(5)] \textit{Misc:} The two approximate MDSs (i.e., $\hat{S}$ and $\hat{S}(G_{MD})$) require $O(n)$ space at most. The silhouette index of the data set takes $O(n)$ in memory.
\end{itemize}
In summary, the space complexities of the KS-CCD, RK-CCD, and UN-CCD algorithms are $O(n^2)$. This complexity arises primarily from the need to store distance matrices.\\

\noindent \textbf{The RU-MCCD and UN-MCCD algorithms:}
\begin{itemize}
    \item[(1)] \textit{Clustering:} Constructing RK-CCDs or UN-CCDs for clustering takes $O(n^2)$ space.
    \item[(2)] \textit{D-MCGs:} The D-MCG algorithm involving constructing KS-CCDs for each cluster, whose space complexity is $O(n)$ at most.
\end{itemize}

Therefore, the space space complexities of the RU-MCCD and UN-MCCD algorithms are $O(n^2)$.\\

\noindent \textbf{The SU-MCCD and SUN-MCCD algorithms:}

Both algorithms are similar to their prototype except that they use multiple covering balls for each cluster, which does not take additional memory. Thus, the space complexity remains $O(n^2)$.

Besides, we summarize the time and space complexity of all CCD-based algorithms in Table \ref{tab:space_time}, which we have proven. 

\begin{table}[htb]
  \center
  \footnotesize{\begin{tabular}{|c|c|}\hline
  \textbf{Algorithms} & \textbf{Time Complexity} \\ \hline
  KS-CCDs & $O(n^3+n^2(d+\log n))$ \\ \cline{1-2}
  RK-CCDs & $O(n^3(\log n+ N)+n^2d)$ \\ \cline{1-2}
  UN-CCDs & $O((N+d)n^2+n^3)$ \\ \cline{1-2}
  RU-MCCDs & $O(n^3(\log n +N)+n^2(d+\log n))$ \\ \cline{1-2}
  SU-MCCDs & $O(n^3(\log n +N)+n^2(d+\log n))$ \\ \cline{1-2}
  UN-MCCDs & $O((N+d+\log n)n^2+n^3)$ \\ \cline{1-2}
  SUN-MCCD & $O((N+d+\log n)n^2+n^3)$ \\ \hline
\end{tabular}}
\caption{The time complexity of all CCD-based algorithms.}
\label{tab:space_time}
\end{table}

\section{Monte Carlo Experiments}
\label{sec:Simul_CCDs}

\subsection{Monte Carlo Experiments: General Settings}
\label{sec:Simul_General}

In this section, we conduct Monte Carlo experiments under various simulation settings to evaluate the performance of the new CCD-based outlier detection algorithms. These experiments are conducted under general settings that involve many factors (e.g., dimensionality, data set sizes, cluster volumes, etc.) whose values vary among different data sets. In the next section, we will conduct an empirical analysis focusing on only one factor each time, and we call it empirical analysis under focus settings.

We will begin with the simulation settings where points within each cluster are uniformly distributed, and we will refer to them as \emph{uniform clusters} in the rest of the section). The simulated data sets involve two clusters, each exhibiting a standard spherical shape whose radius is a uniform random value ranging from 0.7 to 1.3. We aim to assess whether our proposed algorithms can effectively identify local outliers that may not be prominent when considered globally. Additionally, we will consider data sets with dimensionality ($d$) as high as 100, which is particularly challenging as the distance between an outlier and a regular point gets closer to that of any two points due to the low spatial intensity with more dimensions. This effect is particularly pronounced when the size of a data set is small.

Each simulation setting varies in: \label{sec:Uni_General_Settings_Des}

\begin{itemize}
  \item[\romannumeral1.] The dimensionality ($d$) of the simulated data sets with values 2, 3, 5, 10, 20, 50, 100;
  \item[\romannumeral2.] The size of data sets ($n$) with values 50, 100, 200, 500, 1000;
\end{itemize}

On the other hand, all the simulated data sets have the following common features:

\begin{itemize}
  \item[\romannumeral1.] The cluster sizes are equal (although the volume of the supports can be different);
  \item[\romannumeral2.] The radius of each cluster is randomly chosen between 0.7 and 1.3;
  \item[\romannumeral3.] The centers of clusters are: $\bm{\mu_1} = (\underbrace{3,...,3}_{d})$, $\bm{\mu_2} = (6,\underbrace{3,...,3}_{d-1})$;
  \item[\romannumeral4.] The proportion of outliers over the entire data set is fixed to $5\%$;
  \item[\romannumeral5.] The outlier set $C_{outlier}$ is generated uniformly within a much larger hypersphere with radius 5, centered at the mean of the cluster center. and each outlier is at least 2 units away from any cluster center.
\end{itemize}

Two realizations of the simulation settings in 2-dimensional space with data sizes of 100 and 200 are presented in Figure \ref{fig:Demo_2d_2ucls_cont5}.

\begin{figure}[htb]
\centering
\subfigure{
\label{fig:Demo_2d_2ucls_n100_cont5}
\includegraphics[width=0.35\textwidth]{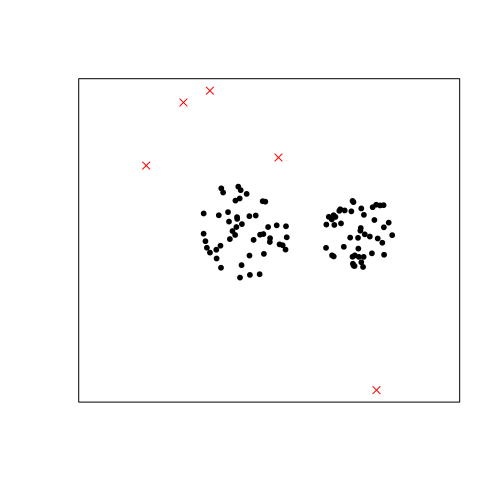}}
\subfigure{
\label{fig:Demo_2d_2ucls_n200_cont5}
\includegraphics[width=0.35\textwidth]{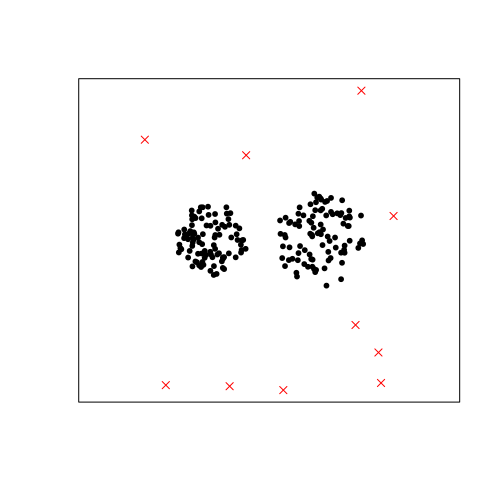}}
\caption{Two realizations of the simulation settings described in Section \ref{sec:Uni_General_Settings_Des} with $n=100$ and $200$ respectively. Each data set has 2 clusters of the same size but different intensities. Black points are regular data points, and red crosses are outliers.}
\label{fig:Demo_2d_2ucls_cont5}
\end{figure}

We repeat each subsequent simulation setting 1000 times to ensure precise and meaningful evaluation. The average TPR for outliers (i.e., the percentage of outliers captured) and TNR for regular data points (i.e., the percentage of regular points falsely identified as outliers) are recorded. However, outlier detection is essentially a classification problem over highly imbalanced data sets. Therefore, In this study, we also use BA and $F_{\beta}$-score with $\beta=2$, indicating recall is two times as important as precision.

Consider the RU-MCCD and SU-MCCD algorithms, which depend on RK-CCDs for clustering. Although RK-CCDs are parameter-free \cite{manukyan2019parameter}, the levels of the SR-MCT based on Ripley's $K$ function can be adjusted. Under moderate and high dimensions, notice that the average inter-cluster distance between any two points increases substantially, and the odds that points located near the border of clusters are substantially higher. As a result, covering balls with much higher volumes is generally preferred. Therefore, we choose the optimal levels of $\alpha$ for each dimensionality $d$. We set $\alpha$ to $1\%$ when $d<10$, and $0.1\%$ when $d\geq10$. This adjustment boosts the performance of both RK-CCD based algorithms under higher-dimensional space. Similarly, we tune the levels of the SR-MCT based on NN distances to optimize the performance of the UN-MCCD and SUN-MCCD algorithms, and we set $\alpha$ to $\{15\%, 10\%, 5\%, 1\%, 0.1\%, 0.1\%, 0.1\% \}$ as the dimensionality $d$ increases from 2 to 100. The simulation results are summarized in Tables \ref{tab:Uni_General_Results1} and \ref{tab:Uni_General_Results2}, providing comprehensive information. For better visualization, we summarize the simulation results in the following line plots (Figures \ref{fig:Uniform_TPR_TNR_Lines} and \ref{fig:Uniform_BA_F_Lines}), illustrating the trend of the performance with varying dimensions and data sizes.

\begin{table}[htb]
  \centering
  \resizebox{\columnwidth}{!}{\begin{tabular}{|c|c|c|c|c|c|c|c|c|c|c|c|}
    \hline
    \multicolumn{2}{|c|}{} & \multicolumn{10}{|c|}{The Size of Data Sets} \\ \cline{3-12}

    \multicolumn{2}{|c|}{} & \multicolumn{2}{|c|}{50} & \multicolumn{2}{|c|}{100} & \multicolumn{2}{|c|}{200} & \multicolumn{2}{|c|}{500} & \multicolumn{2}{|c|}{1000} \\ \cline{3-12}

    \multicolumn{2}{|c|}{} & TPR & TNR & TPR & TNR & TPR & TNR & TPR & TNR & TPR & TNR \\ \hline

    \multirow{4}*{$d=2$} & RU-MCCDs & 0.986 & 0.989 & 0.986 & 0.993 & 0.931 & 0.995 & 0.814 & 0.997 & 0.681 & 0.999 \\ \cline{2-12}
    & SU-MCCDs & 0.973 & 0.993 & 0.994 & 0.998 & 0.997 & 0.999 & 1.000 & 1.000 & 1.000 & 1.000 \\ \cline{2-12}
    & UN-MCCDs & 0.992 & 0.979 & 0.988 & 0.983 & 0.961 & 0.988 & 0.935 & 0.993 & 0.930 & 0.995 \\ \cline{2-12}
    & SUN-MCCDs & 0.979 & 0.987 & 0.995 & 0.992 & 1.000 & 0.994 & 1.000 & 0.996 & 1.000 & 0.997 \\ \cline{1-12}

    \multirow{4}*{$d=3$} & RU-MCCDs & 0.995 & 0.980 & 0.987 & 0.985 & 0.967 & 0.992 & 0.926 & 0.997 & 0.872 & 0.998 \\ \cline{2-12}
    & SU-MCCDs & 0.988 & 0.995 & 0.997 & 0.998 & 1.000 & 0.999 & 1.000 & 1.000 & 1.000 & 1.000 \\ \cline{2-12}
    & UN-MCCDs & 0.997 & 0.979 & 0.991 & 0.986 & 0.983 & 0.991 & 0.963 & 0.996 & 0.922 & 0.998 \\ \cline{2-12}
    & SUN-MCCDs & 0.991 & 0.990 & 0.998 & 0.997 & 1.000 & 0.998 & 1.000 & 0.999 & 1.000 & 0.999 \\ \cline{1-12}

    \multirow{4}*{$d=5$} & RU-MCCDs & 0.998 & 0.950 & 0.999 & 0.972 & 1.000 & 0.988 & 0.999 & 0.996 & 0.996 & 0.999 \\ \cline{2-12}
    & SU-MCCDs & 0.998 & 0.978 & 1.000 & 0.989 & 1.000 & 0.996 & 1.000 & 0.999 & 1.000 & 1.000 \\ \cline{2-12}
    & UN-MCCDs & 0.997 & 0.975 & 0.997 & 0.984 & 0.996 & 0.992 & 0.997 & 0.997 & 0.996 & 0.999 \\ \cline{2-12}
    & SUN-MCCDs & 0.997 & 0.994 & 1.000 & 0.997 & 1.000 & 0.999 & 1.000 & 1.000 & 1.000 & 1.000 \\ \cline{1-12}

    \multirow{4}*{$d=10$} & RU-MCCDs & 1.000 & 0.935 & 1.000 & 0.957 & 1.000 & 0.976 & 1.000 & 0.993 & 1.000 & 0.999 \\ \cline{2-12}
    & SU-MCCDs & 1.000 & 0.961 & 1.000 & 0.975 & 1.000 & 0.991 & 1.000 & 0.996 & 1.000 & 1.000 \\ \cline{2-12}
    & UN-MCCDs & 1.000 & 0.973 & 1.000 & 0.986 & 1.000 & 0.994 & 1.000 & 0.999 & 1.000 & 1.000 \\ \cline{2-12}
    & SUN-MCCDs & 1.000 & 0.998 & 1.000 & 0.999 & 1.000 & 0.999 & 1.000 & 1.000 & 1.000 & 1.000 \\ \cline{1-12}

    \multirow{4}*{$d=20$} & RU-MCCDs & 1.000 & 0.846 & 1.000 & 0.865 & 1.000 & 0.883 & 1.000 & 0.861 & 1.000 & 0.850 \\ \cline{2-12}
    & SU-MCCDs & 1.000 & 0.881 & 1.000 & 0.896 & 1.000 & 0.924 & 1.000 & 0.908 & 1.000 & 0.893 \\ \cline{2-12}
    & UN-MCCDs & 1.000 & 0.951 & 1.000 & 0.971 & 1.000 & 0.984 & 1.000 & 0.992 & 1.000 & 0.994 \\ \cline{2-12}
    & SUN-MCCDs & 1.000 & 0.974 & 1.000 & 0.983 & 1.000 & 0.992 & 1.000 & 0.996 & 1.000 & 1.000 \\ \cline{1-12}

    \multirow{4}*{$d=50$} & RU-MCCDs & 1.000 & 0.567 & 1.000 & 0.542 & 1.000 & 0.534 & 1.000 & 0.534 & 1.000 & 0.543 \\ \cline{2-12}
    & SU-MCCDs & 1.000 & 0.568 & 1.000 & 0.542 & 1.000 & 0.534 & 1.000 & 0.534 & 1.000 & 0.542 \\ \cline{2-12}
    & UN-MCCDs & 1.000 & 0.659 & 1.000 & 0.681 & 1.000 & 0.708 & 1.000 & 0.723 & 1.000 & 0.733 \\ \cline{2-12}
    & SUN-MCCDs & 1.000 & 0.682 & 1.000 & 0.727 & 1.000 & 0.794 & 1.000 & 0.824 & 1.000 & 0.864 \\ \cline{1-12}

    \multirow{4}*{$d=100$} & RU-MCCDs & 1.000 & 0.550 & 1.000 & 0.540 & 1.000 & 0.529 & 1.000 & 0.521 & 1.000 & 0.514 \\ \cline{2-12}
    & SU-MCCDs & 1.000 & 0.550 & 1.000 & 0.541 & 1.000 & 0.530 & 1.000 & 0.522 & 1.000 & 0.515 \\ \cline{2-12}
    & UN-MCCDs & 1.000 & 0.131 & 1.000 & 0.161 & 1.000 & 0.228 & 1.000 & 0.456 & 1.000 & 0.434 \\ \cline{2-12}
    & SUN-MCCDs & 1.000 & 0.131 & 1.000 & 0.161 & 1.000 & 0.228 & 1.000 & 0.456 & 1.000 & 0.435 \\ \cline{1-12}
  \end{tabular}}
  \caption{Summary of the TPR and TNR of all the CCD-based outlier detection algorithms, with the simulation settings elaborated in Section \ref{sec:Uni_General_Settings_Des}.}\label{tab:Uni_General_Results1}
\end{table}

\begin{figure}[htb]
\centering
\subfigure{\includegraphics[width=1\linewidth]{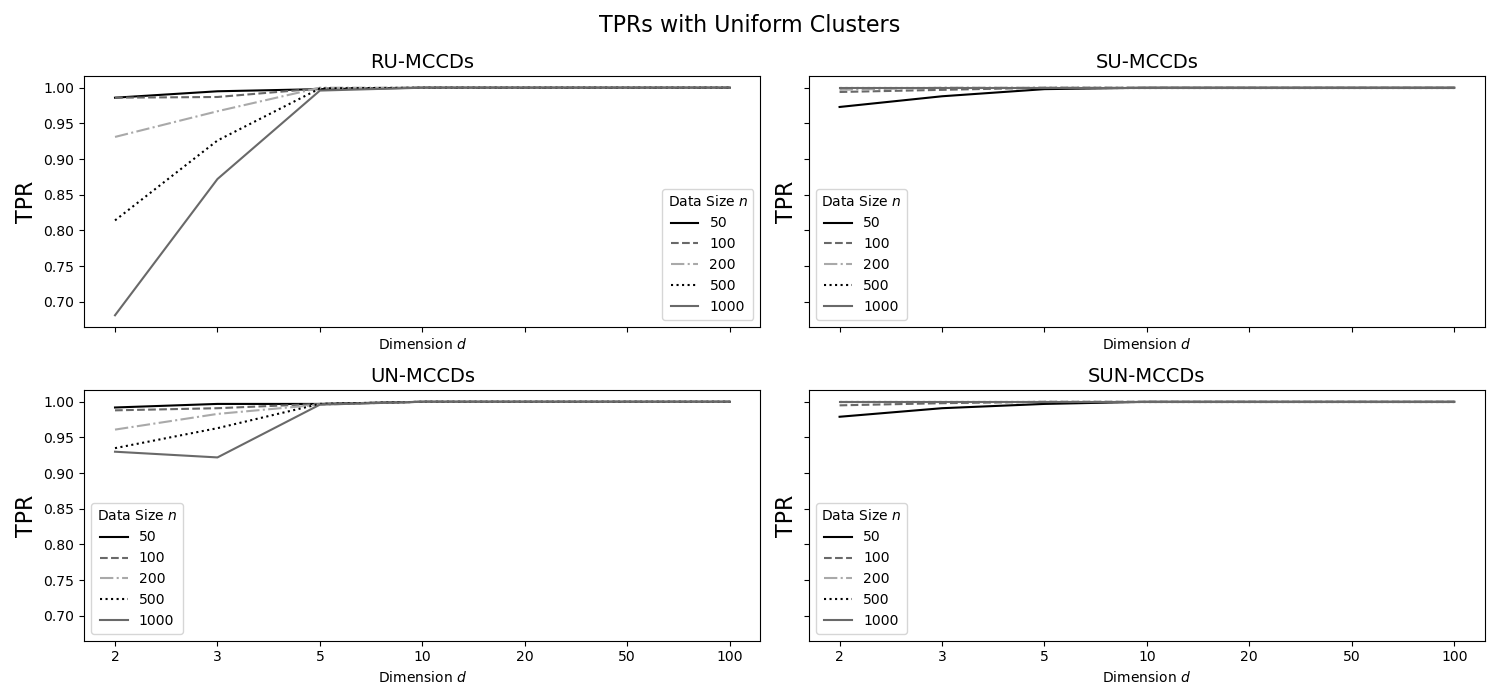}}
\subfigure{\includegraphics[width=1\linewidth]{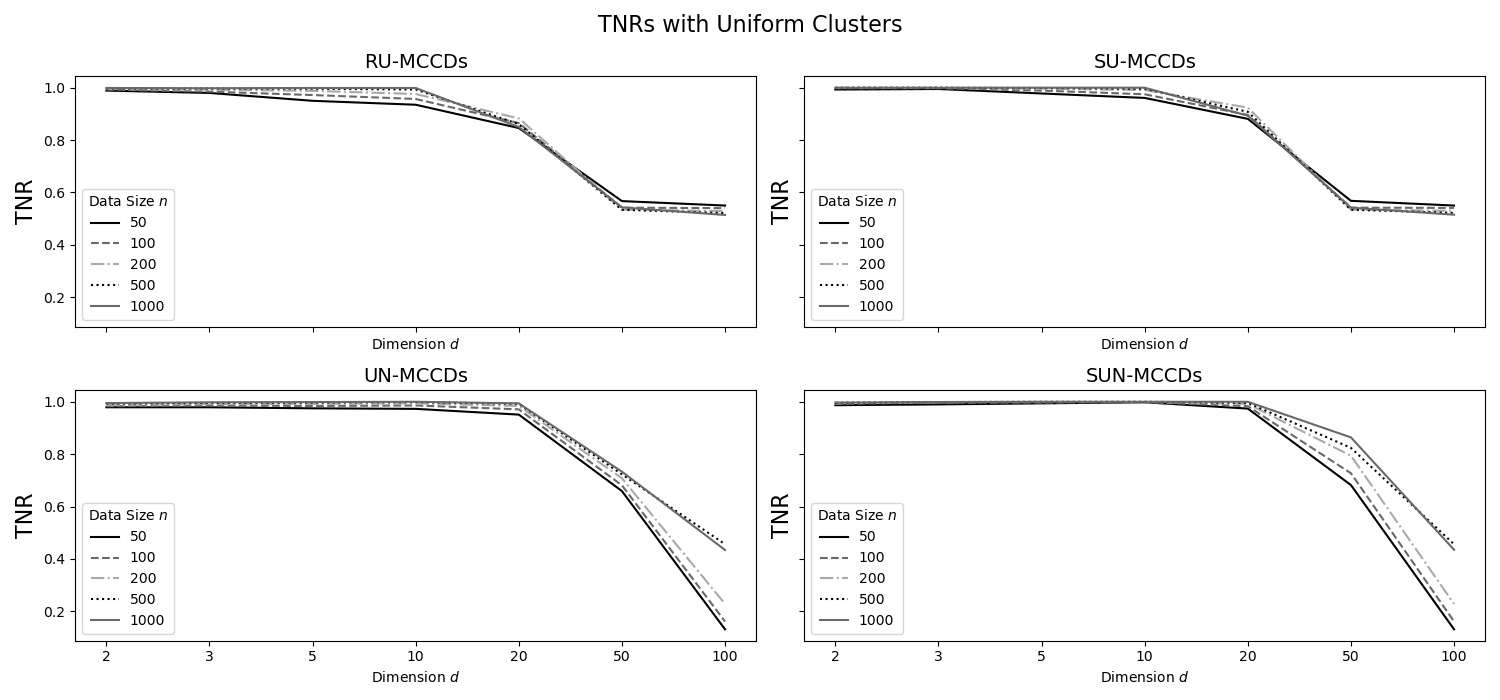}}
\caption{The line plots of the TPRs and TNRs of all CCD-based outlier detection algorithms, under the simulation settings (with uniform clusters) elaborated in Section \ref{sec:Uni_General_Settings_Des}.}
\label{fig:Uniform_TPR_TNR_Lines}
\end{figure}

\begin{table}[htb]
  \centering
  \resizebox{\columnwidth}{!}{\begin{tabular}{|c|c|c|c|c|c|c|c|c|c|c|c|}
    \hline
    \multicolumn{2}{|c|}{} & \multicolumn{10}{|c|}{The Size of Data Sets} \\ \cline{3-12}

    \multicolumn{2}{|c|}{} & \multicolumn{2}{|c|}{50} & \multicolumn{2}{|c|}{100} & \multicolumn{2}{|c|}{200} & \multicolumn{2}{|c|}{500} & \multicolumn{2}{|c|}{1000} \\ \cline{3-12}

    \multicolumn{2}{|c|}{} & BA & $F_2$-score & BA & $F_2$-score & BA & $F_2$-score & BA & $F_2$-score & BA & $F_2$-score \\ \hline

    \multirow{4}*{$d=2$} & RU-MCCDs & 0.980 & 0.949 & 0.990 & 0.963 & 0.963 & 0.926 & 0.906 & 0.836 & 0.840 & 0.724 \\ \cline{2-12}
    & SU-MCCDs & 0.983 & 0.953 & 0.996 & 0.988 & 0.998 & 0.994 & 1.000 & 1.000 & 1.000 & 1.000 \\ \cline{2-12}
    & UN-MCCDs & 0.986 & 0.920 & 0.986 & 0.930 & 0.975 & 0.926 & 0.964 & 0.922 & 0.963 & 0.925 \\ \cline{2-12}
    & SUN-MCCDs & 0.983 & 0.937 & 0.994 & 0.967 & 0.997 & 0.978 & 0.998 & 0.985 & 0.999 & 0.989 \\ \cline{1-12}

    \multirow{4}*{$d=3$} & RU-MCCDs & 0.988 & 0.926 & 0.986 & 0.936 & 0.980 & 0.945 & 0.962 & 0.929 & 0.935 & 0.888 \\ \cline{2-12}
    & SU-MCCDs & 0.992 & 0.972 & 0.998 & 0.990 & 1.000 & 0.996 & 1.000 & 1.000 & 1.000 & 1.000 \\ \cline{2-12}
    & UN-MCCDs & 0.988 & 0.924 & 0.989 & 0.943 & 0.987 & 0.954 & 0.980 & 0.956 & 0.960 & 0.929 \\ \cline{2-12}
    & SUN-MCCDs & 0.991 & 0.956 & 0.998 & 0.987 & 0.999 & 0.992 & 1.000 & 0.996 & 1.000 & 0.996 \\ \cline{1-12}

    \multirow{4}*{$d=5$} & RU-MCCDs & 0.974 & 0.839 & 0.986 & 0.903 & 0.994 & 0.956 & 0.998 & 0.984 & 0.998 & 0.993 \\ \cline{2-12}
    & SU-MCCDs & 0.988 & 0.921 & 0.995 & 0.960 & 0.998 & 0.985 & 1.000 & 0.996 & 1.000 & 1.000 \\ \cline{2-12}
    & UN-MCCDs & 0.986 & 0.911 & 0.991 & 0.940 & 0.994 & 0.967 & 0.997 & 0.986 & 0.998 & 0.993 \\ \cline{2-12}
    & SUN-MCCDs & 0.996 & 0.975 & 0.999 & 0.989 & 1.000 & 0.996 & 1.000 & 1.000 & 1.000 & 1.000 \\ \cline{1-12}

    \multirow{4}*{$d=10$} & RU-MCCDs & 0.968 & 0.802 & 0.979 & 0.860 & 0.988 & 0.916 & 0.997 & 0.974 & 1.000 & 0.996 \\ \cline{2-12}
    & SU-MCCDs & 0.981 & 0.871 & 0.988 & 0.913 & 0.996 & 0.967 & 0.998 & 0.985 & 1.000 & 1.000 \\ \cline{2-12}
    & UN-MCCDs & 0.987 & 0.907 & 0.993 & 0.949 & 0.997 & 0.978 & 1.000 & 0.996 & 1.000 & 1.000 \\ \cline{2-12}
    & SUN-MCCDs & 0.999 & 0.992 & 1.000 & 0.996 & 1.000 & 0.996 & 1.000 & 1.000 & 1.000 & 1.000 \\ \cline{1-12}

    \multirow{4}*{$d=20$} & RU-MCCDs & 0.923 & 0.631 & 0.933 & 0.661 & 0.942 & 0.692 & 0.931 & 0.654 & 0.925 & 0.637 \\ \cline{2-12}
    & SU-MCCDs & 0.941 & 0.689 & 0.948 & 0.717 & 0.962 & 0.776 & 0.954 & 0.741 & 0.947 & 0.711 \\ \cline{2-12}
    & UN-MCCDs & 0.976 & 0.843 & 0.986 & 0.901 & 0.992 & 0.943 & 0.996 & 0.970 & 0.997 & 0.978 \\ \cline{2-12}
    & SUN-MCCDs & 0.987 & 0.910 & 0.992 & 0.939 & 0.996 & 0.970 & 0.998 & 0.985 & 1.000 & 1.000 \\ \cline{1-12}

    \multirow{4}*{$d=50$} & RU-MCCDs & 0.784 & 0.378 & 0.771 & 0.365 & 0.767 & 0.361 & 0.767 & 0.361 & 0.772 & 0.365 \\ \cline{2-12}
    & SU-MCCDs & 0.784 & 0.379 & 0.771 & 0.365 & 0.767 & 0.361 & 0.767 & 0.361 & 0.771 & 0.365 \\ \cline{2-12}
    & UN-MCCDs & 0.830 & 0.436 & 0.841 & 0.452 & 0.854 & 0.474 & 0.862 & 0.487 & 0.867 & 0.496 \\ \cline{2-12}
    & SUN-MCCDs & 0.841 & 0.453 & 0.864 & 0.491 & 0.897 & 0.561 & 0.912 & 0.599 & 0.932 & 0.659 \\ \cline{1-12}

    \multirow{4}*{$d=100$} & RU-MCCDs & 0.775 & 0.369 & 0.770 & 0.364 & 0.765 & 0.358 & 0.761 & 0.355 & 0.757 & 0.351 \\ \cline{2-12}
    & SU-MCCDs & 0.775 & 0.369 & 0.771 & 0.364 & 0.765 & 0.359 & 0.761 & 0.355 & 0.758 & 0.352 \\ \cline{2-12}
    & UN-MCCDs & 0.566 & 0.232 & 0.581 & 0.239 & 0.614 & 0.254 & 0.728 & 0.326 & 0.717 & 0.317 \\ \cline{2-12}
    & SUN-MCCDs & 0.566 & 0.232 & 0.581 & 0.239 & 0.614 & 0.254 & 0.728 & 0.326 & 0.718 & 0.318 \\ \cline{1-12}
  \end{tabular}}
  \caption{Summary of the Balanced Accuracy (BA) and $F_2$-score of all the CCD-based outlier detection algorithms, with the simulation settings elaborated in Section \ref{sec:Uni_General_Settings_Des}.}\label{tab:Uni_General_Results2}
\end{table}

\begin{figure}[htb]
\centering
\subfigure{\includegraphics[width=1\linewidth]{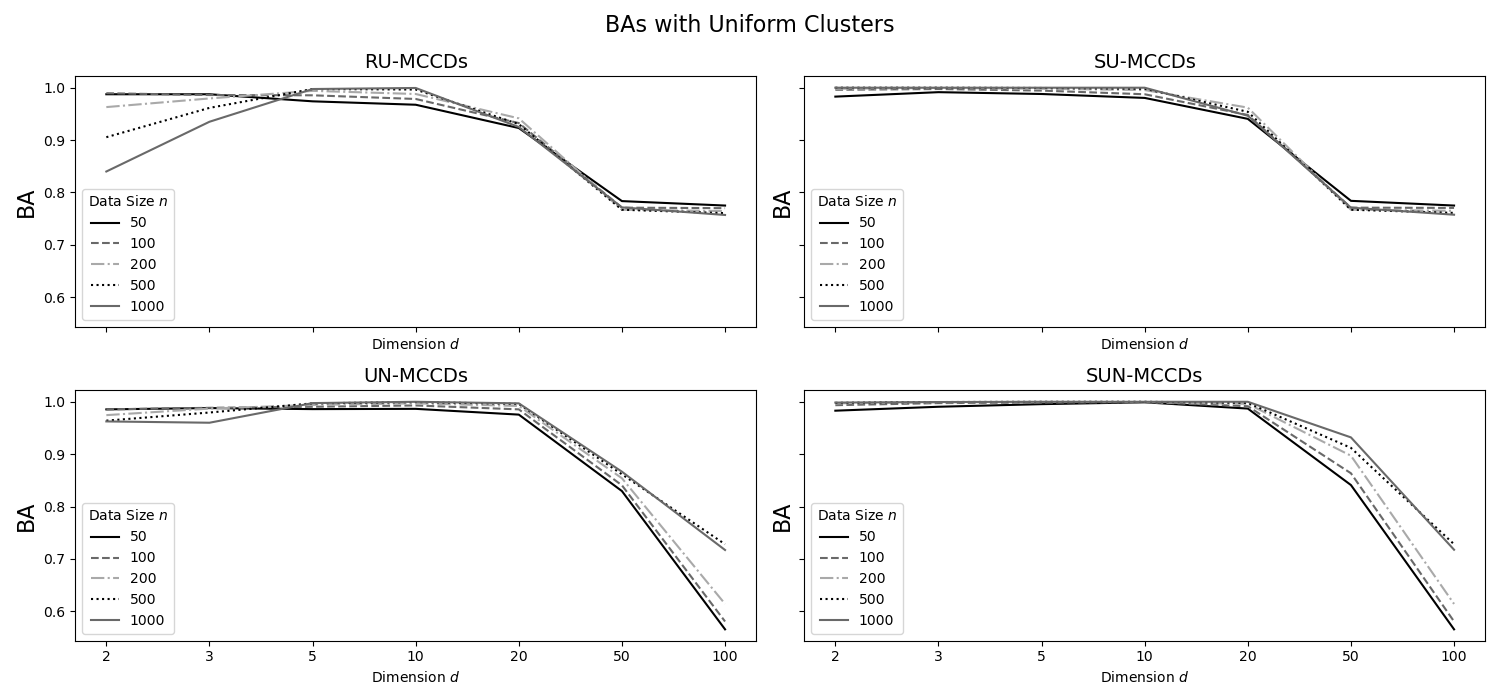}}
\subfigure{\includegraphics[width=1\linewidth]{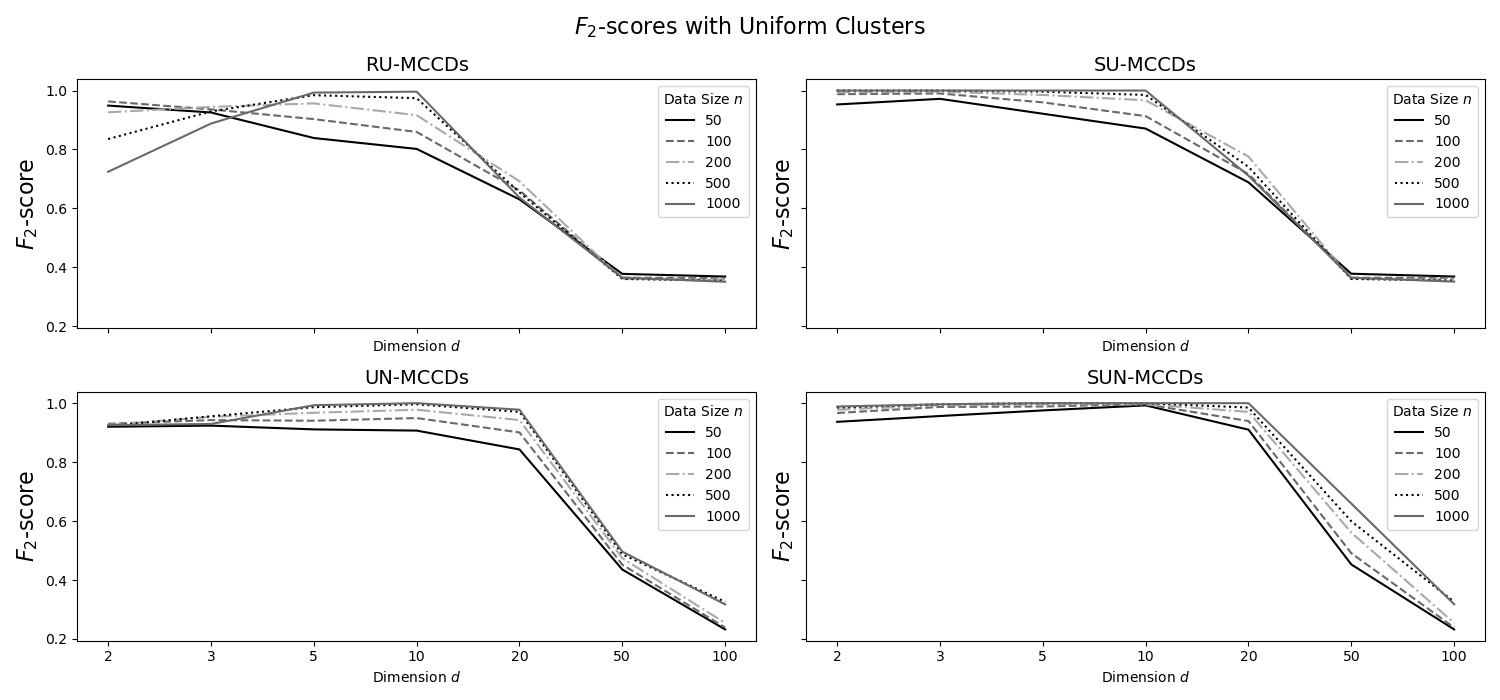}}
\caption{The line plots of the BAs and $F_2$-scores of all CCD-based outlier detection algorithms, under the simulation settings (with uniform clusters) elaborated in Section \ref{sec:Uni_General_Settings_Des}.}
\label{fig:Uniform_BA_F_Lines}
\end{figure}

We first focus on the simulation results under low- and moderate-dimensional space, which are ($d\leq20$) presented in Tables \ref{tab:Uni_General_Results1} and \ref{tab:Uni_General_Results2}.

The RU-MCCD algorithm delivers satisfactory performance considering the percentage of outliers captured. Most of the TPRs are well above $0.9$ or even $0.95$ and equal to $1$ when $d=10,20$; additionally, due to the effectiveness of RK-CCDs on clustering in low dimensions ($d=2,3,5$), the TNRs are well above $0.95$, even when the number of observations is as low as 50. Therefore, the RU-MCCD algorithm also delivers high BAs and $F_2$-score under those dimensions, most of which are above $0.9$. However, there are some exceptions: (1) when $d=2,3$, the effectiveness of the RU-MCCD algorithm declines substantially with a higher number of observations (e.g., the TPRs of the RU-MCCD algorithm are 0.986, 0.986, 0.931, 0.814, 0.681 when $d=2$ as $n$ increases). The declining performance is due to the increasing intensities of outliers, especially in the cases with fewer dimensions ($d=2,3$) where the volume or the area of the support is relatively small. With high intensities, RK-CCDs may falsely construct clusters for a bunch of close outliers and perceive them as regular (i.e., non-outlier) observations, which is called the \emph{masking problem} in outlier detection. Thus, all four measures reduce as the number of observations increases. The lowest readings are observed when $n=1000$ and $d=2$, each falling below $0.9$. (2) While most TNRs are near 1, they fall substantially and are less than $0.9$ when $d=20$. Increasing the number of observations provides little help. We have discussed its reason in Section \ref{sec:UN-CCDs}. In short, several drawbacks of SR-MCT based on the Ripley's $K$ function lead to the limitation, which is negligible when $d$ is small but is greatly exacerbated as $d$ increases when RK-CCDs provide much smaller covering balls and leaves many regular observations uncovered. Therefore, although the BAs are still above $0.9$ when $d=20$, the $F_2$-scores drop to less than $0.7$ since $F_2$-score is much more sensitive to TNR and has less tolerance on false positives.

Next, we consider the performance of the UN-MCCD algorithm. UN-CCDs work similarly to RK-CCDs except for the SR-MCT. Instead of using the Ripley's $K$ function, UN-CCDs conduct SR-MCT based on the average and median NND, which avoids RK-CCDs' shortcomings. Therefore, the UN-MCCD algorithm performs better than the RU-MCCD algorithm across all the simulation settings. However, since both algorithms share almost identical mechanisms, the UN-MCCD algorithm captures almost all outliers with slight errors when $d=2,3$, and the lowest TPR of $0.930$ is observed when $d=2$ and $n=1000$, where BA and $F_2-score$ are $0.963$ and $0.925$ respectively. When $d=20$ and $n=50$, the TNR decreases slightly to $0.951$ due to the low spatial intensity in $\mathbb{R}^{20}$, where BA and $F_2-score$ are $0.976$ and $0.843$. Fortunately, all four measures increase with increasing data sizes ($n$) when $d \leq 5$.

The SU-MCCD and SUN-MCCD algorithms are the flexible adaptations of the RU-MCCD and UN-MCCD algorithms, respectively. Rather than using a single dominating covering ball, they look for a bunch of points connected to the center of a dominating covering ball in the MCG and construct a cluster by taking the union of their covering balls. Theoretically, both could deliver better performance when clusters are arbitrarily shaped (including the cases when the intensities of clusters are uneven). Nevertheless, it is still interesting to compare the performance of the two ``flexible" algorithms with their prototypes when the support of each cluster is standard hyperspheres.

The SU-MCCD and SUN-MCCD algorithms deliver higher TNRs when $d\leq20$, especially the SUN-MCCD algorithms, whose TNRs are close to 1 under all simulation settings. This is expected since using more covering balls leads to better coverage for each cluster. For example, when $d=20$, notice that the SU-MCCD algorithm performs better compared to its prototype (the RU-MCCD algorithm) due to much higher TNRs, the $F_2$-scores of the SU-MCCD algorithms are 0.689, 0.717, 0.776, 0.741 and 0.711 versus 0.631, 0.661, 0.692, 0.654 and 0.637 of the RU-MCCD algorithm. Recall that the effectiveness of both the RU-MCCD and UN-MCCD algorithms declines due to the masking problem as the intensity of outliers grows when $d=2,3$. Fortunately, the SU-MCCD and SUN-MCCD algorithms overcome this problem and yield high TPR even when $n=1000$, attributed to the new mechanism that filters small clusters.

The performance of the two ``flexible" algorithms is comparable when $d \leq 5$, and the SUN-MCCD algorithm delivers slightly greater $F_2$ scores when the data size $n$ is small, and it performs much better when $d=10,20$ due to the disadvantages of the SU-MCCD algorithm under a high-dimensional space. For example, when $d=20$, the $F_2$-scores of the SUN-MCCD algorithms are 0.910, 0.939, 0.970, 0.985, and 1.000 versus 0.689, 0.717, 0.776, 0.741, and 0.711 of the SU-MCCD algorithm.

However, when $d=50,100$, all the four algorithms perform worse. The TNRs become substantially smaller than those with fewer dimensions, particularly when $d=100$, where most BAs are between 0.5 and 0.7, close to random guesses. The $F_2$-scores, sensitive to precision, drop between 0.2 and 0.5. This is because, under high-dimensional space, all the regular points tend to be distributed along the border of the clusters they belong to, even if they are uniformly distributed. Hence, the difficulty in capturing most of them increases substantially as $d$ increases, and few clustering-based outlier detection algorithms could still deliver promising performance without dimensionality reduction techniques. Additionally, it is worth noting that the performance of the two ``flexible" algorithms degrades and is close or equal to the results of their prototypes. It can be explained by the reason that almost every point is isolated points under the MCG constructed on extremely high-dimensional space, and there are none or few points that are connected to the center of dominating covering balls, resulting in only one covering ball for most clusters.

We know that RK-CCDs and UN-CCDs conduct SR-MCT that finds clusters following HPP, which means the points within each constructed cluster are approximately uniformly distributed. Therefore, the CCD-based algorithms prefer the simulation experiments with only uniform clusters, particularly the RU-MCCD and UN-MCCD algorithms. Thus, in addition to the above experiments, we perform similar simulations under Gaussian settings, where regular data points from the same cluster are multivariate-normally distributed (but uncorrelated). We aim to investigate the effectiveness of these CCD-based algorithms when data points within a cluster are nonuniformly distributed. There are two major challenges to finding outliers with Gaussian clusters: capturing the regular data points near the boundary of a cluster where the intensity is much lower than the center while distinguishing outliers with similar intensities. \label{sec:Gau_General_Settings_Des}

To make the simulation experiments with Gaussian clusters comparable to the previous ones with uniform clusters, we choose the scale of the covariance matrix according to the dimensionality $d$ and radius $R$ such that approximate $99\%$ points of the Gaussian cluster located within a hypersphere with radius $R$ (recall $R$ is random number from 0.7 to 1.3), and the approximate $1\%$ points located beyond the hypersphere are perceived to be noise near the cluster (The noise level here represents the percentage of data points that are randomly generated near the range of the clusters. The outliers are data points that are far away from the cluster centers). Similarly, $R$ is a random variable generated uniformly between 0.7 and 1.3, so clusters with different volumes and intensities can be constructed. Except for the way to simulate Gaussian clusters, which we have elaborated on particularly, all the other settings (dimensionality, the sizes of data sets, the centers of clusters, etc.) remain the same. Two realizations with data sizes of 100 and 200 are presented in Figure \ref{fig:Demo_2d_2gcls_cont5}. The performance measures of the four algorithms are summarized in Tables \ref{tab:Gau_General_Results1} and \ref{tab:Gau_General_Results2}. Similarly, we present the line plots of the results in Figures \ref{fig:Gaussian_TPR_TNR_Lines} and \ref{fig:Gaussian_BA_F_Lines}. 

\begin{figure}[htb]
\centering
\subfigure{
\label{Demo_2d_2gcls_n100_cont5}
\includegraphics[width=0.35\textwidth]{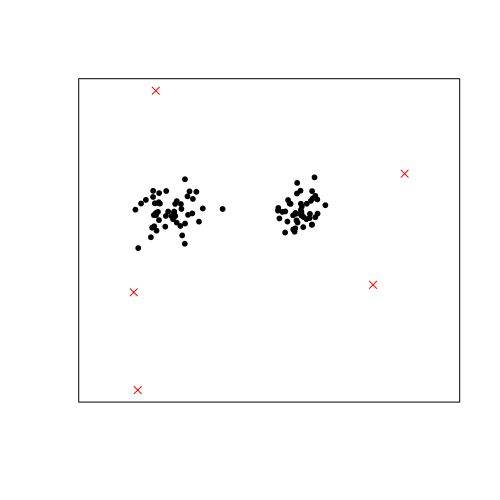}}
\subfigure{
\label{Demo_2d_2gcls_n200_cont5}
\includegraphics[width=0.35\textwidth]{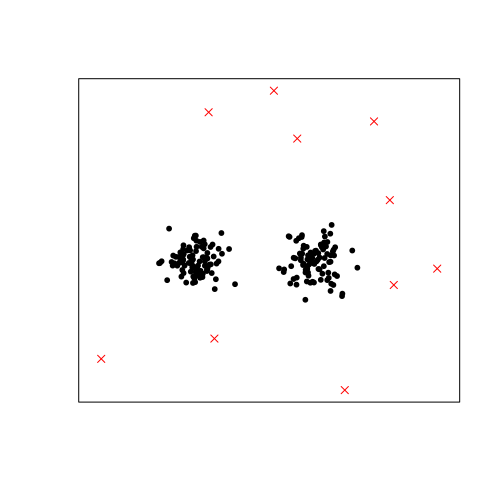}}
\caption{Two realizations of the simulation settings with Gaussian clusters, where $n=100$ and $200$ respectively. Each data set has 2 clusters of the same size but different intensities. Black points are regular data points, and red points are outliers. The numbers of observations are indicated below each sub-figure.}
\label{fig:Demo_2d_2gcls_cont5}
\end{figure}

\begin{table}[htb]
  \centering
  \resizebox{\columnwidth}{!}{\begin{tabular}{|c|c|c|c|c|c|c|c|c|c|c|c|}
    \hline
    \multicolumn{2}{|c|}{} & \multicolumn{10}{|c|}{The Size of Data Sets} \\ \cline{3-12}

    \multicolumn{2}{|c|}{} & \multicolumn{2}{|c|}{50} & \multicolumn{2}{|c|}{100} & \multicolumn{2}{|c|}{200} & \multicolumn{2}{|c|}{500} & \multicolumn{2}{|c|}{1000} \\ \cline{3-12}

    \multicolumn{2}{|c|}{} & TPR & TNR & TPR & TNR & TPR & TNR & TPR & TNR & TPR & TNR \\ \hline

    \multirow{4}*{$d=2$} & RU-MCCDs & 0.994 & 0.918 & 0.998 & 0.886 & 1.000 & 0.851 & 1.000 & 0.818 & 1.000 & 0.792 \\ \cline{2-12}
    & SU-MCCDs & 0.994 & 0.970 & 0.999 & 0.962 & 1.000 & 0.951 & 1.000 & 0.927 & 1.000 & 0.903 \\ \cline{2-12}
    & UN-MCCDs & 0.995 & 0.942 & 0.978 & 0.927 & 0.988 & 0.914 & 0.994 & 0.895 & 0.999 & 0.880 \\ \cline{2-12}
    & SUN-MCCDs & 0.996 & 0.964 & 0.995 & 0.964 & 0.999 & 0.962 & 1.000 & 0.958 & 1.000 & 0.950 \\ \cline{1-12}

    \multirow{4}*{$d=3$} & RU-MCCDs & 0.999 & 0.880 & 1.000 & 0.849 & 1.000 & 0.818 & 1.000 & 0.784 & 1.000 & 0.760 \\ \cline{2-12}
    & SU-MCCDs & 0.999 & 0.948 & 1.000 & 0.943 & 1.000 & 0.937 & 1.000 & 0.921 & 1.000 & 0.903 \\ \cline{2-12}
    & UN-MCCDs & 0.998 & 0.922 & 0.995 & 0.902 & 0.997 & 0.884 & 0.999 & 0.862 & 1.000 & 0.842 \\ \cline{2-12}
    & SUN-MCCDs & 0.998 & 0.957 & 0.999 & 0.958 & 1.000 & 0.959 & 1.000 & 0.955 & 1.000 & 0.947 \\ \cline{1-12}

    \multirow{4}*{$d=5$} & RU-MCCDs & 1.000 & 0.821 & 1.000 & 0.797 & 1.000 & 0.777 & 1.000 & 0.755 & 1.000 & 0.727 \\ \cline{2-12}
    & SU-MCCDs & 1.000 & 0.887 & 1.000 & 0.886 & 1.000 & 0.890 & 1.000 & 0.891 & 1.000 & 0.879 \\ \cline{2-12}
    & UN-MCCDs & 1.000 & 0.888 & 0.999 & 0.865 & 1.000 & 0.846 & 1.000 & 0.820 & 1.000 & 0.794 \\ \cline{2-12}
    & SUN-MCCDs & 1.000 & 0.939 & 1.000 & 0.941 & 1.000 & 0.943 & 1.000 & 0.945 & 1.000 & 0.942 \\ \cline{1-12}

    \multirow{4}*{$d=10$} & RU-MCCDs & 1.000 & 0.748 & 1.000 & 0.715 & 1.000 & 0.698 & 1.000 & 0.693 & 1.000 & 0.684 \\ \cline{2-12}
    & SU-MCCDs & 1.000 & 0.813 & 1.000 & 0.797 & 1.000 & 0.791 & 1.000 & 0.797 & 1.000 & 0.794 \\ \cline{2-12}
    & UN-MCCDs & 1.000 & 0.856 & 1.000 & 0.832 & 1.000 & 0.816 & 1.000 & 0.797 & 1.000 & 0.770 \\ \cline{2-12}
    & SUN-MCCDs & 1.000 & 0.960 & 1.000 & 0.949 & 1.000 & 0.945 & 1.000 & 0.946 & 1.000 & 0.945\\ \cline{1-12}

    \multirow{4}*{$d=20$} & RU-MCCDs & 1.000 & 0.620 & 1.000 & 0.592 & 1.000 & 0.567 & 1.000 & 0.541 & 1.000 & 0.531 \\ \cline{2-12}
    & SU-MCCDs & 1.000 & 0.660 & 1.000 & 0.644 & 1.000 & 0.625 & 1.000 & 0.609 & 1.000 & 0.606 \\ \cline{2-12}
    & UN-MCCDs & 1.000 & 0.736 & 1.000 & 0.701 & 1.000 & 0.668 & 1.000 & 0.627 & 1.000 & 0.604 \\ \cline{2-12}
    & SUN-MCCDs & 1.000 & 0.826 & 1.000 & 0.804 & 1.000 & 0.796 & 1.000 & 0.789 & 1.000 & 0.784 \\ \cline{1-12}

    \multirow{4}*{$d=50$} & RU-MCCDs & 1.000 & 0.580 & 1.000 & 0.562 & 1.000 & 0.552 & 1.000 & 0.542 & 1.000 & 0.544 \\ \cline{2-12}
    & SU-MCCDs & 1.000 & 0.581 & 1.000 & 0.562 & 1.000 & 0.553 & 1.000 & 0.542 & 1.000 & 0.544 \\ \cline{2-12}
    & UN-MCCDs & 1.000 & 0.448 & 1.000 & 0.420 & 1.000 & 0.380 & 1.000 & 0.308 & 1.000 & 0.275 \\ \cline{2-12}
    & SUN-MCCDs & 1.000 & 0.457 & 1.000 & 0.444 & 1.000 & 0.417 & 1.000 & 0.366 & 1.000 & 0.351 \\ \cline{1-12}

    \multirow{4}*{$d=100$} & RU-MCCDs & 1.000 & 0.574 & 1.000 & 0.547 & 1.000 & 0.521 & 1.000 & 0.513 & 1.000 & 0.510 \\ \cline{2-12}
    & SU-MCCDs & 1.000 & 0.575 & 1.000 & 0.547 & 1.000 & 0.521 & 1.000 & 0.513 & 1.000 & 0.511 \\ \cline{2-12}
    & UN-MCCDs & 1.000 & 0.302 & 1.000 & 0.317 & 1.000 & 0.318 & 1.000 & 0.275 & 1.000 & 0.231 \\ \cline{2-12}
    & SUN-MCCDs & 1.000 & 0.302 & 1.000 & 0.317 & 1.000 & 0.319 & 1.000 & 0.276 & 1.000 & 0.232 \\ \cline{1-12}
  \end{tabular}}
  \caption{Summary of the TPR and TNR of all the CCD-based outlier detection algorithms, with the simulation settings elaborated in Section \ref{sec:Gau_General_Settings_Des}.}\label{tab:Gau_General_Results1}
\end{table}

\begin{figure}[htb]
\centering
\subfigure{\includegraphics[width=1\linewidth]{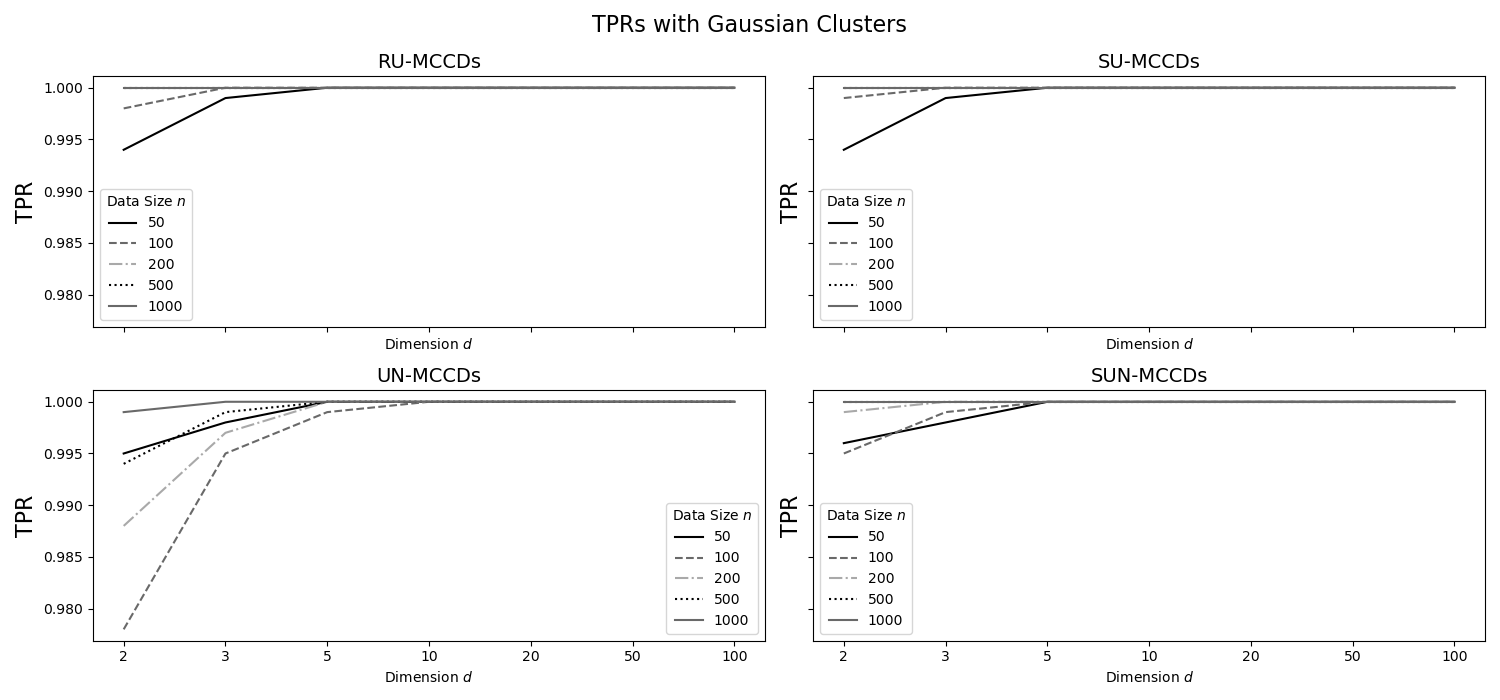}}
\subfigure{\includegraphics[width=1\linewidth]{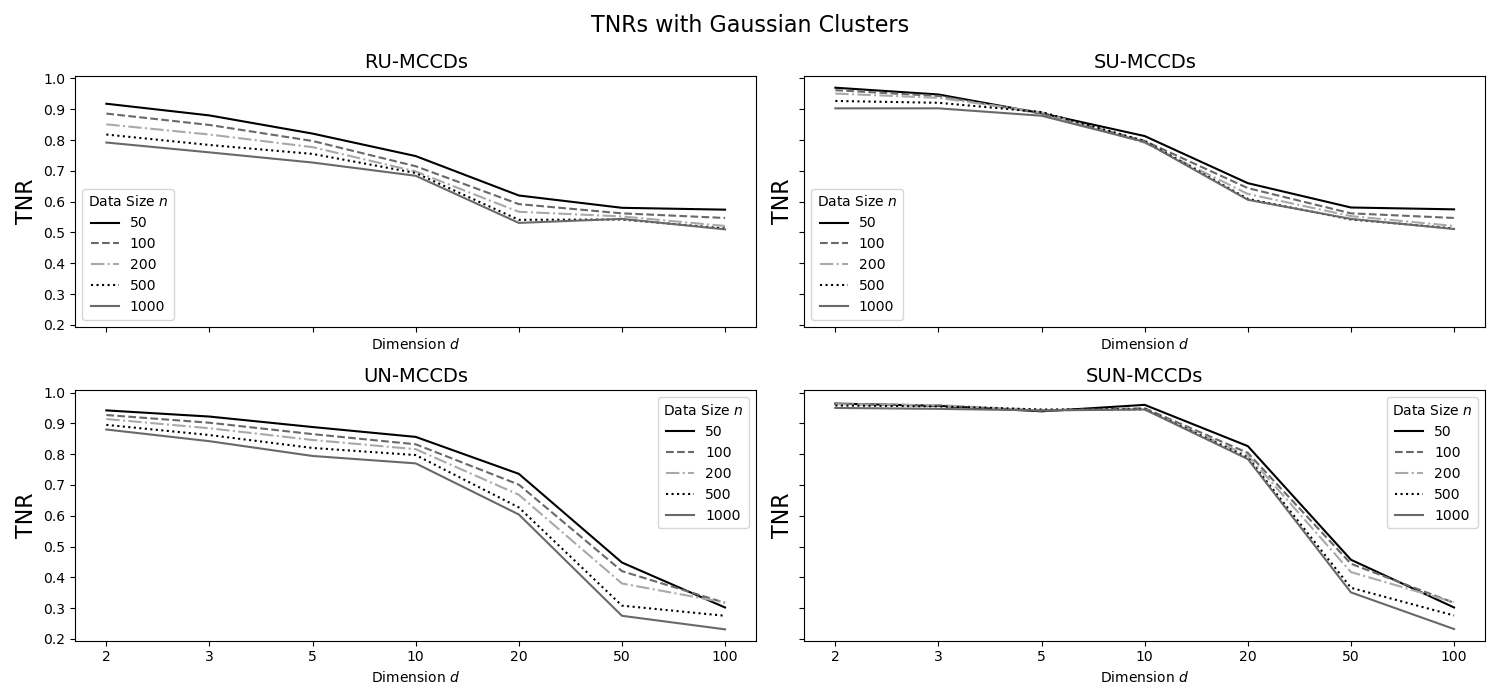}}
\caption{The line plots of the TPRs and TNRs of all CCD-based outlier detection algorithms, under the simulation settings (with Gaussian clusters) elaborated in Section \ref{sec:Gau_General_Settings_Des}.}
\label{fig:Gaussian_TPR_TNR_Lines}
\end{figure}

\begin{table}[htb]
  \centering
  \resizebox{\columnwidth}{!}{\begin{tabular}{|c|c|c|c|c|c|c|c|c|c|c|c|}
    \hline
    \multicolumn{2}{|c|}{} & \multicolumn{10}{|c|}{The Size of Data Sets} \\ \cline{3-12}

    \multicolumn{2}{|c|}{} & \multicolumn{2}{|c|}{50} & \multicolumn{2}{|c|}{100} & \multicolumn{2}{|c|}{200} & \multicolumn{2}{|c|}{500} & \multicolumn{2}{|c|}{1000} \\ \cline{3-12}

    \multicolumn{2}{|c|}{} & BA & $F_2$-score & BA & $F_2$-score & BA & $F_2$-score & BA & $F_2$-score & BA & $F_2$-score \\ \hline

    \multirow{4}*{$d=2$} & RU-MCCDs & 0.956 & 0.759 & 0.942 & 0.697 & 0.926 & 0.638 & 0.909 & 0.591 & 0.896 & 0.559 \\ \cline{2-12}
    & SU-MCCDs & 0.982 & 0.893 & 0.981 & 0.873 & 0.976 & 0.843 & 0.964 & 0.783 & 0.952 & 0.731 \\ \cline{2-12}
    & UN-MCCDs & 0.969 & 0.816 & 0.953 & 0.768 & 0.951 & 0.746 & 0.945 & 0.711 & 0.940 & 0.686 \\ \cline{2-12}
    & SUN-MCCDs & 0.980 & 0.877 & 0.980 & 0.876 & 0.981 & 0.873 & 0.979 & 0.862 & 0.975 & 0.840 \\ \cline{1-12}

    \multirow{4}*{$d=3$} & RU-MCCDs & 0.940 & 0.686 & 0.925 & 0.635 & 0.909 & 0.591 & 0.892 & 0.549 & 0.880 & 0.523 \\ \cline{2-12}
    & SU-MCCDs & 0.974 & 0.834 & 0.972 & 0.822 & 0.969 & 0.807 & 0.961 & 0.769 & 0.952 & 0.731 \\ \cline{2-12}
    & UN-MCCDs & 0.960 & 0.770 & 0.949 & 0.726 & 0.941 & 0.692 & 0.931 & 0.655 & 0.921 & 0.625 \\ \cline{2-12}
    & SUN-MCCDs & 0.978 & 0.858 & 0.979 & 0.862 & 0.980 & 0.865 & 0.978 & 0.854 & 0.974 & 0.832 \\ \cline{1-12}

    \multirow{4}*{$d=5$} & RU-MCCDs & 0.911 & 0.595 & 0.899 & 0.565 & 0.889 & 0.541 & 0.878 & 0.518 & 0.864 & 0.491 \\ \cline{2-12}
    & SU-MCCDs & 0.944 & 0.700 & 0.943 & 0.698 & 0.945 & 0.705 & 0.946 & 0.707 & 0.940 & 0.685 \\ \cline{2-12}
    & UN-MCCDs & 0.944 & 0.701 & 0.932 & 0.660 & 0.923 & 0.631 & 0.910 & 0.594 & 0.897 & 0.561 \\ \cline{2-12}
    & SUN-MCCDs & 0.970 & 0.812 & 0.971 & 0.817 & 0.972	& 0.822 & 0.973 & 0.827 & 0.971 & 0.819 \\ \cline{1-12}

    \multirow{4}*{$d=10$} & RU-MCCDs & 0.874 & 0.511 & 0.858 & 0.480 & 0.849 & 0.466 & 0.847 & 0.462 & 0.842 & 0.454 \\ \cline{2-12}
    & SU-MCCDs & 0.907 & 0.585 & 0.899 & 0.565 & 0.896 & 0.557 & 0.899 & 0.565 & 0.897 & 0.561 \\ \cline{2-12}
    & UN-MCCDs & 0.928 & 0.646 & 0.916 & 0.610 & 0.908 & 0.589 & 0.899 & 0.565 & 0.885 & 0.534 \\ \cline{2-12}
    & SUN-MCCDs & 0.980 & 0.868 & 0.975 & 0.838 & 0.973 & 0.827 & 0.973 & 0.830 & 0.973 & 0.827 \\ \cline{1-12}

    \multirow{4}*{$d=20$} & RU-MCCDs & 0.810 & 0.409 & 0.796 & 0.392 & 0.784 & 0.378 & 0.771 & 0.364 & 0.766 & 0.359 \\ \cline{2-12}
    & SU-MCCDs & 0.830 & 0.436 & 0.822 & 0.425 & 0.813 & 0.412 & 0.805 & 0.402 & 0.803 & 0.400 \\ \cline{2-12}
    & UN-MCCDs & 0.868 & 0.499 & 0.851 & 0.468 & 0.834 & 0.442 & 0.814 & 0.414 & 0.802 & 0.399 \\ \cline{2-12}
    & SUN-MCCDs & 0.913 & 0.602 & 0.902 & 0.573 & 0.898 & 0.563 & 0.895 & 0.555 & 0.892 & 0.549 \\ \cline{1-12}

    \multirow{4}*{$d=50$} & RU-MCCDs & 0.790 & 0.385 & 0.781 & 0.375 & 0.776 & 0.370 & 0.771 & 0.365 & 0.772 & 0.366 \\ \cline{2-12}
    & SU-MCCDs & 0.791 & 0.386 & 0.781 & 0.375 & 0.777 & 0.371 & 0.771 & 0.365 & 0.772 & 0.366 \\ \cline{2-12}
    & UN-MCCDs & 0.724 & 0.323 & 0.710 & 0.312 & 0.690 & 0.298 & 0.654 & 0.276 & 0.638 & 0.266 \\ \cline{2-12}
    & SUN-MCCDs & 0.729 & 0.326 & 0.722 & 0.321 & 0.709 & 0.311 & 0.683 & 0.293 & 0.676 & 0.289 \\ \cline{1-12}

    \multirow{4}*{$d=100$} & RU-MCCDs & 0.787 & 0.382 & 0.774 & 0.367 & 0.761 & 0.355 & 0.757 & 0.351 & 0.755 & 0.349 \\ \cline{2-12}
    & SU-MCCDs & 0.788 & 0.382 & 0.774 & 0.367 & 0.761 & 0.355 & 0.757 & 0.351 & 0.756 & 0.350 \\ \cline{2-12}
    & UN-MCCDs & 0.651 & 0.274 & 0.659 & 0.278 & 0.659 & 0.278 & 0.638 & 0.266 & 0.616 & 0.255 \\ \cline{2-12}
    & SUN-MCCDs & 0.651 & 0.274 & 0.659 & 0.278 & 0.660 & 0.279 & 0.638 & 0.267 & 0.616 & 0.255\\ \cline{1-12}
  \end{tabular}}
  \caption{Summary of the Balanced Accuracy (BA) and $F_2$-score of all the CCD-based outlier detection algorithms, with the simulation settings elaborated in Section \ref{sec:Gau_General_Settings_Des}.}\label{tab:Gau_General_Results2}
\end{table}

\begin{figure}[htb]
\centering
\subfigure{\includegraphics[width=1\linewidth]{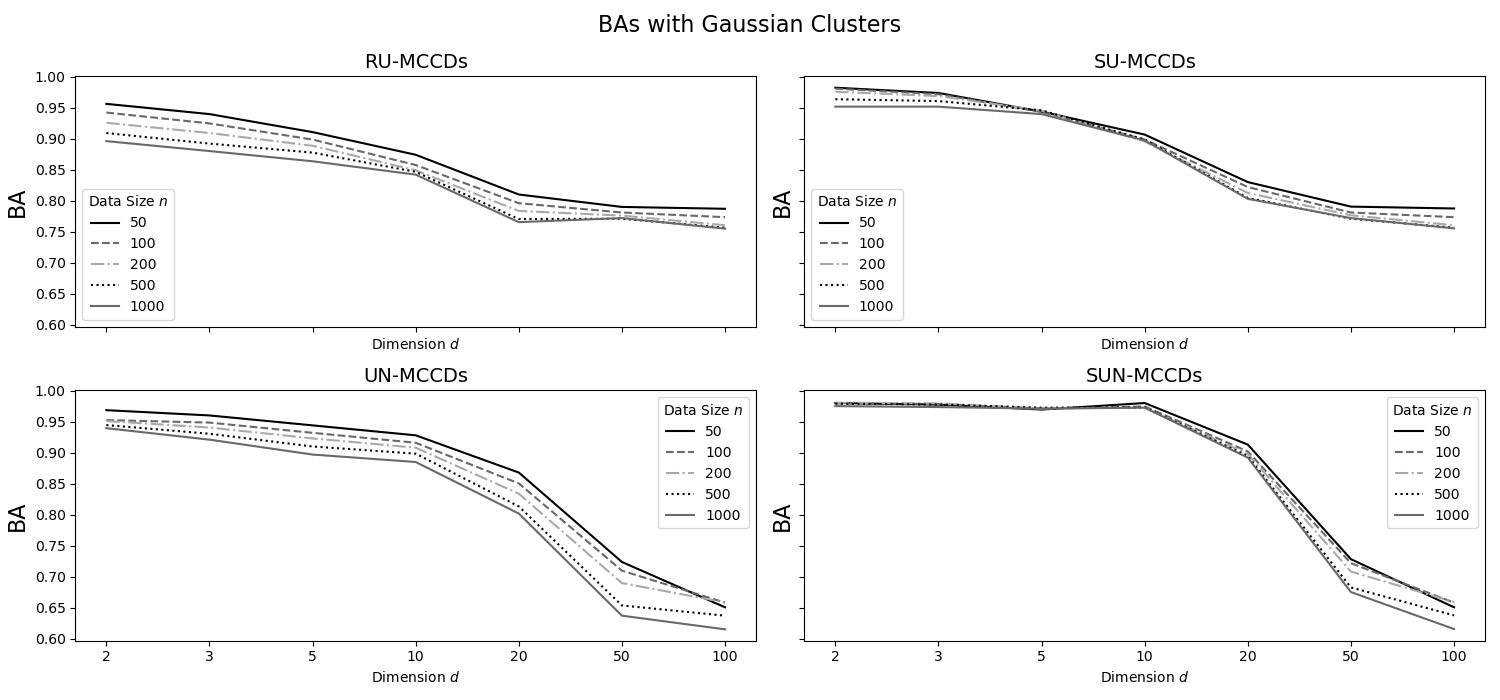}}
\subfigure{\includegraphics[width=1\linewidth]{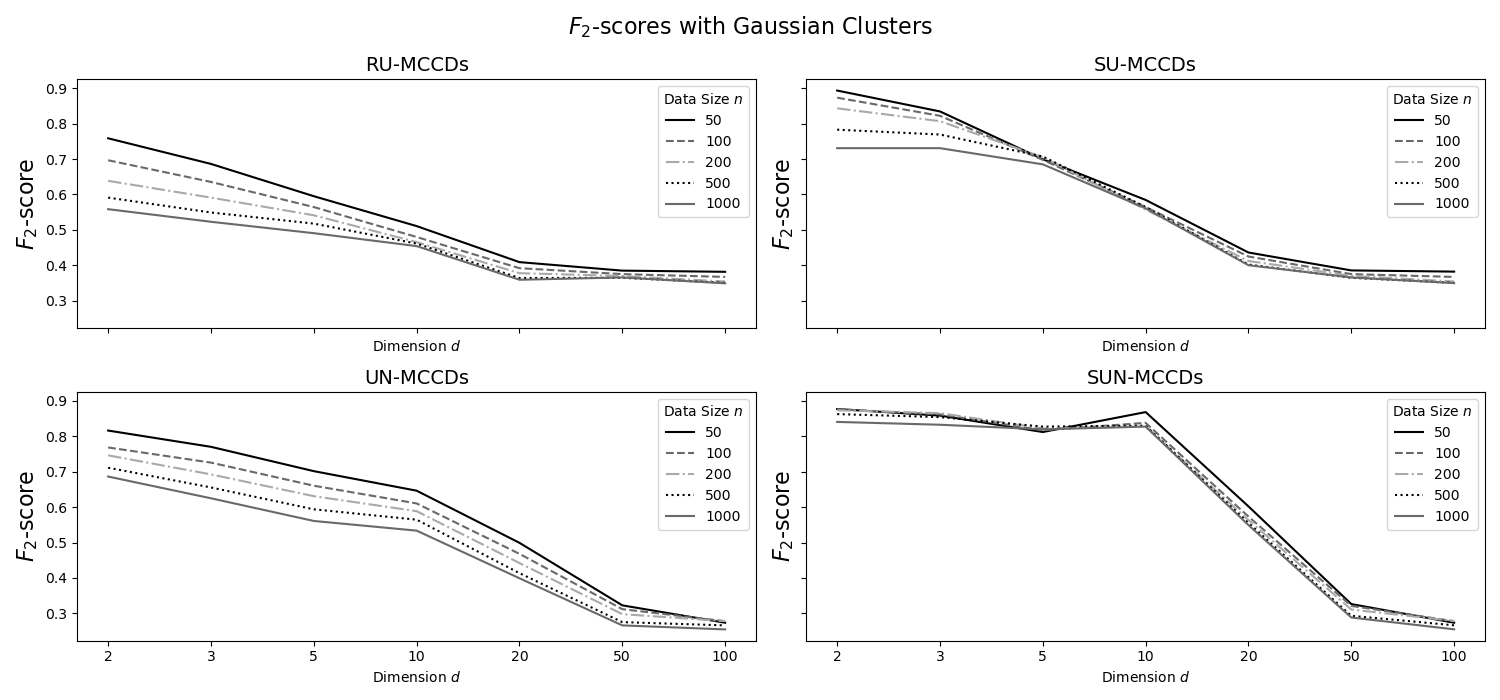}}
\caption{The line plots of the TPRs and TNRs of all CCD-based outlier detection algorithms, under the simulation settings (with Gaussian clusters) elaborated in Section \ref{sec:Gau_General_Settings_Des}.}
\label{fig:Gaussian_BA_F_Lines}
\end{figure}

As in the previous cases, we discuss the results under low and moderate dimensionality ($d\leq20$) and consider the RU-MCCD algorithm first.

The RU-MCCD algorithm generally performs much worse with Gaussian clusters; although it can still capture most outliers and provide high TPRs, the TNPs exhibit a substantial decrease. For example, when $d=3$, the RU-MCCD algorithm delivers TNRs of 0.880, 0.849, 0.818, 0.784, and 0.760, which show a major drop from 0.980, 0.985, 0.992, 0.997, and 0.998 under similar simulation settings with uniform clusters. It is within our expectation because RK-CCDs find support for each cluster by conducting SR-MCT; the point pattern of each constructed cluster is close to a uniform distribution, deviating from Gaussian clusters with uneven intensities. Furthermore, a Gaussian density has unbounded support, but each covering ball has bounded volume. Consequently, the resulting dominating covering balls tend to be smaller than the scope of Gaussian clusters and generally located around the center, leaving many regular points of less intensity uncovered. It is unlikely for the D-MCG algorithm to connect these relatively sparse uncovered points to the points of dominating covering balls, which generally have much higher intensities. Furthermore, notice that as the number of observations increases from $50$ to $1000$, the TNR decreases from 0.880 to 0.760, and as a result, the $F_2$-score decreases from 0.686 to 0.523. The reason can be explained as follows: the larger the size of a Gaussian cluster, the more deviation of its point pattern from a uniform density. Therefore, it becomes more difficult for the RU-MCCD and UN-MCCD algorithms to capture regular observations. Also, it is worth noting that the RU-MCCD algorithm performs worse with more dimensions $d$. For instance, when $n$ is fixed to 200, the RU-MCCD algorithm delivers $F_2$-scores of 0.638, 0.591, 0.541, 0.466, and 0.378 as $d$ increases from 2 to 20; it is due to the same reason that the effectiveness of RK-CCDs degenerates rapidly with increasing number of dimensions.

In Tables \ref{tab:Gau_General_Results1} and \ref{tab:Gau_General_Results2}, observe that the UN-MCCD algorithm also exhibits a performance drop in the simulation cases with Gaussian clusters; e.g., when $d=5$, the TNRs are 0.888, 0.865, 0.846, 0.820, and 0.794, compared to 0.975, 0.984, 0.992, 0.997, and 0.999 with uniform clusters. The corresponding $F_2$-scores also decrease substantially from 0.911, 0.940, 0.967, 0.986, and 0.993 under uniform setting to 0.701, 0.660, 0.631, 0.594, and 0.561 under Gaussian setting. For the same reason as the RU-MCCD algorithm, the $F_2$-score of the UN-MCCD algorithm decreases when $n$ increases. The performance of the UN-MCCD algorithm also shows a downward trend with increasing dimensionality $d$ (e.g., for $n=100$, the $F_2$-scores are 0.768, 0.726, 0.660, 0.610, and 0.468), but much less severely affected than the UN-MCCD algorithm. In summary, although the performance of the UN-MCCD algorithm deteriorates from uniform to Gaussian clusters, it still outperforms compared to the RU-MCCD algorithm thanks to the improved SR-MCT with NND.

Next, we consider the SU-MCCD and SUN-MCCD algorithms, both of which yield promising results compared with the two prototypes because they provide much better coverage for Gaussian clusters with multiple covering balls. For instance, when $d=10$, the TNRs of the SUN-MCCD algorithm are 0.960, 0.949, 0.945, 0.946, and 0.945, much higher than those of the UN-MCCD algorithm, therefore, the SUN-MCCD algorithm deliver $F_2$-scores of 0.868, 0.838, 0.827, 0.830, and 0.828, versus 0.646, 0.610, 0.589, 0.565, and 0.534 of the UN-MCCD algorithm. A similar performance gap is observed from the RU-MCCD to the SU-MCCD algorithms. Additionally, unlike the RU-MCCD and UN-MCCD algorithms, the two ``flexible" algorithms perform better when $n$ is larger. The reason is that when multiple covering balls are allowed for a single cluster, increasing the size of a cluster results in performance gain since the point pattern is easier to capture with more observations.

When $d\leq3$, the SU-MCCD algorithm slightly outperforms the SUN-MCCD algorithm; e.g., when $d=3$, the $F_2$-scores of the SUN-MCCD algorithm are 0.858, 0.862, 0.865, 0.854 and 0.832, higher than the $F_2$-scores of the SU-MCCD algorithm, which are 0.834, 0.822, 0.807, 0.769, and 0.731. Starting from $d=5$, the SUN-MCCD algorithm outperforms the SU-MCCD algorithm substantially. The most substantial performance difference is observed when $d=10$, where the $F_2$-scores of the SUN-MCCD algorithm are 0.868, 0.835, 0.827, 0.830, and 0.843, substantially higher than those of the SU-MCCD algorithm, which are less than $0.6$. This is due to the same reason for the degeneration of the RU-MCCD algorithm when $d$ is large.

For a similar reason explained under the simulation settings with only uniform clusters, all four CCD-based algorithms fail to deliver promising results without dimensionality reduction when $d=50,100$.

\subsection{Monte Carlo Experiments: Focus Settings}
\label{sec:Simul_Focus}
In the simulations we conducted in the previous section, we set up two clusters of data points with $5\%$ outliers and $1\%$ noise (the latter is only for Gaussian clusters). We fixed the distances between the cluster centers and the minimal distances between the cluster centers and the outliers to 3 and 2 units, respectively. We compared the balanced accuracies and $F_2$-scores of the CCD-based outlier detection algorithms on this setting. Next, we will investigate how the performance of these algorithms changes with varying factors such as the number of clusters, the noise level, the outlier percentage, and the distances between the clusters and the outliers, which we call focus settings. We conduct such simulation analysis to get a better understanding of the robustness and behaviors of the four CCD-based algorithms under different simulation settings and to identify the sensitivity of each algorithm. \label{sec:Simul_Changing_Factors}

\subsubsection{Varying the Number of Clusters}
\label{sec:Change_Clusters}
After assessing the effectiveness of CCD-based outlier detection algorithms on data sets with two distinct clusters, the next goal involves examining how their performance changes as the number of clusters increases from 2 to 5, while keeping other factors constant as in Section \ref{sec:Gau_General_Settings_Des}. We conduct two series of simulations, one with uniform clusters and another with Gaussian clusters. Additionally, we simulate both 3-dimensional and 10-dimensional data sets to understand how $d$ impacts performance on data sets of both small and high dimensions. Specific details are outlined below.

\begin{itemize}
  \item[\romannumeral1.] The dimensionality ($d$) of the simulated data sets: 3, 10;
  \item[\romannumeral2.] The size of data sets ($n$): 200;
  \item[\romannumeral3.] The size of each cluster is equal (although the volume of the supports is different), and we conduct two series of simulations with uniform clusters and Gaussian clusters, respectively;
  \item[\romannumeral4.] Number of clusters: 2, 3, 4, and 5 (the study of focus in this section);
  \item[\romannumeral5.] The radius of each cluster is randomly chosen between 0.7 and 1.3;
  \item[\romannumeral6.] When $d=3$, the centers of clusters are: (1) Two clusters: $\bm{\mu_1} = (3,3,3)$ and $\bm{\mu_2} = (6,3,3)$; (2) Three clusters: $\bm{\mu_1} = (3,3,3)$, $\bm{\mu_2} = (6,3,3)$, and $\bm{\mu_3} = (3,6,3)$; (3) Four clusters: $\bm{\mu_1} = (3,3,3)$, $\bm{\mu_2} = (6,3,3)$, $\bm{\mu_3} = (3,6,3)$, and $\bm{\mu_4} = (3,3,6)$; (4) Five clusters: $\bm{\mu_1} = (3,3,3)$, $\bm{\mu_2} = (6,3,3)$, $\bm{\mu_3} = (3,6,3)$, $\bm{\mu_4} = (3,3,6)$, and $\bm{\mu_5} = (6,6,3)$;
  \item[\romannumeral7.] When $d=10$, the centers of clusters are: (1) Two clusters: $\bm{\mu_1} = (\underbrace{3,...,3}_{d})$ and $\bm{\mu_2} = (6,\underbrace{3,...,3}_{d-1})$; (2) Three clusters: $\bm{\mu_1} = (\underbrace{3,...,3}_{d})$, $\bm{\mu_2} = (6,\underbrace{3,...,3}_{d-1})$, and $\bm{\mu_3} = (3,6,\underbrace{3,...,3}_{d-2})$; (3) Four clusters: $\bm{\mu_1} = (\underbrace{3,...,3}_{d})$, $\bm{\mu_2} = (6,\underbrace{3,...,3}_{d-1})$, $\bm{\mu_3} = (3,6,\underbrace{3,...,3}_{d-2})$, and $\bm{\mu_4} = (3,3,6,\underbrace{3,...,3}_{d-3})$; (4) Five clusters: $\bm{\mu_1} = (\underbrace{3,...,3}_{d})$, $\bm{\mu_2} = (6,\underbrace{3,...,3}_{d-1})$, $\bm{\mu_3} = (3,6,\underbrace{3,...,3}_{d-2})$, $\bm{\mu_4} = (3,3,6,\underbrace{3,...,3}_{d-3})$, and $\bm{\mu_5} = (3,3,3,6,\underbrace{3,...,3}_{d-4})$;
  \item[\romannumeral8.] The proportion of outliers is fixed to $5\%$;
  \item[\romannumeral9.] The outlier set $C_{outlier}$ is generated uniformly within a much larger hypersphere of radius 5, centered at the mean of the cluster center. and each outlier is at least 2 units away from any cluster center;
  \item[\romannumeral10.] The noise level of each Gaussian cluster is set to $1\%$.
  \label{sec:N_Cls_Setting}
\end{itemize}

The simulation results are summarized from Tables \ref{tab:N_U_Cls1} to \ref{tab:N_G_Cls2}. For better visualization, we present the results of BAs and $F_2$-scores (Tables \ref{tab:N_U_Cls2} and \ref{tab:N_G_Cls2}) as barplots in Figures \ref{fig:Barplot_NClusters} and \ref{fig:Barplot_GClusters}, respectively.

\begin{table}[htb]
  \centering
  \resizebox{0.7\columnwidth}{!}{\begin{tabular}{|c|c|c|c|c|c|c|c|c|c|}
    \hline
    \multicolumn{2}{|c|}{} & \multicolumn{8}{|c|}{Number of Clusters} \\ \cline{3-10}

    \multicolumn{2}{|c|}{} & \multicolumn{2}{|c|}{2} & \multicolumn{2}{|c|}{3} & \multicolumn{2}{|c|}{4} & \multicolumn{2}{|c|}{5} \\ \cline{3-10}

    \multicolumn{2}{|c|}{} & TPR & TNR & TPR & TNR & TPR & TNR & TPR & TNR \\ \hline

    \multirow{4}*{$d=3$} & RU-MCCDs & 0.970 & 0.992 & 0.984 & 0.988 & 0.993 & 0.985 & 0.989 & 0.982 \\ \cline{2-10}
    & SU-MCCDs & 1.000 & 0.999 & 0.999 & 0.999 & 0.997 & 0.998 & 0.994 & 0.996 \\ \cline{2-10}
    & UN-MCCDs & 0.983 & 0.991 & 0.989 & 0.989 & 0.997 & 0.985 & 0.993 & 0.984 \\ \cline{2-10}
    & SUN-MCCDs & 1.000 & 0.998 & 0.998 & 0.997 & 0.995 & 0.996 & 0.990 & 0.995 \\ \cline{1-10}

    \multirow{4}*{$d=10$} & RU-MCCDs & 1.000 & 0.976 & 1.000 & 0.916 & 1.000 & 0.900 & 1.000 & 0.900 \\ \cline{2-10}
    & SU-MCCDs & 1.000 & 0.991 & 1.000 & 0.933 & 1.000 & 0.916 & 1.000 & 0.917 \\ \cline{2-10}
    & UN-MCCDs & 1.000 & 0.995 & 1.000 & 0.990 & 1.000 & 0.985 & 1.000 & 0.984 \\ \cline{2-10}
    & SUN-MCCDs & 1.000 & 0.999 & 1.000 & 0.999 & 1.000 & 0.998 & 1.000 & 0.998 \\ \cline{1-10}
  \end{tabular}}
  \caption{The TPRs and TNRs of the CCD-based algorithms as the number of uniform clusters increases from 2 to 5.}
  \label{tab:N_U_Cls1}
\end{table}

\begin{table}[htb]
  \centering
  \resizebox{0.7\columnwidth}{!}{\begin{tabular}{|c|c|c|c|c|c|c|c|c|c|}
    \hline
    \multicolumn{2}{|c|}{} & \multicolumn{8}{|c|}{Number of Clusters} \\ \cline{3-10}

    \multicolumn{2}{|c|}{} & \multicolumn{2}{|c|}{2} & \multicolumn{2}{|c|}{3} & \multicolumn{2}{|c|}{4} & \multicolumn{2}{|c|}{5} \\ \cline{3-10}

    \multicolumn{2}{|c|}{} & BA & $F_2$-score & BA & $F_2$-score & BA & $F_2$-score & BA & $F_2$-score \\ \hline

    \multirow{4}*{$d=3$} & RU-MCCDs & 0.981 & 0.947 & 0.986 & 0.944 & 0.989 & 0.941 & 0.986 & 0.928 \\ \cline{2-10}
    & SU-MCCDs & 1.000 & 0.996 & 0.999 & 0.995 & 0.998 & 0.990 & 0.995 & 0.980 \\ \cline{2-10}
    & UN-MCCDs & 0.987 & 0.954 & 0.989 & 0.951 & 0.991 & 0.944 & 0.989 & 0.937 \\ \cline{2-10}
    & SUN-MCCDs & 0.999 & 0.992 & 0.998 & 0.987 & 0.996 & 0.981 & 0.993 & 0.973 \\ \cline{1-10}

    \multirow{4}*{$d=10$} & RU-MCCDs & 0.988 & 0.916 & 0.958 & 0.758 & 0.950 & 0.725 & 0.950 & 0.725 \\ \cline{2-10}
    & SU-MCCDs & 0.996 & 0.967 & 0.967 & 0.797 & 0.958 & 0.758 & 0.959 & 0.760 \\ \cline{2-10}
    & UN-MCCDs & 0.998 & 0.981 & 0.995 & 0.963 & 0.993 & 0.946 & 0.992 & 0.943 \\ \cline{2-10}
    & SUN-MCCDs & 1.000 & 0.996 & 1.000 & 0.996 & 0.999 & 0.992 & 0.999 & 0.992 \\ \cline{1-10}
  \end{tabular}}
  \caption{The TPRs and TNRs of the CCD-based algorithms as the number of uniform clusters increases from 2 to 5.}
  \label{tab:N_U_Cls2}
\end{table}

\begin{figure}[htb]
\centering
\subfigure[The BAs when $d=3$.]{
\label{fig:Barplot3_NClusters_BA}
\includegraphics[width=0.45\textwidth]{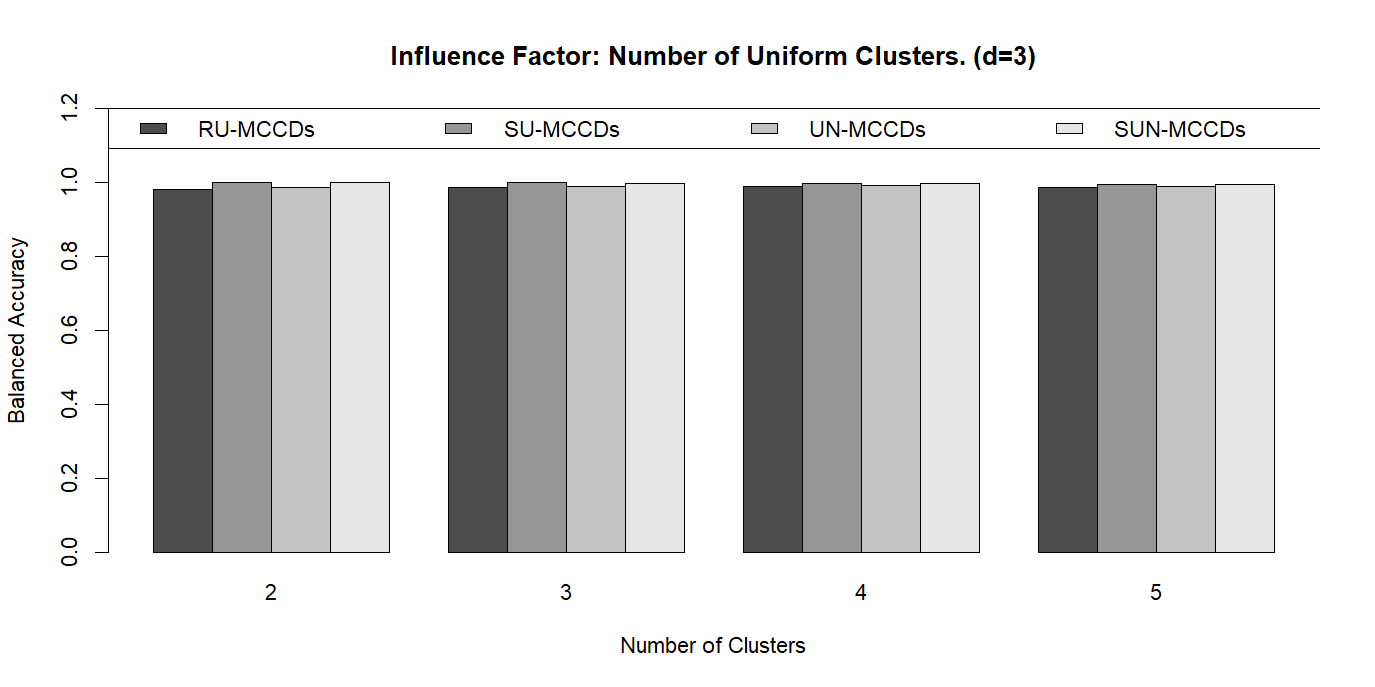}}
\subfigure[The $F_2$-scores when $d=3$.]{
\label{fig:Barplot3_NClusters_FS}
\includegraphics[width=0.45\textwidth]{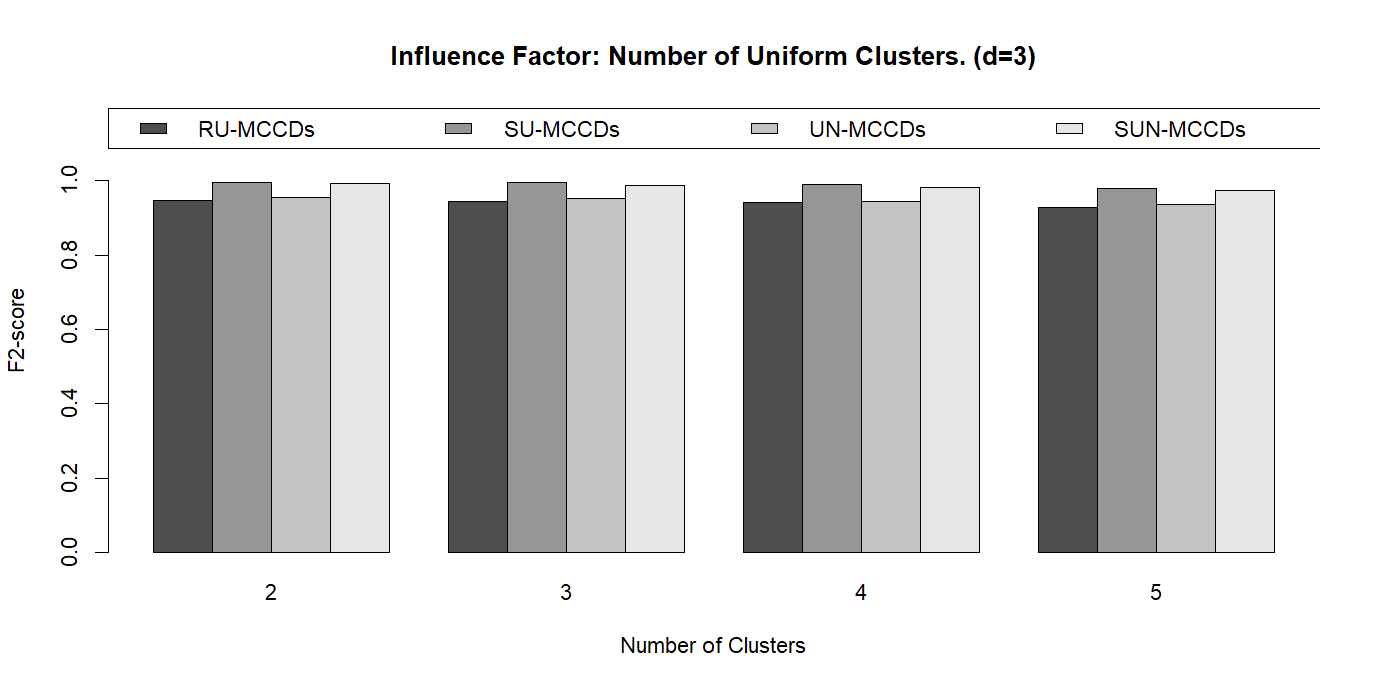}}

\subfigure[The BAs when $d=10$.]{
\label{fig:Barplot10_NClusters_BA}
\includegraphics[width=0.45\textwidth]{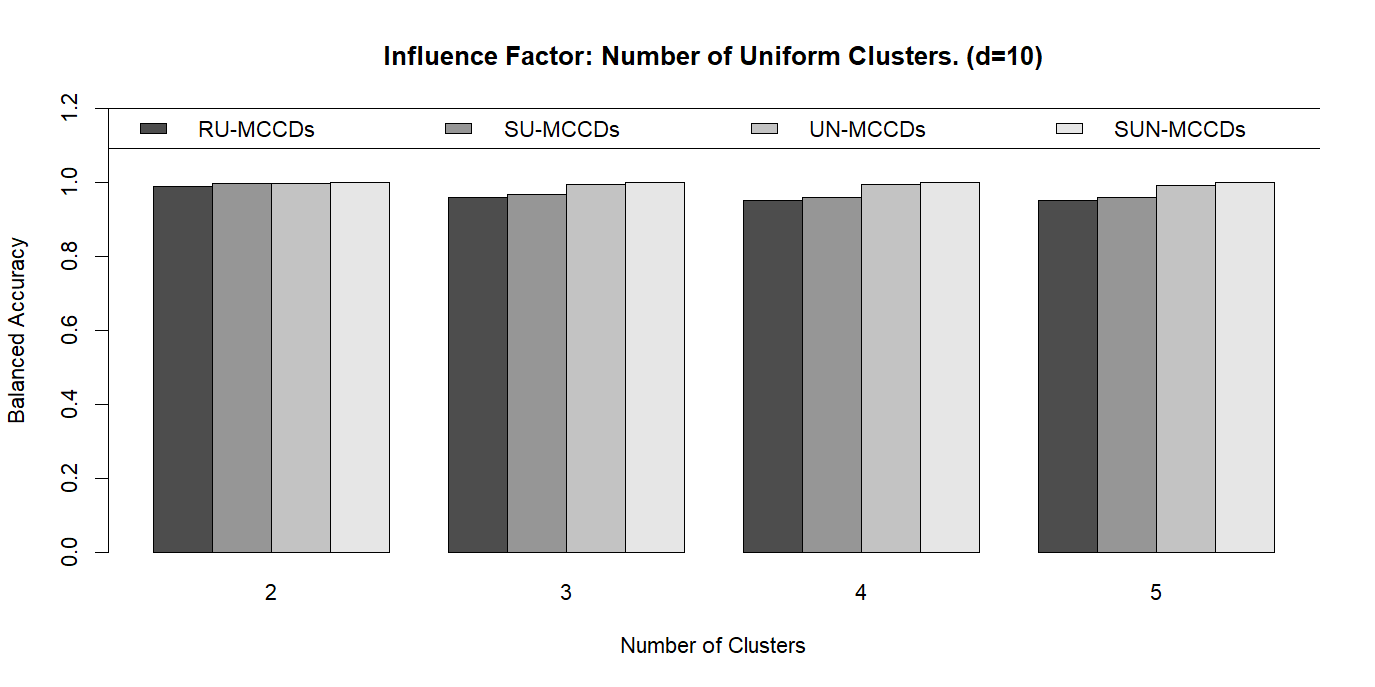}}
\subfigure[The $F_2$-scores when $d=10$.]{
\label{fig:Barplot10_NClusters_FS}
\includegraphics[width=0.45\textwidth]{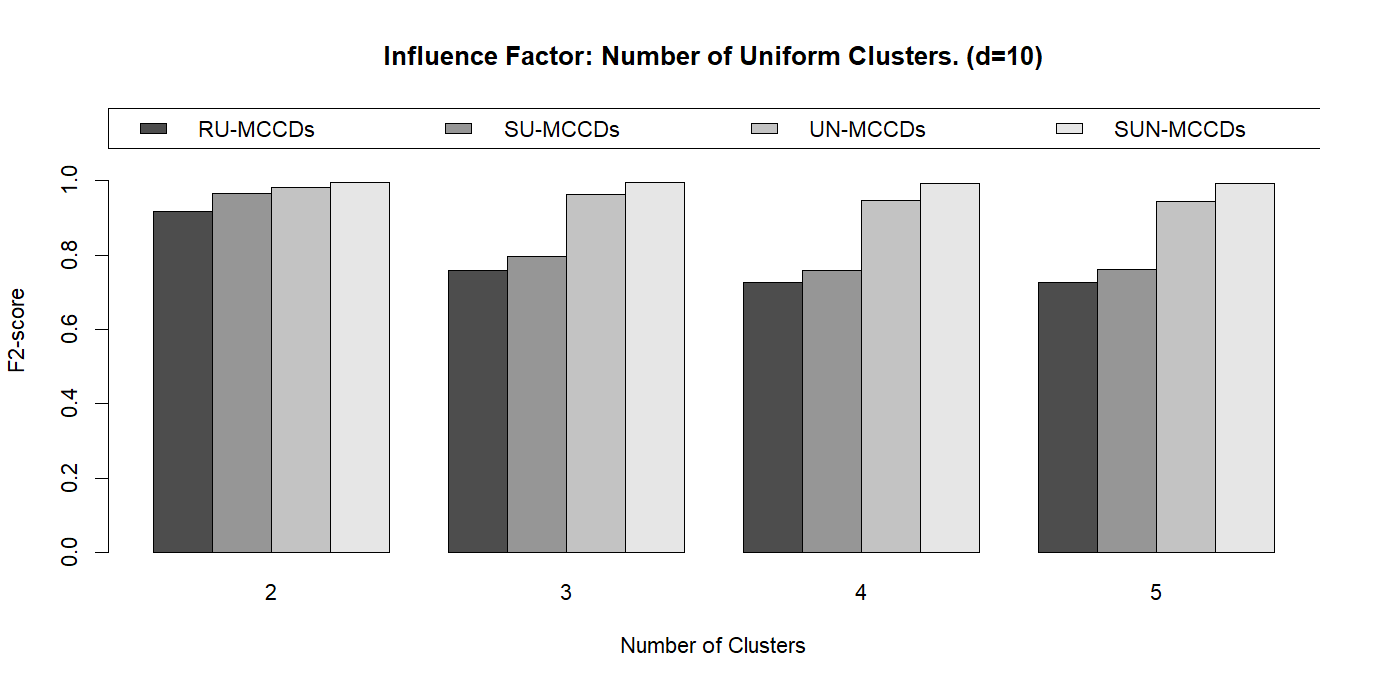}}
\caption{The barplots summarizing the performances of the CCD-based outlier detection algorithms as the number of uniform clusters increases. (a) The BAs for $d=3$. (b) The $F_2$-scores for $d=3$. (c) The BAs for $d=10$. (d) The $F_2$-scores for $d=10$.}
\label{fig:Barplot_NClusters}
\end{figure}

\begin{table}[htb]
  \centering
  \resizebox{0.7\columnwidth}{!}{\begin{tabular}{|c|c|c|c|c|c|c|c|c|c|}
    \hline
    \multicolumn{2}{|c|}{} & \multicolumn{8}{|c|}{Number of Clusters} \\ \cline{3-10}

    \multicolumn{2}{|c|}{} & \multicolumn{2}{|c|}{2} & \multicolumn{2}{|c|}{3} & \multicolumn{2}{|c|}{4} & \multicolumn{2}{|c|}{5} \\ \cline{3-10}

    \multicolumn{2}{|c|}{} & TPR & TNR & TPR & TNR & TPR & TNR & TPR & TNR \\ \hline

    \multirow{4}*{$d=3$} & RU-MCCDs & 1.000 & 0.818 & 1.000 & 0.836 & 1.000 & 0.847 & 1.000 & 0.860 \\ \cline{2-10}
    & SU-MCCDs & 1.000 & 0.938 & 1.000 & 0.941 & 1.000 & 0.945 & 1.000 & 0.947 \\ \cline{2-10}
    & UN-MCCDs & 0.997 & 0.884 & 0.995 & 0.986 & 0.996 & 0.902 & 0.997 & 0.908 \\ \cline{2-10}
    & SUN-MCCDs & 1.000 & 0.959 & 1.000 & 0.958 & 1.000 & 0.958 & 1.000 & 0.958 \\ \cline{1-10}

    \multirow{4}*{$d=10$} & RU-MCCDs & 1.000 & 0.698 & 1.000 & 0.689 & 1.000 & 0.700 & 1.000 & 0.708 \\ \cline{2-10}
    & SU-MCCDs & 1.000 & 0.791 & 1.000 & 0.771 & 1.000 & 0.779 & 1.000 & 0.782 \\ \cline{2-10}
    & UN-MCCDs & 1.000 & 0.817 & 1.000 & 0.825 & 1.000 & 0.832 & 1.000 & 0.836 \\ \cline{2-10}
    & SUN-MCCDs & 1.000 & 0.945 & 1.000 & 0.947 & 1.000 & 0.950 & 1.000 & 0.952 \\ \cline{1-10}
  \end{tabular}}
  \caption{The TPRs and TNRs of the CCD-based algorithms as the number of Gaussian clusters increases from 2 to 5.}
  \label{tab:N_G_Cls1}
\end{table}

\begin{table}[htb]
  \centering
  \resizebox{0.7\columnwidth}{!}{\begin{tabular}{|c|c|c|c|c|c|c|c|c|c|}
    \hline
    \multicolumn{2}{|c|}{} & \multicolumn{8}{|c|}{Number of Clusters} \\ \cline{3-10}

    \multicolumn{2}{|c|}{} & \multicolumn{2}{|c|}{2} & \multicolumn{2}{|c|}{3} & \multicolumn{2}{|c|}{4} & \multicolumn{2}{|c|}{5} \\ \cline{3-10}

    \multicolumn{2}{|c|}{} & BA & $F_2$-score & BA & $F_2$-score & BA & $F_2$-score & BA & $F_2$-score \\ \hline

    \multirow{4}*{$d=3$} & RU-MCCDs & 0.909 & 0.591 & 0.918 & 0.616 & 0.924 & 0.632 & 0.930 & 0.653 \\ \cline{2-10}
    & SU-MCCDs & 0.969 & 0.809 & 0.971 & 0.817 & 0.973 & 0.827 & 0.974 & 0.832 \\ \cline{2-10}
    & UN-MCCDs & 0.941 & 0.692 & 0.946 & 0.714 & 0.949 & 0.726 & 0.953 & 0.739 \\ \cline{2-10}
    & SUN-MCCDs & 0.980 & 0.865 & 0.979 & 0.862 & 0.979 & 0.862 & 0.979 & 0.862 \\ \cline{1-10}

    \multirow{4}*{$d=10$} & RU-MCCDs & 0.849 & 0.466 & 0.845 & 0.458 & 0.850 & 0.467 & 0.854 & 0.474 \\ \cline{2-10}
    & SU-MCCDs & 0.896 & 0.557 & 0.886 & 0.535 & 0.890 & 0.544 & 0.891 & 0.547 \\ \cline{2-10}
    & UN-MCCDs & 0.909 & 0.590 & 0.913 & 0.601 & 0.916 & 0.610 & 0.918 & 0.616 \\ \cline{2-10}
    & SUN-MCCDs & 0.973 & 0.827 & 0.974 & 0.832 & 0.975 & 0.840 & 0.976 & 0.846 \\ \cline{1-10}
  \end{tabular}}
  \caption{The BAs and $F_2$-scores of the CCD-based algorithms as the number of Gaussian clusters increases from 2 to 5.}
  \label{tab:N_G_Cls2}
\end{table}

\begin{figure}[htb]
\centering
\subfigure[The BAs when $d=3$.]{
\label{fig:Barplot3_GClusters_BA}
\includegraphics[width=0.45\textwidth]{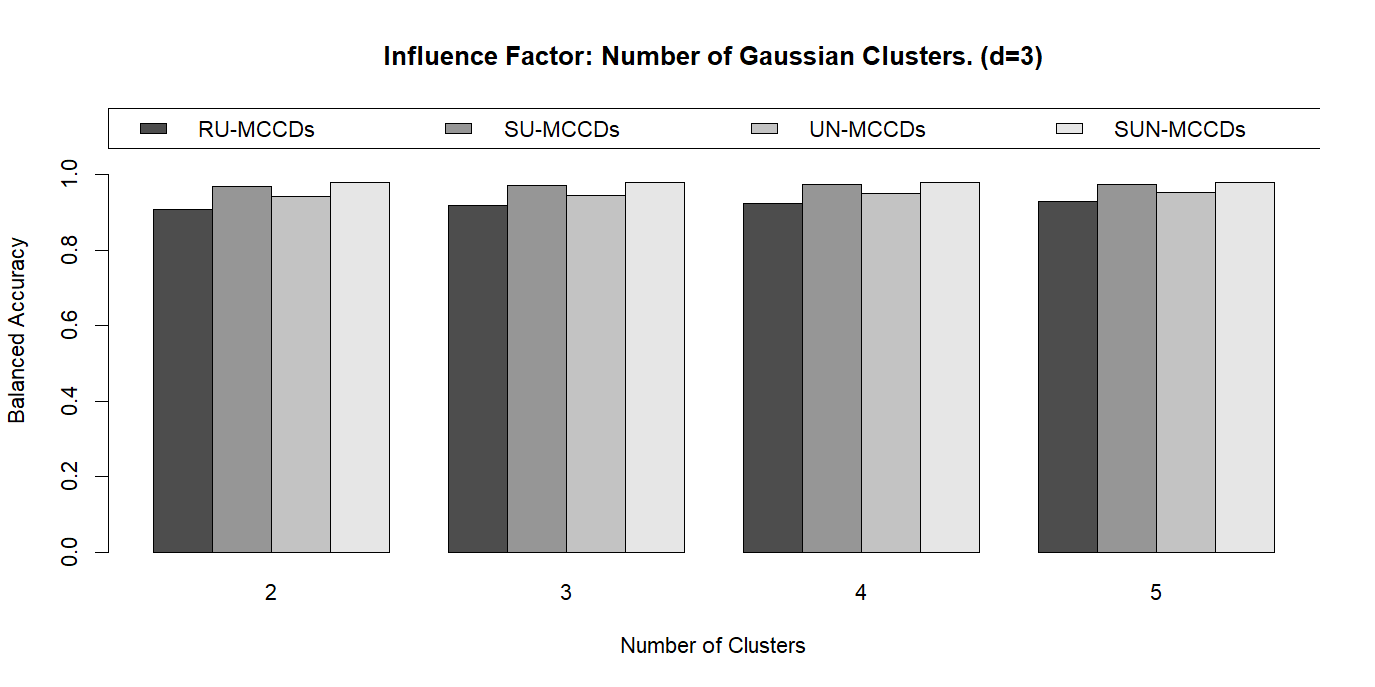}}
\subfigure[The $F_2$-scores when $d=3$.]{
\label{fig:Barplot3_GClusters_FS}
\includegraphics[width=0.45\textwidth]{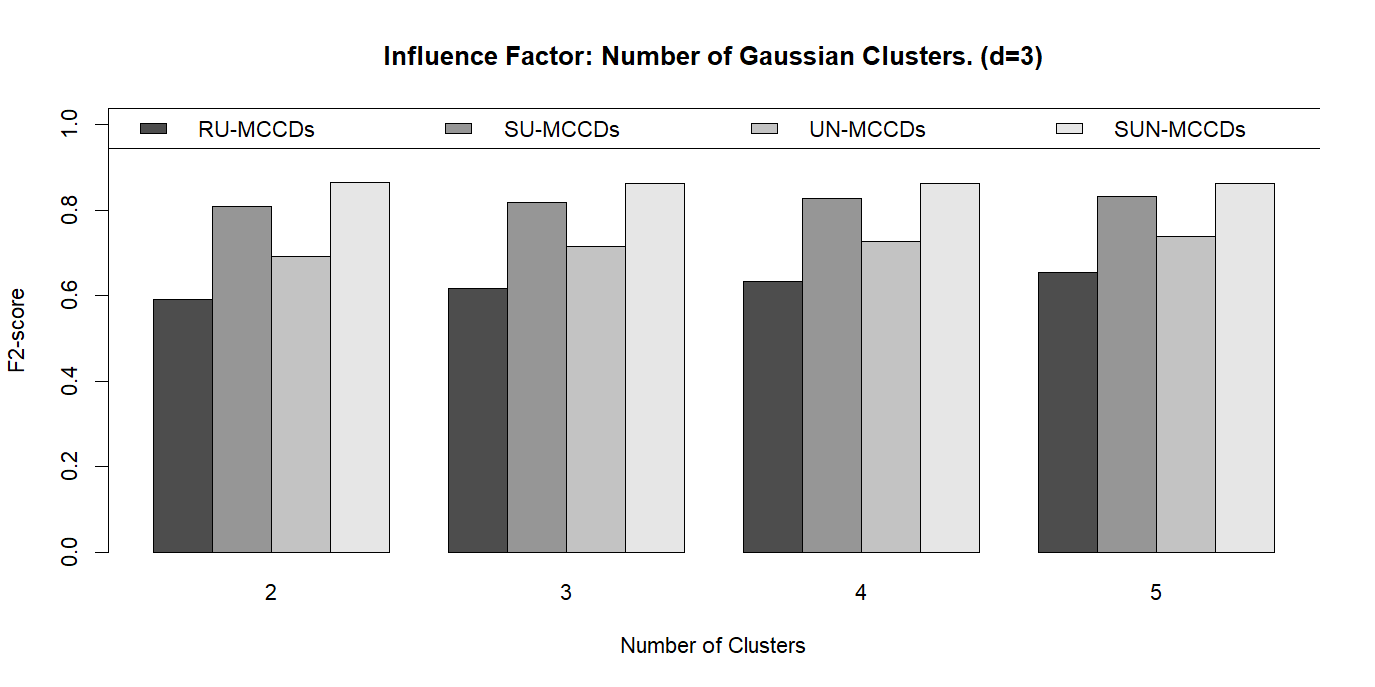}}

\subfigure[The BAs when $d=10$.]{
\label{fig:Barplot10_GClusters_BA}
\includegraphics[width=0.45\textwidth]{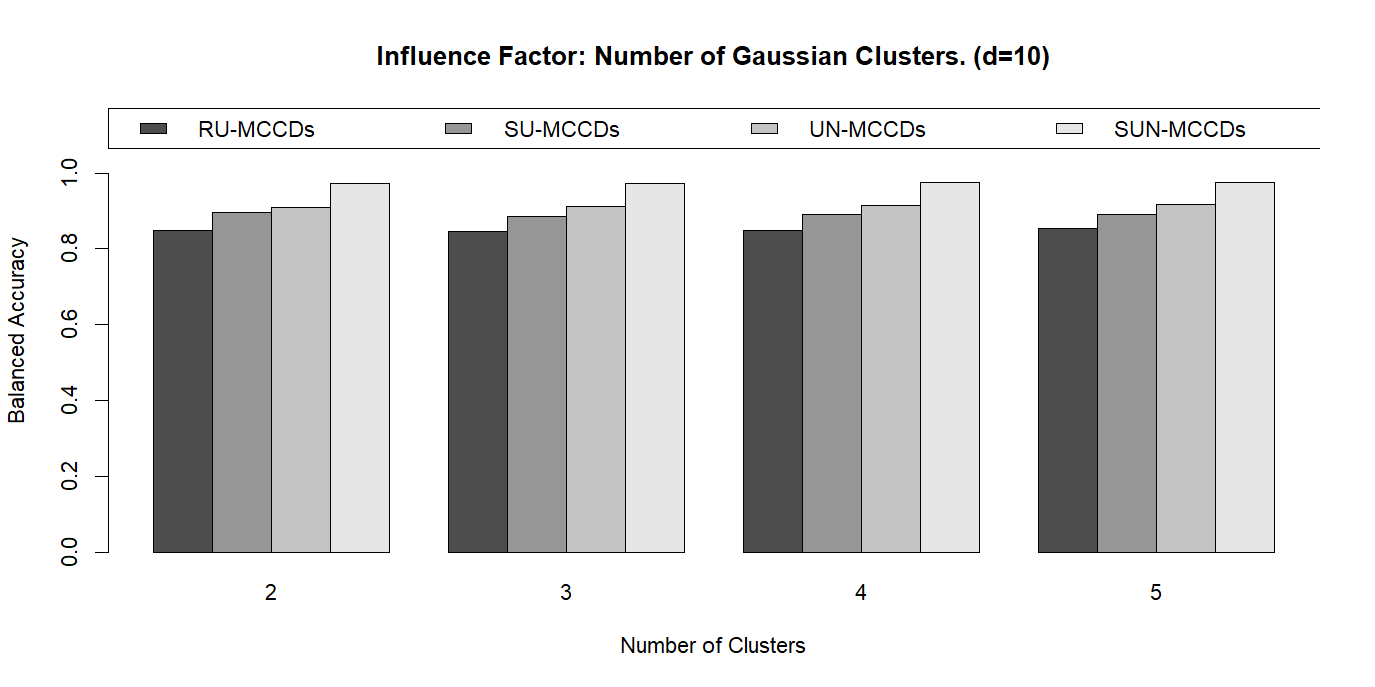}}
\subfigure[The $F_2$-scores when $d=10$.]{
\label{fig:Barplot10_GClusters_FS}
\includegraphics[width=0.45\textwidth]{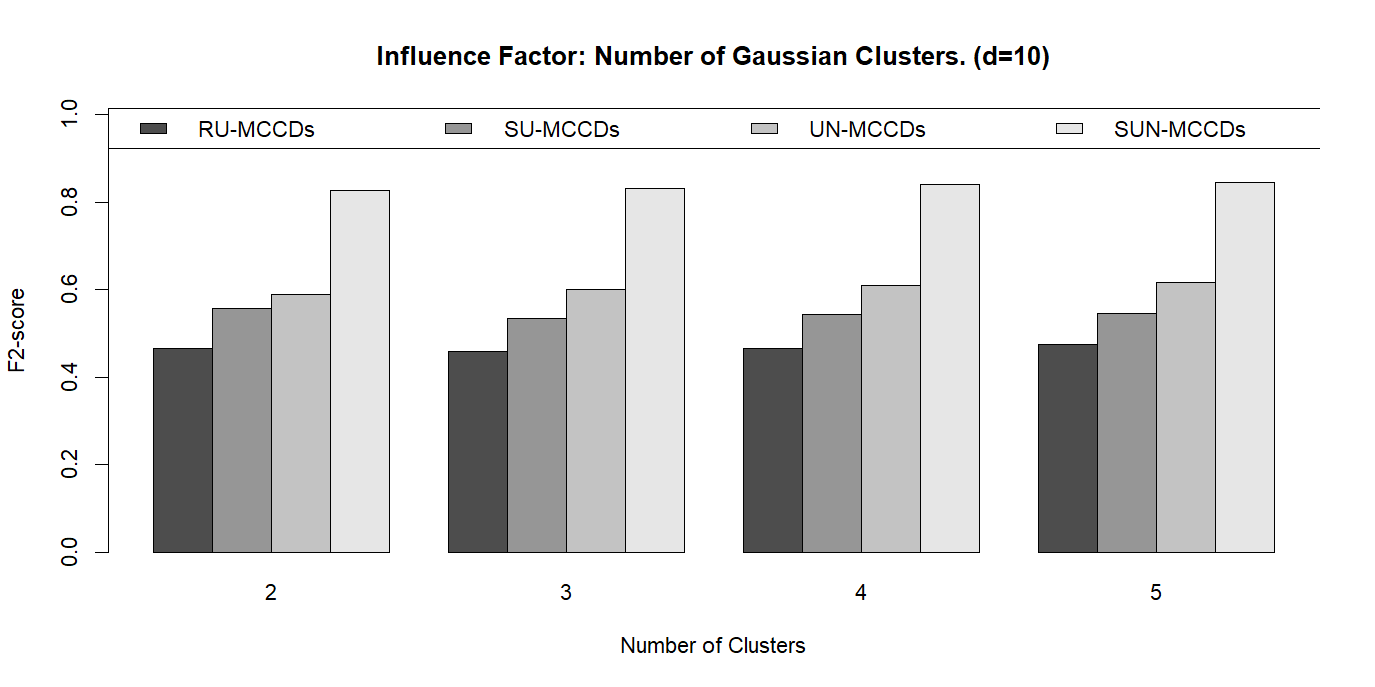}}
\caption{The barplots summarizing the performances of the CCD-based outlier detection algorithms as the number of Gaussian clusters increases. (a) The BAs for $d=3$. (b) The $F_2$-scores for $d=3$. (c) The BAs for $d=10$. (d) The $F_2$-scores for $d=10$.}
\label{fig:Barplot_GClusters}
\end{figure}

Considering the simulation settings with uniform clusters (Tables \ref{tab:N_U_Cls1} and \ref{tab:N_U_Cls2}), observe that almost all the algorithms perform well with $F_2$-scores exceeding $90\%$ except the RU-MCCD and SU-MCCD algorithms, which tends to have low TNRs when $d=10$ (for the same reason that has been discussed in Section \ref{sec:UN-CCDs}). The performances of these algorithms decrease slightly as the number of clusters increases because when we fix $n$ to 200, more clusters indicate less intensity for each uniform cluster; thus, the difficulty level to identify the correct number of clusters and capture an entire cluster increases.

With Gaussian clusters, similar to the results we obtained in the previous section, the SU-MCCD and SUN-MCCD algorithms outperform their prototypes by a large margin, especially the SUN-MCCD algorithm, which delivers high $F_2$-scores of 0.827, 0.832, 0.840, and 0.846 when $d=10$ as the number of clusters increases. It is interesting to find that the $F_2$-scores of the RU-MCCD and SU-MCCD algorithms increase with the cluster numbers when $d=3$, e.g., the $F_2$-score of the RU-MCCD algorithm rises from 0.591 to 0.653 when the cluster number increases; because when the intensities of Gaussian clusters decrease, their point patterns are closer to uniform clusters, which give advantage to the performance of the two algorithms and outweigh the effect of intensity drops.

In summary, the effectiveness of all four algorithms is relatively robust against the number of clusters. With other factors fixed, although their performance tends to decrease as the number of clusters increases, the decrease is minimal. The SUN-MCCD algorithm offers better overall performance and could deliver promising results even if there are 5 Gaussian clusters.

\subsubsection{Varying the Outliers' Percentage}

The main goal of this section is to evaluate the performance of the four CCD-based algorithms under different levels of contamination. In Section \ref{sec:Uni_General_Settings_Des}, we present the results of the data sets with $5\%$ outliers, which is a moderate level of contamination. In this section, we aim to investigate the sensitivity of these algorithms by conducting a series of simulations with the percentage of outliers increasing from $2\%$ to $15\%$. To increase complexity, we set the number of clusters to 3 rather than 2; all the other factors, such as the number of observations, the distances between cluster centers, noise level, etc., are fixed at the same values as in Section \ref{sec:Change_Clusters}. We expect that the algorithms show different degrees of sensitivity to the presence of outliers.

Similar to Section \ref{sec:N_Cls_Setting}, we conduct two sets of simulations with uniform and Gaussian clusters, and we choose to simulate data sets with 3 and 10 dimensions. Details are presented below, it is worth noting that we only list the difference and skip the common parts compared to the simulation setting in Section \ref{sec:N_Cls_Setting}. Some realizations of data sets with Gaussian clusters in 2-dimensional space (although the simulation experiments are conducted on 3 and 10-dimensional space) are presented in Figure \ref{fig:Demo_2d_cont}. (for illustration purposes)

\begin{itemize}
   \item[\romannumeral1.] The proportion of outliers: $2\%$, $5\%$, $7\%$, $10\%$, and $15\%$ (the study of focus in this section).
  \label{sec:Cont_Setting}
\end{itemize}

\begin{figure}[htb]
  \centering
  \subfigure[$2\%$]{
  \includegraphics[width=0.30\textwidth]{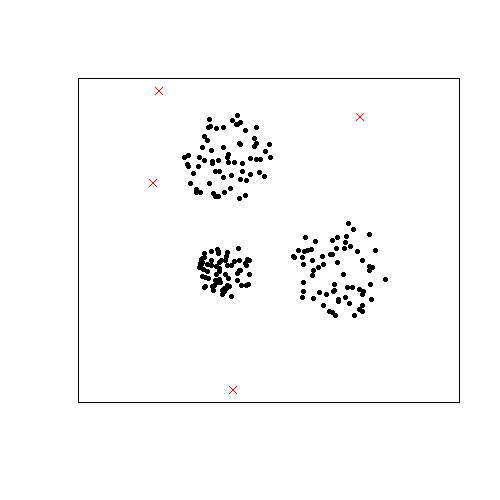}}
  \subfigure[$5\%$]{
  \includegraphics[width=0.30\textwidth]{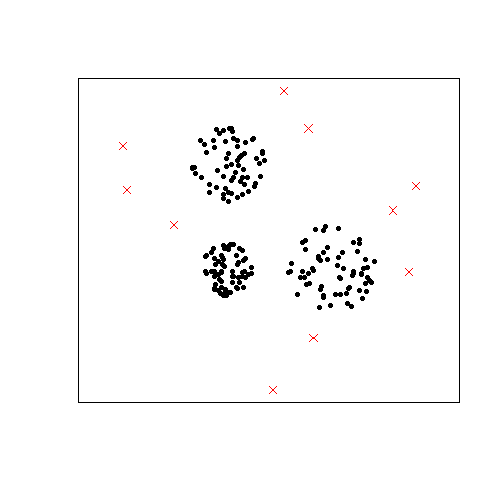}}
  \subfigure[$7\%$]{
  \includegraphics[width=0.30\textwidth]{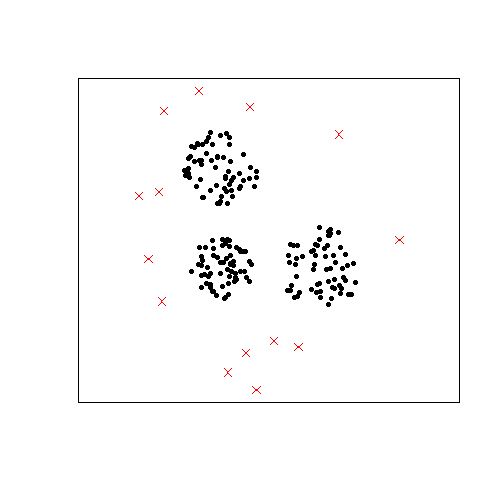}}
  \subfigure[$10\%$]{
  \includegraphics[width=0.30\textwidth]{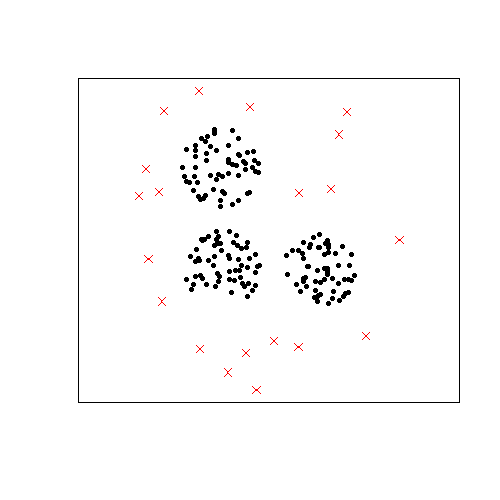}}
  \subfigure[$15\%$]{
  \includegraphics[width=0.30\textwidth]{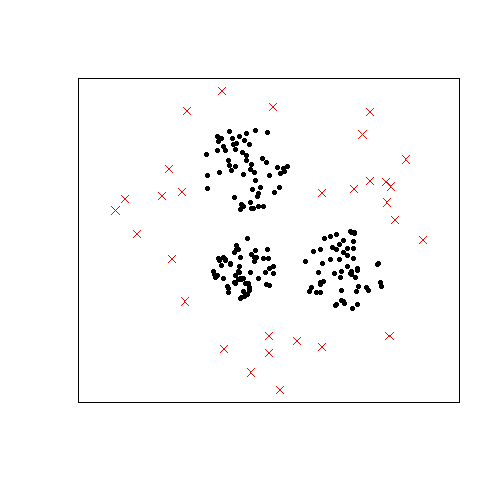}}
  \caption{Some realizations (with Gaussian clusters) of the simulation setting in Section \ref{sec:Cont_Setting}, the contamination level increases from $2\%$ to $15\%$. Red crosses are outliers, black points are regular observations. Contamination levels are indicated below each sub-figure.}\label{fig:Demo_2d_cont}
\end{figure}

The simulation results are summarized from Tables \ref{tab:Otl_U_Cls1} to \ref{tab:Otl_G_Cls2}. We also present the results of BAs and $F_2$-scores (Tables \ref{tab:Otl_U_Cls2} and \ref{tab:Otl_G_Cls2}) as barplots in Figures \ref{fig:Barplot_UCon} and \ref{fig:Barplot_GCon}, respectively.

\begin{table}[htb]
  \centering
  \resizebox{0.7\columnwidth}{!}{\begin{tabular}{|c|c|c|c|c|c|c|c|c|c|c|c|}
    \hline
    \multicolumn{2}{|c|}{} & \multicolumn{10}{|c|}{Percentage of Outliers} \\ \cline{3-12}

    \multicolumn{2}{|c|}{} & \multicolumn{2}{|c|}{$2\%$} & \multicolumn{2}{|c|}{$5\%$} & \multicolumn{2}{|c|}{$7\%$} & \multicolumn{2}{|c|}{$10\%$} & \multicolumn{2}{|c|}{$15\%$} \\ \cline{3-12}

    \multicolumn{2}{|c|}{} & TPR & TNR & TPR & TNR & TPR & TNR & TPR & TNR & TPR & TNR \\ \hline

    \multirow{4}*{$d=3$} & RU-MCCDs & 0.999 & 0.988 & 0.985 & 0.988 & 0.961 & 0.989 & 0.943 & 0.987 & 0.878 & 0.987 \\ \cline{2-12}
    & SU-MCCDs & 0.999 & 0.998 & 0.999 & 0.998 & 0.997 & 0.998 & 0.991 & 0.998 & 0.961 & 0.998 \\ \cline{2-12}
    & UN-MCCDs & 0.999 & 0.989 & 0.991 & 0.990 & 0.977 & 0.989 & 0.965 & 0.989 & 0.919 & 0.987 \\ \cline{2-12}
    & SUN-MCCDs & 0.999 & 0.997 & 0.998 & 0.998 & 0.996 & 0.997 & 0.996 & 0.997 & 0.973 & 0.998 \\ \cline{1-12}

    \multirow{4}*{$d=10$} & RU-MCCDs & 1.000 & 0.911 & 1.000 & 0.920 & 1.000 & 0.907 & 1.000 & 0.918 & 1.000 & 0.914 \\ \cline{2-12}
    & SU-MCCDs & 1.000 & 0.926 & 1.000 & 0.934 & 1.000 & 0.924 & 1.000 & 0.933 & 1.000 & 0.931 \\ \cline{2-12}
    & UN-MCCDs & 1.000 & 0.990 & 1.000 & 0.990 & 1.000 & 0.991 & 1.000 & 0.988 & 1.000 & 0.989 \\ \cline{2-12}
    & SUN-MCCDs & 1.000 & 0.999 & 1.000 & 0.999 & 1.000 & 0.999 & 1.000 & 0.999 & 0.999 & 0.998 \\ \cline{1-12}
  \end{tabular}}
  \caption{The TPRs and TNRs of the CCD-based algorithms as the percentage of outliers over the entire simulated data set increases from $2\%$ to $15\%$ (for simulations with uniform clusters).}
  \label{tab:Otl_U_Cls1}
\end{table}

\begin{table}[htb]
  \centering
  \resizebox{0.7\columnwidth}{!}{\begin{tabular}{|c|c|c|c|c|c|c|c|c|c|c|c|}
    \hline
    \multicolumn{2}{|c|}{} & \multicolumn{10}{|c|}{Percentage of Outliers} \\ \cline{3-12}

    \multicolumn{2}{|c|}{} & \multicolumn{2}{|c|}{$2\%$} & \multicolumn{2}{|c|}{$5\%$} & \multicolumn{2}{|c|}{$7\%$} & \multicolumn{2}{|c|}{$10\%$} & \multicolumn{2}{|c|}{$15\%$} \\ \cline{3-12}

    \multicolumn{2}{|c|}{} & BA & $F_2$-score & BA & $F_2$-score & BA & $F_2$-score & BA & $F_2$-score & BA & $F_2$-score \\ \hline

    \multirow{4}*{$d=3$} & RU-MCCDs & 0.994 & 0.894 & 0.987 & 0.945 & 0.975 & 0.941 & 0.965 & 0.932 & 0.933 & 0.887 \\ \cline{2-12}
    & SU-MCCDs & 0.999 & 0.980 & 0.999 & 0.992 & 0.998 & 0.992 & 0.995 & 0.989 & 0.980 & 0.966 \\ \cline{2-12}
    & UN-MCCDs & 0.994 & 0.902 & 0.991 & 0.956 & 0.983 & 0.954 & 0.977 & 0.953 & 0.953 & 0.920 \\ \cline{2-12}
    & SUN-MCCDs & 0.998 & 0.971 & 0.998 & 0.991 & 0.997 & 0.989 & 0.997 & 0.991 & 0.986 & 0.976 \\ \cline{1-12}

    \multirow{4}*{$d=10$} & RU-MCCDs & 0.956 & 0.534 & 0.960 & 0.767 & 0.954 & 0.802 & 0.959 & 0.871 & 0.957 & 0.911 \\ \cline{2-12}
    & SU-MCCDs & 0.963 & 0.580 & 0.967 & 0.799 & 0.962 & 0.832 & 0.967 & 0.892 & 0.966 & 0.927 \\ \cline{2-12}
    & UN-MCCDs & 0.995 & 0.911 & 0.995 & 0.963 & 0.996 & 0.977 & 0.994 & 0.979 & 0.995 & 0.988 \\ \cline{2-12}
    & SUN-MCCDs & 1.000 & 0.990 & 1.000 & 0.996 & 1.000 & 0.997 & 1.000 & 0.998 & 0.999 & 0.997 \\ \cline{1-12}
  \end{tabular}}
  \caption{The BAs and $F_2$-scores of the CCD-based algorithms as the percentage of outliers over the entire simulated data set increases from $2\%$ to $15\%$ (for simulations with uniform clusters).}
  \label{tab:Otl_U_Cls2}
\end{table}

\begin{figure}[htb]
\centering
\subfigure[The BAs when $d=3$.]{
\label{fig:Barplot3_NCon_BA}
\includegraphics[width=0.45\textwidth]{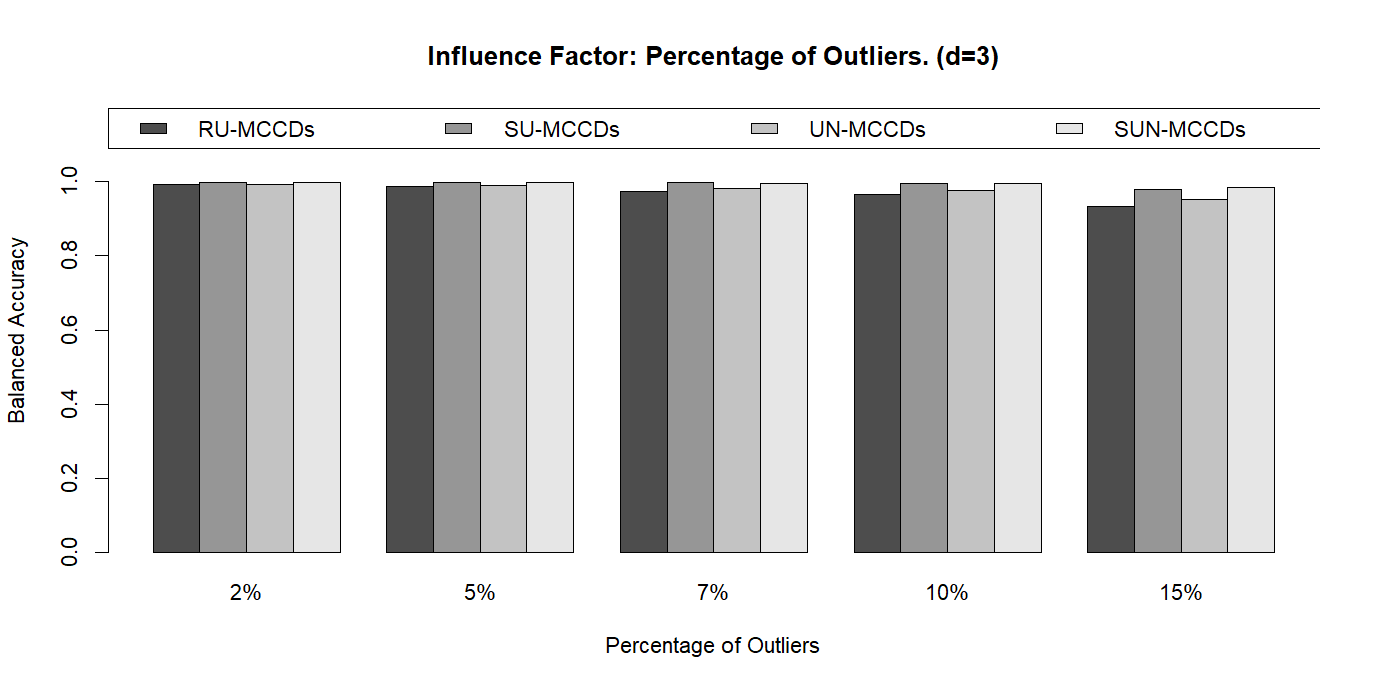}}
\subfigure[The $F_2$-scores when $d=3$.]{
\label{fig:Barplot3_NCon_FS}
\includegraphics[width=0.45\textwidth]{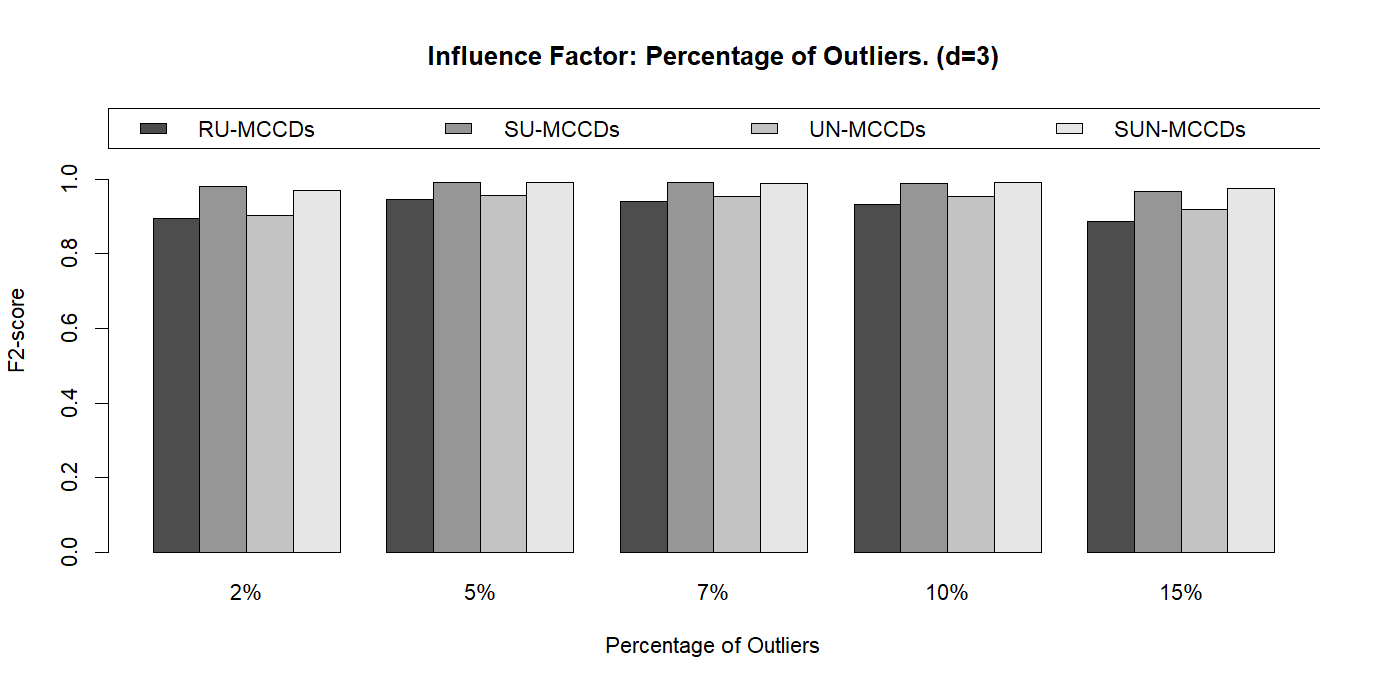}}

\subfigure[The BAs when $d=10$.]{
\label{fig:Barplot10_NCon_BA}
\includegraphics[width=0.45\textwidth]{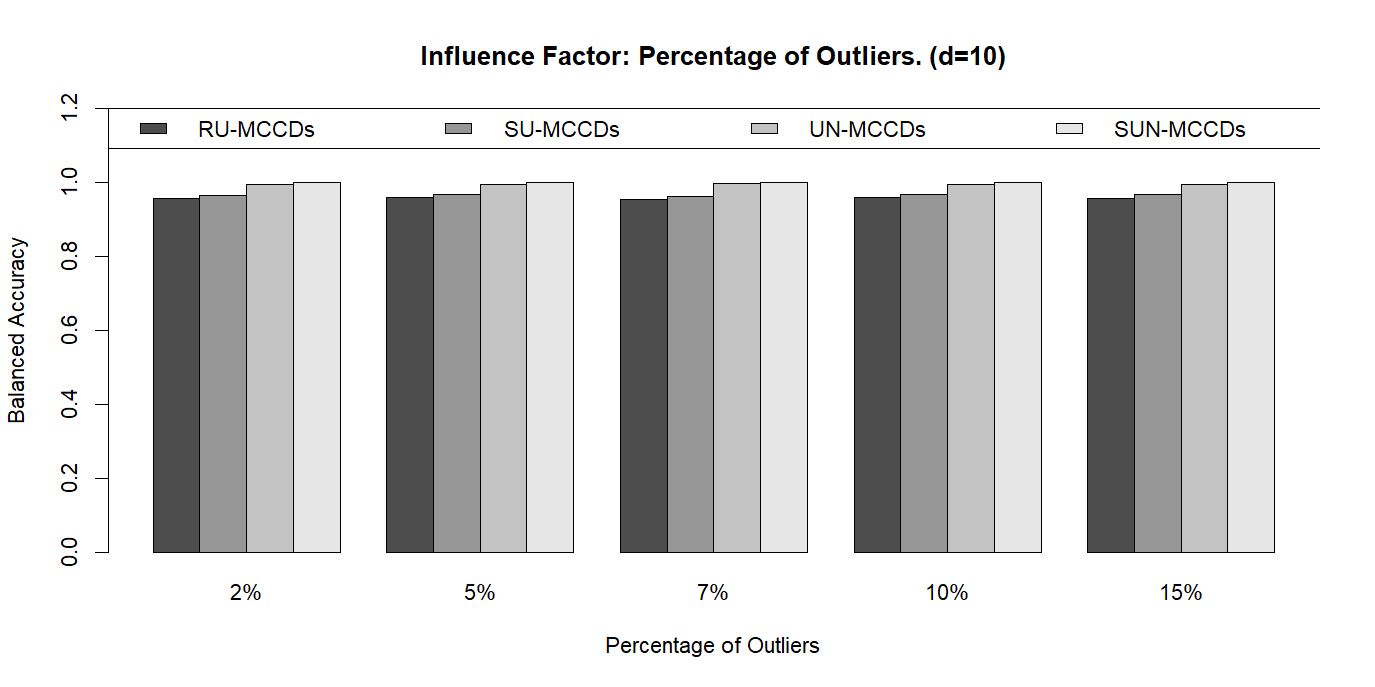}}
\subfigure[The $F_2$-scores when $d=10$.]{
\label{fig:Barplot10_NCon_FS}
\includegraphics[width=0.45\textwidth]{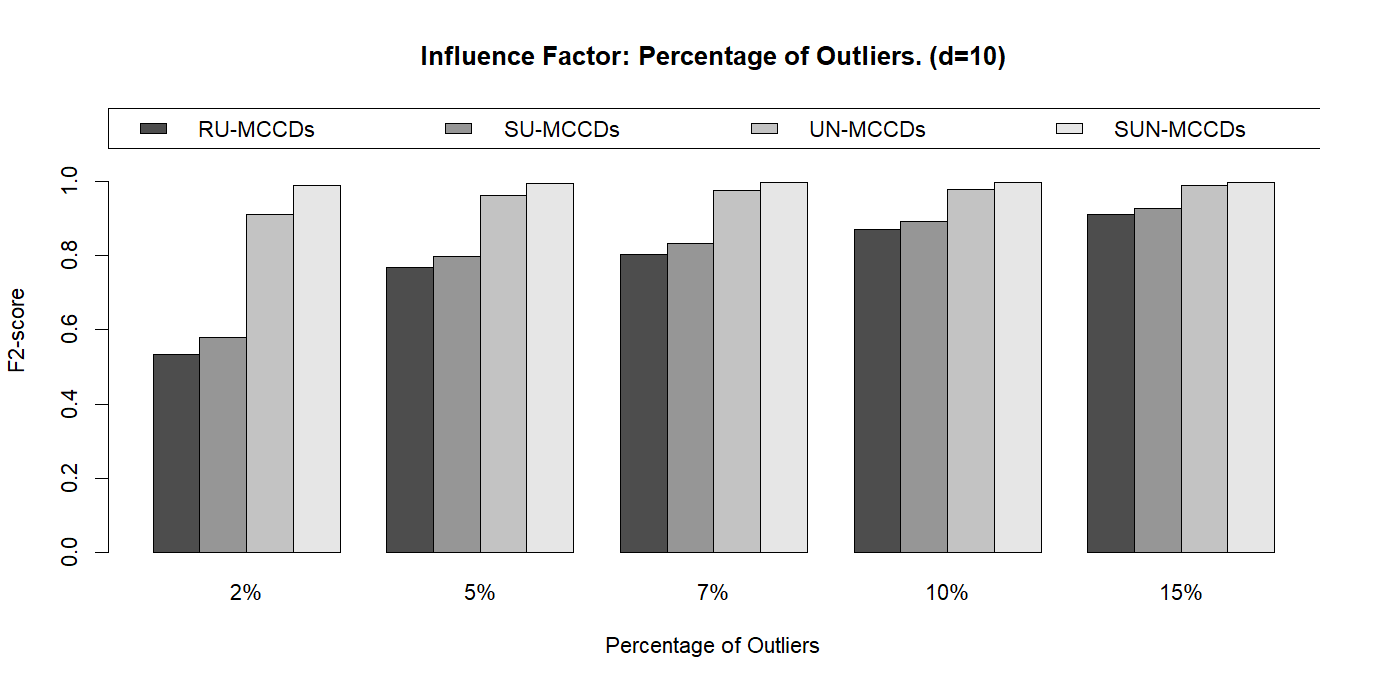}}
\caption{The barplots summarizing the performances of the CCD-based outlier detection algorithms as the percentage of outlier increases (points within each clusters are uniformly distributed). (a) The BAs for $d=3$. (b) The $F_2$-scores for $d=3$. (c) The BAs for $d=10$. (d) The $F_2$-scores for $d=10$.}
\label{fig:Barplot_UCon}
\end{figure}

\begin{table}[htb]
  \footnotesize
  \centering
  \resizebox{0.7\columnwidth}{!}{\begin{tabular}{|c|c|c|c|c|c|c|c|c|c|c|c|}
    \hline
    \multicolumn{2}{|c|}{} & \multicolumn{10}{|c|}{Percentage of Outliers} \\ \cline{3-12}

    \multicolumn{2}{|c|}{} & \multicolumn{2}{|c|}{$2\%$} & \multicolumn{2}{|c|}{$5\%$} & \multicolumn{2}{|c|}{$7\%$} & \multicolumn{2}{|c|}{$10\%$} & \multicolumn{2}{|c|}{$15\%$} \\ \cline{3-12}

    \multicolumn{2}{|c|}{} & TPR & TNR & TPR & TNR & TPR & TNR & TPR & TNR & TPR & TNR \\ \hline

    \multirow{4}*{$d=3$} & RU-MCCDs & 1.000 & 0.834 & 1.000 & 0.833 & 1.000 & 0.838 & 1.000 & 0.839 & 0.998 & 0.840 \\ \cline{2-12}
    & SU-MCCDs & 1.000 & 0.940 & 1.000 & 0.941 & 1.000 & 0.943 & 1.000 & 0.943 & 0.997 & 0.945 \\ \cline{2-12}
    & UN-MCCDs & 1.000 & 0.893 & 0.997 & 0.890 & 0.992 & 0.895 & 0.991 & 0.898 & 0.971 & 0.899 \\ \cline{2-12}
    & SUN-MCCDs & 1.000 & 0.957 & 1.000 & 0.956 & 1.000 & 0.958 & 0.998 & 0.959 & 0.983 & 0.960 \\ \cline{1-12}

    \multirow{4}*{$d=10$} & RU-MCCDs & 1.000 & 0.684 & 1.000 & 0.698 & 1.000 & 0.690 & 1.000 & 0.689 & 1.000 & 0.693 \\ \cline{2-12}
    & SU-MCCDs & 1.000 & 0.767 & 1.000 & 0.777 & 1.000 & 0.772 & 1.000 & 0.767 & 1.000 & 0.773 \\ \cline{2-12}
    & UN-MCCDs & 1.000 & 0.827 & 1.000 & 0.829 & 1.000 & 0.826 & 1.000 & 0.827 & 1.000 & 0.825 \\ \cline{2-12}
    & SUN-MCCDs & 1.000 & 0.947 & 1.000 & 0.947 & 1.000 & 0.948 & 1.000 & 0.948 & 1.000 & 0.948 \\ \cline{1-12}
  \end{tabular}}
  \caption{The TPRs and TNRs of the CCD-based algorithms as the percentage of outliers over the entire simulated data set increases from $2\%$ to $15\%$ (for simulations with Gaussian clusters).}
  \label{tab:Otl_G_Cls1}
\end{table}

\begin{table}[htb]
  \footnotesize
  \centering
  \resizebox{0.7\columnwidth}{!}{\begin{tabular}{|c|c|c|c|c|c|c|c|c|c|c|c|}
    \hline
    \multicolumn{2}{|c|}{} & \multicolumn{10}{|c|}{Percentage of Outliers} \\ \cline{3-12}

    \multicolumn{2}{|c|}{} & \multicolumn{2}{|c|}{$2\%$} & \multicolumn{2}{|c|}{$5\%$} & \multicolumn{2}{|c|}{$7\%$} & \multicolumn{2}{|c|}{$10\%$} & \multicolumn{2}{|c|}{$15\%$} \\ \cline{3-12}

    \multicolumn{2}{|c|}{} & BA & $F_2$-score & BA & $F_2$-score & BA & $F_2$-score & BA & $F_2$-score & BA & $F_2$-score \\ \hline

    \multirow{4}*{$d=3$} & RU-MCCDs & 0.917 & 0.381 & 0.917 & 0.612 & 0.919 & 0.699 & 0.920 & 0.775 & 0.919 & 0.845 \\ \cline{2-12}
    & SU-MCCDs & 0.970 & 0.630 & 0.971 & 0.817 & 0.972 & 0.868 & 0.972 & 0.907 & 0.971 & 0.939 \\ \cline{2-12}
    & UN-MCCDs & 0.947 & 0.488 & 0.944 & 0.703 & 0.944 & 0.777 & 0.945 & 0.839 & 0.935 & 0.876 \\ \cline{2-12}
    & SUN-MCCDs & 0.979 & 0.704 & 0.978 & 0.857 & 0.979 & 0.900 & 0.979 & 0.930 & 0.972 & 0.943 \\ \cline{1-12}

    \multirow{4}*{$d=10$} & RU-MCCDs & 0.842 & 0.244 & 0.849 & 0.466 & 0.845 & 0.548 & 0.845 & 0.641 & 0.847 & 0.742 \\ \cline{2-12}
    & SU-MCCDs & 0.884 & 0.305 & 0.889 & 0.541 & 0.886 & 0.623 & 0.884 & 0.705 & 0.887 & 0.795 \\ \cline{2-12}
    & UN-MCCDs & 0.914 & 0.371 & 0.915 & 0.606 & 0.913 & 0.684 & 0.914 & 0.763 & 0.913 & 0.834 \\ \cline{2-12}
    & SUN-MCCDs & 0.974 & 0.658 & 0.974 & 0.832 & 0.974 & 0.879 & 0.974 & 0.914 & 0.974 & 0.944 \\ \cline{1-12}
  \end{tabular}}
  \caption{The BAs and $F_2$-scores of the CCD-based algorithms as the percentage of outliers over the entire simulated data set increases from $2\%$ to $15\%$ (for simulations with Gaussian clusters).}
  \label{tab:Otl_G_Cls2}
\end{table}

\begin{figure}[htb]
\centering
\subfigure[The BAs when $d=3$.]{
\label{Barplot3_GCon_BA}
\includegraphics[width=0.45\textwidth]{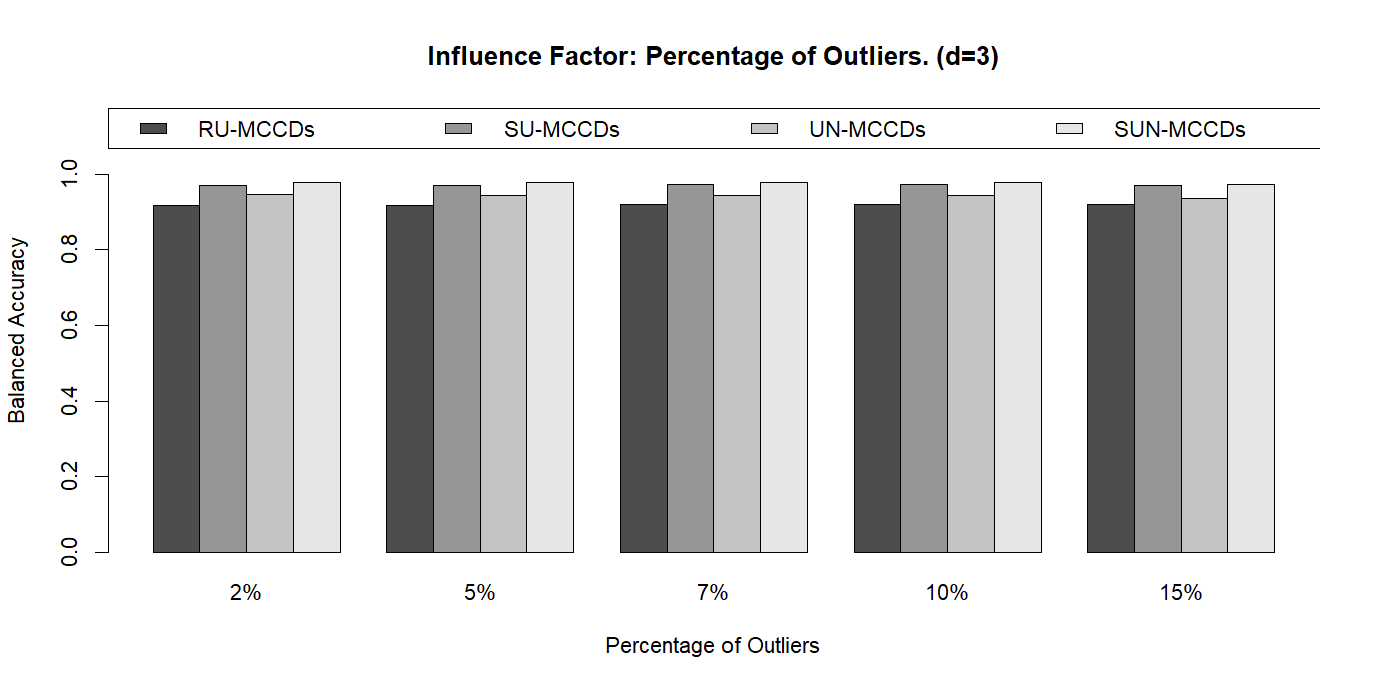}}
\subfigure[The $F_2$-scores when $d=3$.]{
\label{Barplot3_GCon_FS}
\includegraphics[width=0.45\textwidth]{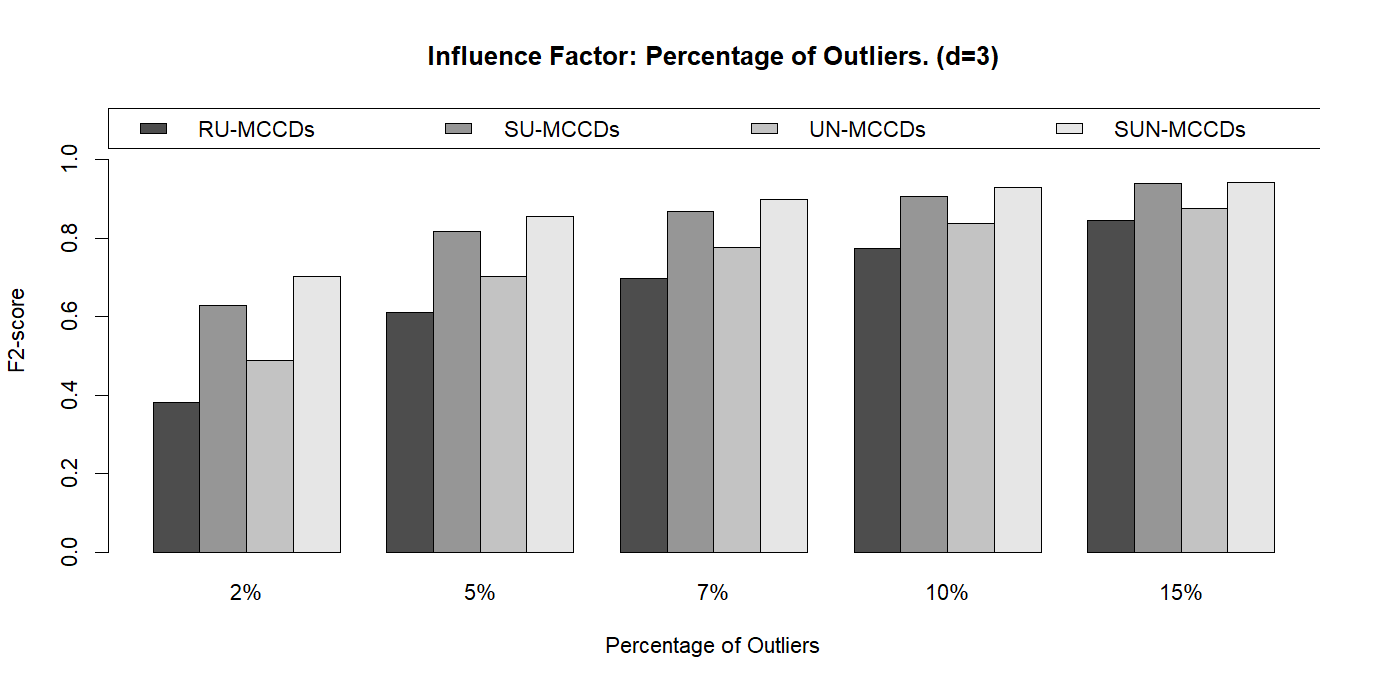}}

\subfigure[The BAs when $d=10$.]{
\label{Barplot10_GCon_BA}
\includegraphics[width=0.45\textwidth]{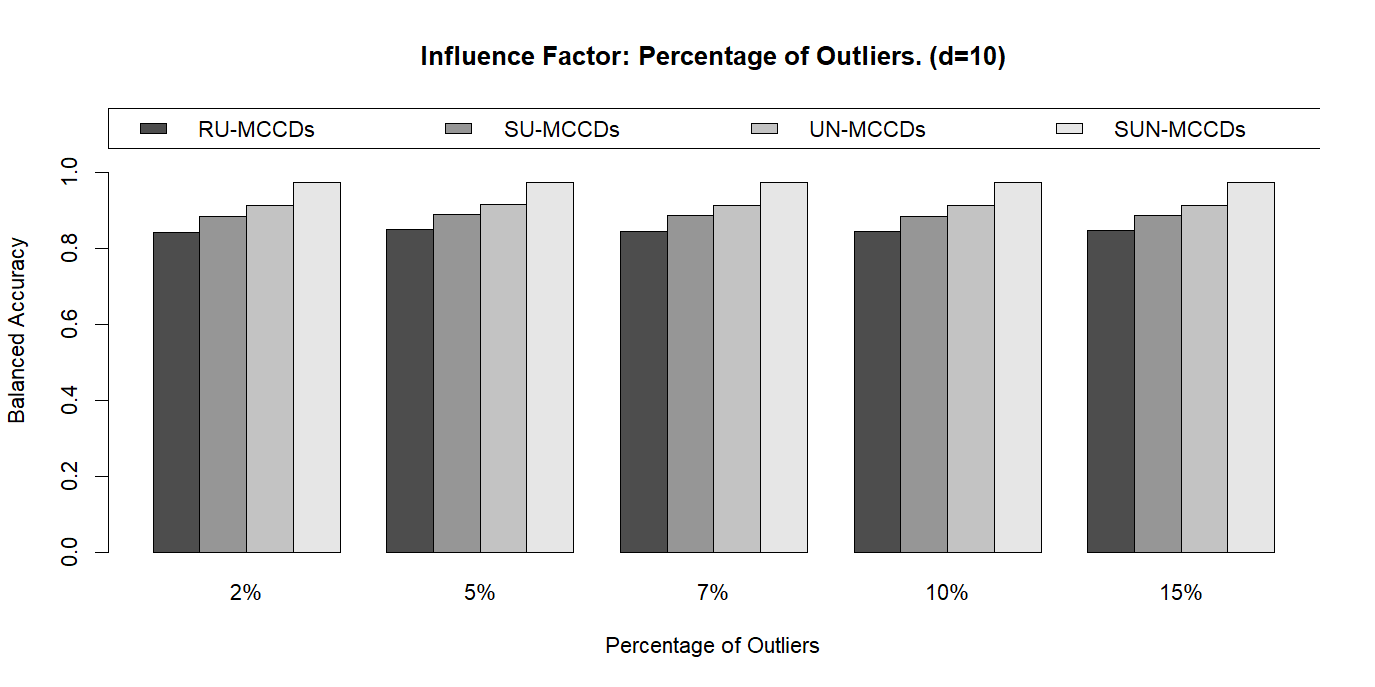}}
\subfigure[The $F_2$-scores when $d=10$.]{
\label{Barplot10_GCon_FS}
\includegraphics[width=0.45\textwidth]{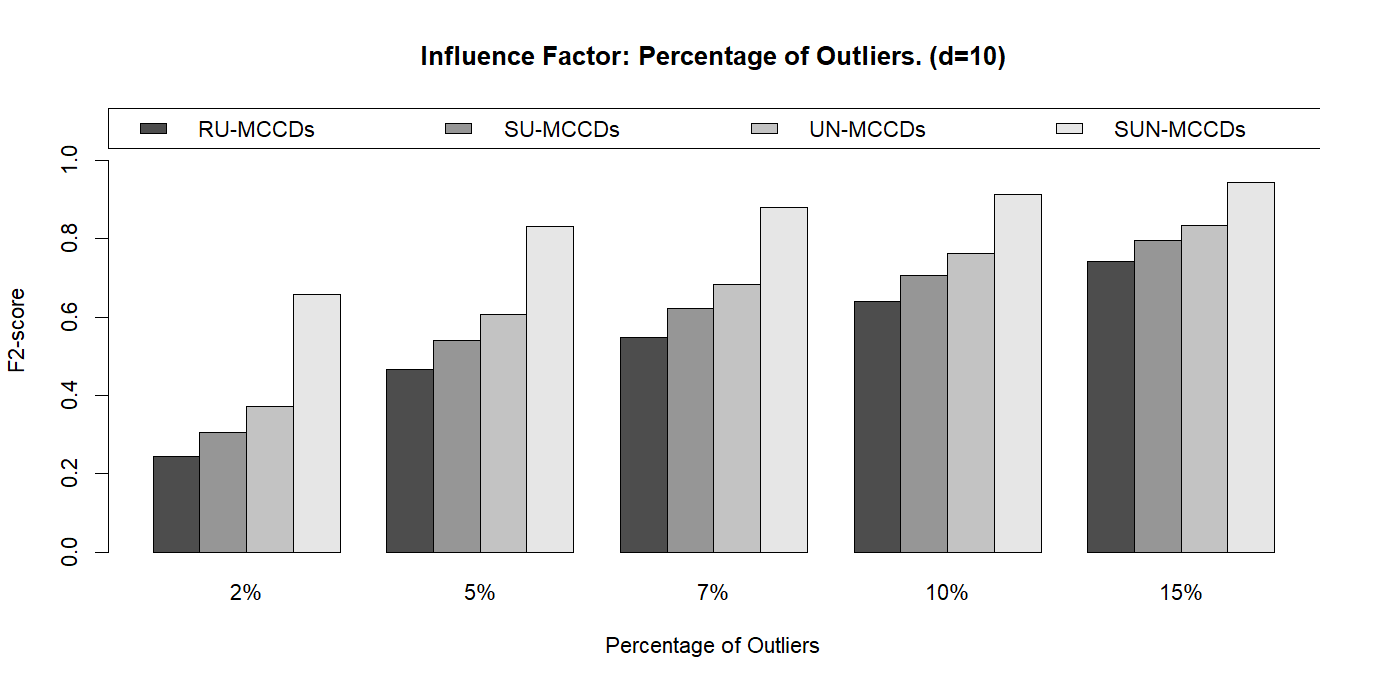}}
\caption{The barplots summarizing the performances of the CCD-based outlier detection algorithms as the percentage of outlier increases (points within each clusters are (multivariate) normally distributed). (a) The BAs for $d=3$. (b) The $F_2$-scores for $d=3$. (c) The BAs for $d=10$. (d) The $F_2$-scores for $d=10$.}
\label{fig:Barplot_GCon}
\end{figure}

In the current setting, the percentage of outliers is not fixed. As a result, the $F_2$-score is not an appropriate measure to compare the efficiency across the data sets with different outlier contamination levels, because precision is highly dependent on the size of outliers. For instance, suppose we have two data sets, each with 100 observations. The first data set has one outlier, and the second has 20 outliers. If an algorithm captures all the outliers and returns one false positive for the first data set and 20 false positives for the second, then the algorithm performs better on the first data set because it has much fewer false positives and higher overall accuracy ($99\%$ versus $80\%$). However, the algorithm would have the same $F_2$-score of 0.882 for both data sets, which is misleading. Therefore, we consider accuracies only instead of $F_2$-scores in the current setting.

We first consider the settings with uniform clusters, whose results are summarized in Tables \ref{tab:Otl_U_Cls1} and \ref{tab:Otl_U_Cls2}. All the algorithms achieve good performance with BAs close to 1. Similar to the previous simulation results, the RU-MCCD and SU-MCCD algorithms lag behind the other two when $d=10$. Furthermore, observe that the TPRs of the RU-MCCD and UN-MCCD algorithm decreases at a faster rate than the other two ``flexible" algorithms when the contamination level increases, e.g., when $d=3$, the TPRs of the RU-MCCD algorithm are 0.999, 0.985, 0.961, 0.943, and 0.878 as the contamination level rises from $2\%$ to $15\%$. It is due to the masking problem that we have explained in Section \ref{sec:Uni_General_Settings_Des}, which happens more frequently when the intensity of outliers is high. Fortunately, thanks to their mechanism that filters small clusters, the SU-MCCD and SUN-MCCD algorithms exhibit more robustness against a high percentage of outliers, e.g., when $d=3$, the SUN-MCCD algorithm can still provide a TPR of 0.973 when the contamination level is as high as $15\%$.

Consider the simulations with Gaussian clusters (Tables \ref{tab:Otl_G_Cls1} and \ref{tab:Otl_G_Cls2}), the SU-MCCD and SUN-MCCD algorithms are slightly better than the other two prototypes and perform similarly when $d=3$, and deliver BAs of at least $95\%$. When $d=10$, the SUN-MCCD algorithm offers substantially better results than the others. Furthermore, all the algorithms are insensitive to the changing contamination level under Gaussian simulation settings with the cost of some false positives.

\subsubsection{Varying the Minimal Distance Between Outliers and Cluster Centers}

In the previous simulation settings, the distances between outliers and cluster centers are at least 2. Given the fact that the support of each cluster is a hypersphere with a radius that varies from 0.7 to 1.3, there is a noticeable distance between an outlier and a regular observation. Under those settings, all four CCD-based algorithms can separate most outliers from regular observations in the majority of cases (except the RU-MCCD and UN-MCCD algorithms, which are affected by the masking problem when the intensity of outliers is relatively high). In this section, instead of fixing the minimal distance to 2, we simulate data sets with outliers and clusters being much closer in proximity. We conduct five simulations with the minimal distance between outliers and any cluster centers increasing from 1.25 to 2.25 and investigate the performance of all 4 CCD-based algorithms. We expect the difficulty of capturing most outliers to increase substantially, especially when the minimal distance is set to 1.25, where the outlier sets may even overlap with some clusters.

Similarly, all other factors are set to the same values as in the previous simulations, and details are presented below. Again, we only list the differences and skip the common parts compared to the simulation setting in Section \ref{sec:N_Cls_Setting}. Some realizations of data sets with uniform clusters in 2-dimensional space (although the simulation experiments are conducted on 3 and 10-dimensional space) are presented in Figure \ref{fig:Demo_2d_Dis_Otl} (for illustration purposes),

\begin{itemize}
  \item[\romannumeral1.] The minimal distance between an outliers and any cluster center varies with values: 1.25, 1.5, 1.75, 2, and 2.25 (the study of focus in this section).
  \label{sec:Olt_Dis_Setting}
\end{itemize}

\begin{figure}[htb]
  \centering
  \subfigure[1.25]{
  \includegraphics[width=0.30\textwidth]{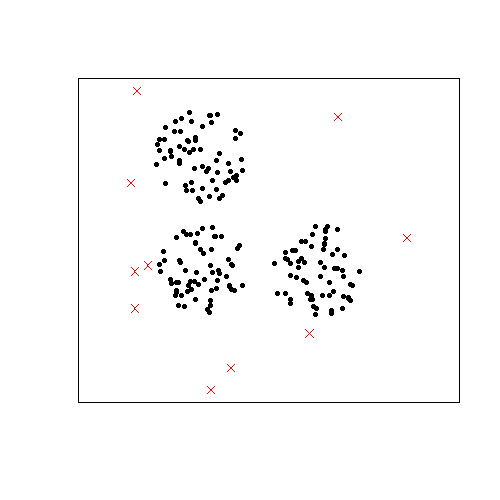}}
  \subfigure[1.5]{
  \includegraphics[width=0.30\textwidth]{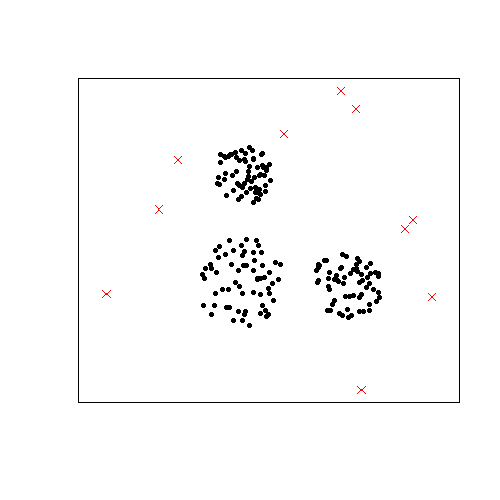}}
  \subfigure[1.75]{
  \includegraphics[width=0.30\textwidth]{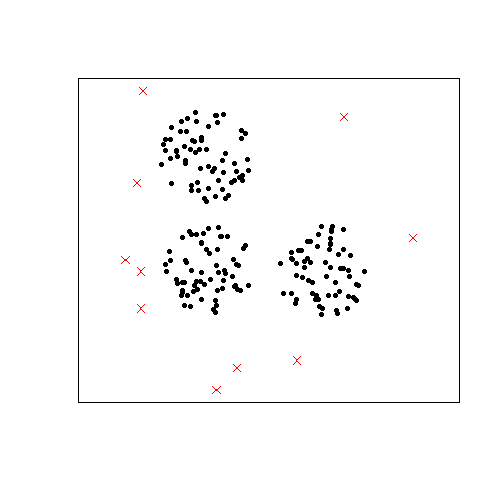}}
  \subfigure[2]{
  \includegraphics[width=0.30\textwidth]{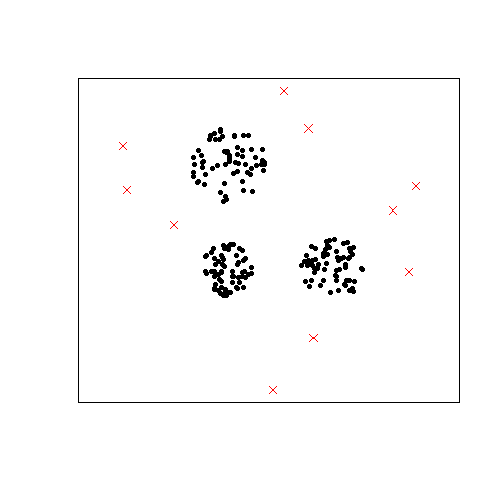}}
  \subfigure[2.25]{
  \includegraphics[width=0.30\textwidth]{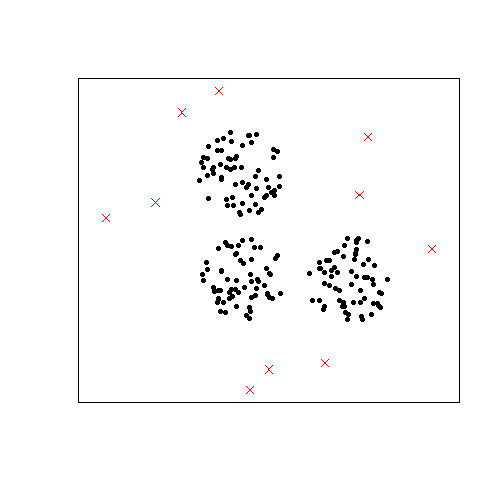}}
  \caption{Some realizations (with uniform clusters) of the simulation setting in Section \ref{sec:Olt_Dis_Setting} , the minimal distance between outliers and cluster centers increases from 1.25 to 2.25. Red crosses are outliers, black points are regular observations. The minimal distances are indicated below each sub-figure.}
  \label{fig:Demo_2d_Dis_Otl}
\end{figure}

The simulation results are summarized from Tables \ref{tab:Otl_Dis_U_Cls1} to \ref{tab:Otl_Dis_G_Cls2}. BAs and $F_2$-scores (Tables \ref{tab:Otl_Dis_U_Cls2} and \ref{tab:Otl_Dis_G_Cls2}) are also presented as barplots in Figures \ref{fig:Barplot_NOutlierDist} and \ref{fig:Barplot_GOutlierDist}, respectively.

\begin{table}[htb]
  \centering
  \resizebox{0.7\columnwidth}{!}{\begin{tabular}{|c|c|c|c|c|c|c|c|c|c|c|c|}
    \hline
    \multicolumn{2}{|c|}{} & \multicolumn{10}{|c|}{Minimal Distances between Outliers and Cluster Centers} \\ \cline{3-12}

    \multicolumn{2}{|c|}{} & \multicolumn{2}{|c|}{$1.25$} & \multicolumn{2}{|c|}{$1.5$} & \multicolumn{2}{|c|}{$1.75$} & \multicolumn{2}{|c|}{$2.00$} & \multicolumn{2}{|c|}{$2.25$} \\ \cline{3-12}

    \multicolumn{2}{|c|}{} & TPR & TNR & TPR & TNR & TPR & TNR & TPR & TNR & TPR & TNR \\ \hline

    \multirow{4}*{$d=3$} & RU-MCCDs & 0.979 & 0.988 & 0.982 & 0.988 & 0.986 & 0.988 & 0.985 & 0.988 & 0.984 & 0.988 \\ \cline{2-12}
    & SU-MCCDs & 0.966 & 0.998 & 0.983 & 0.999 & 0.997 & 0.998 & 0.999 & 0.998 & 1.000 & 0.998 \\ \cline{2-12}
    & UN-MCCDs & 0.980 & 0.989 & 0.989 & 0.989 & 0.985 & 0.988 & 0.991 & 0.990 & 0.990 & 0.989 \\ \cline{2-12}
    & SUN-MCCDs & 0.962 & 0.997 & 0.976 & 0.997 & 0.992 & 0.997 & 0.998 & 0.996 & 1.000 & 0.997 \\ \cline{1-12}

    \multirow{4}*{$d=10$} & RU-MCCDs & 1.000 & 0.914 & 1.000 & 0.914 & 1.000 & 0.914 & 1.000 & 0.920 & 1.000 & 0.920 \\ \cline{2-12}
    & SU-MCCDs & 1.000 & 0.929 & 1.000 & 0.929 & 1.000 & 0.929 & 1.000 & 0.934 & 1.000 & 0.934 \\ \cline{2-12}
    & UN-MCCDs & 1.000 & 0.990 & 1.000 & 0.990 & 1.000 & 0.990 & 1.000 & 0.990 & 1.000 & 0.990 \\ \cline{2-12}
    & SUN-MCCDs & 1.000 & 0.999 & 1.000 & 0.999 & 1.000 & 0.999 & 1.000 & 0.999 & 1.000 & 0.999 \\ \cline{1-12}
  \end{tabular}}
  \caption{The TPRs and TNRs of the CCD-based algorithms as the minimal distance from outliers to any cluster centers increases from 1.25 to 2.25 (for simulations with uniform clusters).}
  \label{tab:Otl_Dis_U_Cls1}
\end{table}

\begin{table}[htb]
  \centering
  \resizebox{0.7\columnwidth}{!}{\begin{tabular}{|c|c|c|c|c|c|c|c|c|c|c|c|}
    \hline
    \multicolumn{2}{|c|}{} & \multicolumn{10}{|c|}{Minimal Distances between Outliers and Cluster Centers} \\ \cline{3-12}

    \multicolumn{2}{|c|}{} & \multicolumn{2}{|c|}{$1.25$} & \multicolumn{2}{|c|}{$1.5$} & \multicolumn{2}{|c|}{$1.75$} & \multicolumn{2}{|c|}{$2.00$} & \multicolumn{2}{|c|}{$2.25$} \\ \cline{3-12}

    \multicolumn{2}{|c|}{} & BA & $F_2$-score & BA & $F_2$-score & BA & $F_2$-score & BA & $F_2$-score & BA & $F_2$-score \\ \hline

    \multirow{4}*{$d=3$} & RU-MCCDs & 0.984 & 0.940 & 0.985 & 0.942 & 0.987 & 0.946 & 0.987 & 0.945 & 0.986 & 0.944 \\ \cline{2-12}
    & SU-MCCDs & 0.982 & 0.964 & 0.991 & 0.983 & 0.998 & 0.990 & 0.999 & 0.992 & 0.999 & 0.992 \\ \cline{2-12}
    & UN-MCCDs & 0.985 & 0.895 & 0.989 & 0.951 & 0.987 & 0.945 & 0.991 & 0.956 & 0.990 & 0.952 \\ \cline{2-12}
    & SUN-MCCDs & 0.980 & 0.953 & 0.987 & 0.970 & 0.995 & 0.982 & 0.997 & 0.983 & 0.999 & 0.989 \\ \cline{1-12}

    \multirow{4}*{$d=10$} & RU-MCCDs & 0.957 & 0.754 & 0.957 & 0.754 & 0.957 & 0.754 & 0.960 & 0.767 & 0.960 & 0.767 \\ \cline{2-12}
    & SU-MCCDs & 0.965 & 0.788 & 0.965 & 0.788 & 0.965 & 0.788 & 0.967 & 0.799 & 0.967 & 0.799 \\ \cline{2-12}
    & UN-MCCDs & 0.995 & 0.963 & 0.995 & 0.963 & 0.995 & 0.963 & 0.995 & 0.963 & 0.995 & 0.963 \\ \cline{2-12}
    & SUN-MCCDs & 1.000 & 0.996 & 1.000 & 0.996 & 1.000 & 0.996 & 1.000 & 0.996 & 1.000 & 0.996 \\ \cline{1-12}
  \end{tabular}}
  \caption{The BAs and $F_2$-scores of the CCD-based algorithms as the minimal distance from outliers to any cluster centers increases from 1.25 to 2.25 (for simulations with uniform clusters).}
  \label{tab:Otl_Dis_U_Cls2}
\end{table}

\begin{figure}[htb]
\centering
\subfigure[The BAs when $d=3$.]{
\label{Barplot3_NOutlierDist_BA}
\includegraphics[width=0.45\textwidth]{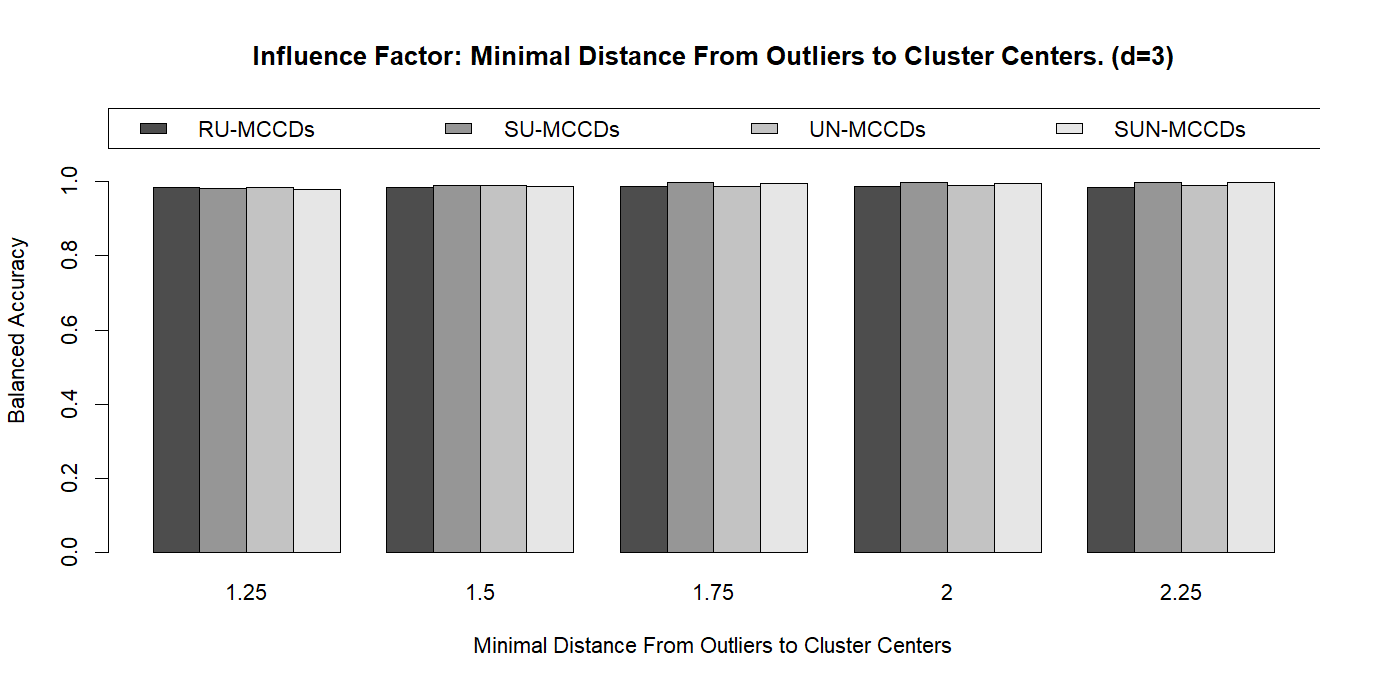}}
\subfigure[The $F_2$-scores when $d=3$.]{
\label{Barplot3_NOutlierDist_FS}
\includegraphics[width=0.45\textwidth]{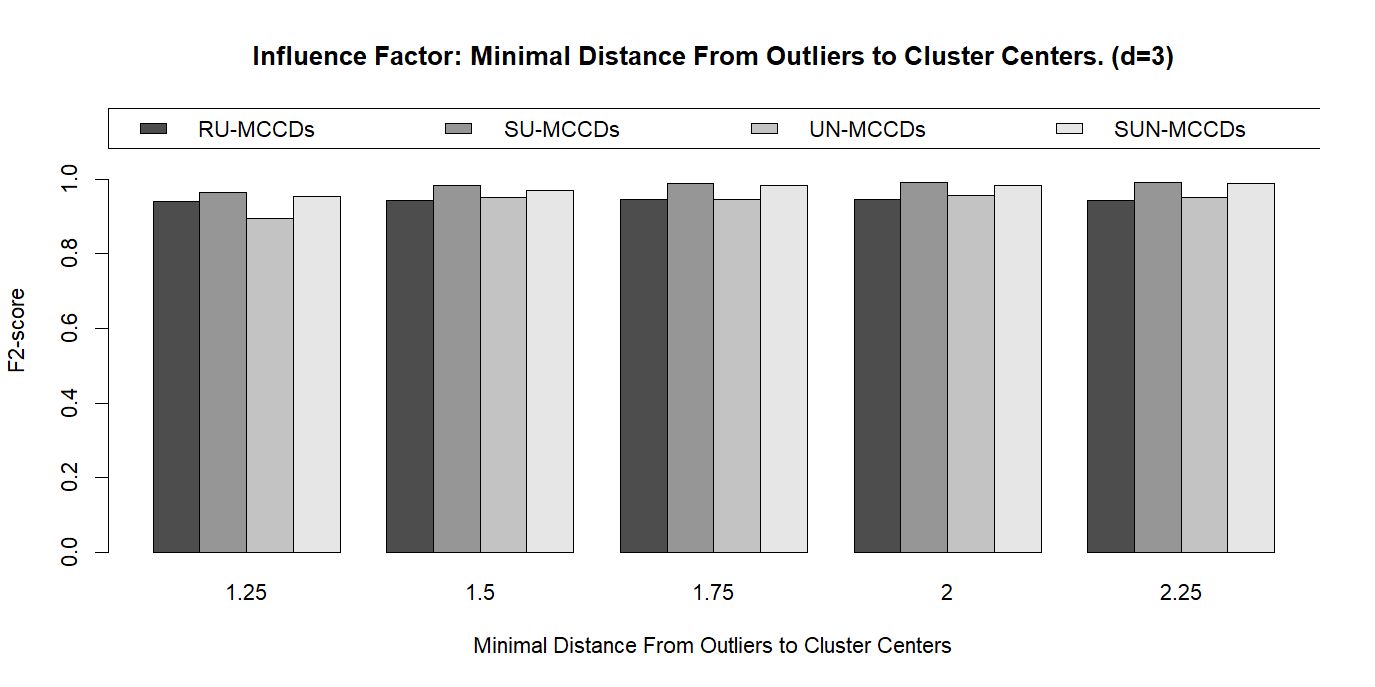}}

\subfigure[The BAs when $d=10$.]{
\label{Barplot10_NOutlierDist_BA}
\includegraphics[width=0.45\textwidth]{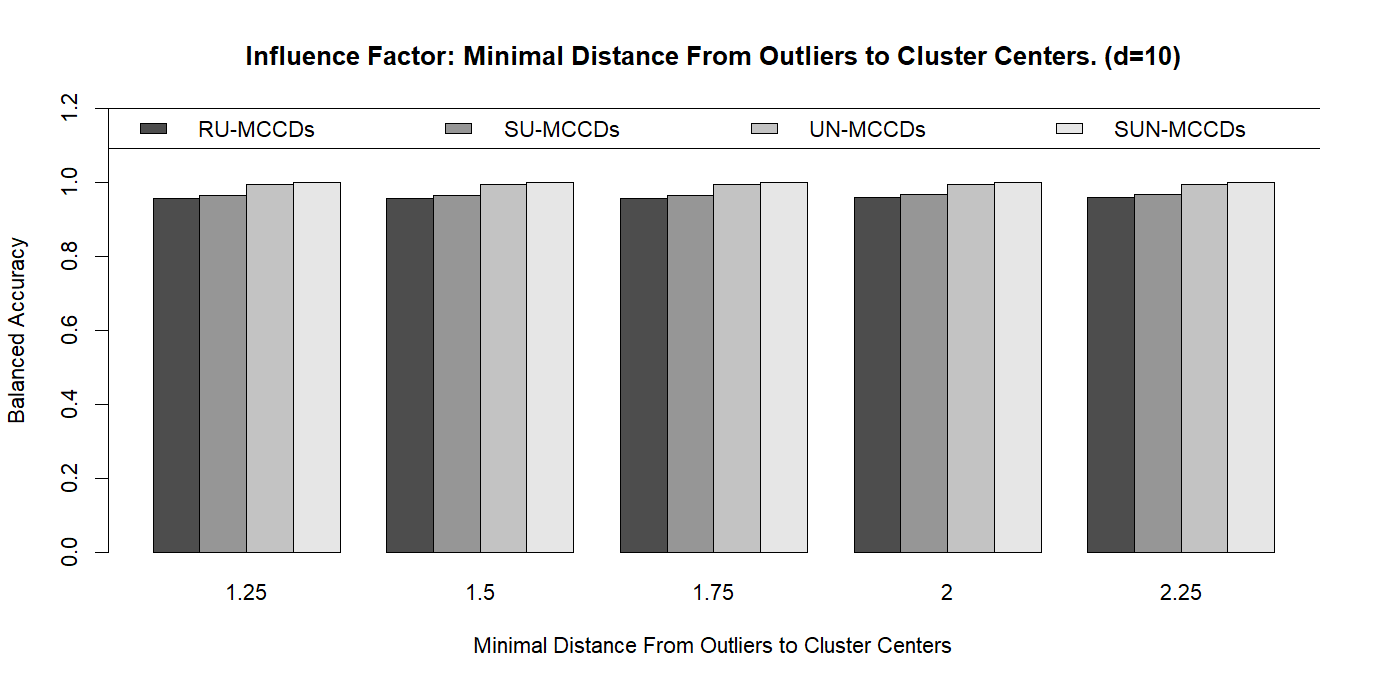}}
\subfigure[The $F_2$-scores when $d=10$.]{
\label{Barplot10_NOutlierDist FS}
\includegraphics[width=0.45\textwidth]{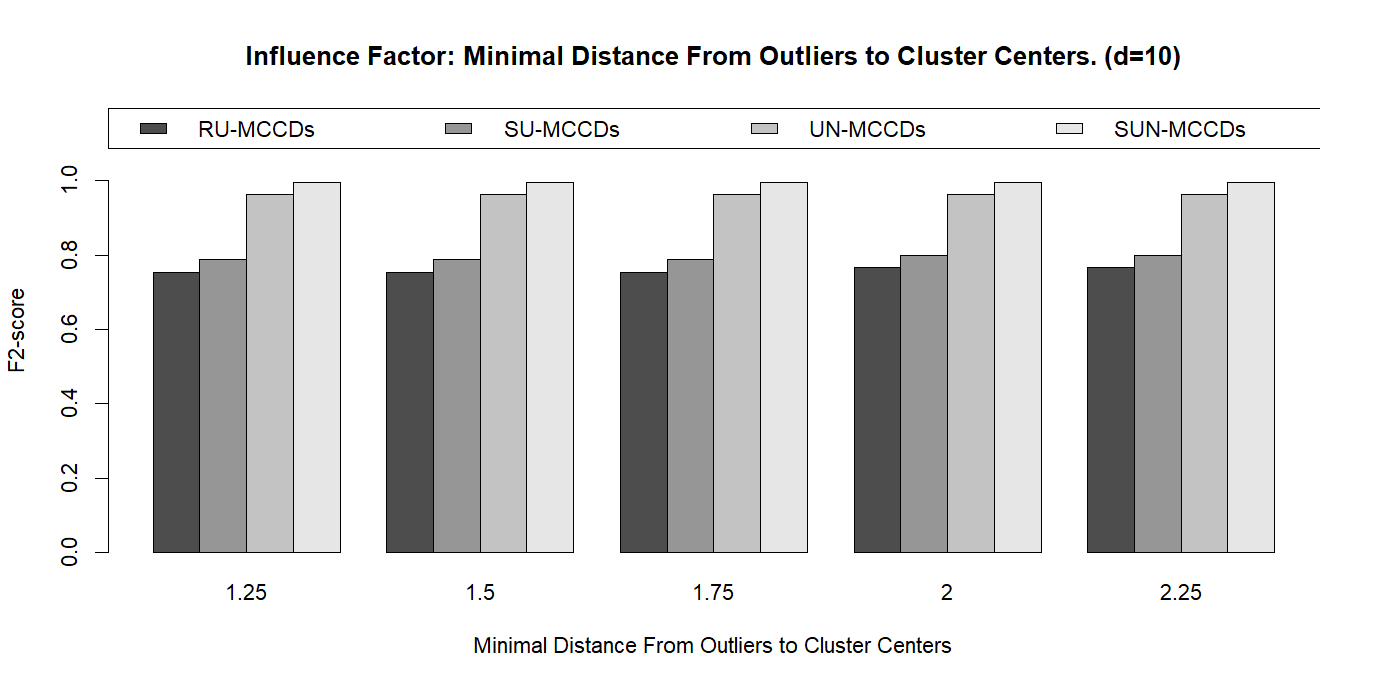}}
\caption{The barplots summarizing the performances of the CCD-based outlier detection algorithms as the minimal distance from outliers to any cluster centers increases (points within each clusters are uniformly distributed). (a) The BAs for $d=3$. (b) The $F_2$-scores for $d=3$. (c) The BAs for $d=10$. (d) The $F_2$-scores for $d=10$.}
\label{fig:Barplot_NOutlierDist}
\end{figure}

\begin{table}[htb]
  \centering
  \resizebox{0.7\columnwidth}{!}{\begin{tabular}{|c|c|c|c|c|c|c|c|c|c|c|c|}
    \hline
    \multicolumn{2}{|c|}{} & \multicolumn{10}{|c|}{Minimal Distances between Outliers and Cluster Centers} \\ \cline{3-12}

    \multicolumn{2}{|c|}{} & \multicolumn{2}{|c|}{$1.25$} & \multicolumn{2}{|c|}{$1.50$} & \multicolumn{2}{|c|}{$1.75$} & \multicolumn{2}{|c|}{$2.00$} & \multicolumn{2}{|c|}{$2.25$} \\ \cline{3-12}

    \multicolumn{2}{|c|}{} & TPR & TNR & TPR & TNR & TPR & TNR & TPR & TNR & TPR & TNR \\ \hline

    \multirow{4}*{$d=3$} & RU-MCCDs & 1.000 & 0.835 & 1.000 & 0.837 & 1.000 & 0.833 & 1.000 & 0.838 & 1.000 & 0.836 \\ \cline{2-12}
    & SU-MCCDs & 1.000 & 0.941 & 1.000 & 0.941 & 1.000 & 0.942 & 1.000 & 0.941 & 1.000 & 0.942 \\ \cline{2-12}
    & UN-MCCDs & 0.987 & 0.893 & 0.983 & 0.895 & 0.983 & 0.893 & 0.983 & 0.896 & 0.985 & 0.894 \\ \cline{2-12}
    & SUN-MCCDs & 0.999 & 0.957 & 1.000 & 0.957 & 1.000 & 0.959 & 1.000 & 0.956 & 1.000 & 0.959 \\ \cline{1-12}

    \multirow{4}*{$d=10$} & RU-MCCDs & 1.000 & 0.694 & 1.000 & 0.694 & 1.000 & 0.689 & 1.000 & 0.698 & 1.000 & 0.691 \\ \cline{2-12}
    & SU-MCCDs & 1.000 & 0.775 & 1.000 & 0.775 & 1.000 & 0.769 & 1.000 & 0.777 & 1.000 & 0.772 \\ \cline{2-12}
    & UN-MCCDs & 1.000 & 0.826 & 1.000 & 0.826 & 1.000 & 0.826 & 1.000 & 0.827 & 1.000 & 0.828 \\ \cline{2-12}
    & SUN-MCCDs & 1.000 & 0.948 & 1.000 & 0.948 & 1.000 & 0.948 & 1.000 & 0.947 & 1.000 & 0.949 \\ \cline{1-12}
  \end{tabular}}
  \caption{The TPRs and TNRs of the CCD-based algorithms as the minimal distance from outliers to any cluster centers increases from 1.25 to 2.25 (for simulations with Gaussian clusters).}
  \label{tab:Otl_Dis_G_Cls1}
\end{table}

\begin{table}[htb]
  \centering
  \resizebox{0.7\columnwidth}{!}{\begin{tabular}{|c|c|c|c|c|c|c|c|c|c|c|c|}
    \hline
    \multicolumn{2}{|c|}{} & \multicolumn{10}{|c|}{Minimal Distances between Outliers and Cluster Centers} \\ \cline{3-12}

    \multicolumn{2}{|c|}{} & \multicolumn{2}{|c|}{$1.25$} & \multicolumn{2}{|c|}{$1.50$} & \multicolumn{2}{|c|}{$1.75$} & \multicolumn{2}{|c|}{$2.00$} & \multicolumn{2}{|c|}{$2.25$} \\ \cline{3-12}

    \multicolumn{2}{|c|}{} & BA & $F_2$-score & BA & $F_2$-score & BA & $F_2$-score & BA & $F_2$-score & BA & $F_2$-score \\ \hline

    \multirow{4}*{$d=3$} & RU-MCCDs & 0.918 & 0.615 & 0.919 & 0.618 & 0.917 & 0.612 & 0.919 & 0.619 & 0.918 & 0.616 \\ \cline{2-12}
    & SU-MCCDs & 0.971 & 0.817 & 0.971 & 0.817 & 0.971 & 0.819 & 0.971 & 0.817 & 0.971 & 0.819 \\ \cline{2-12}
    & UN-MCCDs & 0.940 & 0.703 & 0.939 & 0.704 & 0.938 & 0.701 & 0.940 & 0.706 & 0.940 & 0.704 \\ \cline{2-12}
    & SUN-MCCDs & 0.978 & 0.859 & 0.979 & 0.860 & 0.980 & 0.865 & 0.978 & 0.857 & 0.980 & 0.865 \\ \cline{1-12}

    \multirow{4}*{$d=10$} & RU-MCCDs & 0.847 & 0.462 & 0.847 & 0.462 & 0.845 & 0.458 & 0.849 & 0.466 & 0.846 & 0.460 \\ \cline{2-12}
    & SU-MCCDs & 0.888 & 0.539 & 0.888 & 0.539 & 0.885 & 0.533 & 0.889 & 0.541 & 0.886 & 0.536 \\ \cline{2-12}
    & UN-MCCDs & 0.913 & 0.602 & 0.913 & 0.602 & 0.913 & 0.602 & 0.914 & 0.603 & 0.914 & 0.605 \\ \cline{2-12}
    & SUN-MCCDs & 0.974 & 0.835 & 0.974 & 0.835 & 0.974 & 0.835 & 0.974 & 0.832 & 0.975 & 0.838 \\ \cline{1-12}
  \end{tabular}}
  \caption{The BAs and $F_2$-scores of the CCD-based algorithms as the minimal distance from outliers to any cluster centers increases from 1.25 to 2.25 (for simulations with Gaussian clusters).}
  \label{tab:Otl_Dis_G_Cls2}
\end{table}

\begin{figure}[htb]
\centering
\subfigure[The BAs when $d=3$.]{
\label{Barplot3_GOutlierDist_BA}
\includegraphics[width=0.45\textwidth]{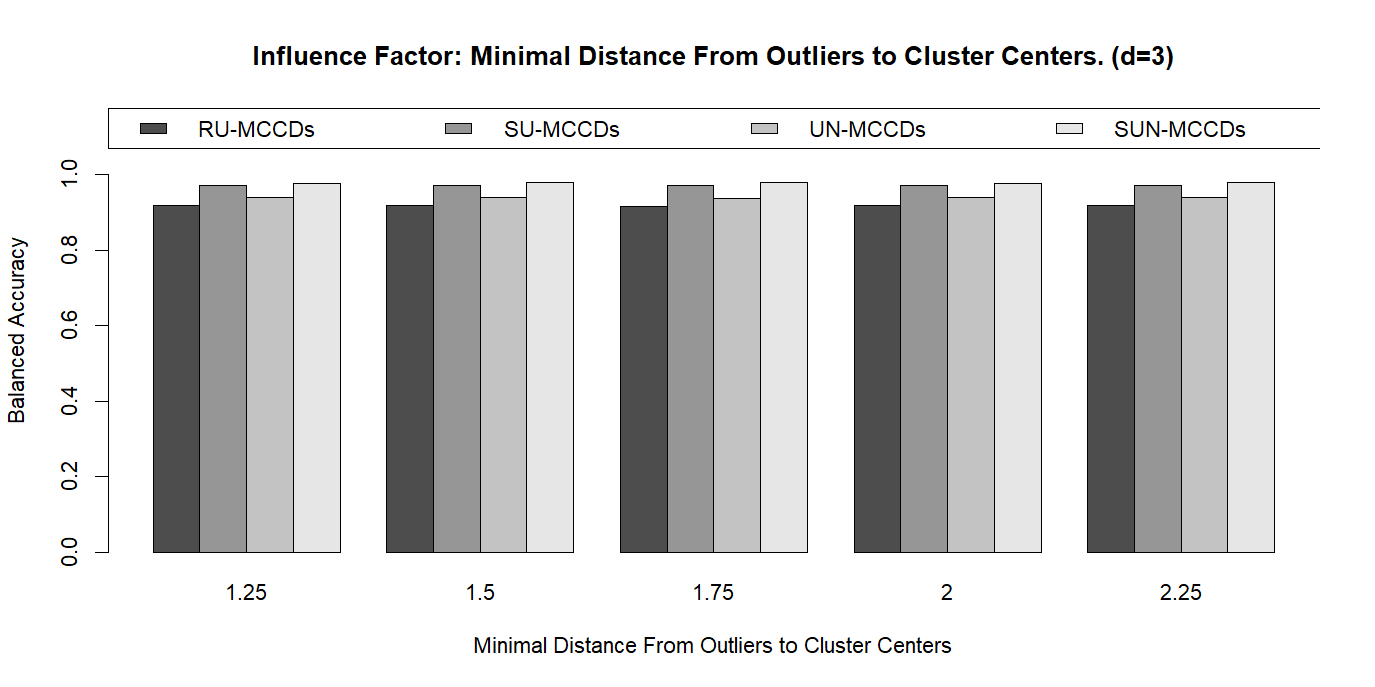}}
\subfigure[The $F_2$-scores when $d=3$.]{
\label{Barplot3_GOutlierDist_FS}
\includegraphics[width=0.45\textwidth]{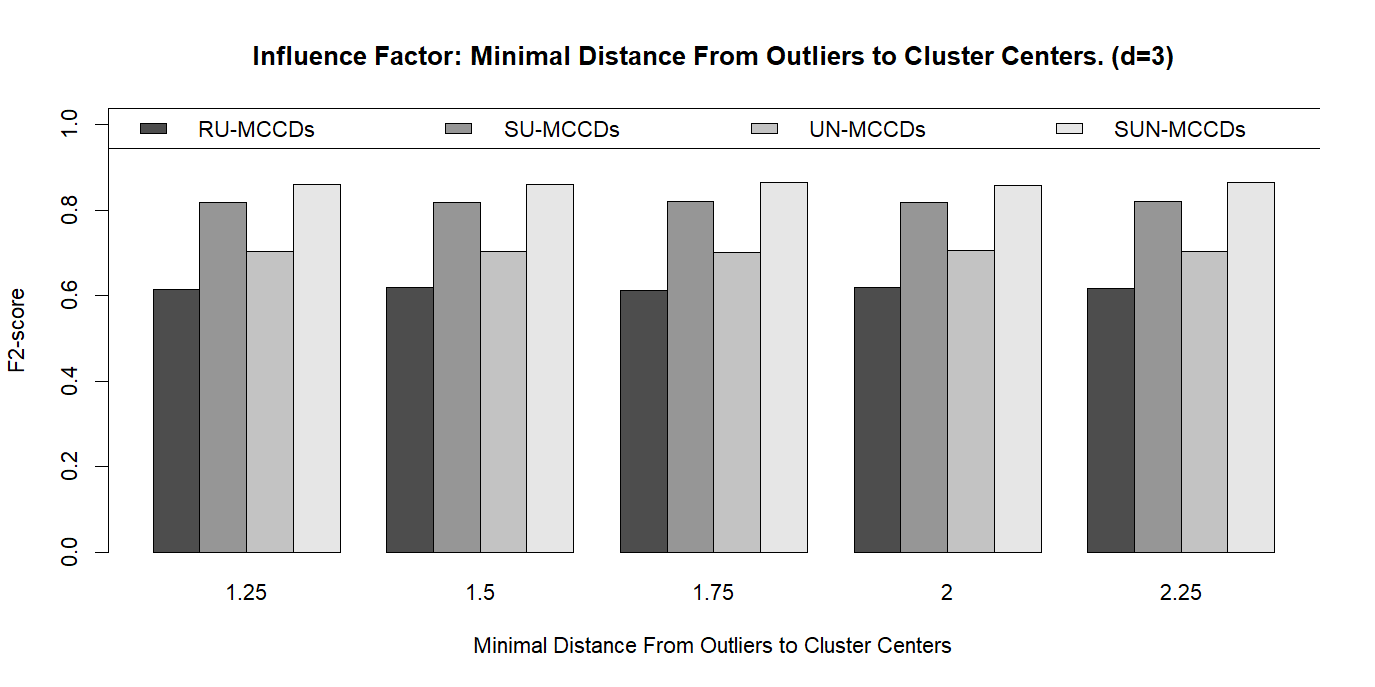}}

\subfigure[The BAs when $d=10$.]{
\label{Barplot10_GOutlierDist_BA}
\includegraphics[width=0.45\textwidth]{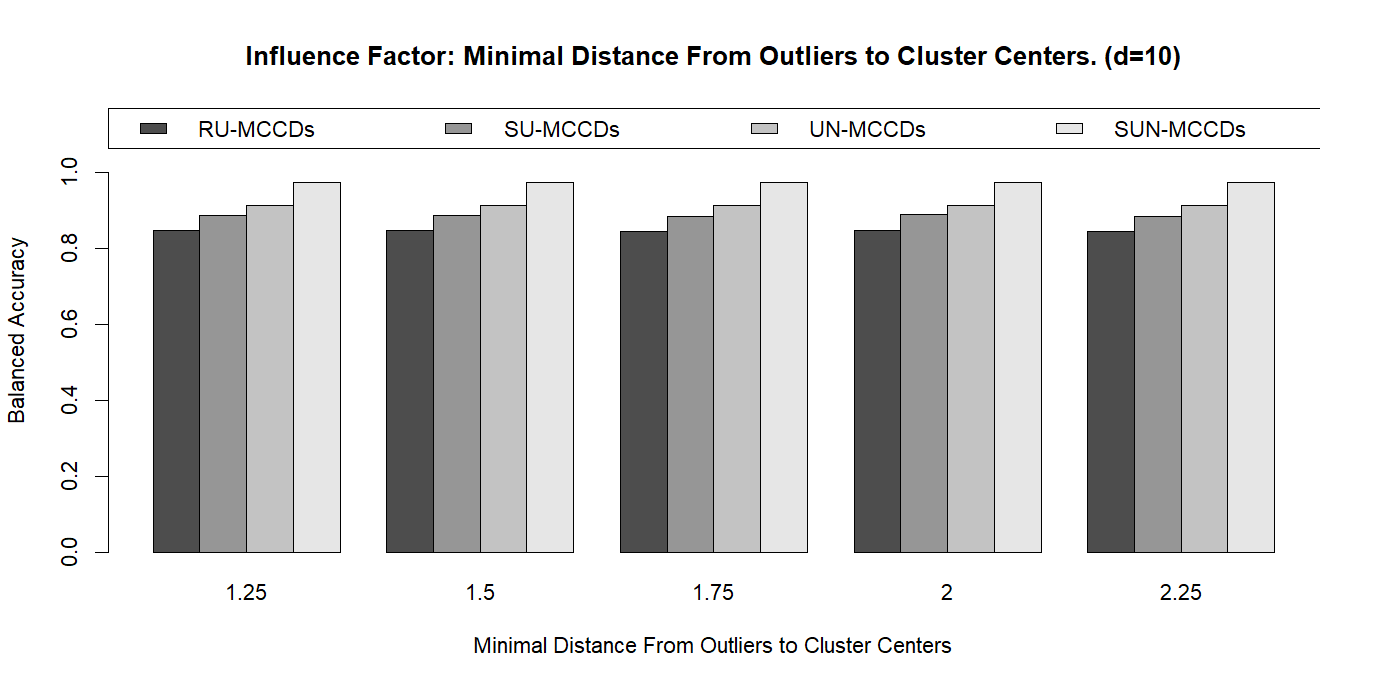}}
\subfigure[The $F_2$-scores when $d=10$.]{
\label{Barplot10_GOutlierDist FS}
\includegraphics[width=0.45\textwidth]{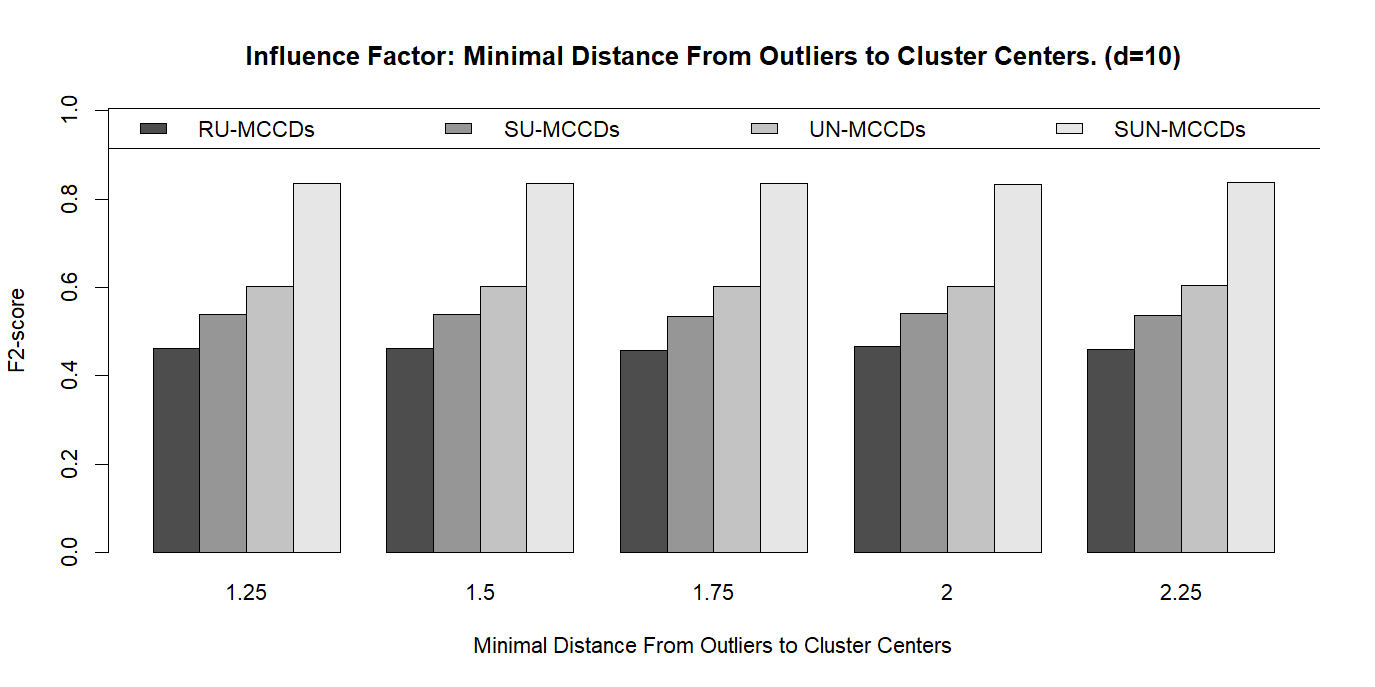}}
\caption{The barplots summarizing the performances of the CCD-based outlier detection algorithms as the minimal distance from outliers to any cluster centers increases (points within each clusters are (multivariate) normally distributed). (a) The BAs for $d=3$. (b) The $F_2$-scores for $d=3$. (c) The BAs for $d=10$. (d) The $F_2$-scores for $d=10$.}
\label{fig:Barplot_GOutlierDist}
\end{figure}

In the simulations with only uniform clusters, observe that when $d=3$ and the minimal distance is 1.25, the four algorithms yield TPRs of 0.979, 0.966, 0.980, and 0.962, slightly lower than those in other scenarios. It aligns with our expectations since a few outliers may fall into the range of regular clusters. Additionally, although in small margins, it is worth noting that the RU-MCCD and UN-MCCD algorithms achieve higher TPRs than the other two algorithms (0.979, 0.980 versus 0.966 and 0.962). Their different mechanisms can explain the reason. The two ``flexible" algorithms construct clusters using multiple covering balls. Consequently, the odds of outliers incorporated by multiple covering balls increase when they approach clusters. When $d=10$, the results remain promising even with a minimal distance of 1.25. All four algorithms seem unaffected by the minimal distance as long as it exceeds 1.5,  when the outlier group and clusters are separable.

With Gaussian clusters, the four CCD-based algorithms exhibit stable performance despite the varying distances between outliers and clusters (e.g., when $d=3$, the $F_2$-scores of the SUN-MCCD algorithm range from 0.948 to 0.949 when the minimal distance increases from 1.25 to 2.25). The reason is that all the algorithms determine the radius of each covering ball by SR-MCT; therefore, they cannot capture an entire Gaussian cluster, and some regular observations that are relatively far from the cluster center tend to be uncovered, especially under high dimensions. Thus, even when an outlier set is close to or overlaps with some regular observations near the border, they get labeled as positives (i.e., outliers). In other words, with Gaussian clusters, these algorithms identify most or all of the outliers, even if the outliers are close to regular observations at the cost of some false positives along the border of each cluster. Echoing the results of previous simulations, the ``cost" is much lower for the SU-MCCD and SUN-MCCD algorithms than their prototypes because these two ``flexible" algorithms generally end up with more than one covering ball for each cluster, which has better coverage for the regular observations.

\subsubsection{Varying The Distances Between Cluster Centers}

In this section, we investigate whether the distance between clusters affects the performance of the four CCD-based outlier detection algorithms. Previously, the first cluster center is $(\underbrace{3,...,3}_{d})$, and others are obtained by shifting three units from the first one in various directions. Therefore, these simulated clusters are always distinct and easy to separate. As a result, the four CCD-based algorithms could identify each cluster without difficulty in most cases, which is helpful for the subsequent steps in outlier detection. In this setting, we alter the difficulty level of clustering by changing the inter-cluster distances (the distances between pairs of points from different clusters). We keep the first cluster centered at $(\underbrace{3,...,3}_{d})$, but we vary its distances to other cluster centers from 1.5 to 4. Here is where things get interesting: when the distance is smaller than 2, the chance that two or more clusters overlap is high, making it challenging to figure out the correct number of clusters and their locations. We are curious to see if the increasing difficulty level of clustering will affect the accuracy of outlier detection. Some challenges include (1) capturing the outliers close to two or more overlapping clusters with different intensities and (2) dealing with the swapping problem when clusters with different intensities overlap, since some regular observations from low-intensity clusters could be located near high-intensity clusters or the overlapping area, which could lead to many false positives for some outlier detection algorithms. Similar to the previous simulations, all other irrelevant factors are fixed, and we only list the relevant parts below. Again, we present realizations of synthetic data sets with uniform clusters in 2-dimensional space (although the simulation experiments are conducted on 3 and 10-dimensional space) in Figure \ref{fig:Demo_2d_Dis_Otl} (for illustration proposes). It is not hard to see that when the distance is equal to 1.5 (Figure \ref{fig:Demo_2d_Dis_Otl} (a)), the three clusters are highly overlapping, and separating them from each other is a challenging task.

\begin{itemize}
  \item[\romannumeral1.] The centers of clusters are: $\bm{\mu_1} = (\underbrace{3,...,3}_{d})$, $\bm{\mu_2} = (3+s,\underbrace{3,...,3}_{d-1})$, and $\bm{\mu_3} = (3,3+s,\underbrace{3,...,3}_{d-2})$, where $s$ could be 1.5, 2, 2.5, 3.0, 3.5, and 4 (the study of focus in this section);
  \label{Cls_Dis_Setting}
\end{itemize}

\begin{figure}[htb]
  \centering
  \subfigure[1.5]{
  \includegraphics[width=0.30\textwidth]{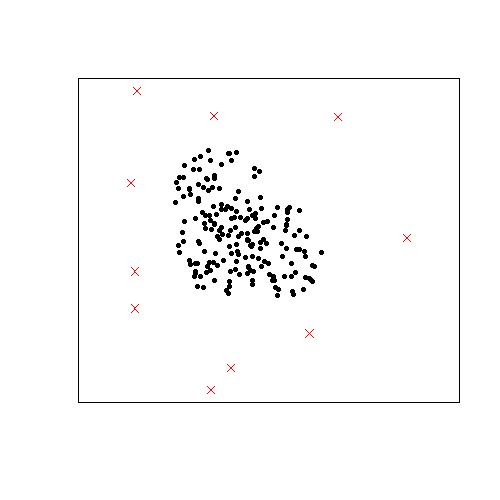}}
  \subfigure[2]{
  \includegraphics[width=0.30\textwidth]{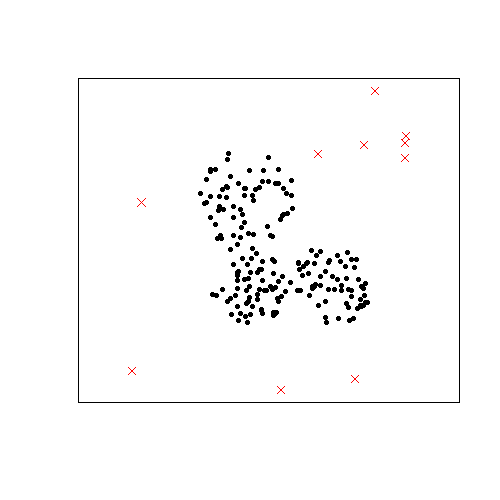}}
  \subfigure[2.5]{
  \includegraphics[width=0.30\textwidth]{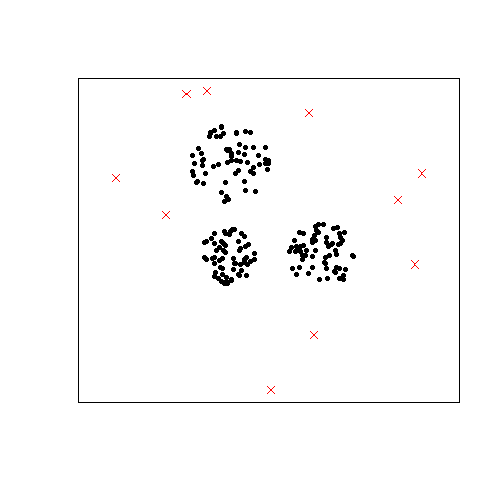}}
  \subfigure[3]{
  \includegraphics[width=0.30\textwidth]{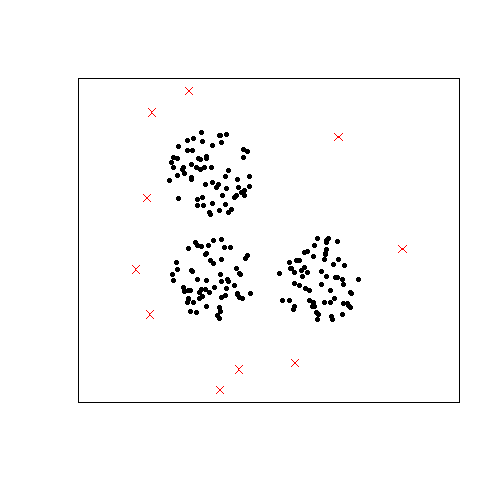}}
  \subfigure[3.5]{
  \includegraphics[width=0.30\textwidth]{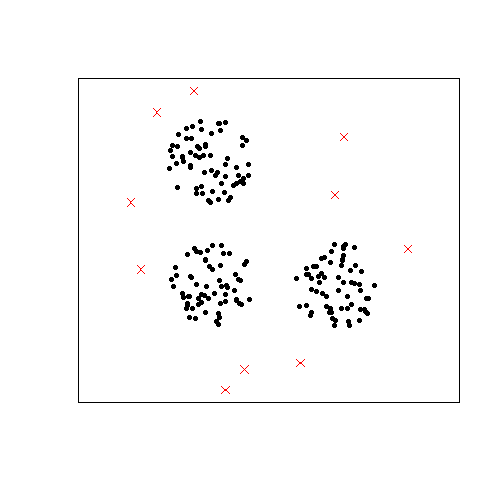}}
  \subfigure[4]{
  \includegraphics[width=0.30\textwidth]{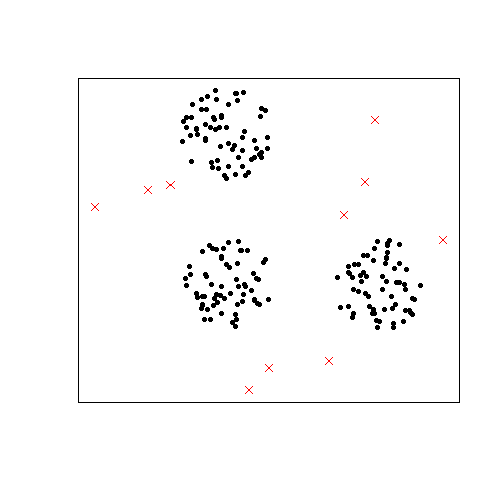}}
  \caption{Some realizations (with uniform clusters) of the simulation setting in Section \ref{Cls_Dis_Setting}, the distance between cluster centers increases from 1.5 to 4. Red crosses are outliers, black points are regular observations. The distance between clusters are indicated below the sub-figures.}
  \label{fig:Demo_2d_Dis_cls}
\end{figure}

We summarize the results we obtained from Tables \ref{tab:Cls_Dis_U_Cls1} to \ref{tab:Cls_Dis_G_Cls2}. The same as before, the BAs and $F_2$-scores (Tables \ref{tab:Cls_Dis_U_Cls1} and \ref{tab:Cls_Dis_G_Cls1}) are also presented as barplots in Figures \ref{fig:Barplot_NCluster_Dist} and \ref{fig:Barplot_GCluster_Dist}, respectively.

\begin{table}[htb]
  \centering
  \resizebox{0.7\columnwidth}{!}{\begin{tabular}{|c|c|c|c|c|c|c|c|c|c|c|c|c|c|}
    \hline
    \multicolumn{2}{|c|}{} & \multicolumn{12}{|c|}{Distances Between Cluster Centers} \\ \cline{3-14}

    \multicolumn{2}{|c|}{} & \multicolumn{2}{|c|}{$1.5$} & \multicolumn{2}{|c|}{$2$} & \multicolumn{2}{|c|}{$2.5$} & \multicolumn{2}{|c|}{$3$} & \multicolumn{2}{|c|}{$3.5$} & \multicolumn{2}{|c|}{$4$} \\ \cline{3-14}

    \multicolumn{2}{|c|}{} & TPR & TNR & TPR & TNR & TPR & TNR & TPR & TNR & TPR & TNR & TPR & TNR \\ \hline

    \multirow{4}*{$d=3$} & RU-MCCDs & 0.987 & 0.969 & 0.961 & 0.982 & 0.970 & 0.987 & 0.985 & 0.988 & 0.985 & 0.988 & 0.986 & 0.988 \\ \cline{2-14}
    & SU-MCCDs & 1.000 & 0.984 & 0.999 & 0.992 & 0.999 & 0.997 & 1.000 & 0.998 & 0.999 & 0.999 & 0.995 & 0.998 \\ \cline{2-14}
    & UN-MCCDs & 0.991 & 0.972 & 0.984 & 0.982 & 0.983 & 0.987 & 0.991 & 0.990 & 0.993 & 0.989 & 0.991 & 0.990 \\ \cline{2-14}
    & SUN-MCCDs & 0.999 & 0.985 & 0.998 & 0.991 & 0.999 & 0.996 & 0.998 & 0.998 & 0.998 & 0.998 & 0.997 & 0.998 \\ \cline{1-14}

    \multirow{4}*{$d=10$} & RU-MCCDs & 1.000 & 0.922 & 1.000 & 0.933 & 1.000 & 0.928 & 1.000 & 0.920 & 1.000 & 0.920 & 1.000 & 0.914 \\ \cline{2-14}
    & SU-MCCDs & 1.000 & 0.941 & 1.000 & 0.947 & 1.000 & 0.940 & 1.000 & 0.939 & 1.000 & 0.929 & 1.000 & 0.925 \\ \cline{2-14}
    & UN-MCCDs & 1.000 & 0.965 & 1.000 & 0.987 & 1.000 & 0.990 & 1.000 & 0.990 & 1.000 & 0.990 & 1.000 & 0.991 \\ \cline{2-14}
    & SUN-MCCDs & 1.000 & 0.982 & 1.000 & 0.994 & 1.000 & 0.998 & 1.000 & 0.999 & 1.000 & 0.999 & 1.000 & 0.999 \\ \cline{1-14}
  \end{tabular}}
  \caption{The TPRs and TNRs of the CCD-based algorithms as the distance between cluster centers increases from 1.5 to 4 (for simulations with uniform clusters).}
  \label{tab:Cls_Dis_U_Cls1}
\end{table}

\begin{table}[htb]
  \centering
  \resizebox{0.7\columnwidth}{!}{\begin{tabular}{|c|c|c|c|c|c|c|c|c|c|c|c|c|c|}
    \hline
    \multicolumn{2}{|c|}{} & \multicolumn{12}{|c|}{Distances Between Cluster Centers} \\ \cline{3-14}

    \multicolumn{2}{|c|}{} & \multicolumn{2}{|c|}{$1.5$} & \multicolumn{2}{|c|}{$2$} & \multicolumn{2}{|c|}{$2.5$} & \multicolumn{2}{|c|}{$3$} & \multicolumn{2}{|c|}{$3.5$} & \multicolumn{2}{|c|}{$4$} \\ \cline{3-14}

    \multicolumn{2}{|c|}{} & BA & $F_2$-score & BA & $F_2$-score & BA & $F_2$-score & BA & $F_2$-score & BA & $F_2$-score & BA & $F_2$-score \\ \hline

    \multirow{4}*{$d=3$} & RU-MCCDs & 0.978 & 0.885 & 0.972 & 0.906 & 0.979 & 0.930 & 0.987 & 0.945 & 0.987 & 0.945 & 0.987 & 0.946 \\ \cline{2-14}
    & SU-MCCDs & 0.992 & 0.943 & 0.996 & 0.970 & 0.998 & 0.988 & 0.999 & 0.992 & 0.999 & 0.995 & 0.997 & 0.988 \\ \cline{2-14}
    & UN-MCCDs & 0.982 & 0.897 & 0.983 & 0.924 & 0.985 & 0.940 & 0.991 & 0.956 & 0.991 & 0.954 & 0.991 & 0.956 \\ \cline{2-14}
    & SUN-MCCDs & 0.992 & 0.945 & 0.995 & 0.965 & 0.998 & 0.984 & 0.998 & 0.991 & 0.998 & 0.991 & 0.998 & 0.990 \\ \cline{1-14}

    \multirow{4}*{$d=10$} & RU-MCCDs & 0.961 & 0.771 & 0.967 & 0.797 & 0.964 & 0.785 & 0.960 & 0.767 & 0.960 & 0.767 & 0.957 & 0.754 \\ \cline{2-14}
    & SU-MCCDs & 0.971 & 0.817 & 0.974 & 0.832 & 0.970 & 0.814 & 0.970 & 0.812 & 0.965 & 0.788 & 0.963 & 0.778 \\ \cline{2-14}
    & UN-MCCDs & 0.983 & 0.883 & 0.994 & 0.953 & 0.995 & 0.963 & 0.995 & 0.963 & 0.995 & 0.963 & 0.996 & 0.967 \\ \cline{2-14}
    & SUN-MCCDs & 0.991 & 0.936 & 0.997 & 0.978 & 0.999 & 0.992 & 1.000 & 0.996 & 1.000 & 0.996 & 1.000 & 0.996 \\ \cline{1-14}
  \end{tabular}}
  \caption{The BAs and $F_2$-scores of the CCD-based algorithms as the distance between cluster centers increases from 1.5 to 4 (for simulations with uniform clusters).}
  \label{tab:Cls_Dis_U_Cls2}
\end{table}

\begin{figure}[htb]
\centering
\subfigure[The BAs when $d=3$.]{
\label{Barplot3_NCluster_Dist_BA}
\includegraphics[width=0.45\textwidth]{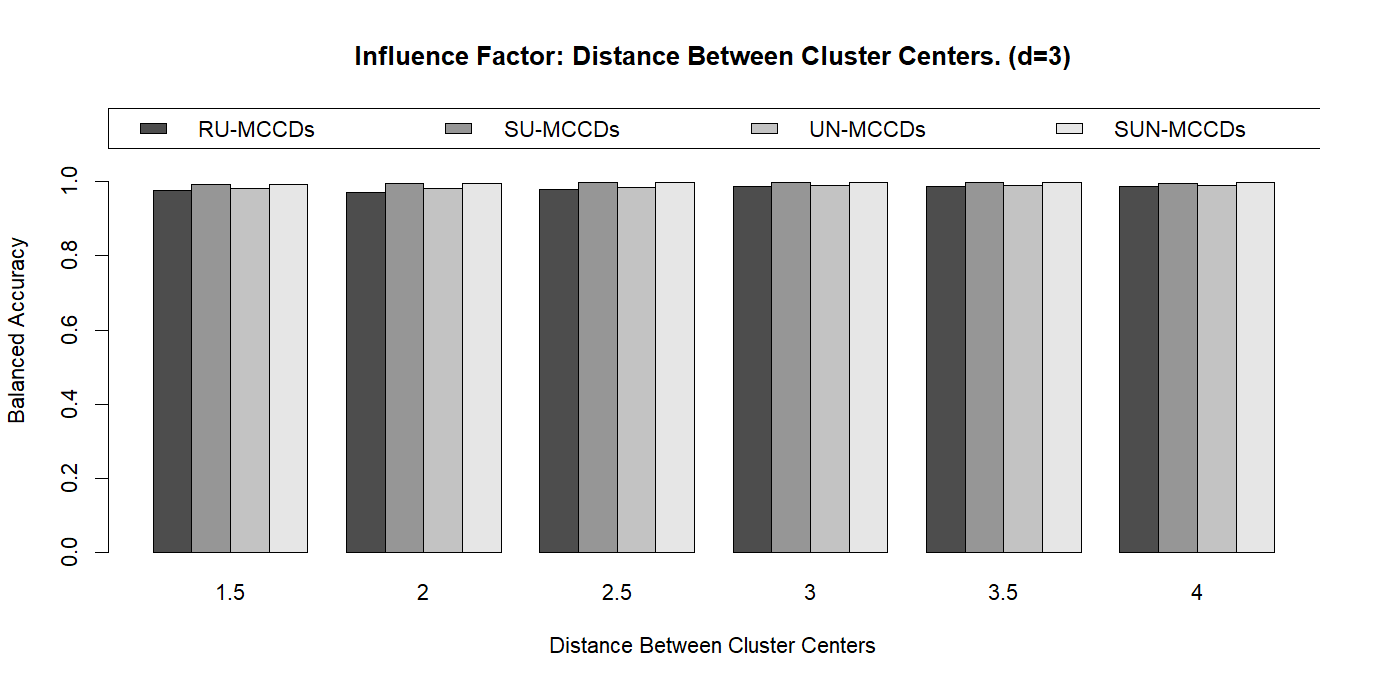}}
\subfigure[The $F_2$-scores when $d=3$.]{
\label{Barplot3_NCluster_Dist_FS}
\includegraphics[width=0.45\textwidth]{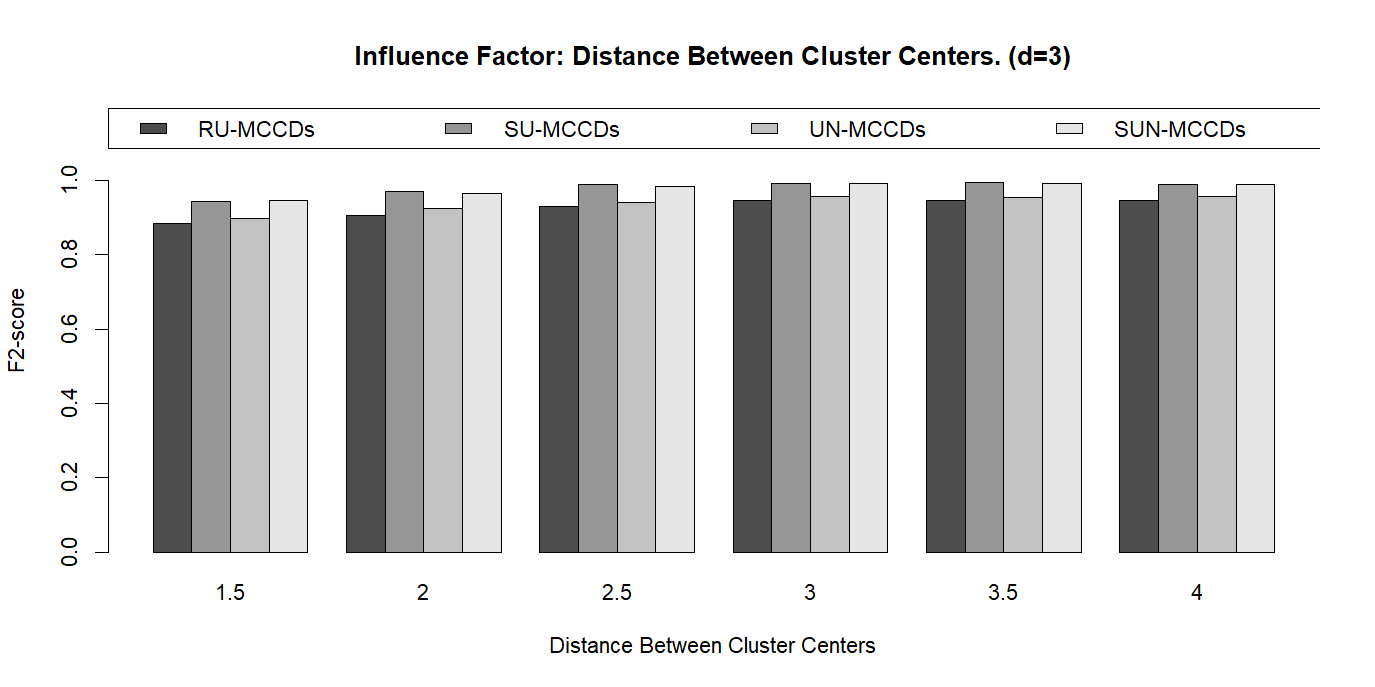}}

\subfigure[The BAs when $d=10$.]{
\label{Barplot10_NCluster_Dist_BA}
\includegraphics[width=0.45\textwidth]{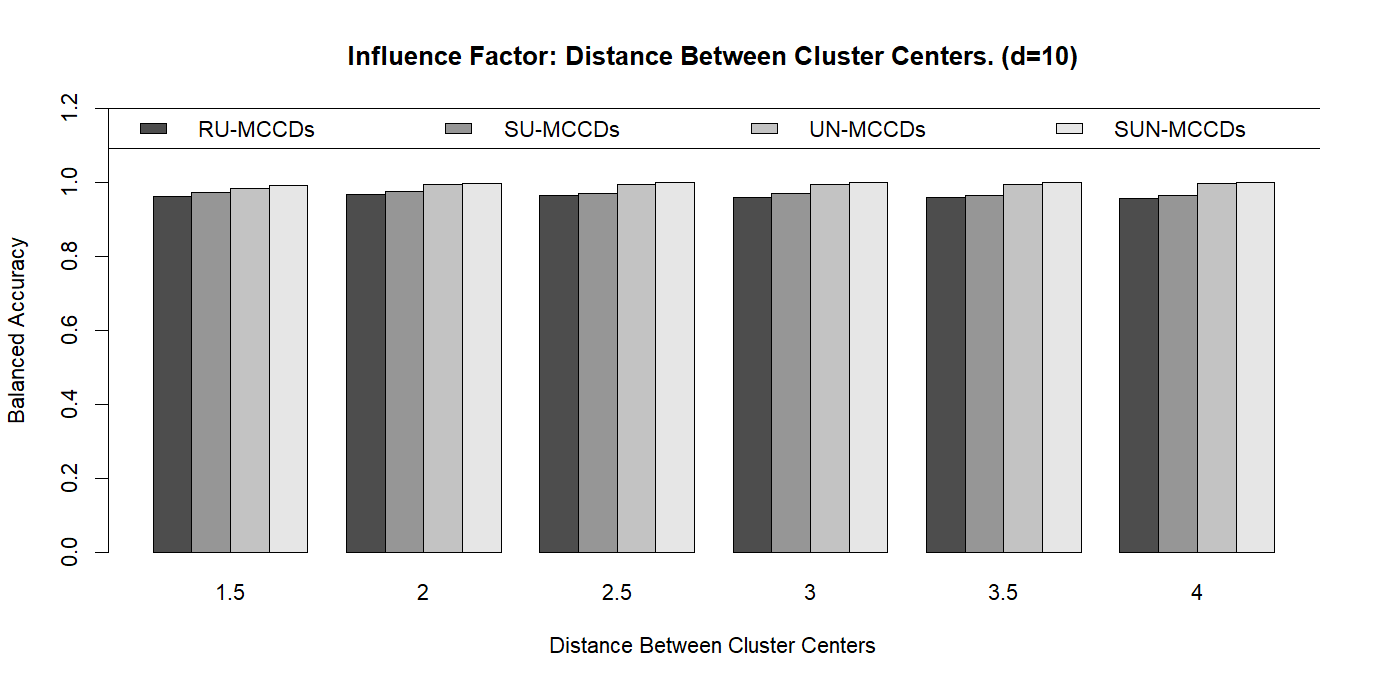}}
\subfigure[The $F_2$-scores when $d=10$.]{
\label{Barplot10_NCluster_Dist FS}
\includegraphics[width=0.45\textwidth]{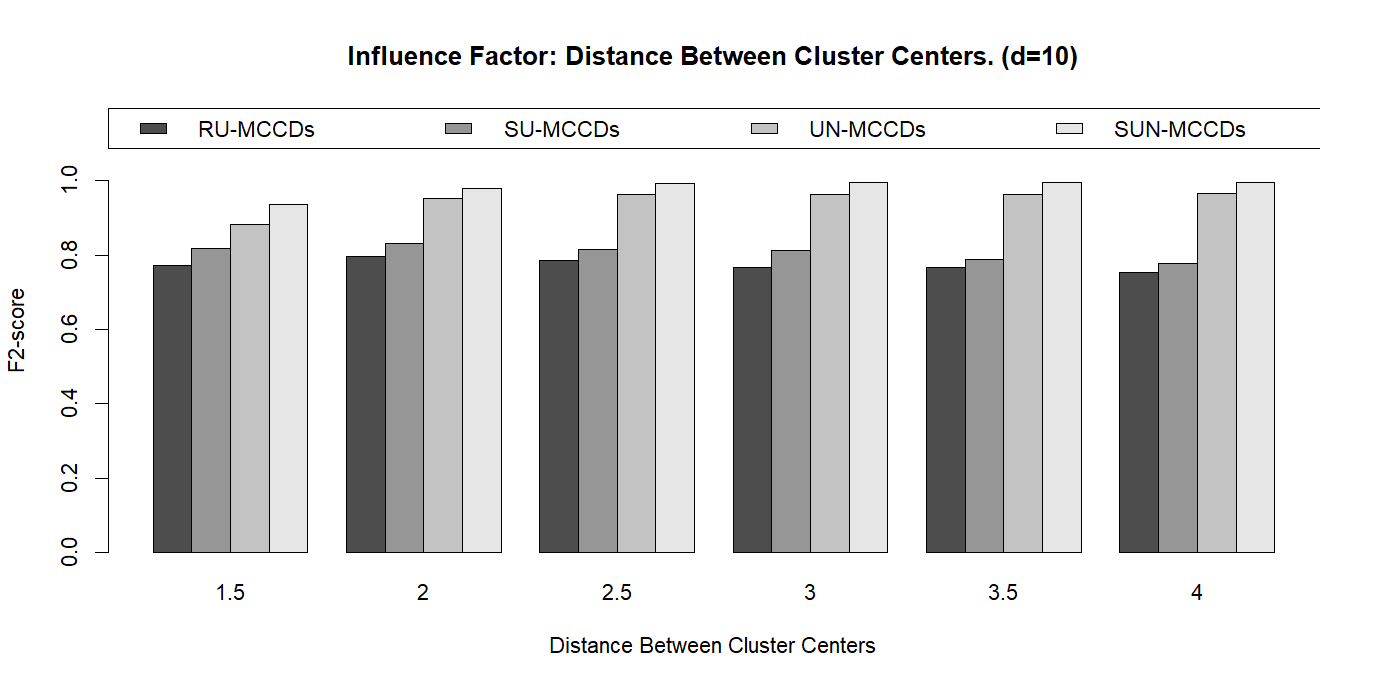}}
\caption{The barplots summarizing the performances of the CCD-based outlier detection algorithms as the distance between cluster centers increases (points within each clusters are uniformly distributed). (a) The BAs for $d=3$. (b) The $F_2$-scores for $d=3$. (c) The BAs for $d=10$. (d) The $F_2$-scores for $d=10$.}
\label{fig:Barplot_NCluster_Dist}
\end{figure}

\begin{table}[htb]
  \centering
  \resizebox{0.7\columnwidth}{!}{\begin{tabular}{|c|c|c|c|c|c|c|c|c|c|c|c|c|c|}
    \hline
    \multicolumn{2}{|c|}{} & \multicolumn{12}{|c|}{Distances Between Cluster Centers} \\ \cline{3-14}

    \multicolumn{2}{|c|}{} & \multicolumn{2}{|c|}{$1.5$} & \multicolumn{2}{|c|}{$2$} & \multicolumn{2}{|c|}{$2.5$} & \multicolumn{2}{|c|}{$3$} & \multicolumn{2}{|c|}{$3.5$} & \multicolumn{2}{|c|}{$4$} \\ \cline{3-14}

    \multicolumn{2}{|c|}{} & TPR & TNR & TPR & TNR & TPR & TNR & TPR & TNR & TPR & TNR & TPR & TNR \\ \hline

    \multirow{4}*{$d=3$} & RU-MCCDs & 1.000 & 0.849 & 1.000 & 0.835 & 1.000 & 0.835 & 1.000 & 0.833 & 1.000 & 0.834 & 1.000 & 0.836 \\ \cline{2-14}
    & SU-MCCDs & 1.000 & 0.948 & 1.000 & 0.942 & 1.000 & 0.941 & 1.000 & 0.944 & 1.000 & 0.945 & 1.000 & 0.941 \\ \cline{2-14}
    & UN-MCCDs & 0.961 & 0.901 & 0.983 & 0.896 & 0.988 & 0.895 & 0.997 & 0.890 & 0.999 & 0.895 & 0.999 & 0.894 \\ \cline{2-14}
    & SUN-MCCDs & 0.999 & 0.961 & 1.000 & 0.959 & 1.000 & 0.959 & 1.000 & 0.956 & 1.000 & 0.959 & 1.000 & 0.958 \\ \cline{1-14}

    \multirow{4}*{$d=10$} & RU-MCCDs & 1.000 & 0.687 & 1.000 & 0.691 & 1.000 & 0.698 & 1.000 & 0.697 & 1.000 & 0.690 & 1.000 & 0.690 \\ \cline{2-14}
    & SU-MCCDs & 1.000 & 0.769 & 1.000 & 0.771 & 1.000 & 0.777 & 1.000 & 0.775 & 1.000 & 0.780 & 1.000 & 0.773 \\ \cline{2-14}
    & UN-MCCDs & 1.000 & 0.817 & 1.000 & 0.827 & 1.000 & 0.829 & 1.000 & 0.829 & 1.000 & 0.829 & 1.000 & 0.829 \\ \cline{2-14}
    & SUN-MCCDs & 1.000 & 0.939 & 1.000 & 0.948 & 1.000 & 0.947 & 1.000 & 0.947 & 1.000 & 0.947 & 1.000 & 0.947 \\ \cline{1-14}
  \end{tabular}}
  \caption{The TPRs and TNRs of the CCD-based algorithms as the distance between cluster centers increases from 1.5 to 4 (for simulations with Gaussian clusters).}
  \label{tab:Cls_Dis_G_Cls1}
\end{table}

\begin{table}[htb]
  \centering
  \resizebox{0.7\columnwidth}{!}{\begin{tabular}{|c|c|c|c|c|c|c|c|c|c|c|c|c|c|}
    \hline
    \multicolumn{2}{|c|}{} & \multicolumn{12}{|c|}{Distances Between Cluster Centers} \\ \cline{3-14}

    \multicolumn{2}{|c|}{} & \multicolumn{2}{|c|}{$1.5$} & \multicolumn{2}{|c|}{$2$} & \multicolumn{2}{|c|}{$2.5$} & \multicolumn{2}{|c|}{$3$} & \multicolumn{2}{|c|}{$3.5$} & \multicolumn{2}{|c|}{$4$} \\ \cline{3-14}

    \multicolumn{2}{|c|}{} & BA & $F_2$-score & BA & $F_2$-score & BA & $F_2$-score & BA & $F_2$-score & BA & $F_2$-score & BA & $F_2$-score \\ \hline

    \multirow{4}*{$d=3$} & RU-MCCDs & 0.925 & 0.635 & 0.918 & 0.615 & 0.918 & 0.615 & 0.917 & 0.612 & 0.917 & 0.613 & 0.918 & 0.616 \\ \cline{2-14}
    & SU-MCCDs & 0.974 & 0.835 & 0.971 & 0.819 & 0.971 & 0.817 & 0.972 & 0.825 & 0.973 & 0.827 & 0.971 & 0.817 \\ \cline{2-14}
    & UN-MCCDs & 0.931 & 0.702 & 0.940 & 0.706 & 0.942 & 0.707 & 0.944 & 0.703 & 0.947 & 0.714 & 0.947 & 0.712 \\ \cline{2-14}
    & SUN-MCCDs & 0.980 & 0.870 & 0.980 & 0.865 & 0.980 & 0.865 & 0.978 & 0.857 & 0.980 & 0.865 & 0.979 & 0.862 \\ \cline{1-14}

    \multirow{4}*{$d=10$} & RU-MCCDs & 0.844 & 0.457 & 0.846 & 0.460 & 0.849 & 0.466 & 0.849 & 0.465 & 0.845 & 0.459 & 0.845 & 0.459 \\ \cline{2-14}
    & SU-MCCDs & 0.885 & 0.533 & 0.886 & 0.535 & 0.889 & 0.541 & 0.888 & 0.539 & 0.890 & 0.545 & 0.887 & 0.537 \\ \cline{2-14}
    & UN-MCCDs & 0.909 & 0.590 & 0.914 & 0.603 & 0.915 & 0.606 & 0.915 & 0.606 & 0.915 & 0.606 & 0.915 & 0.606 \\ \cline{2-14}
    & SUN-MCCDs & 0.970 & 0.812 & 0.974 & 0.835 & 0.974 & 0.832 & 0.974 & 0.832 & 0.974 & 0.832 & 0.974 & 0.832 \\ \cline{1-14}
  \end{tabular}}
  \caption{The BAs and $F_2$-scores of the CCD-based algorithms as the distance between cluster centers increases from 1.5 to 4 (for simulations with Gaussian clusters).}
  \label{tab:Cls_Dis_G_Cls2}
\end{table}

\begin{figure}[htb]
\centering
\subfigure[The BAs when $d=3$.]{
\label{Barplot3_GCluster_Dist_BA}
\includegraphics[width=0.45\textwidth]{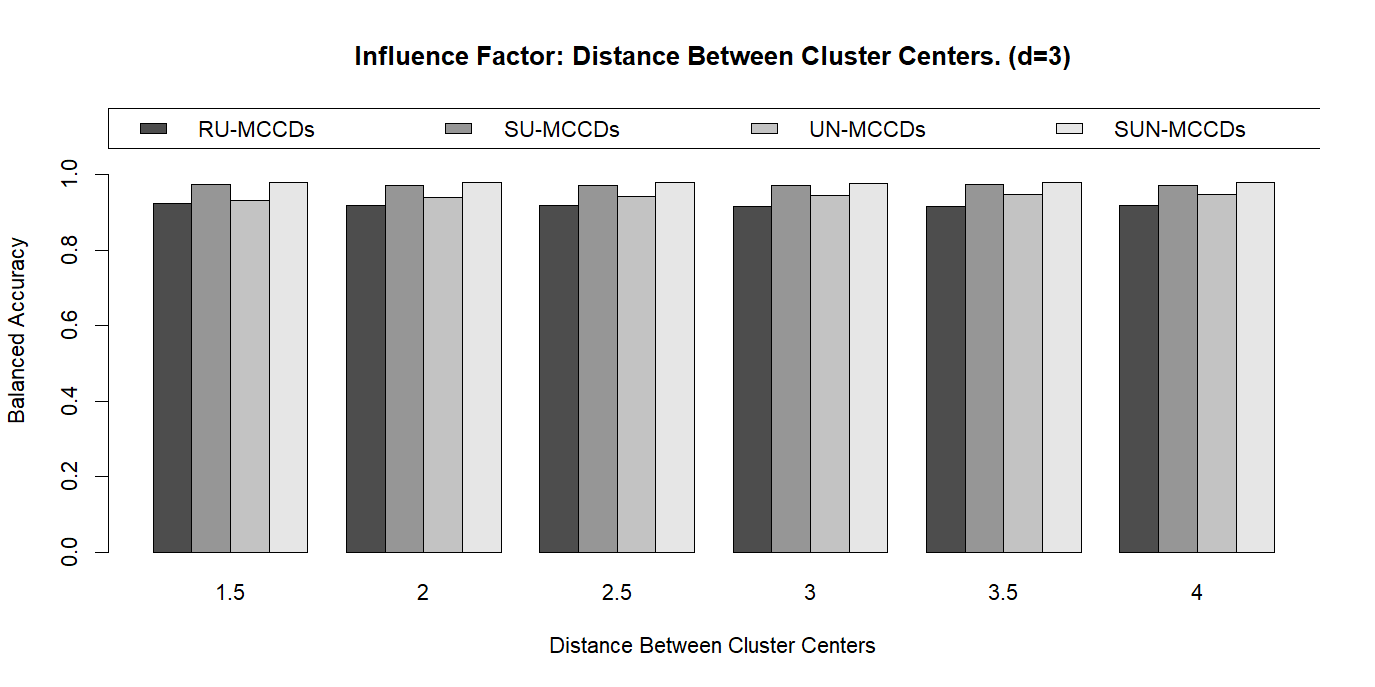}}
\subfigure[The $F_2$-scores when $d=3$.]{
\label{Barplot3_GCluster_Dist_FS}
\includegraphics[width=0.45\textwidth]{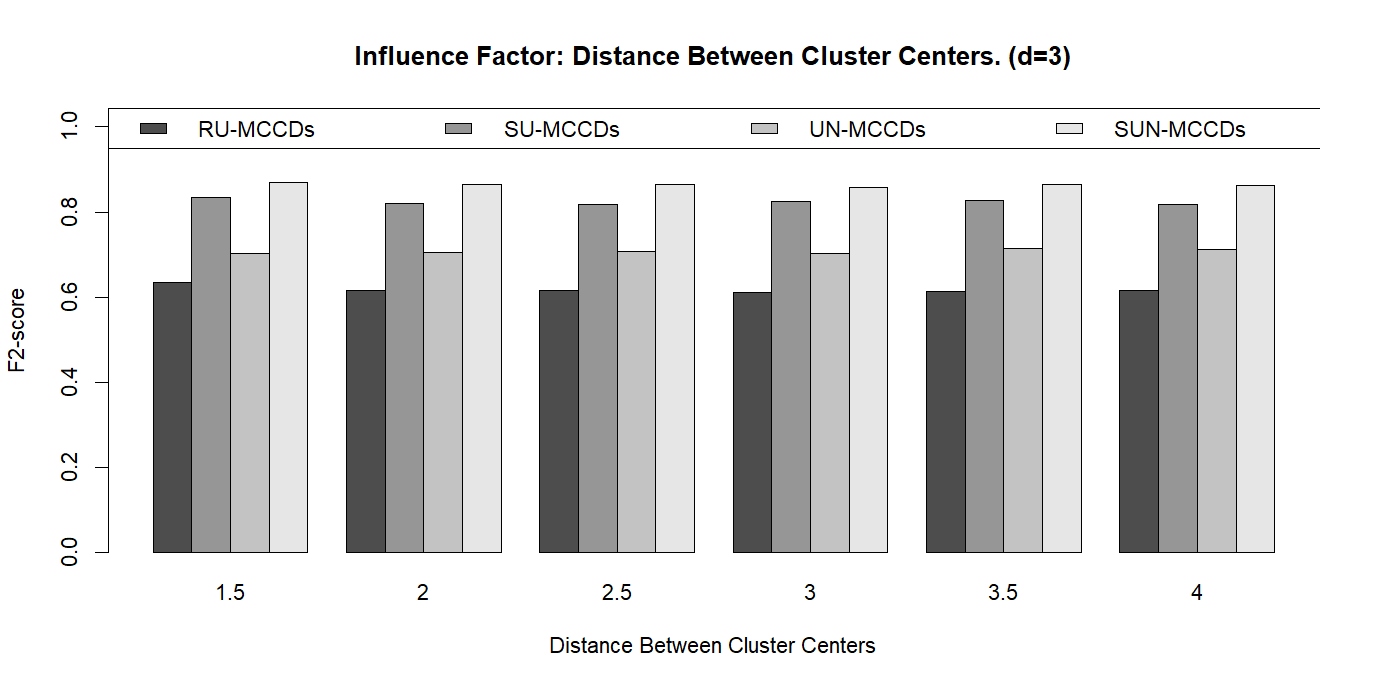}}

\subfigure[The BAs when $d=10$.]{
\label{Barplot10_GCluster_Dist_BA}
\includegraphics[width=0.45\textwidth]{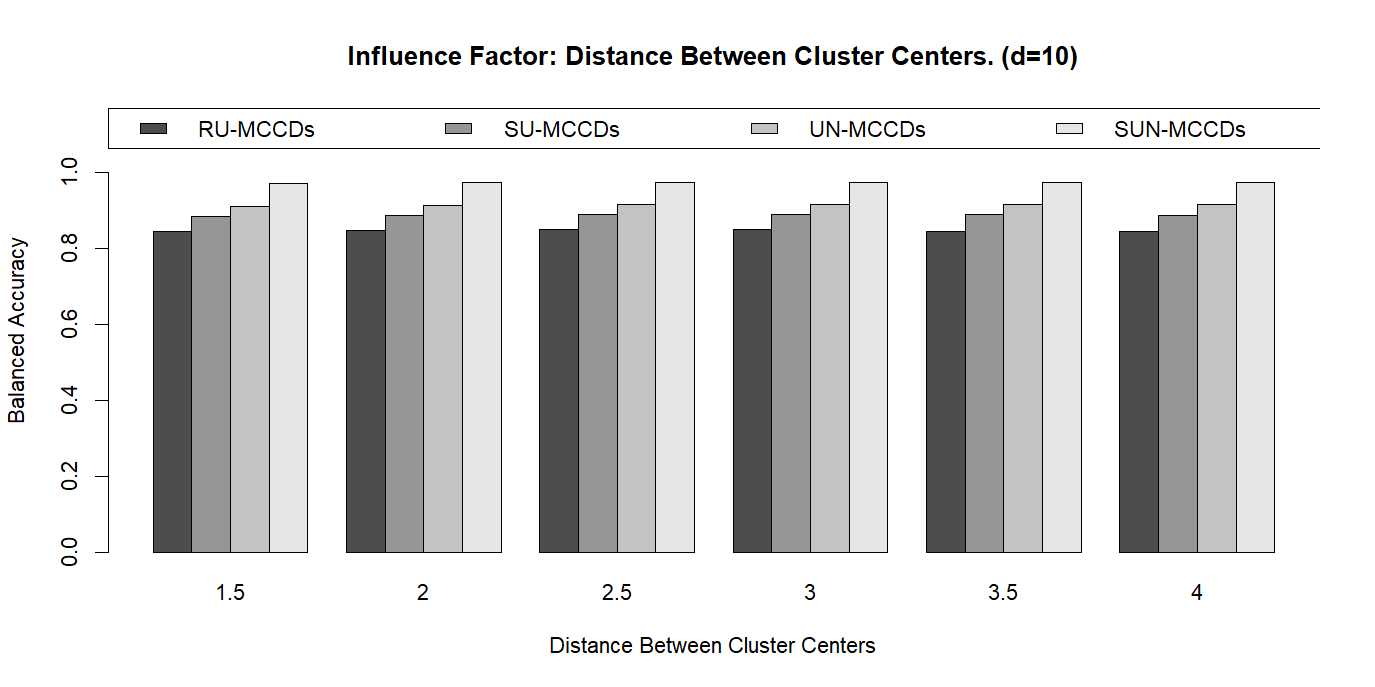}}
\subfigure[The $F_2$-scores when $d=10$.]{
\label{Barplot10_GCluster_Dist FS}
\includegraphics[width=0.45\textwidth]{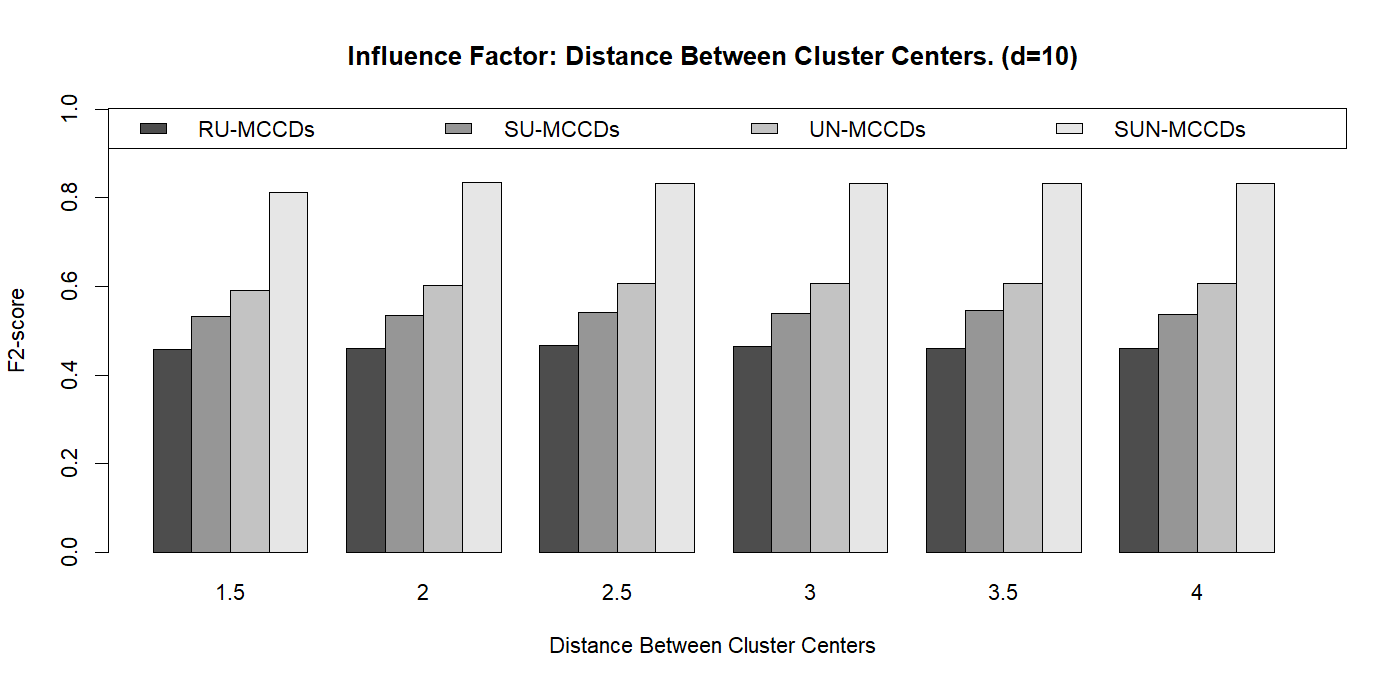}}
\caption{The barplots summarizing the performances of the CCD-based outlier detection algorithms as the distance between cluster centers increases (points within each clusters are (multivariate) normally distributed). (a) The BAs for $d=3$. (b) The $F_2$-scores for $d=3$. (c) The BAs for $d=10$. (d) The $F_2$-scores for $d=10$.}
\label{fig:Barplot_GCluster_Dist}
\end{figure}

Recall that when clusters with different intensities overlap (the inter-cluster center distance $s\leq2$), the challenges include identifying the outliers near overlapping clusters with different intensities and addressing the swapping problem. Thus, we compare the performance of these algorithms under the simulation settings when $s\leq2$.

Firstly, we explore the simulations with uniform clusters. When $d=3$, all four algorithms address the two challenges effectively. The SU-MCCD and SUN-MCCD algorithms exhibit stable behavior regardless of the cluster distances. However, the TPRs and TNRs of the RU-MCCD and UN-MCCD algorithms are slightly lower when $s\leq2$, compared to the other cases where clusters are distinct. When increasing the number of dimensions to 10, all the algorithms become insensitive to cluster distances, even when clusters overlap. For example, the $F_2$-scores of the RU-MCCD algorithm are stable (0.771, 0.797, 0.785, 0.767, 0.767, and 0.754), although they lag behind other algorithms.

Then, we consider the simulation settings with Gaussian clusters. It is interesting to see that the cluster distance has minimal influence on the performance, no matter how close the simulated clusters are. This could be explained as follows: the two challenges we discussed at the beginning of this section exist for Gaussian clusters even when they do not overlap because outlier and regular points can be close due to the wide span of Gaussian clusters. Similar to the previous simulations, the two ``flexible" algorithms perform better than the others when $d=3$, and the SUN-MCCD algorithms deliver the best results and outperform other algorithms by a large gap when $d=10$.

\subsubsection{Varying the Noise Level of Gaussian Clusters}

The second last factor to study is the noise level for Gaussian clusters. Therefore, this simulations are conducted only on data sets with Gaussian clusters. In the previous study, ``noise" is defined as the points close to the clusters, typically exhibiting much lower vicinity intensity than the observations deep in the clusters. In the previous work, we constructed the support with a radius randomly chosen between 0.7 and 1.3 for a Gaussian cluster. We tune the covariance such that approximately $1\%$ of the regular observations fell beyond the desired support and were thus perceived as noise. In other words, each support is a $99^{th}$ percentile contour of an uncorrelated Gaussian density. In the current setting, without changing the range of the radii, we conduct simulations with the noise level increasing from $1\%$ to $10\%$. Different noise levels can be achieved by adjusting the scale of the covariance matrix. Once the radius of the support is known, the desired scale can be obtained via a $\chi^2_d$ distribution. All the other factors remain consistent with previous simulations. Some realizations in a 2-dimensional space are presented in Figure \ref{fig:Demo_2d_Noise}. Observe that the Gaussian clusters have a wider span as the noise level increases, and the noise and outliers get much closer. Therefore, we expect the severity level of the swamping problem to rise incrementally, and we are particularly interested in the behaviour of all four CCD-based algorithms under these conditions.

\begin{itemize}
  \item[\romannumeral1.] We only conduct the simulations with Gaussian clusters since we study the noise level in this section.
  \item[\romannumeral2.] The noise level of each Gaussian cluster is set to $1\%$, $3\%$, $5\%$, $7\%$, and $10\%$ (the study of focus in this section).
  \label{sec:Noise_Setting}
\end{itemize}

\begin{figure}[htb]
  \centering
  \subfigure[$1\%$]{
  \includegraphics[width=0.30\textwidth]{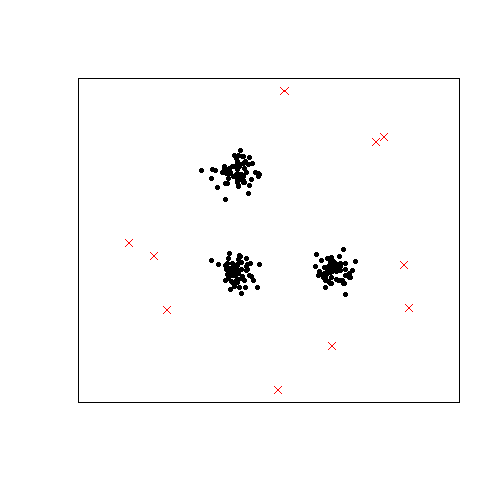}}
  \subfigure[$3\%$]{
  \includegraphics[width=0.30\textwidth]{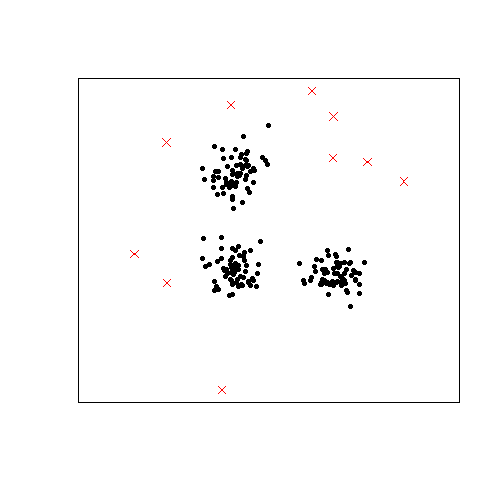}}
  \subfigure[$5\%$]{
  \includegraphics[width=0.30\textwidth]{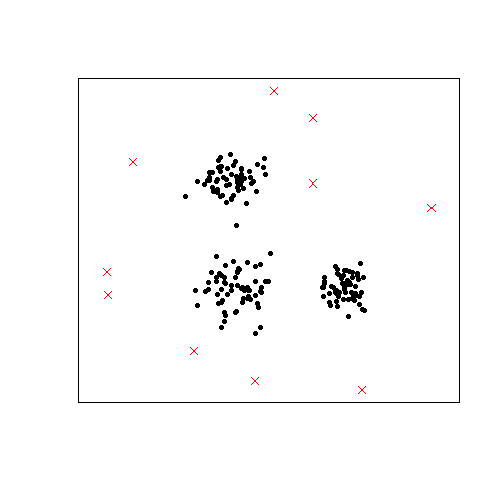}}
  \subfigure[$7\%$]{
  \includegraphics[width=0.30\textwidth]{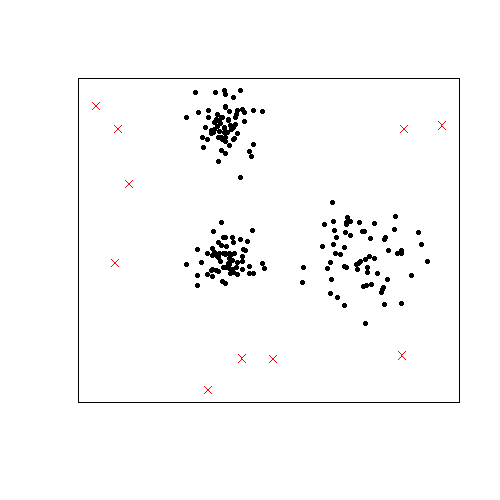}}
  \subfigure[$10\%$]{
  \includegraphics[width=0.30\textwidth]{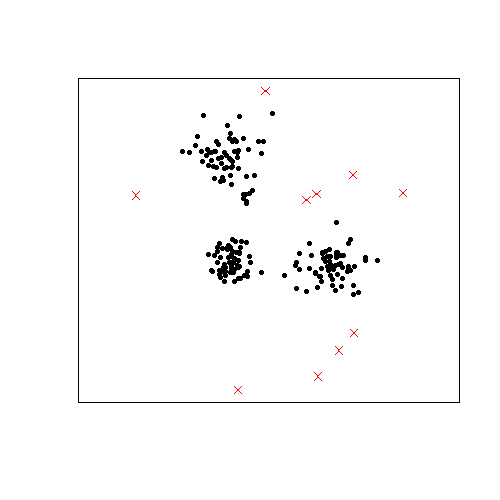}}
  \caption{Some realizations of the simulation setting in Section \ref{sec:Noise_Setting}, the noise level of Gaussian cluster centers increases from $1\%$ to $10\%$. Red crosses are outliers, black points are regular observations. The noise levels are indicated below each sub-figure.}
  \label{fig:Demo_2d_Noise}
\end{figure}

The results obtained from this simulation setting are summarized in Tables \ref{tab:Noise_G_Cls1} and \ref{tab:Noise_G_Cls2}. The BAs and $F_2$-scores, which can be found in Table \ref{tab:Noise_G_Cls2}, are also represented as a barplot in Figure \ref{fig:Barplot_GNoise}.

\begin{table}[htb]
  \centering
  \resizebox{0.7\columnwidth}{!}{\begin{tabular}{|c|c|c|c|c|c|c|c|c|c|c|c|}
    \hline
    \multicolumn{2}{|c|}{} & \multicolumn{10}{|c|}{Level of Noise} \\ \cline{3-12}

    \multicolumn{2}{|c|}{} & \multicolumn{2}{|c|}{$1\%$} & \multicolumn{2}{|c|}{$3\%$} & \multicolumn{2}{|c|}{$5\%$} & \multicolumn{2}{|c|}{$7\%$} & \multicolumn{2}{|c|}{$10\%$} \\ \cline{3-12}

    \multicolumn{2}{|c|}{} & TPR & TNR & TPR & TNR & TPR & TNR & TPR & TNR & TPR & TNR \\ \hline

    \multirow{4}*{$d=3$} & RU-MCCDs & 1.000 & 0.833 & 1.000 & 0.833 & 1.000 & 0.833 & 1.000 & 0.833 & 1.000 & 0.833 \\ \cline{2-12}
    & SU-MCCDs & 1.000 & 0.941 & 1.000 & 0.941 & 1.000 & 0.941 & 1.000 & 0.941 & 1.000 & 0.941 \\ \cline{2-12}
    & UN-MCCDs & 0.997 & 0.890 & 0.998 & 0.890 & 0.997 & 0.890 & 0.997 & 0.890 & 0.998 & 0.890 \\ \cline{2-12}
    & SUN-MCCDs & 1.000 & 0.957 & 0.999 & 0.957 & 0.998 & 0.957 & 0.997 & 0.957 & 0.994 & 0.957 \\ \cline{1-12}

    \multirow{4}*{$d=10$} & RU-MCCDs & 1.000 & 0.697 & 1.000 & 0.697 & 1.000 & 0.697 & 1.000 & 0.698 & 1.000 & 0.698 \\ \cline{2-12}
    & SU-MCCDs & 1.000 & 0.777 & 1.000 & 0.777 & 1.000 & 0.776 & 1.000 & 0.777 & 1.000 & 0.777 \\ \cline{2-12}
    & UN-MCCDs & 1.000 & 0.829 & 1.000 & 0.829 & 1.000 & 0.829 & 1.000 & 0.829 & 1.000 & 0.829 \\ \cline{2-12}
    & SUN-MCCDs & 1.000 & 0.948 & 1.000 & 0.948 & 1.000 & 0.948 & 1.000 & 0.948 & 1.000 & 0.948 \\ \cline{1-12}
  \end{tabular}}
  \caption{The TPRs and TNRs of the CCD-based algorithms as the approximate noise level of each Gaussian cluster increases from $1\%$ to $10\%$.}
  \label{tab:Noise_G_Cls1}
\end{table}

\begin{table}[htb]
  \centering
  \resizebox{0.7\columnwidth}{!}{\begin{tabular}{|c|c|c|c|c|c|c|c|c|c|c|c|}
    \hline
    \multicolumn{2}{|c|}{} & \multicolumn{10}{|c|}{Level of Noise} \\ \cline{3-12}

    \multicolumn{2}{|c|}{} & \multicolumn{2}{|c|}{$1\%$} & \multicolumn{2}{|c|}{$3\%$} & \multicolumn{2}{|c|}{$5\%$} & \multicolumn{2}{|c|}{$7\%$} & \multicolumn{2}{|c|}{$10\%$} \\ \cline{3-12}

    \multicolumn{2}{|c|}{} & BA & $F_2$-score & BA & $F_2$-score & BA & $F_2$-score & BA & $F_2$-score & BA & $F_2$-score \\ \hline

    \multirow{4}*{$d=3$} & RU-MCCDs & 0.917 & 0.612 & 0.917 & 0.612 & 0.917 & 0.612 & 0.917 & 0.612 & 0.917 & 0.612 \\ \cline{2-12}
    & SU-MCCDs & 0.971 & 0.817 & 0.971 & 0.817 & 0.971 & 0.817 & 0.971 & 0.817 & 0.971 & 0.817 \\ \cline{2-12}
    & UN-MCCDs & 0.944 & 0.703 & 0.944 & 0.704 & 0.944 & 0.703 & 0.944 & 0.703 & 0.944 & 0.704 \\ \cline{2-12}
    & SUN-MCCDs & 0.979 & 0.860 & 0.978 & 0.859 & 0.978 & 0.858 & 0.977 & 0.857 & 0.976 & 0.855 \\ \cline{1-12}

    \multirow{4}*{$d=10$} & RU-MCCDs & 0.849 & 0.465 & 0.849 & 0.465 & 0.849 & 0.465 & 0.849 & 0.466 & 0.849 & 0.466 \\ \cline{2-12}
    & SU-MCCDs & 0.889 & 0.541 & 0.889 & 0.541 & 0.888 & 0.540 & 0.889 & 0.541 & 0.889 & 0.541 \\ \cline{2-12}
    & UN-MCCDs & 0.915 & 0.606 & 0.915 & 0.606 & 0.915 & 0.606 & 0.915 & 0.606 & 0.915 & 0.606 \\ \cline{2-12}
    & SUN-MCCDs & 0.974 & 0.835 & 0.974 & 0.835 & 0.974 & 0.835 & 0.974 & 0.835 & 0.974 & 0.835 \\ \cline{1-12}
  \end{tabular}}
  \caption{The BAs and $F_2$-scores of the CCD-based algorithms as the approximate noise level of each Gaussian cluster increases from $1\%$ to $10\%$.}
  \label{tab:Noise_G_Cls2}
\end{table}

\begin{figure}[htb]
\centering
\subfigure[The BAs when $d=3$.]{
\label{Barplot3_GNoise_BA}
\includegraphics[width=0.45\textwidth]{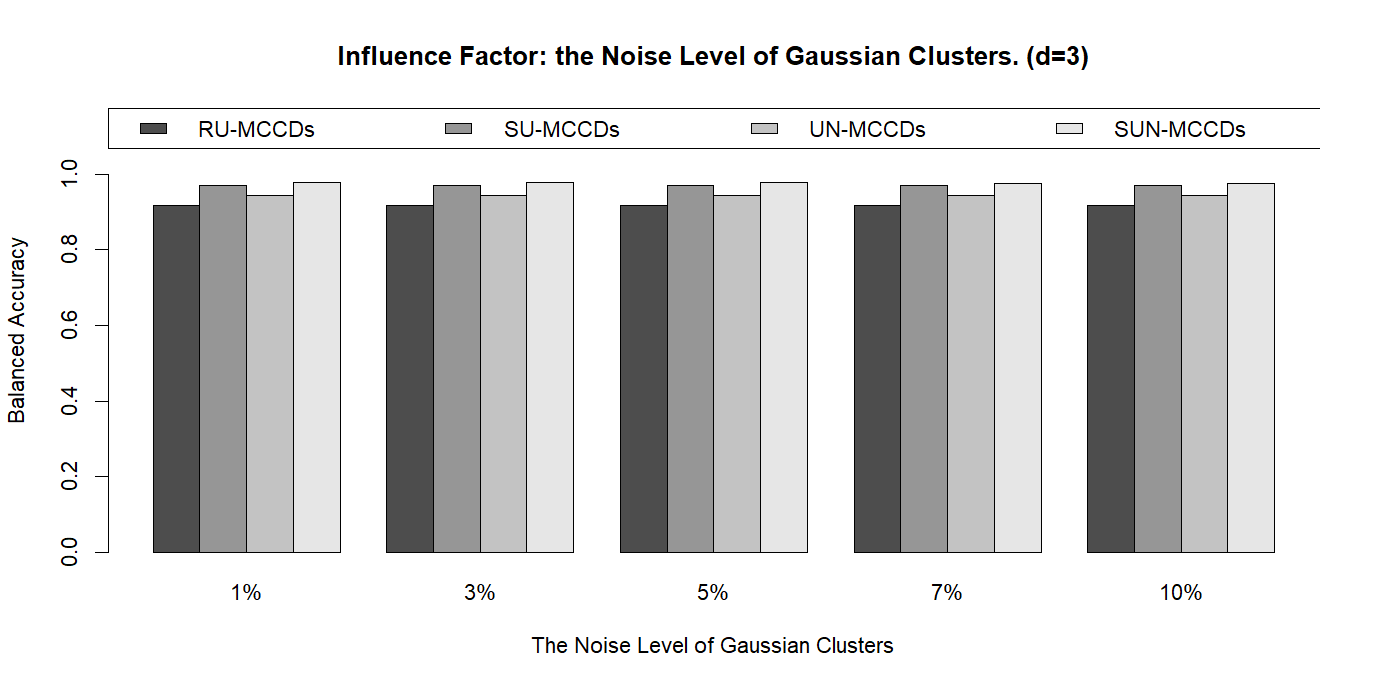}}
\subfigure[The $F_2$-scores when $d=3$.]{
\label{Barplot3_GNoise_FS}
\includegraphics[width=0.45\textwidth]{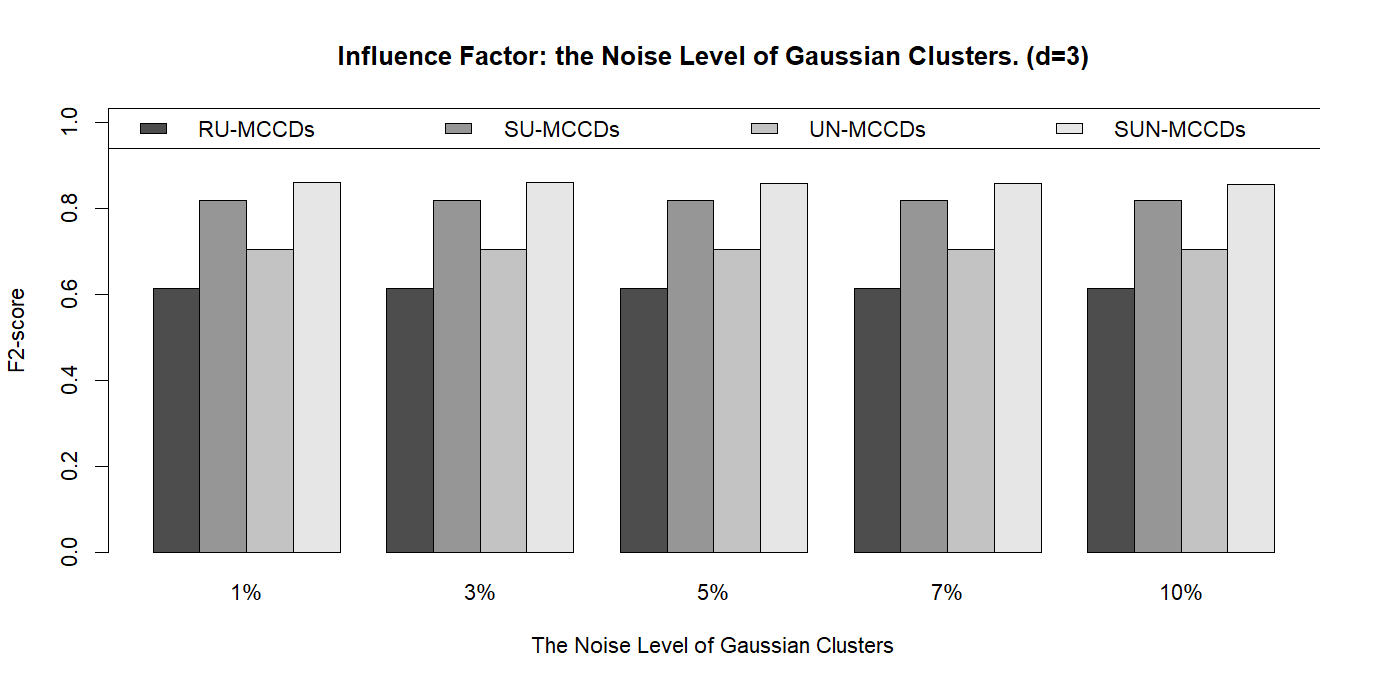}}

\subfigure[The BAs when $d=10$.]{
\label{Barplot10_GNoise_BA}
\includegraphics[width=0.45\textwidth]{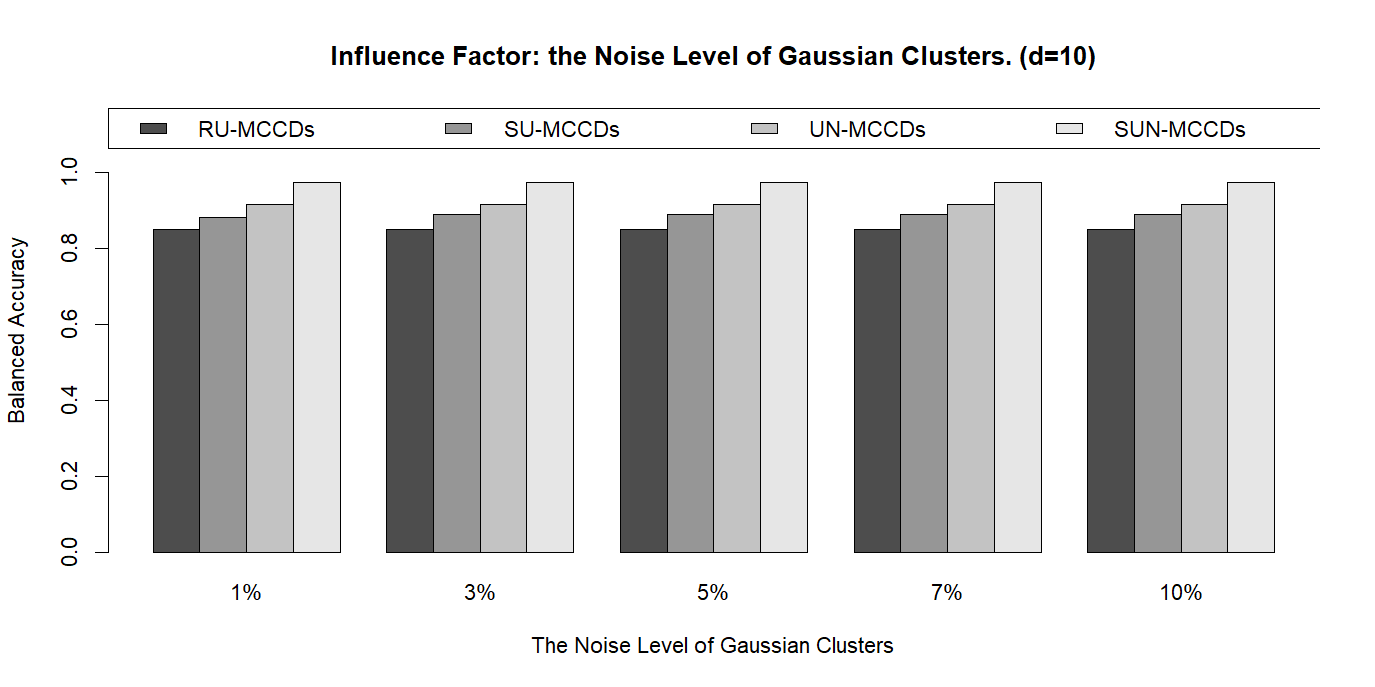}}
\subfigure[The $F_2$-scores when $d=10$.]{
\label{Barplot10_GNoise FS}
\includegraphics[width=0.45\textwidth]{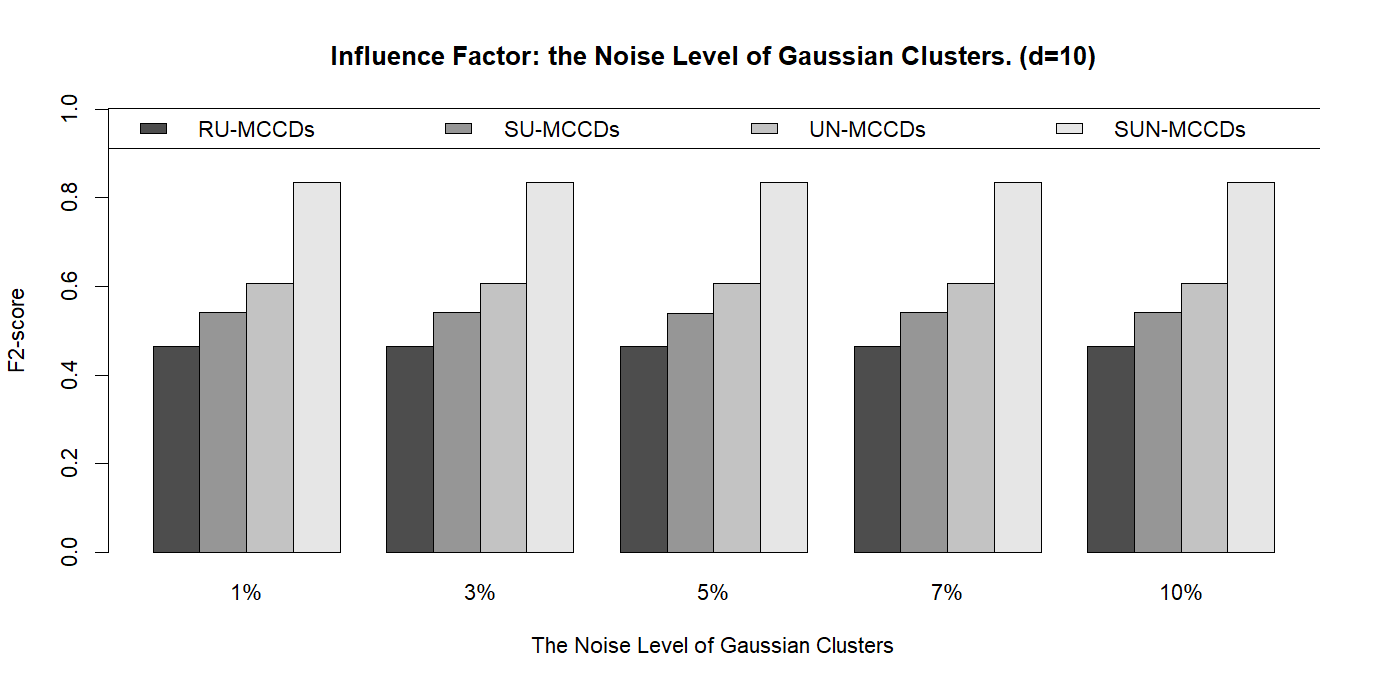}}
\caption{The barplots summarizing the performances of the CCD-based outlier detection algorithms as the approximate noise level increases (points within each clusters are (multivariate) normally distributed). (a) The BAs for $d=3$. (b) The $F_2$-scores for $d=3$. (c) The BAs for $d=10$. (d) The $F_2$-scores for $d=10$.}
\label{fig:Barplot_GNoise}
\end{figure}

Observe that all four CCD-based algorithms perform stably, regardless of the noise level. For instance, when $d=3$, the $F_2$-scores of the SUN-MCCD algorithm are 0.860, 0.859, 0.858, 0.857, and 0.855, presenting a slight downtrend, it suggests that all the algorithms are highly adaptable to the span of Gaussian clusters and their distances to outliers. This phenomenon can be attributed to a similar reason discussed in Section \ref{sec:Olt_Dis_Setting}. Notably, the TPRs of all the algorithms are 1 or close to 1, while the TNRs are substantially lower, particularly when $d=10$. Therefore, all the CCD-based algorithms isolate outliers from regular observations at the expense of some false positives, and this mechanism dynamically adapts to the scale of the covariance matrix of a Gaussian cluster. Moreover, the four algorithms achieve different levels of TNRs, with the SUN-MCCD algorithm performing the best and the RU-MCCD algorithm comparatively inferior (the worst).

\subsubsection{Collective Outliers in Convex Hull}

In all the previous simulation settings, the outliers are scattered around the ground truth clusters as they are drawn from a large hypersphere of radius 5. That said, most outliers are isolates far from one another, except when the contamination level is exceptionally high (we investigated the cases when the contamination level is as high as $15\%$ in Section \ref{sec:Cont_Setting}). In this section, we study the scenarios when outliers form a small group, called collective outliers. We want to explore the robustness of all the CCD-based algorithms to the mask problem, which usually emerges when collective outliers exist. Therefore, in the artificial data sets of this section, outliers are generated within a hypersphere of radius 1. To add more challenges, the hypersphere covering outliers is located inside the convex hull of regular points, with the distance between the hypersphere and cluster centers varying. We conduct the simulations with only uniform clusters to ensure all the outliers are within the convex hull. Simulation details are as follows. Similarly, only the different factors (compared to the first focus study in Section \ref{sec:N_Cls_Setting}) are presented.

\begin{itemize}
  \item[\romannumeral1.] Number of clusters: 2;
  \item[\romannumeral2.] The centers of clusters are: $\bm{\mu_1} = (3,\underbrace{3,...,3}_{d-1})$ and $\bm{\mu_2} = (9,\underbrace{3,...,3}_{d-1})$ (where $d=3,10$);
  \item[\romannumeral3.] The outlier set $C_{outlier}$ is generated uniformly within a hypersphere of radius 1. The center of the hypersphere is $\bm{\mu_0} = (3+s,\underbrace{3,...,3}_{d-1})$, where $s$ represents the distance of it to the first cluster center, and it is set to $1.5$, $2$, $2.5$, and $3$, respectively. When $s \leq 2$, the outlier set and the first cluster overlap, and there is no minimal distance between outliers and any cluster centers.
  \label{Convex_Hull}
\end{itemize}

Figure \ref{fig:Demo_2d_Noise} illustrates some realizations of the data set in a 2-dimensional space. Apparently, in the first two sub-figures where $s \leq 2$, the support of the outlier set and the left cluster overlap, and separating them is challenging.

\begin{figure}[htb]
  \centering
  \subfigure[$s$: 1.5]{
  \includegraphics[width=0.23\textwidth]{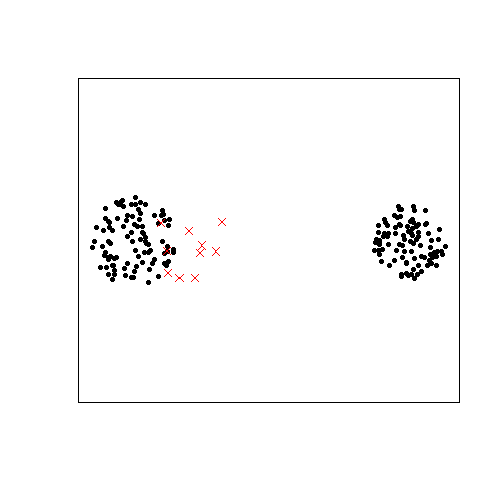}}
  \subfigure[$s$: 2]{
  \includegraphics[width=0.23\textwidth]{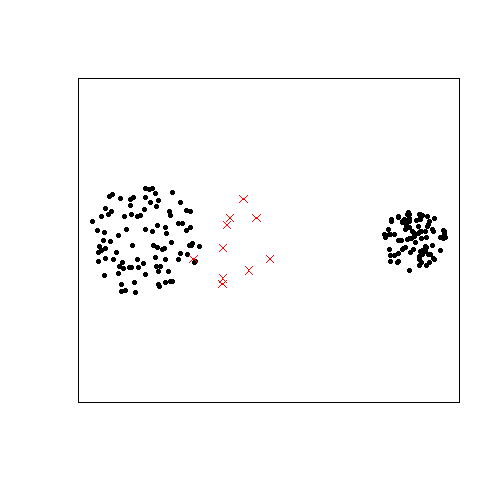}}
  \subfigure[$s$: 2.5]{
  \includegraphics[width=0.23\textwidth]{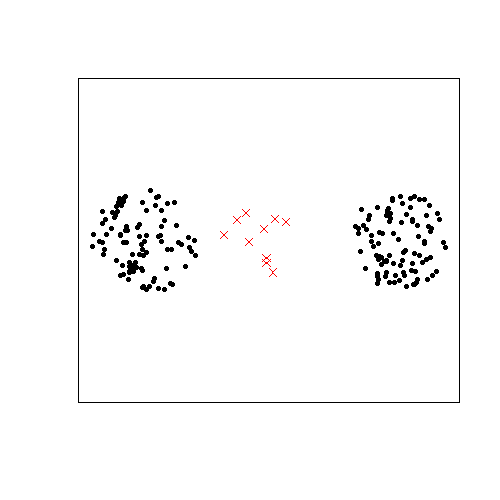}}
  \subfigure[$s$: 3]{
  \includegraphics[width=0.23\textwidth]{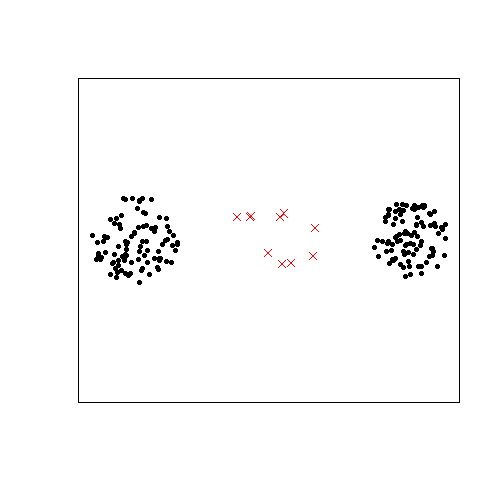}}
  \caption{Some realizations of the simulation setting containing collective outliers, where $s$ (indicated below each sub-figure) represents the distance between the left cluster center and the outlier center, and it increases from 1.5 to 3. Red crosses are outliers, black points are regular observations. All the outliers are within the convex hull of regular points.}
  \label{Demo_Collective_Outliers}
\end{figure}

The simulation results are summarized in Tables \ref{tab:Collective_N_Cls1} and \ref{tab:Collective_N_Cls2}. Similarly, the BAs and $F_2$-scores are also represented as a barplot in Figure \ref{fig:Barplot_Collective}.

\begin{table}[htb]
  \centering
  \resizebox{0.7\columnwidth}{!}{\begin{tabular}{|c|c|c|c|c|c|c|c|c|c|}
    \hline
    \multicolumn{2}{|c|}{} & \multicolumn{8}{|c|}{Distance} \\ \cline{3-10}

    \multicolumn{2}{|c|}{} & \multicolumn{2}{|c|}{$1.5$} & \multicolumn{2}{|c|}{$2$} & \multicolumn{2}{|c|}{$2.5$} & \multicolumn{2}{|c|}{$3$} \\ \cline{3-10}

    \multicolumn{2}{|c|}{} & TPR & TNR & TPR & TNR & TPR & TNR & TPR & TNR \\ \hline

    \multirow{4}*{$d=3$} & RU-MCCDs & 0.756 & 0.993 & 0.943 & 0.992 & 0.998 & 0.992 & 1.000 & 0.992 \\ \cline{2-10}
    & SU-MCCDs & 0.639 & 0.999 & 0.873 & 0.999 & 0.985 & 0.999 & 1.000 & 0.999 \\ \cline{2-10}
    & UN-MCCDs & 0.740 & 0.993 & 0.931 & 0.993 & 0.997 & 0.992 & 1.000 & 0.992 \\ \cline{2-10}
    & SUN-MCCDs & 0.604 & 0.999 & 0.837 & 0.999 & 0.971 & 0.999 & 0.999 & 0.999 \\ \cline{1-10}

    \multirow{4}*{$d=10$} & RU-MCCDs & 0.730 & 0.977 & 0.900 & 0.977 & 0.996 & 0.977 & 1.000 & 0.977 \\ \cline{2-10}
    & SU-MCCDs & 0.710 & 0.991 & 0.904 & 0.991 & 0.997 & 0.991 & 1.000 & 0.991 \\ \cline{2-10}
    & UN-MCCDs & 0.695 & 0.995 & 0.880 & 0.994 & 0.998 & 0.994 & 1.000 & 0.994 \\ \cline{2-10}
    & SUN-MCCDs & 0.608 & 0.999 & 0.837 & 0.999 & 0.992 & 0.999 & 1.000 & 0.999 \\ \cline{1-10}
  \end{tabular}}
  \caption{The TPRs and TNRs of the CCD-based algorithms as the distance between the collective outlier center and one of the cluster centers increases from $1.5$ to $2$.}
  \label{tab:Collective_N_Cls1}
\end{table}

\begin{table}[htb]
  \centering
  \resizebox{0.7\columnwidth}{!}{\begin{tabular}{|c|c|c|c|c|c|c|c|c|c|}
    \hline
    \multicolumn{2}{|c|}{} & \multicolumn{8}{|c|}{Distance} \\ \cline{3-10}

    \multicolumn{2}{|c|}{} & \multicolumn{2}{|c|}{$1.5$} & \multicolumn{2}{|c|}{$2$} & \multicolumn{2}{|c|}{$2.5$} & \multicolumn{2}{|c|}{$3$} \\ \cline{3-10}

    \multicolumn{2}{|c|}{} & BA & $F_2$-score & BA & $F_2$-score & BA & $F_2$-score & BA & $F_2$-score \\ \hline

    \multirow{4}*{$d=3$} & RU-MCCDs & 0.875 & 0.773 & 0.968 & 0.926 & 0.995 & 0.969 & 0.996 & 0.970 \\ \cline{2-10}
    & SU-MCCDs & 0.819 & 0.686 & 0.936 & 0.892 & 0.992 & 0.984 & 1.000 & 0.996 \\ \cline{2-10}
    & UN-MCCDs & 0.867 & 0.759 & 0.962 & 0.919 & 0.995 & 0.968 & 0.996 & 0.970 \\ \cline{2-10}
    & SUN-MCCDs & 0.802 & 0.653 & 0.918 & 0.862 & 0.985 & 0.973 & 0.999 & 0.995 \\ \cline{1-10}

    \multirow{4}*{$d=10$} & RU-MCCDs & 0.854 & 0.706 & 0.939 & 0.843 & 0.997 & 0.917 & 0.999 & 0.920 \\ \cline{2-10}
    & SU-MCCDs & 0.851 & 0.727 & 0.948 & 0.891 & 0.994 & 0.965 & 0.996 & 0.967 \\ \cline{2-10}
    & UN-MCCDs & 0.845 & 0.725 & 0.937 & 0.881 & 0.996 & 0.976 & 0.997 & 0.978 \\ \cline{2-10}
    & SUN-MCCDs & 0.804 & 0.657 & 0.918 & 0.862 & 0.996 & 0.990 & 1.000 & 0.996 \\ \cline{1-10}
  \end{tabular}}
  \caption{The BAs and $F_2$-scores of the CCD-based algorithms as the distance between the collective outlier center and one of the cluster centers increases from $1.5$ to $2$.}
  \label{tab:Collective_N_Cls2}
\end{table}

\begin{figure}[htb]
\centering
\subfigure[The BAs when $d=3$.]{
\label{Barplot3_Collective_BA}
\includegraphics[width=0.45\textwidth]{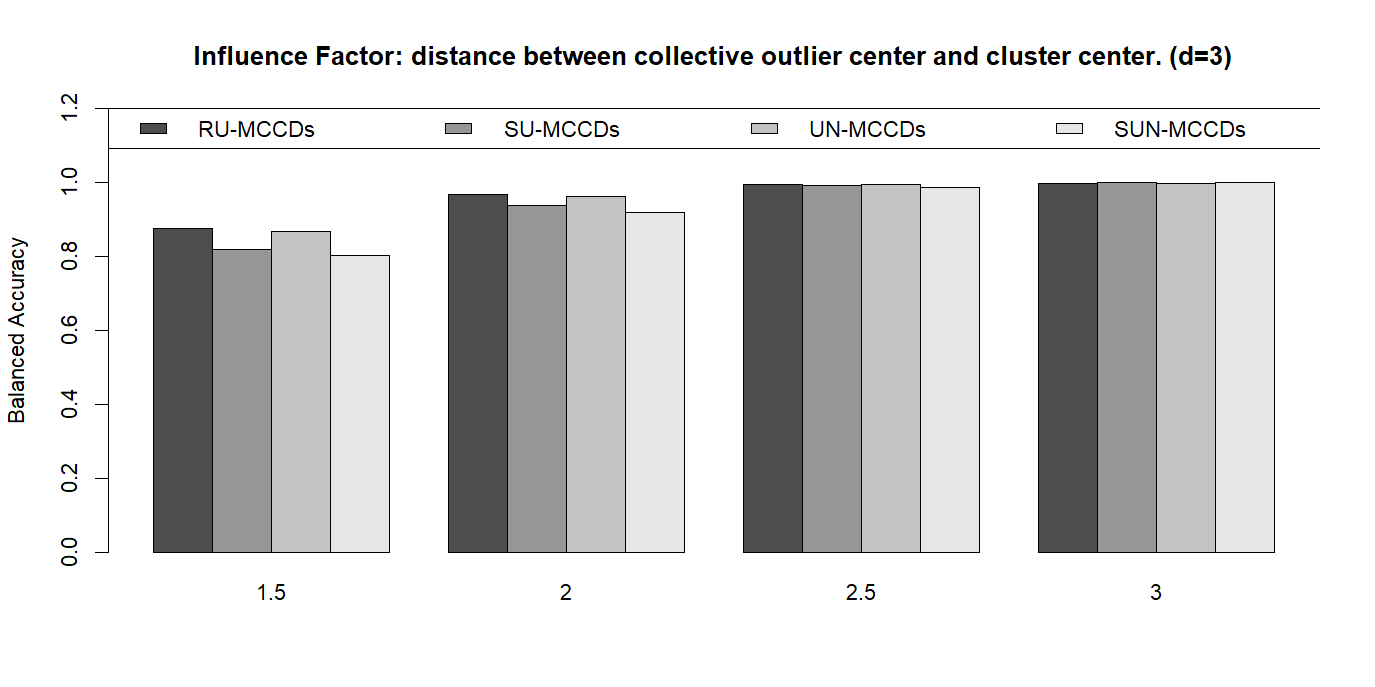}}
\subfigure[The $F_2$-scores when $d=3$.]{
\label{Barplot3_Collective_FS}
\includegraphics[width=0.45\textwidth]{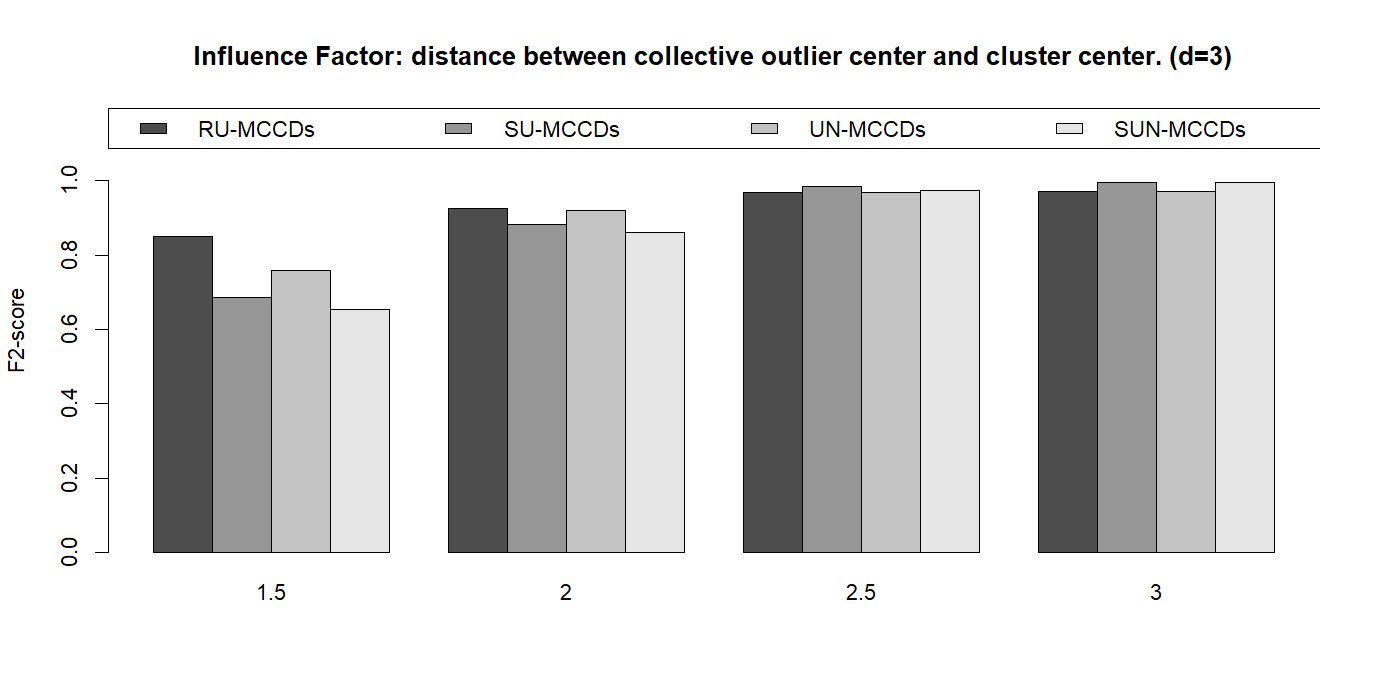}}

\subfigure[The BAs when $d=10$.]{
\label{Barplot10_Collective_BA}
\includegraphics[width=0.45\textwidth]{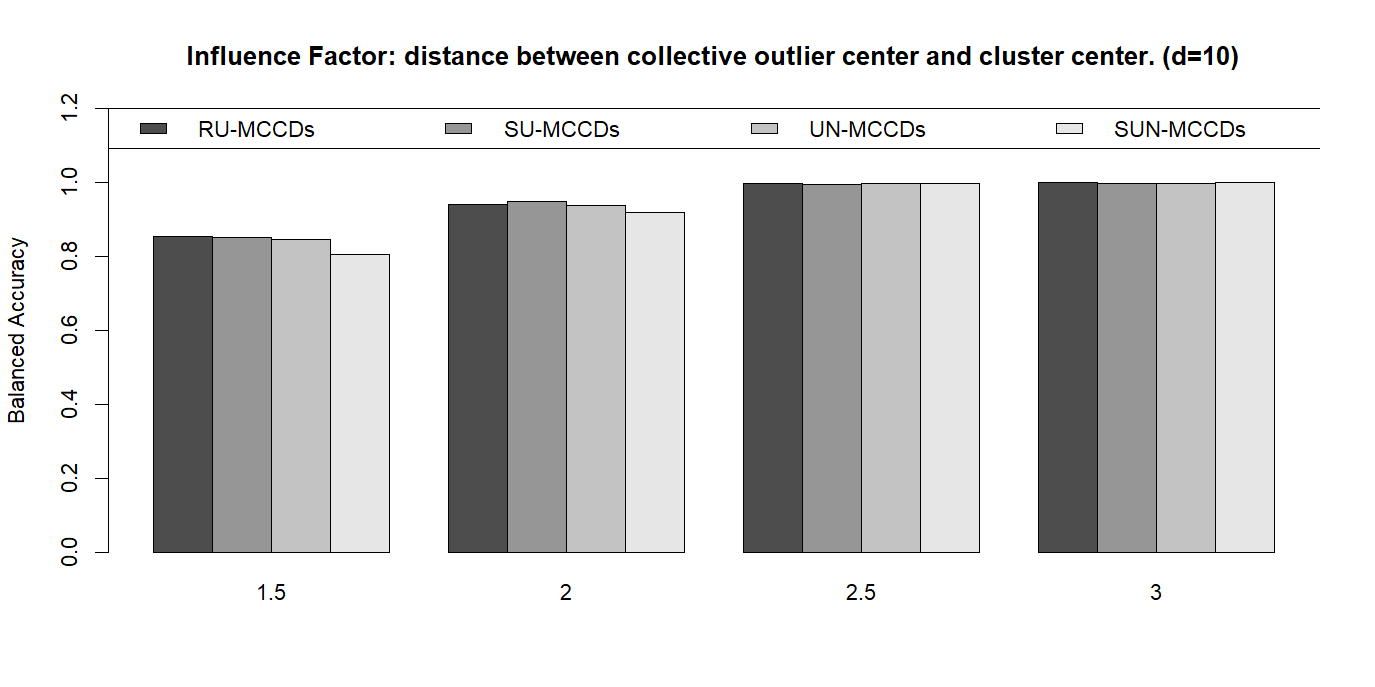}}
\subfigure[The $F_2$-scores when $d=10$.]{
\label{Barplot10_Collective_FS}
\includegraphics[width=0.45\textwidth]{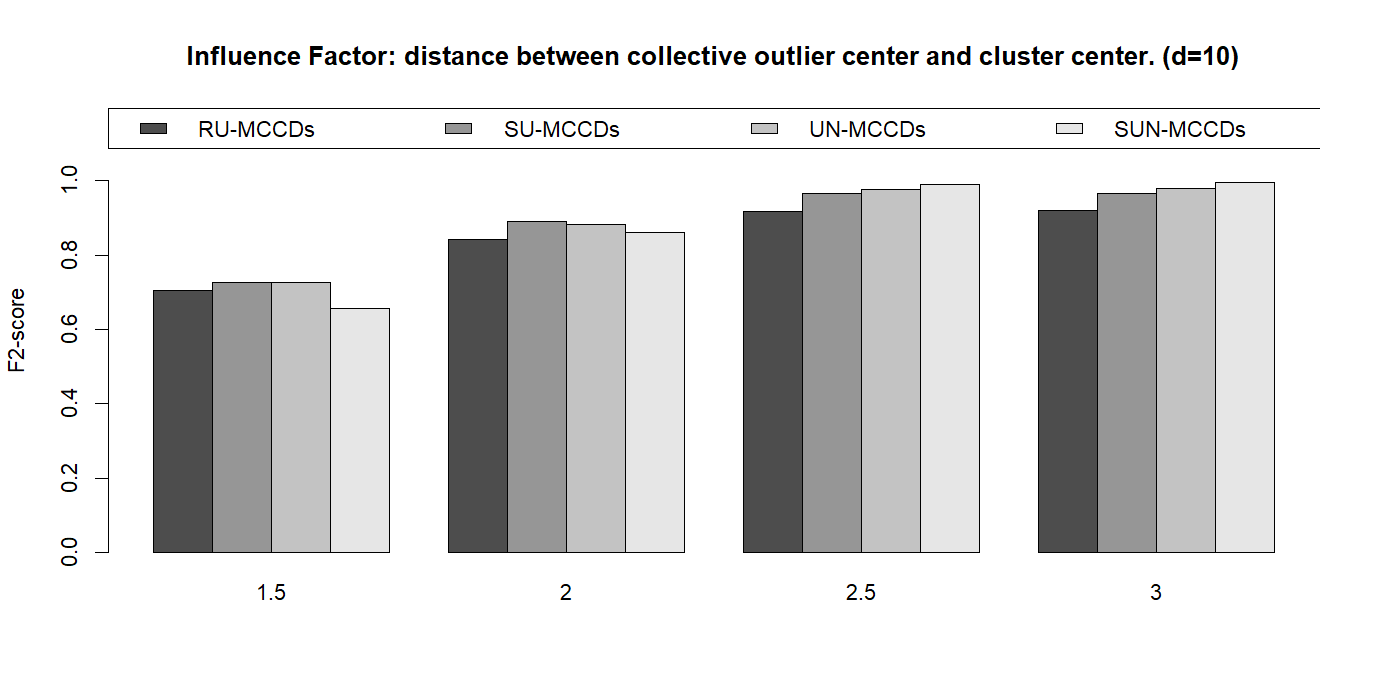}}
\caption{The barplots summarizing the performances of the CCD-based outlier detection algorithms as the approximate distance between the center of collective outliers and the center of one of the clusters increases from 1.5 to 3 (points within each clusters are normally distributed). (a) The BAs for $d=3$. (b) The $F_2$-scores for $d=3$. (c) The BAs for $d=10$. (d) The $F_2$-scores for $d=10$.}
\label{fig:Barplot_Collective}
\end{figure}

When $s \leq 2$, the simulation results show that the RU-MCCD and UN-MCCD algorithms perform comparably and are superior to the other two ``flexible" algorithms when $d=3$ or $d=10$. For example, when $d=3$ and $s=1.5$ or $s=2$, the $F_2$-scores of the RU-MCCD algorithm are 0.850 and 0.926, substantially higher than 0.686 and 0.892 delivered by the SU-MCCD algorithm. The reason is that the SU-MCCD and SUN-MCCD algorithms use multiple covering balls for each cluster. Thus, the chance of capturing the outliers close to regular points is much higher, yielding more false negatives. When $s>2$, the outlier set and regular points are well separate, and all four algorithms deliver similar performance and handle the collective outliers well with high $F_2$-scores (at least 0.9). Generally, the two ``flexible" algorithms perform slightly better in these cases.

\section{Monte Carlo Experiments Under Random Cluster Process}
\label{sec:Flex_Simul}
In the previous sections, we conducted Monte Carlo experiments to evaluate the performance of each proposed outlier detection algorithm. The UN-MCCD algorithm delivers comparable or better performance compared to the RU-MCCD algorithm when the dimensionality $d$ is small ($d \leq 5 $) and superior when $d=10$ and $20$. The conclusion is similar when comparing the SU-MCCD and SUN-MCCD algorithms. Additionally, the two ``shape-adaptive" algorithms outperform their ``vanilla versions" under the simulation cases with Gaussian clusters (except the simulation settings when $d \geq 50$), especially the SUN-MCCD algorithm, which outperforms other CCD-based algorithms when $d=5,\ 10$, and $20$.

However, the previous simulation settings (including the general simulation settings in Section \ref{sec:Simul_General} and the focus simulation settings in Section \ref{sec:Simul_Focus}) are relatively simplistic as the cluster centers are fixed. Additionally, the sizes of data sets, the number of clusters, the inter-cluster distances, the contamination levels, etc., are also fixed values under each simulation setting. In order to evaluate the CCD-based algorithms we proposed thoroughly and compare them with existing outlier detection algorithms, we conduct additional Monte Carlo experiments with more flexible settings.

Unlike previous simulation settings with levels of factors predetermined (e.g., $n=50, 100, ..., 500$, number of cluster$=2,3,4$.), converting those factors to random variables is a good solution towards our objective. To approach this goal, we try to simulate the \emph{Neyman-Scott cluster process} \cite{neyman1958statistical}, a class of cluster generation mechanisms with great randomness used widely in general practice. The realization of a general Neyman-Scott cluster process consists of two major steps, which are described as follows \cite{baddeley2023robust},

\begin{itemize}
\item[(1)] Firstly, a point set $\mathbf{S} = \{s_1,s_2,...,s_m\}$ is generated from an HPP with intensity parameter $\kappa>0$, these points are called ``parents". In the second step, each cluster is generated around one of the parents.

\item[(2)] A finite set/cluster $C_i=\{y_{i1},y_{i2},...,y_{in_i}\}$ is generated around each $s_i \in \mathbf{S}$, the size of $C_i$ (i.e., $n_i$) follows a Poisson distribution with mean $\mu$. The set of points $\{y_{i1},y_{i2},...,y_{in_i}\}$ are generated \emph{i.i.d} from the following probability density function, which depends on the distances (or similarities) to their parent $s_i$ \cite{baddeley2023robust},
    \begin{equation}\label{Neyman-Scott}
      P(x|s_i) = \frac{1}{\sigma^2} h \left( \frac{||x-s_i||}{\sigma} \right),
    \end{equation}
    where $||x-s_i||$ represents a distance measure between $x$ and $s_i$, $\sigma$ is a scale parameter, and $h$ is called the \emph{kernel} function of the Neyman-Scott cluster process. The points generated for cluster $C_i$ are also called the ``offspring" or the ``children" of $s_i$.

    Finally, the union of all offspring points $\cup_{s_i \in \mathbf{S}} C_i$ is a realization of a general Neyman-Scott cluster process, and the parent point set $\mathbf{S}$ will be dropped from the simulated data sets eventually.
\end{itemize}

One of the advantages of using the Neyman-Scott cluster process is the randomness of the intensity, location, and number of clusters. To simulate a general Neyman-Scott cluster process, $\kappa$, $\mu$, and the kernel function $h$ need to be specified. We shall consider two standard models, the \emph{Mat\'{e}rn cluster process} \cite{Matrn1960SpatialV} and the \emph{Thomas cluster process} \cite{thomas1949generalization, diggle1976statistical}. They only differ on the kernel function $h$.\\

\noindent \textbf{The Mat\'{e}rn and Thomas Cluster Processes}
\begin{itemize}
\item[(1)] The kernel of Mat\'{e}rn cluster process is $h(x)=\frac{1}{\pi} \mathbf{1}\{||x|| \leq 1\}$, i.e., a uniform density on a unit disc. The scale parameter $\sigma$ of Equation \eqref{Neyman-Scott} represents the radius of the disc.
\item[(2)] On the other hand, the Thomas cluster process employs the Gaussian kernel $h(x)=\frac{1}{2\pi}\exp(-||x||^2)$, and $\sigma$ is the standard deviation, controlling the intensity of each cluster.
\item[(3)] The formulas of the kernels above are for $\mathbb{R}^2$. We generalize and adopt kernels in the subsequent Monte Carlo experiments for high-dimensional spaces.
\end{itemize}

With the Mat\'{e}rn and Thomas cluster processes, we consider the following $3$ simulation settings within a unit (hyper) square across a different number of dimensions ($d=2,3,5,10,20$). $(\kappa_{M}, \mu_{M}, \sigma_{M})$ and $(\kappa_{T}, \mu_{T}, \sigma_{T})$ are the parameter sets of the two cluster processes. It is worth noting that any offspring falling beyond the unit (hyper) square will be dropped. For compensation, the values of $\mu_{M}$ and $\mu_{T}$ vary for different dimensions, such that the expected sizes of generated data sets are approximately 200. Except for the pure Mat\'{e}rn or Thomas cluster process, we consider the hybrid of them as the third simulation setting and call it the ``mixed" point process. The details of each simulation setting are presented below,

\begin{itemize}
\item[\uppercase\expandafter{\romannumeral1}] Simulate a Mat\'{e}rn cluster process with parents intensity $\kappa_{M}=6$, radius $\sigma_{M}=0.1$. The mean size of each cluster $\mu_{M}$ is set to be 33.00, 35.26, 37.45, 40.37, and 44.48 as the number of dimensions $d$ increases from 2 to 20.

\item[\uppercase\expandafter{\romannumeral2}] Conduct a Thomas cluster process with $\kappa_{T}=6$, $\sigma_{T}=0.07$ (the covariance matrix is $\sigma_{T}^2I_d$). The mean size of each cluster $\mu_{T}$ is set to 33.70, 36.13, 42.38, 55.16, and 90.54 as the dimensionality $d$ increases from 2 to 20.

\item[\uppercase\expandafter{\romannumeral3}] Conduct a Mat\'{e}rn cluster process and a Thomas cluster process synchronously with $\kappa_{M} = \kappa_{T}=3$, $\sigma_{M}=0.1$, and $\sigma_{T}=0.07$. $\mu_{M}$ and $\mu_{T}$ are set to 33.30, 36.15, 39.72182, 46.78, and 60.31 as $d$ increases from 2 to 20.
\label{sec:Complex_Settings}
\end{itemize}

Under the above simulation settings, latent outliers follow an HPP with an intensity of 20. Outliers have certain distances to parents depending on the type of the corresponding cluster process (the minimum distance to any parents in the Mat\'{e}rn and Thomas cluster processes are $2\sigma_{M}$ and $3.33\sigma_{T}$, respectively). Additionally, to avoid generating data sets where the sizes of regular observations and outliers are close, we set the lower bound of the size of regular observations to 80. We want to ensure that every simulated data set is strictly imbalanced (regular points outnumber outliers by a large margin).

Figures \ref{fig:Matern_Plot}, \ref{fig:Thomas_Plot}, and \ref{fig:Matern-Thomas_Plot} present some realizations of the three simulation settings on $\mathbb{R}^2$. We compare the performance with some other existing outlier detection algorithms, including \emph{Local Outlier Factor} (LOF) \cite{breunig2000lof}, \emph{Density Based Spatial Clustering of Applications with Noise} (DBSCAN) \cite{ester1996density}, the \emph{Minimal Spanning Tree} (MST) Method \cite{MST}, \emph{Outlier Detection using In-degree Number} (ODIN) \cite{hautamaki2004outlier} and \emph{isolation Forest} for outlier detection \cite{liu2008isolation}.

LOF \cite{breunig2000lof} is a density-based outlier detection algorithm. It measures the outlyingness of points by comparing their \emph{local reachability density} with their nearest neighbors. The number of nearest neighbors is an input parameter, denoted as $k$. Rather than choosing only one value for $k$, Breunig \emph{et al.} provided a heuristic that considers a range of $k$ value instead and computes the corresponding LOF values; then all the points are ranked by their highest LOF values \cite{breunig2000lof}. We conduct our experiment following this heuristic and choose the lower and upper bound of $k$ to be 11 and 30, respectively, consistent with the guidelines provided by Breunig \emph{et al.}. After several experiments, we found the optimal threshold is 1.5, which is as expected, given the fact that the LOFs of most regular points are close to 1 \cite{breunig2000lof}.

DBSCAN \cite{ester1996density} is a density-based clustering method proposed by Ester \textit{et al.}, tuned for data sets with noise or outliers. Thus, it can also be used for outlier detection. This approach is constructed based on the idea that points deep inside a cluster generally have a minimum number (denoted \textit{MinPts}) of neighbors within a given radius (denoted \textit{Eps}); Ester \textit{et al.} call these points \emph{core points} or \emph{seeds}. To find a cluster, DBSCAN starts with an arbitrary seed, denoted as $p$; then it builds a cluster with $p$ by finding all the points that are \emph{density-reachable} from it; after that, the above steps are repeated on the next unassigned seed until no more new seed can be found; finally, the points that are not connected to any seeds are labeled as noise or outliers. To determine the value of the input parameters \textit{MinPts} and \textit{Eps}, Ester \textit{et al.} offered a heuristic which sets \emph{MinPts} to 4, then sorts the $4$-dist (the distance of a point to its $4^{th}$ nearest neighbor) of the entire data set, and find the value at the first ``elbow", setting it to be \emph{Eps} \cite{ester1996density}. Although finding the first ``elbow" point is easy with the naked eye, it is not feasible in our Monte Carlo experiments with 1000 data sets. Fortunately, Ester \textit{et al.} provided another heuristic allowing users to enter the estimated percentage of outliers to derive proper value for \emph{Eps}. To give DBSCAN some advantages, we adopt the second heuristic and set the percentage of outliers $9\%$.

The MST method \cite{MST} is a graph-based approach used for clustering. It can label any minority clusters or isolates as outliers. First, it constructs a graph with data points as nodes and the distance between any two points as edge weight. The MST is then created by linking all nodes with the minimum sum of weights while avoiding cycles. Then, the edges with substantially larger weights than the average weight of their adjacent edges are considered ``inconsistent" and are removed, effectively breaking the MST into subtrees that correspond to clusters. Clustering based on MST helps identify clusters with arbitrary shapes. However, constructing the MST can be computationally expensive for large data sets, and its performance is sensitive to the choice of threshold for identifying inconsistent edges \cite{zhang2013advancements}. In the subsequent Monte Carlo experiments, we tested several thresholds ranging from 1.1 to 3 and found the optimal thresholds are 1.7, 1.7, 1.4, 1.2, and 1.1 as $d$ increases from 2 to 20, that they deliver the best overall performance. Additionally, we label any clusters with sizes smaller than 4\% of the size of the entire data set as outliers, which is consistent with our CCD-based algorithms. We want to explore the performance of a typical clustering algorithm on outlier detection.

ODIN \cite{hautamaki2004outlier} is a graph-based outlier detection algorithm using $k$-nearest-neighbor graphs. The in-degree of an observation refers to the number of times that point appears within the $k$ nearest-neighbor sets of other points. The main idea of ODIN is based on the assumption that outliers typically have lower in-degrees because they deviate from regular observations. The observations with in-degrees smaller than a pre-specified threshold $T$ are labeled as outliers. ODIN is simple, computationally efficient, and can work without assumptions on data distribution \cite{wang2019survey}. However, like most other algorithms, it is sensitive to the choice of $k$ and $T$, and the optimal values depend on the specific data set and domain knowledge. We make $k$ and $T$ in the following Monte Carlo simulations dynamic. ODIN delivers decent overall performance when setting the two input parameters to 0.5 and 0.33 degrees of the size of the corresponding data set.

iForest \cite{liu2008isolation} is an unsupervised graph-based outlier detection algorithm. The main idea is based on the fact that outliers are rare and generally distinctive and are more likely to be separated from other regular points in a binary tree. Specifically, this is done by constructing a random decision tree (called \emph{iTrees}) and partitioning a random sub-sample based on randomly chosen features and split values; the depth of the tree is determined by sample size. iForest (also called \emph{iForest}) is a collection of iTree, and the outlyingness score of a point is determined by the average path length to the root; regular points will be more easily isolated near the root of these trees, leading to shorter average path lengths and smaller outlyingness score. We construct an iForest with 1000 iTrees with sub-sample size 64 for each to ensure the convergence of the outlyingness scores, which align with the guidance offered by Liu \emph{et al.} \cite{liu2008isolation}. Additionally, we found that a threshold of 0.57 (for the outlyingness score) delivers decent overall performance and is close to Liu \emph{et al.}'s choice.

The mean performance (out of 1000 repetitions) of each outlier detection algorithm under the three simulation settings are summarized in the subsequent tables (Tables \ref{tab:Matern_Result1} to \ref{tab:Matern-Thomas_Result1}, and \ref{tab:Matern-Thomas_Result2}). The same as the previous simulation settings, we select TPR, TNR, BA, and $F_2$-score to assess their performance.

\begin{figure}[htb]
\centering
\subfigure{
\label{Matern1}
\includegraphics[width=0.30\textwidth]{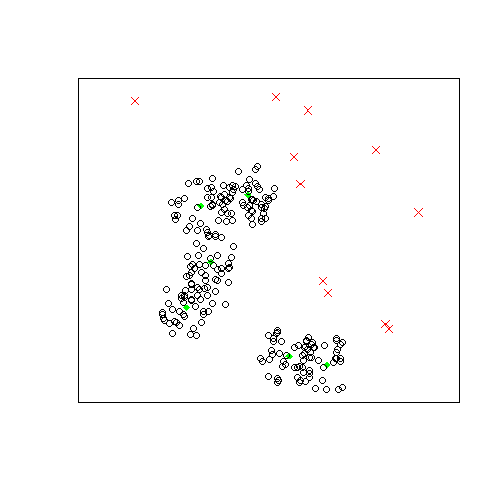}}
\subfigure{
\label{Matern2}
\includegraphics[width=0.30\textwidth]{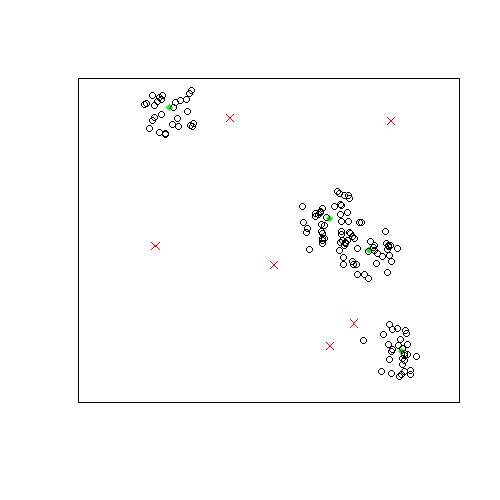}}
\subfigure{
\label{Matern3}
\includegraphics[width=0.30\textwidth]{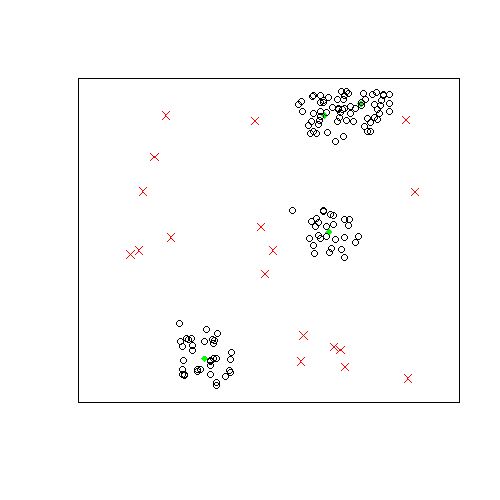}}

\caption{Three realizations of a Mat\'{e}rn cluster process (the simulation setting \uppercase\expandafter{\romannumeral1} on Section \ref{sec:Complex_Settings}) on a 2-dimensional plane with $\kappa_{M}=6$, $\sigma_{M}=0.1$, and $\mu_{M}=33$, where black dots are regular points, green dots are parents, and red dots are outliers.}\label{fig:Matern_Plot}
\end{figure}

\begin{figure}[htb]
\centering
\subfigure{
\label{Thomas1}
\includegraphics[width=0.30\textwidth]{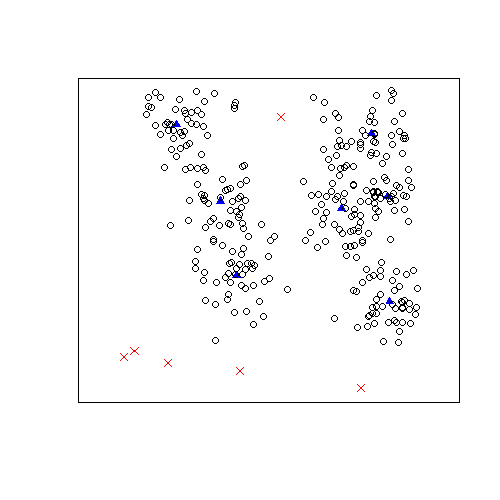}}
\subfigure{
\label{Thomas2}
\includegraphics[width=0.30\textwidth]{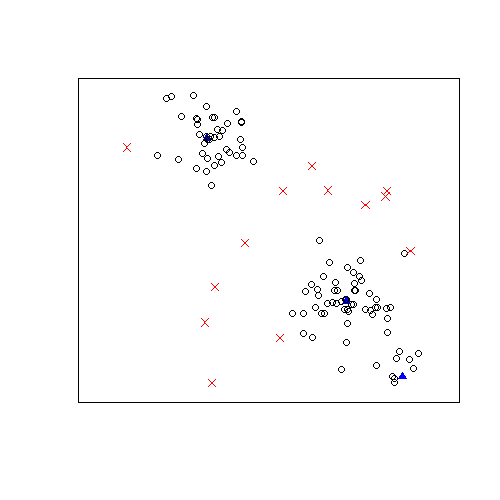}}
\subfigure{
\label{Thomas3}
\includegraphics[width=0.30\textwidth]{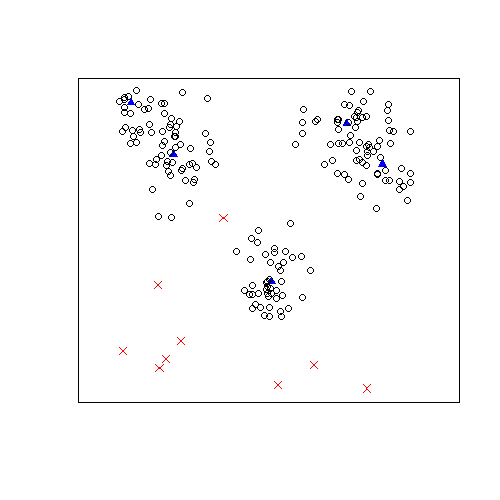}}

\caption{Three realizations of a Thomas cluster process (the simulation setting \uppercase\expandafter{\romannumeral2} on Section \ref{sec:Complex_Settings}) on a 2-dimensional plane with $\kappa_{T}=6$, $\sigma_{T}=0.005$, and $\mu_{T}=33.7$, where black dots are regular points, blue dots are parents, and red dots are outliers.}\label{fig:Thomas_Plot}
\end{figure}

\begin{figure}[htb]
\centering
\subfigure{
\label{Matern-Thomas1}
\includegraphics[width=0.30\textwidth]{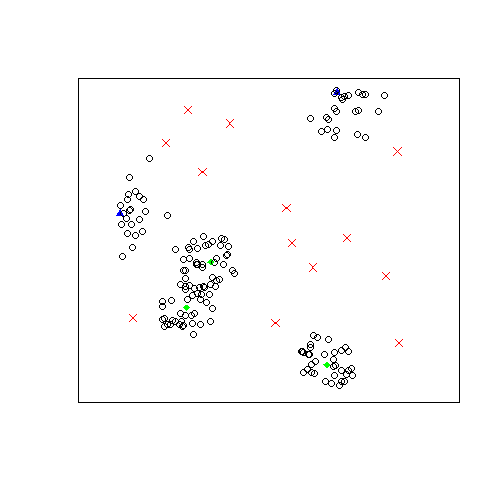}}
\subfigure{
\label{Matern-Thomas2}
\includegraphics[width=0.30\textwidth]{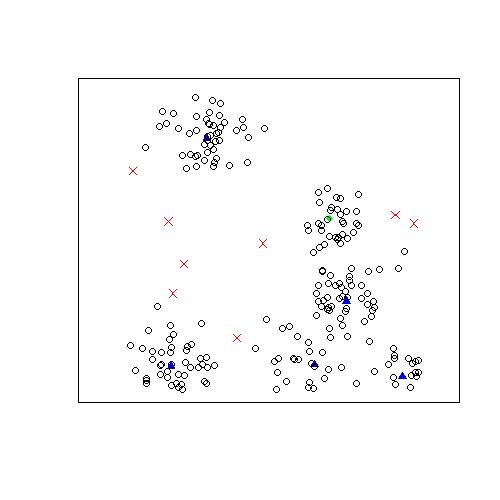}}
\subfigure{
\label{Matern-Thomas3}
\includegraphics[width=0.30\textwidth]{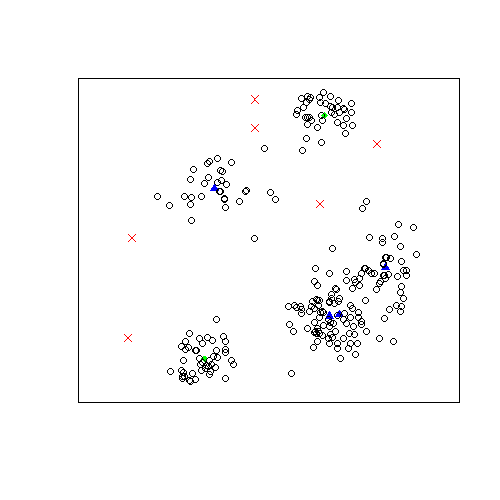}}

\caption{Three realizations of a mixed cluster process (the simulation setting \uppercase\expandafter{\romannumeral3} on Section \ref{sec:Complex_Settings}) on a 2-dimensional plane with $\kappa_{M}=\kappa_{T}=3$, $\sigma_{M}=0.1$, $\sigma_{T}=0.005$, and $\mu_{M}=\mu_{T}=33.40$, where black dots are regular points, green and blue dots are parents, and red dots are outliers.}\label{fig:Matern-Thomas_Plot}
\end{figure}

\begin{table}[htb]
  \resizebox{\textwidth}{!}{\begin{tabular}{|c|c|c|c|c|c|c|c|c|c|c|}
  \hline
  \multirow{2}*{Algorithms} & \multicolumn{2}{|c|}{$d=2$} & \multicolumn{2}{|c|}{$d=3$} & \multicolumn{2}{|c|}{$d=5$} & \multicolumn{2}{|c|}{$d=10$} & \multicolumn{2}{|c|}{$d=20$} \\ \cline{2-11}
  & TPR & TNR & TPR & TNR & TPR & TNR & TPR & TNR & TPR & TNR \\ 
  \hline
  RU-MCCDs  & 0.949 & 0.926	& 0.941 & 0.932 & 0.973 & 0.923 & 0.982 & 0.828 & 0.981 & 0.654 \\ \hline
  SU-MCCDs  & 0.969 & 0.954 & 0.970 & 0.940 & 0.971 & 0.945 & 0.982 & 0.849 & 0.979 & 0.678 \\ \hline
  UN-MCCDs  & 0.939 & 0.931 & 0.940 & 0.936 & 0.942 & 0.957 & 0.978 & 0.948 & 0.978 & 0.841 \\ \hline
  SUN-MCCDs & 0.952 & 0.948 & 0.970 & 0.932 & 0.940 & 0.973 & 0.977 & 0.961 & 0.977 & 0.853 \\ \hline
  LOF       & 0.999 & 0.962 & 0.999 & 0.962 & 1.000 & 0.927 & 0.999 & 0.866 & 0.999 & 0.842 \\ \hline
  DBSCAN    & 0.891 & 0.988 & 0.789 & 0.996 & 0.768 & 1.000 & 0.771 & 1.000 & 0.750 & 1.000 \\ \hline
  MST       & 0.659 & 0.661 & 0.558 & 0.875 & 0.623 & 0.881 & 0.713 & 0.855 & 0.757 & 0.802 \\ \hline
  ODIN      & 0.912 & 0.937 & 0.918 & 0.977 & 0.905 & 0.988 & 0.898 & 0.991 & 0.870 & 0.999 \\ \hline
  iForest   & 0.855 & 0.904 & 0.756 & 0.946 & 0.800 & 0.967 & 0.915 & 0.974 & 0.982 & 0.972	\\ \hline
\end{tabular}}
  \caption{The TPRs and TNRs of selected outlier detection algorithms under a Mat\'{e}rn cluster process (the simulation setting \uppercase\expandafter{\romannumeral1} in Section \ref{sec:Complex_Settings}).}\label{tab:Matern_Result1}
\end{table}

\begin{table}[htb]
  \resizebox{\textwidth}{!}{\begin{tabular}{|c|c|c|c|c|c|c|c|c|c|c|}
  \hline
  \multirow{2}*{Algorithms} & \multicolumn{2}{|c|}{$d=2$} & \multicolumn{2}{|c|}{$d=3$} & \multicolumn{2}{|c|}{$d=5$} & \multicolumn{2}{|c|}{$d=10$} & \multicolumn{2}{|c|}{$d=20$} \\ \cline{2-11}
  & BA & $F_2$-score & BA & $F_2$-score & BA & $F_2$-score & BA & $F_2$-score & BA & $F_2$-score \\ 
  \hline
  RU-MCCDs  & 0.938 & 0.732 & 0.937 & 0.824 & 0.948 & 0.853 & 0.905 & 0.747 & 0.818 & 0.595 \\ \hline
  SU-MCCDs  & 0.962 & 0.822 & 0.955 & 0.863 & 0.958 & 0.886 & 0.916 & 0.886 & 0.829 & 0.610 \\ \hline
  UN-MCCDs  & 0.935 & 0.730 & 0.938 & 0.833 & 0.950 & 0.875 & 0.963 & 0.892 & 0.910 & 0.755 \\ \hline
  SUN-MCCDs & 0.950 & 0.787 & 0.951 & 0.851 & 0.957 & 0.902 & 0.969 & 0.912 & 0.915 & 0.768 \\ \hline
  LOF       & 0.981 & 0.866 & 0.981 & 0.926 & 0.964 & 0.884 & 0.933 & 0.802 & 0.921 & 0.774 \\ \hline
  DBSCAN    & 0.940 & 0.827 & 0.893 & 0.794 & 0.884 & 0.786 & 0.886 & 0.789 & 0.875 & 0.767 \\ \hline
  MST       & 0.660 & 0.283 & 0.717 & 0.450 & 0.752 & 0.525 & 0.784 & 0.556 & 0.780 & 0.536 \\ \hline
  ODIN      & 0.925 & 0.783 & 0.948 & 0.882 & 0.947 & 0.901 & 0.945 & 0.901 & 0.932 & 0.879 \\ \hline
  iForest   & 0.880 & 0.615 & 0.851 & 0.691 & 0.884 & 0.775 & 0.945 & 0.877 & 0.977 & 0.925 \\ \hline
\end{tabular}}
  \caption{The BAs and $F_2$-scores of selected outlier detection algorithms under a Mat\'{e}rn cluster process (the simulation setting \uppercase\expandafter{\romannumeral1} in Section \ref{sec:Complex_Settings}).}\label{Matern_Result2}
\end{table}

\begin{table}[htb]
  \resizebox{\textwidth}{!}{\begin{tabular}{|c|c|c|c|c|c|c|c|c|c|c|}
  \hline
  \multirow{2}*{Algorithms} & \multicolumn{2}{|c|}{$d=2$} & \multicolumn{2}{|c|}{$d=3$} & \multicolumn{2}{|c|}{$d=5$} & \multicolumn{2}{|c|}{$d=10$} & \multicolumn{2}{|c|}{$d=20$} \\ \cline{2-11}
  & TPR & TNR & TPR & TNR & TPR & TNR & TPR & TNR & TPR & TNR \\ 
  \hline
  RU-MCCDs  & 0.924 & 0.907 & 0.966 & 0.852 & 0.987 & 0.772 & 0.976 & 0.669 & 0.974 & 0.497 \\ \hline
  SU-MCCDs  & 0.880 & 0.959 & 0.942 & 0.922 & 0.983 & 0.849 & 0.976 & 0.734 & 0.973 & 0.534 \\ \hline
  UN-MCCDs  & 0.875	& 0.932 & 0.943 & 0.897 & 0.979 & 0.860 & 0.990 & 0.822 & 0.980 & 0.660 \\ \hline
  SUN-MCCDs & 0.824 & 0.963 & 0.918 & 0.941 & 0.970 & 0.922 & 0.989 & 0.889 & 0.979 & 0.744 \\ \hline
  LOF       & 0.979 & 0.943 & 0.960 & 0.960 & 0.967 & 0.961 & 0.997 & 0.921 & 0.996 & 0.862 \\ \hline
  DBSCAN    & 0.824 & 0.990 & 0.684 & 0.998 & 0.728 & 0.999 & 0.726 & 0.999 & 0.707 & 0.999 \\ \hline
  MST       & 0.602 & 0.697 & 0.485 & 0.875 & 0.668 & 0.868 & 0.769 & 0.809 & 0.869 & 0.739 \\ \hline
  ODIN      & 0.891 & 0.930 & 0.903 & 0.917 & 0.916 & 0.907 & 0.899 & 0.895 & 0.859 & 0.879 \\ \hline
  iForest   & 0.857 & 0.892 & 0.708 & 0.938 & 0.644 & 0.961 & 0.716 & 0.975 & 0.789 & 0.972 \\ \hline
\end{tabular}}
  \caption{The TPRs and TNRs of selected outlier detection algorithms under a Thomas cluster process (the simulation setting \uppercase\expandafter{\romannumeral2} in Section \ref{sec:Complex_Settings}).}\label{Thomas_Result1}
\end{table}

\begin{table}[htb]
  \resizebox{\textwidth}{!}{\begin{tabular}{|c|c|c|c|c|c|c|c|c|c|c|}
  \hline
  \multirow{2}*{Algorithms} & \multicolumn{2}{|c|}{$d=2$} & \multicolumn{2}{|c|}{$d=3$} & \multicolumn{2}{|c|}{$d=5$} & \multicolumn{2}{|c|}{$d=10$} & \multicolumn{2}{|c|}{$d=20$} \\ \cline{2-11}
  & BA & $F_2$-score & BA & $F_2$-score & BA & $F_2$-score & BA & $F_2$-score & BA & $F_2$-score \\ 
  \hline
  RU-MCCDs  & 0.916 & 0.611 & 0.909 & 0.706 & 0.880 & 0.682 & 0.823 & 0.603 & 0.736 & 0.509 \\ \hline
  SU-MCCDs  & 0.920 & 0.711 & 0.932 & 0.794 & 0.916 & 0.763 & 0.855 & 0.652 & 0.754 & 0.526 \\ \hline
  UN-MCCDs  & 0.904 & 0.639 & 0.920 & 0.751 & 0.920 & 0.756 & 0.906 & 0.743 & 0.820 & 0.601 \\ \hline
  SUN-MCCDs & 0.894 & 0.687 & 0.930 & 0.806 & 0.946 & 0.845 & 0.939 & 0.822 & 0.862 & 0.664 \\ \hline
  LOF       & 0.961 & 0.741 & 0.960 & 0.877 & 0.964 & 0.908 & 0.959 & 0.876 & 0.929 & 0.802 \\ \hline
  DBSCAN    & 0.907 & 0.755 & 0.841 & 0.708 & 0.864 & 0.751 & 0.863 & 0.744 & 0.853 & 0.726 \\ \hline
  MST       & 0.650 & 0.240 & 0.680 & 0.384 & 0.768 & 0.557 & 0.789 & 0.569 & 0.804 & 0.582 \\ \hline
  ODIN      & 0.911 & 0.634 & 0.910 & 0.749 & 0.912 & 0.778 & 0.897 & 0.759 & 0.869 & 0.713 \\ \hline
  iForest   & 0.875 & 0.545 & 0.823 & 0.632 & 0.803 & 0.633 & 0.846 & 0.717 & 0.881 & 0.774 \\ \hline
\end{tabular}}
  \caption{The BAs and $F_2$-scores of selected outlier detection algorithms under Thomas cluster process (the simulation setting \uppercase\expandafter{\romannumeral2} in Section \ref{sec:Complex_Settings}).}\label{Thomas_Result2}
\end{table}

\begin{table}[htb]
  \resizebox{\textwidth}{!}{\begin{tabular}{|c|c|c|c|c|c|c|c|c|c|c|}
  \hline
  \multirow{2}*{Algorithms} & \multicolumn{2}{|c|}{$d=2$} & \multicolumn{2}{|c|}{$d=3$} & \multicolumn{2}{|c|}{$d=5$} & \multicolumn{2}{|c|}{$d=10$} & \multicolumn{2}{|c|}{$d=20$} \\ \cline{2-11}
  & TPR & TNR & TPR & TNR & TPR & TNR & TPR & TNR & TPR & TNR \\ 
  \hline
  RU-MCCDs  & 0.942 & 0.907 & 0.953 & 0.884 & 0.978 & 0.856 & 0.975 & 0.745 & 0.980 & 0.588 \\ \hline
  SU-MCCDs  & 0.925 & 0.952 & 0.955 & 0.921 & 0.975 & 0.900 & 0.974 & 0.775 & 0.974 & 0.617 \\ \hline
  UN-MCCDs  & 0.910 & 0.926 & 0.927 & 0.912 & 0.952 & 0.914 & 0.984 & 0.883 & 0.975 & 0.757 \\ \hline
  SUN-MCCDs & 0.891 & 0.952 & 0.939 & 0.932 & 0.946 & 0.950 & 0.983 & 0.915 & 0.974 & 0.779 \\ \hline
  LOF       & 0.990 & 0.948 & 0.984 & 0.957 & 0.984 & 0.942 & 0.998 & 0.893 & 0.993 & 0.857 \\ \hline
  DBSCAN    & 0.849 & 0.988 & 0.789 & 0.996 & 0.749 & 0.998 & 0.746 & 0.998 & 0.725 & 0.997 \\ \hline
  MST       & 0.639 & 0.682 & 0.525 & 0.875 & 0.657 & 0.880 & 0.736 & 0.836 & 0.809 & 0.784 \\ \hline
  ODIN      & 0.899	& 0.943 & 0.906 & 0.944 & 0.911 & 0.952 & 0.885 & 0.956 & 0.827 & 0.968 \\ \hline
  iForest   & 0.851 & 0.898 & 0.730 & 0.941 & 0.708 & 0.960 & 0.837 & 0.963 & 0.955 & 0.943 \\ \hline
\end{tabular}}
  \caption{The TPRs and TNRs of selected outlier detection algorithms under a mixed cluster process (the simulation setting \uppercase\expandafter{\romannumeral3} in Section \ref{sec:Complex_Settings}).}\label{tab:Matern-Thomas_Result1}
\end{table}

\begin{table}[htb]
  \resizebox{\textwidth}{!}{\begin{tabular}{|c|c|c|c|c|c|c|c|c|c|c|}
  \hline
  \multirow{2}*{Algorithms} & \multicolumn{2}{|c|}{$d=2$} & \multicolumn{2}{|c|}{$d=3$} & \multicolumn{2}{|c|}{$d=5$} & \multicolumn{2}{|c|}{$d=10$} & \multicolumn{2}{|c|}{$d=20$} \\ \cline{2-11}
  & BA & $F_2$-score & BA & $F_2$-score & BA & $F_2$-score & BA & $F_2$-score & BA & $F_2$-score \\ 
  \hline
  RU-MCCDs  & 0.925 & 0.660 & 0.919 & 0.752 & 0.917 & 0.763 & 0.860 & 0.658 & 0.784 & 0.550 \\ \hline
  SU-MCCDs  & 0.939 & 0.756 & 0.938 & 0.813 & 0.938 & 0.819 & 0.875 & 0.685 & 0.796 & 0.582 \\ \hline
  UN-MCCDs  & 0.918 & 0.678 & 0.920 & 0.769 & 0.933 & 0.815 & 0.934 & 0.806 & 0.866 & 0.663 \\ \hline
  SUN-MCCDs & 0.922 & 0.736 & 0.936 & 0.816 & 0.948 & 0.866 & 0.949 & 0.850 & 0.877 & 0.682 \\ \hline
  LOF       & 0.969 & 0.794 & 0.971 & 0.899 & 0.963 & 0.889 & 0.946 & 0.835 & 0.925 & 0.785 \\ \hline
  DBSCAN    & 0.919 & 0.776 & 0.893 & 0.794 & 0.874 & 0.764 & 0.872 & 0.762 & 0.861 & 0.736 \\ \hline
  MST       & 0.661 & 0.266 & 0.700 & 0.419 & 0.769 & 0.535 & 0.786 & 0.566 & 0.797 & 0.561 \\ \hline
  ODIN      & 0.921 & 0.699 & 0.925 & 0.804 & 0.932 & 0.842 & 0.921 & 0.832 & 0.898 & 0.802 \\ \hline
  iForest   & 0.875 & 0.580 & 0.836 & 0.659 & 0.834 & 0.685 & 0.900 & 0.798 & 0.949 & 0.854 \\ \hline
\end{tabular}}
  \caption{The BAs and $F_2$-scores of selected outlier detection algorithms under a mixed cluster process (the simulation setting \uppercase\expandafter{\romannumeral3} in Section \ref{sec:Complex_Settings}).}\label{tab:Matern-Thomas_Result2}
\end{table}

LOF delivers excellent overall performance, outperforming other algorithms under most simulation settings as the TPRs exceed 0.95 and TNRs larger than 0.85 substantially, with $F_2$-scores approximately equal to or larger than 0.8 regardless of the type of point process. The results align with our expectation because the outliers generated have low local density, and LOF has the advantage of identifying those low-density points thanks to its mechanism involving \emph{Local Reachability Density} (LRD). Furthermore, unlike most clustering-based algorithms, the performance of LOF does not depend on the quality of the clustering result. However, its performance declines gradually when $d \geq 10$, e.g., under the Mat\'{e}rn cluster process, the $F_2$-scores are 0.866, 0.926, 0.844, 0.802, and 0.774 as $d$ goes from 2 to 20. Here is the reasoning: With the data size remaining at the same level, increasing the number of dimensions results in more regular observations with low local densities. Therefore, the chance that LOF misclassifies regular points as outliers increases, leading to higher False Positive Rates (FPRs).

DBSCAN exhibits strong performance when $d=2$, e.g., under the Thomas cluster process, its $F_2$-score reaches 0.755, outperforming all other algorithms. Additionally, the TNRs are almost 1 under all simulation cases thanks to the exceptional clustering quality. However, its TPRs decrease gradually as $d$ increases, particularly under the simulation settings with Gaussian clusters. For example, the TPRs are 0.849, 0.789, 0.749, 0.746, and 0.725 under the mixed cluster process. DBSCAN's distance-based mechanism can explain this issue. The algorithm labels outliers by identifying points whose $4^{th}$-nearest-neighbor distances ($4^{th}$-dists) are substantially greater than others. However, the $4^{th}$-dists of outliers become close to those of regular points located along the edges of Gaussian clusters, making it challenging to differentiate them with DBSCAN, and this issue deteriorates as the number of dimensions increases and distances between points become close. Consequently, even if an outlier has a slight chance of being a ``seed", the likelihood that this outlier being \emph{density-reachable} to an existing seed grows with higher dimensionality, leading to smaller TPRs. Nonetheless, DBSCAN remains a top-performing algorithm when $d \leq 5$.

The MST algorithm consistently delivers the poorest performance under each simulation setting. For instance, under the Thomas cluster process, its $F_2$-scores are 0.240, 0.384, 0.557, 0.569, and 0.582 – substantially lower than those of the other algorithms, even after carefully tuning the thresholds for different dimensions. The MST algorithm generally possesses several inherent weaknesses in clustering and outlier detection. Firstly, it lacks robustness against noise or outliers when identifying and removing ``inconsistent" edges. For example, given two distinct clusters of points and a few noise points or outliers between them, the MST algorithm might falsely link them, misinterpreting two clusters as one. Secondly, it lacks the mechanisms to address the masking problem. Since the distances between closely grouped outliers can be similar, the MST algorithm may retain most edges connecting them, resulting in low TPRs.

The performance of the ODIN algorithm is stable across different dimensions. It delivers the best performance under the Mat\'{e}rn cluster process, where the $F_2$-scores are 0.783, 0.882, 0.901, 0.901, and 0.879, close to or even higher than those by LOF. However, its performance degrades when there are Gaussian clusters, which is still comparable to LOF under the mixed point process but substantially worse under the Thomas cluster process. It is expected when considering the characteristics of the $k$NN graph, where the points along the border of Gaussian clusters tend to have low in-degree numbers.

Unlike other algorithms, iForest behaves uniquely compared to other algorithms: its performance improves incrementally as $d$ increases. For instance, under the Thomas cluster process, the $F_2$-scores progress from 0.545 to 0.774 as the dimensions increase from 2 to 20. While delivering mediocre performance when $d \leq 5$, it outperforms most other algorithms under most simulation settings when $d$ exceeds 10. For example, under the Mat\'{e}rn cluster process, the $F_2$-scores reach 0.925 when $d=20$, substantially higher than any other algorithms. This behavior can be explained by its sensitivity to swamping and masking problems, which are prevalent in low-dimensional space where the data points are relatively dense. Although building iTrees on smaller subsets of the data reduces the intensity, making it easier to isolate outliers, it is not a perfect solution. If swamping or masking is severe within a data set, even iTrees with sub-samples struggle to differentiate outliers effectively.

Now, we focus on the four CCD-based algorithms. Due to the reasons outlined earlier, the SUN-MCCD and SU-MCCD algorithms consistently perform better than their prototypes (the RU-MCCD and UN-MCCD algorithms) under all the simulation settings, which is consistent with the result of the previous Monte Carlo simulations. For instance, under the Thomas cluster process, the SUN-MCCD algorithm attains $F_2$-scores of 0.687, 0.806, 0.845, 0.822, and 0.664, surpassing those of the UN-MCCD algorithm. On the other hand, the SUN-MCCD and SU-MCCD algorithms exhibit similar performance when $d\leq3$ due to the same mechanisms they share. However, once $d$ exceeds 5, the SUN-MCCD algorithm achieves superior performance thanks to its better adaptability in high dimensions. For example, their $F_2$- scores are 0.850 and 0.685 under the mixed cluster process when $d=20$. Consequently, our primary comparison will focus on the SUN-MCCD algorithms against other established approaches.

Under the Mat\'{e}rn cluster process, the SUN-MCCD algorithm delivers decent results. When $d \leq 3$, its $F_2$-scores are 0.787 and 0.851, following ODIN closely and slightly lower than those of LOF and DBSCAN. When $d=5$ and 10, it attains the highest $F_2$-scores among all the algorithms, with both surpassing 0.9. It performs worse than ODIN and iForest when $d$ increases to 20; this is because SUN-MCCD is distribution-based, and capturing the distribution patterns in a data set with limited size within high-dimensional space poses challenges. Nonetheless, its performance remains comparable to LOF and DBSCAN.

Considering the Thomas cluster process, nearly all the algorithms degrade due to the non-uniformity of Gaussian clusters. LOF achieves the highest $F_2$ scores across all dimensions: 0.741, 0.877, 0.908, 0.876, and 0.802. In comparison, the SUN-MCCD algorithm achieves the second-best overall performance with $F_2$-scores of 0.687, 0.806, 0.845, 0.822, and 0.664; when $d=3,\ 5,\ $ and $10$, it closely follows LOF while substantially outperforming other existing algorithms.

The situation under the mixed cluster process resembles those of the Mat\'{e}rn cluster process. When $d=2$, the SUN-MCCD algorithm performs slightly below LOF and DBSCAN; when $d=3$ and 5, it delivers the second best results, closely aligned with LOF's performance; when d=10, the SUN-MCCD algorithm achieves a marginal advantage over LOF with the highest $F_2$-score of 0.850.

For each simulation setting, we rank the performance of algorithms by their $F_2$-scores in the following table (Table \ref{Complex_Cluster_Ranking}), and the top 3 are highlighted in bold.


\begin{table}[htb]
  \center
  \resizebox{0.7\columnwidth}{!}{\begin{tabular}{|c|c|c|c|c|c|c|c|c|c|c|c|c|c|c|c|}
  \hline
  & \multicolumn{5}{|c|}{Mat\'{e}rn} & \multicolumn{5}{|c|}{Thomas} & \multicolumn{5}{|c|}{Mixed} \\ \cline{1-16}
  $d$ & 2 & 3 & 5 & 10 & 20 & 2 & 3 & 5 & 10 & 20 & 2 & 3 & 5 & 10 & 20 \\ 
  \hline
  RU-MCCDs  & \emph{6} & \emph{6} & \emph{6} & \emph{8} & \emph{8} & \emph{7} & \emph{7} & \emph{7} & \emph{8} & \emph{9} & \emph{7} & \emph{7} & \emph{7} & \emph{8} & \emph{9} \\ \hline
  SU-MCCDs  & \emph{\textbf{3}} & \emph{\textbf{3}} & \emph{\textbf{3}} & \emph{4} & \emph{7} & \emph{\textbf{3}} & \emph{\textbf{3}} & \emph{4} & \emph{7} & \emph{8} & \emph{\textbf{3}} & \emph{\textbf{3}} & \emph{4} & \emph{7} & \emph{7} \\ \hline
  UN-MCCDs  & \emph{7} & \emph{5} & \emph{5} & \emph{\textbf{3}} & \emph{6} & \emph{5} & \emph{4} & \emph{5} & \emph{5} & \emph{7} & \emph{6} & \emph{6} & \emph{5} & \emph{4} & \emph{6} \\ \hline
  SUN-MCCDs & \emph{4} & \emph{4} & \emph{\textbf{1}} & \emph{\textbf{1}} & \emph{4} & \emph{4} & \emph{\textbf{2}} & \emph{\textbf{2}} & \emph{\textbf{2}} & \emph{5} & \emph{4} & \emph{\textbf{2}} & \emph{\textbf{2}} & \emph{\textbf{1}} & \emph{5} \\ \hline
  LOF       & \emph{\textbf{\textbf{1}}} & \emph{\textbf{1}} & \emph{4} & \emph{6} & \emph{\textbf{3}} & \emph{\textbf{1}} & \emph{\textbf{1}} & \emph{\textbf{1}} & \emph{\textbf{1}} & \emph{\textbf{1}} & \emph{\textbf{1}} & \emph{\textbf{1}} & \emph{\textbf{1}} & \emph{\textbf{2}} & \emph{\textbf{3}} \\ \hline
  DBSCAN    & \emph{\textbf{2}} & \emph{7} & \emph{7} & \emph{7} & \emph{5} & \emph{\textbf{2}} & \emph{6} & \emph{6} & \emph{4} & \emph{\textbf{3}} & \emph{\textbf{2}} & \emph{5} & \emph{6} & \emph{6} & \emph{4} \\ \hline
  MST       & \emph{9} & \emph{9} & \emph{9} & \emph{9} & \emph{9} & \emph{9} & \emph{9} & \emph{9} & \emph{9} & \emph{6} & \emph{9} & \emph{9} & \emph{9} & \emph{9} & \emph{8} \\ \hline
  ODIN      & \emph{5} & \emph{\textbf{2}} & \emph{\textbf{2}} & \emph{\textbf{2}} & \emph{\textbf{2}} & \emph{6} & \emph{5} & \emph{\textbf{3}} & \emph{\textbf{3}} & \emph{4} & \emph{5} & \emph{4} & \emph{\textbf{3}} & \emph{\textbf{3}} & \emph{\textbf{2}} \\ \hline
  iForest   & \emph{8} & \emph{8} & \emph{8} & \emph{5} & \emph{\textbf{1}} & \emph{8} & \emph{8} & \emph{8} & \emph{6} & \emph{\textbf{2}} & \emph{8} & \emph{8} & \emph{8} & \emph{5} & \emph{\textbf{1}} \\ \hline
\end{tabular}}
  \caption{The rankings (by $F_2$-scores) of all the algorithms under each simulation setting of this section, top 3 are highlighted in bold.}\label{Complex_Cluster_Ranking}
\end{table}

In summary, the SUN-MCCD algorithm consistently ranks among the top-performing algorithms with the ``flexible" simulation settings. It performs better than other cluster-based algorithms, such as DBSCAN and MST, while comparable to or better than ODIN and iForest. Although LOF delivers the best overall performance, the SUN-MCCD algorithm remains a compelling choice. Moreover, the SUN-MCCD algorithm simultaneously produces clustering results, a capability absent in LOF. Furthermore, the SUN-MCCD algorithm is almost input parameter-free, strengthening its appeal compared to other algorithms.

\section{Real Data Examples}
\label{sec:Real-Data-Examples}
In this section, we evaluate the performance of all four CCD-based algorithms in real-life data and compare them with the state-of-the-art methods. Real-life data are much more complicated than the artificial data sets in Sections \ref{sec:Simul_General}, \ref{sec:Simul_Focus}, and \ref{sec:Flex_Simul}. Those data sets are obtained from \emph{Outlier Detection Datasets (ODDS)} \cite{Rayana2016} and \emph{ELKI Outlier Datasets} \cite{DBLP:journals/corr/abs-1902-03616}. Before outlier detection, we need to normalize all the features. A traditional way of normalization is subtracting the sample mean and dividing by the sample standard deviation, which is not robust to outliers exhibiting extreme feature values \cite{maronna2019robust}. Therefore, we employ a robust alternative way with mean and standard deviation replaced by the median (Med) and the Normalized Median Absolute Deviation about the median (MADN). The details of the data sets are summarized below.

\begin{itemize}
  \item [] \textbf{Brief descriptions of each real-life data set.}
  \item \emph{hepatitis}: A data set contains patients suffering from hepatitis that have died (outliers) or survived (inliers).
  \item \emph{glass}: This data set consists of 6 types of glass, and the $6^{th}$ type is a minority class, thus marked as outliers, while all other points are inliers.
  \item \emph{vertebral}: A data set with six bio-mechanical features, which are used to classify orthopedic patients either as normal (inliers) or abnormal (outliers).
  \item \emph{ecoli}: A data set consists of eight classes, three of which are the minority classes and are used as outliers.
  \item \emph{stamps}: A data set with each observation representing forged (photocopied or scanned+printed) stamps (outliers) or genuine (ink) stamps (inlier). The features are based on the color and printing properties of the stamps.
  \item \emph{vowels}: Four male speakers (classes) uttered two Japanese vowels successively; class (speaker) 1 is used as an outlier. The other speakers (classes) are considered inliers.
  \item \emph{waveform}: This data set represents three classes of waves, where class 1 was defined as an outlier/minority class.
  \item \emph{wilt}: This data set differentiates diseased trees (outliers) from other land covers (inliers).
\end{itemize}

\begin{table}[htb]
  \center
  \footnotesize{\begin{tabular}{|c|c|c|c|}
  \hline
  & $n$ & $d$ & \# of outliers \\ \hline
  hepatitis & 74   & 19 & 7 (9.5\%)   \\ \hline
  glass     & 214  & 9  & 10 (4.5\%)  \\ \hline
  vertebral & 240  & 6  & 30 (12.5\%) \\ \hline
  ecoli     & 336  & 7  & 9 (2.6\%)   \\ \hline
  stamps    & 340  & 9  & 31 (9.1\%)  \\ \hline
  vowels    & 1456 & 12 & 50 (3.4\%)  \\ \hline
  waveform  & 3443 & 21 & 100 (2.9\%) \\ \hline
  wilt      & 4735 & 5  & 257 (5.4\%) \\ \hline
\end{tabular}}
 \caption{The size ($n$), dimensionality ($d$), and contamination level of each real-life data set.}\label{Real_Data}
\end{table}

Similar to the parameter selection in Section \ref{sec:Flex_Simul}, for LOF, we choose the lower and upper bound of $k$ to be 11 and 30, finding the highest LOF for each point, and setting the threshold to 1.5. Considering DBSCAN, to get an appropriate cutoff value for the $4$-dist, we assume the percentage of outliers is known when conducting DBSCAN. When conducting MST, we set the threshold value for ``inconsistent" edges to 1.2, and we label any clusters with sizes smaller than 2\% of the size of the entire data set as outliers. As for ODIN, we set the input parameters $k$ and $T$ to 0.5 and 0.33 degrees of the size of the data set; finally, we construct iForests with 1000 iTrees with the sub-sample size of 256 for each, a threshold of 0.55 (for the outlyingness score) is used to capture the outliers. We record TPRs, TNRs, BAs, and $F_2$-scores in Tables \ref{Real_Data_Result2.1} and \ref{Real_Data_Result2.2}.

\begin{table}[htb]
  \resizebox{\textwidth}{!}{\begin{tabular}{|c|c|c|c|c|c|c|c|c|c|c|c|c|c|c|c|c|}
  \hline
  \multirow{2}*{} & \multicolumn{2}{|c|}{hepatitis} & \multicolumn{2}{|c|}{glass} & \multicolumn{2}{|c|}{vertebral} & \multicolumn{2}{|c|}{ecoli} & \multicolumn{2}{|c|}{stamps} & \multicolumn{2}{|c|}{vowels} & \multicolumn{2}{|c|}{waveform} & \multicolumn{2}{|c|}{wilt} \\ \cline{2-17}
   & TPR & TNR & TPR & TNR & TPR & TNR & TPR & TNR & TPR & TNR & TPR & TNR & TPR & TNR & TPR & TNR \\
  \hline
  RU-MCCDs  & 0.286 & 0.881 & 1.000 & 0.363 & 0.467 & 0.643 & 0.750 & 0.558 & 0.065 & 0.958 & 1.000 & 0.327 & 0.870 & 0.678 & 0.763 & 0.630 \\ \hline
  SU-MCCDs  & 0.286 & 0.925 & 1.000 & 0.363 & 0.200 & 0.576 & 0.750 & 0.680 & 0.516	& 0.883 & 1.000 & 0.373 & 0.830 & 0.774 & 0.300 & 0.785 \\ \hline
  UN-MCCDs  & 0.714 & 0.657 & 0.222 & 0.765 & 0.033 & 0.914 & 0.750 & 0.668 & 0.484 & 0.812 & 1.000 & 0.541 & 0.860 & 0.664 & 0.140 & 0.897 \\ \hline
  SUN-MCCDs & 0.714 & 0.657 & 1.000 & 0.540 & 0.100 & 0.928 & 0.750 & 0.741 & 0.516	& 0.884 & 0.978 & 0.676 & 0.620 & 0.898 & 0.366 & 0.745 \\ \hline
  LOF       & 0.000 & 0.985 & 0.778 & 0.618 & 0.033 & 0.938 & 0.500 & 0.930 & 0.161 & 0.919 & 0.370 & 0.985 & 0.000 & 1.000 & 0.031 & 0.973 \\ \hline
  DBSCAN    & 0.000 & 0.955 & 0.000 & 0.980 & 0.000 & 0.943 & 0.000 & 0.988 & 0.161 & 0.955 & 0.304 & 0.996 & 0.090 & 0.996 & 0.000 & 0.959 \\ \hline
  MST       & 0.429 & 0.866 & 0.778 & 0.662 & 0.367 & 0.695 & 0.875 & 0.546 & 0.774 & 0.437 & 0.652 & 0.553 & 0.670 & 0.484 & 0.553 & 0.672 \\ \hline
  ODIN      & 0.429 & 0.746 & 0.111 & 0.848 & 0.167 & 0.848 & 0.750 & 0.857 & 0.290 & 0.874 & 0.587 & 0.925 & 0.370 & 0.844 & 0.062 & 0.976 \\ \hline
  iForest   & 0.143 & 0.821 & 0.111 & 0.936 & 0.000 & 0.957 & 0.750 & 0.976 & 0.097 & 0.961 & 0.022 & 0.999 & 0.000 & 0.999 & 0.004 & 0.953 \\ \hline
\end{tabular}}
 \caption{The TPRs and TNRs of selected outlier detection algorithms on real-life data sets.}\label{Real_Data_Result2.1}
\end{table}

\begin{table}[htb]
  \resizebox{\textwidth}{!}{\begin{tabular}{|c|c|c|c|c|c|c|c|c|c|c|c|c|c|c|c|c|}
  \hline
  \multirow{2}*{} & \multicolumn{2}{|c|}{hepatitis} & \multicolumn{2}{|c|}{glass} & \multicolumn{2}{|c|}{vertebral} & \multicolumn{2}{|c|}{ecoli} & \multicolumn{2}{|c|}{stamps} & \multicolumn{2}{|c|}{vowels} & \multicolumn{2}{|c|}{waveform} & \multicolumn{2}{|c|}{wilt} \\ \cline{2-17}
   & BA & $F_2$-score & BA & $F_2$-score & BA & $F_2$-score & BA & $F_2$-score & BA & $F_2$-score & BA & $F_2$-score & BA & $F_2$-score & BA & $F_2$-score \\
  \hline
  RU-MCCDs  & 0.583	& 0.263 & 0.681 & 0.257 & 0.555 & 0.335 & 0.654 & 0.164 & 0.511 & 0.072 & 0.664	& 0.196 & 0.774 & 0.278 & 0.696 & 0.336 \\ \hline
  SU-MCCDs  & 0.606 & 0.286 & 0.681 & 0.257 & 0.388 & 0.140 & 0.715 & 0.210 & 0.700 & 0.455 & 0.686	& 0.207 & 0.802 & 0.335 & 0.542 & 0.185 \\ \hline
  UN-MCCDs  & 0.686 & 0.446 & 0.493 & 0.116 & 0.474 & 0.036 & 0.709 & 0.204 & 0.648 & 0.381 & 0.771	& 0.263 & 0.762 & 0.267 & 0.519 & 0.117 \\ \hline
  SUN-MCCDs & 0.686 & 0.446 & 0.770 & 0.324 & 0.514 & 0.109 & 0.745 & 0.244 & 0.701 & 0.457 & 0.827	& 0.328 & 0.759 & 0.387 & 0.555 & 0.206 \\ \hline
  LOF       & 0.493 & 0.000 & 0.697 & 0.289 & 0.488 & 0.037 & 0.711 & 0.328 & 0.540 & 0.162 & 0.677	& 0.383 & 0.500 & 0.000 & 0.502 & 0.035 \\ \hline
  DBSCAN    & 0.478 & 0.000 & 0.490 & 0.000 & 0.471 & 0.000 & 0.494 & 0.000 & 0.557 & 0.178 & 0.650	& 0.343 & 0.543 & 0.107 & 0.673 & 0.381 \\ \hline
  MST       & 0.647 & 0.375 & 0.720 & 0.313 & 0.531 & 0.282 & 0.710 & 0.186 & 0.606 & 0.373 & 0.603	& 0.178 & 0.577 & 0.153 & 0.612 & 0.266 \\ \hline
  ODIN      & 0.587 & 0.313 & 0.480 & 0.074 & 0.507 & 0.159 & 0.803 & 0.353 & 0.582 & 0.262 & 0.756	& 0.427 & 0.607 & 0.193 & 0.519 & 0.069 \\ \hline
  iForest   & 0.482 & 0.122 & 0.524 & 0.100 & 0.479 & 0.000 & 0.863 & 0.652 & 0.529 & 0.108 & 0.510	& 0.027 & 0.500 & 0.000 & 0.479 & 0.004 \\ \hline
\end{tabular}}
 \caption{The BAs and $F_2$-scores of selected outlier detection algorithms on real-life data sets.}\label{Real_Data_Result2.2}
\end{table}

The UN-MCCD and SUN-MCCD algorithms perform the best with the hepatitis data set. Both achieve TPR and $F_2$-Scores of 0.714 and 0.446, respectively. All the other algorithms deliver much lower TPRs, leading to worse performance.

For the glass data set, the SUN-MCCD algorithm and MST achieve the highest $F_2$-scores of 0.313 and 0.324. DBSCAN fails to capture any outliers, resulting in 0 $F_2$-score. ODIN and iForest can only capture $11\%$ of outliers. Although the RU-MCCD and SU-MCCD algorithms can identify all the outliers, their TNRs are merely 0.363. 

The RU-MCCD algorithm obtains the highest $F_2$-score of 0.335 under the vertebral data set, while most other algorithms can hardly identify any outliers. 

The performance of the CCD-based algorithms is worse than other algorithms under the ecoli data set, with $F_2$-scores of approximately 0.2. Here is the reason, the intensity varies greatly across each cluster or class of the ecoli data set, making clustering and density-based algorithms unsuitable, as all of them perform badly (including MST and DBSCAN). iForest achieves the highest $F_2$-score of 0.652, with TPR and TNR of 0.750 and 0.976, respectively. LOF performs the second best, with a $F_2$-score of 0.328.

For the stamps data set, the SU-MCCD and SUN-MCCD algorithms achieve the best $F_2$ Scores of 0.455 and 0.457, respectively. All the other algorithms can barely distinguish the outliers from the regular points. 

ODIN performs the best with $F_2$-scores of 0.427 under the vowels data set, and LOF delivers comparable performance after it. Meanwhile, the performance of all the CCD-based algorithms is mediocre.

The CCD-based algorithms outperform others for the waveform data set, and the SUN-MCCD algorithm achieves the best $F_2$-score of 0.384.

Finally, considering the wilt data set, DBSCAN and the RU-MCCD algorithm get the highest $F_2$-scores of 0.381 and 0.336, respectively, while all the other algorithms perform much worse.

To summarize, the CCD-based algorithms can deliver comparable or better performance under most of the eight real-life data sets.

\section{Summary and Conclusion}
\label{sec:conclusions}

In this paper, we have developed and applied Cluster Catch Digraphs (CCDs) for outlier detection, 
aiming to identify points that deviate substantially from regular points.
One of our algorithms, the RU-MCCD algorithm, utilizes RK-CCDs to partition the data into clusters, 
followed by the D-MCG algorithm to detect outliers within each cluster 
by identifying the largest connected components. 
This method effectively captures outliers that lie outside the dominant covering balls, representing the primary clusters.

Despite its effectiveness, the RU-MCCD algorithm exhibits limitations when dealing with non-spherical clusters or clusters of varying intensities, often leading to many false positives. 
To address this, we proposed the SU-MCCD algorithm, which extends cluster coverage by including additional mutually-caught covering balls, 
thus enhancing its ability to handle clusters of arbitrary shapes or varying intensities. 
We also introduced a threshold $S_{\text{min}}$ to filter small clusters, improving robustness against the masking problem.
Monte Carlo simulations demonstrated that the SU-MCCD algorithm achieves substantially higher TNRs compared to the RU-MCCD algorithm, especially with Gaussian clusters.

However, both RU-MCCD and SU-MCCD algorithms face performance degradation in high-dimensional spaces (when $d > 10$), 
due to the intrinsic properties of the Spatial Randomness Monte Carlo Test (SR-MCT) with Ripley's \textit{K} function. 
To overcome this, we formulated the SR-MCT using Nearest Neighbor Distances (NND), 
resulting in the UN-CCDs for clustering. 
By integrating UN-CCDs into the RU-MCCD and SU-MCCD frameworks, 
we developed the UN-MCCD and SUN-MCCD algorithms, respectively. 
Monte Carlo simulations showed that these new algorithms maintain high performance in low-dimensional spaces and substantially improve $F_2$-scores when the number of dimensions exceeds 10.

In Sections \ref{sec:Simul_CCDs} and \ref{sec:Flex_Simul}, 
we compared the performance of the four CCD-based algorithms with existing outlier detection methods through extensive Monte Carlo simulations using artificially generated data. 
Among the CCD-based algorithms, 
the SUN-MCCD algorithm consistently delivered the best overall performance, 
particularly in terms of robustness and adaptability across various simulation settings. 
While the $F_2$-scores were comparable to or slightly lower than those of the Local Outlier Factor (LOF), 
the SUN-MCCD algorithm outperformed other cluster-based methods like DBSCAN and MST, 
and was on par with or better than ODIN and iForest. 
Additionally, the SUN-MCCD algorithm's near parameter-free nature makes it a compelling choice.
In Section \ref{sec:Real-Data-Examples}, 
we evaluated the algorithms using eight real-life data sets. 
Despite some performance degradation due to the increased complexity of real-world data, 
the CCD-based algorithms still delivered comparable or superior overall performance compared to other outlier detection methods.

Future research will focus on further enhancing the robustness and scalability of CCD-based algorithms for outlier detection. 
This includes developing adaptive mechanisms to dynamically determine optimal parameters, 
improving computational efficiency for large-scale data sets, and exploring hybrid approaches that combine CCDs with other advanced machine learning techniques. 

\section{Acknowledgements}

Most of the Monte Carlo simulations in this paper were completed in part with the computing resource provided by the Auburn University Easley Cluster. The remaining computations were conducted on a Ubuntu desktop powered by Intel Core i9-13900K and 64 gigabyte 6400MHz DDR5 memory, sponsored by the Department of Mathematics and Statistics of Auburn University. The authors are grateful to Art{\"u}r Manukyan for sharing the codes of KS-CCDs and RK-CCDs.

\bibliographystyle{plain}
\bibliography{rf}

\end{document}